\documentclass[11pt,a4paper]{article}
\newcommand{\blind}{1} 

\usepackage{tabularx, enumerate,soul}
\usepackage[dvipsnames]{xcolor}
\usepackage{amsmath,amssymb,dsfont,subcaption,comment,enumitem, enumerate}
\usepackage{mathtools, graphicx,color,dsfont,amsthm,bbm, amsfonts}
\usepackage{lipsum,bm,comment}
\usepackage[round]{natbib}
\usepackage{longtable}
\usepackage{array}
\usepackage{multirow}
\usepackage{multicol}
\usepackage{caption,booktabs,tablefootnote}
\usepackage{bm}
\usepackage{bbm}
\usepackage[colorlinks=true,linkcolor=blue,anchorcolor=blue,citecolor=blue,filecolor=black,menucolor=black,runcolor=black,urlcolor=black]{hyperref}
\usepackage{aliascnt}
\usepackage{cleveref}

\allowdisplaybreaks

\makeatletter
\newcommand*{\Rom}[1]{\expandafter\@slowromancap\romannumeral #1@}
\makeatother
\newcommand*{\rom}[1]{\romannumeral #1}

\usepackage{algorithm}
\usepackage{algpseudocode}
\algnewcommand\algorithmicinput{\textbf{INPUT:}}
\algnewcommand\INPUT{\item[\algorithmicinput]}
\algnewcommand\algorithmicoutput{\textbf{OUTPUT:}}
\algnewcommand\OUTPUT{\item[\algorithmicoutput]}

\newcommand{\RN}[1]{%
  (\textup{\uppercase\expandafter{\romannumeral#1}})%
}
\newcolumntype{H}{>{\setbox0=\hbox\bgroup}c<{\egroup}@{}}

\newcommand{\tp}{\mathbb{P}}

\newcommand{\mL}{\mathcal{L}}
\newcommand{\mA}{\mathcal{A}}

\newcommand{\bXk}[1]{\bm{X}^{(#1)}}
\newcommand{\Xk}[1]{X^{(#1)}}
\newcommand{\Yk}[1]{Y^{(#1)}}
\newcommand{\bYk}[1]{\bm{Y}^{(#1)}}
\newcommand{\bw}{w}
\newcommand{\bmg}{g}

\newcommand{\bX}{\bm{X}}

\newcommand{\bSigmak}[1]{\Sigma^{(#1)}}
\newcommand{\bSigma}{\Sigma}

\newcommand{\hSigma}{\hat{\bSigma}}
\newcommand{\bbeta}{\beta}

\newcommand{\wbbetak}[1]{\widetilde{\beta}^{(#1)}}
\newcommand{\bbetak}[1]{\beta^{(#1)}}
\newcommand{\bbetaks}[1]{\beta^{(#1)}}
\newcommand{\hbeta}{\hat{\beta}}
\newcommand{\zeronorm}[1]{\|#1\|_0}

\newcommand{\norm}[1]{|#1|}
\newcommand{\norma}[1]{\left|#1\right|}
\newcommand{\twonorm}[1]{\|#1\|_2}
\newcommand{\twonorma}[1]{\left\|#1\right\|_2}
\newcommand{\onenorm}[1]{\|#1\|_1}

\newcommand{\infnorm}[1]{\|#1\|_{\infty}}
\newcommand{\infnorma}[1]{\left\|#1\right\|_{\infty}}
\newcommand{\lambdamin}{\lambda_{\min}}
\newcommand{\lambdamax}{\lambda_{\max}}
\newcommand{\<}{\langle}
\newcommand{\aLoss}{\overline{L}}
\renewcommand{\>}{\rangle}

\newcommand{\tetheta}{\hat{\theta}^{(0)}}
\newcommand{\ttheta}{{\theta}^{(0)}}
\newcommand{\setheta}{\hat{\theta}^{(k)}}
\newcommand{\stheta}{{\theta}^{(k)}}
\newcommand{\stset}{\{0\}\cup \mathcal{A}}
\newcommand{\tkset}{\{0\} \cup [K]}
\theoremstyle{plain}
\newtheorem{thm}{Theorem}
\newaliascnt{lemma}{thm}
\newtheorem{lemma}[lemma]{Lemma}
\aliascntresetthe{lemma}
\newtheorem{asp}{Assumption}
\newaliascnt{prop}{thm}
\newtheorem{prop}[prop]{Proposition}
\aliascntresetthe{prop}
\newaliascnt{cor}{thm}
\newtheorem{cor}[cor]{Corollary}
\aliascntresetthe{cor}

\crefname{thm}{theorem}{theorems}
\Crefname{thm}{Theorem}{Theorems}
\crefname{lemma}{lemma}{lemmas}
\Crefname{lemma}{Lemma}{Lemmas}
\crefname{prop}{proposition}{propositions}
\Crefname{prop}{Proposition}{Propositions}
\crefname{cor}{corollary}{corollaries}
\Crefname{cor}{Corollary}{Corollaries}
\crefname{asp}{assumption}{assumptions}
\Crefname{asp}{Assumption}{Assumptions}

\theoremstyle{definition}
\newtheorem{defn}{Definition}

\newtheorem{remark}{Remark}

\DeclareMathOperator*{\argmax}{arg\,max}
\DeclareMathOperator*{\argmin}{arg\,min}

\DeclarePairedDelimiter\floor{\lfloor}{\rfloor}

\allowdisplaybreaks

\title{Federated Transfer Learning with Differential Privacy}

\if1\blind
{
\author{Mengchu Li$^{1}$, Ye Tian$^{2}$, Yang Feng$^{3}$, Yi Yu$^{4}$ \\
$^1$School of Mathematics, University of Birmingham \\
$^2$Department of Statistics, Columbia University \\
$^3$Department of Biostatistics, School of Global Public Health, New York University\\
$^4$Department of Statistics, University of Warwick 
}
   
} \fi

\date{}
\usepackage[colorlinks=true,linkcolor=blue,anchorcolor=blue,citecolor=blue,filecolor=black,menucolor=black,runcolor=black,urlcolor=black]{hyperref}

\renewcommand{\arraystretch}{1.5}
\begin{document}
\def\spacingset#1{\renewcommand{\baselinestretch}%
{#1}\small\normalsize} \spacingset{1}

\maketitle

\spacingset{1.83}
\begin{abstract}
Federated learning has emerged as a powerful framework for analysing distributed data, yet two challenges remain pivotal: \emph{heterogeneity} across sites and \emph{privacy} of local data. In this paper, we address both challenges within a federated transfer learning framework, aiming to enhance learning on a target data set by leveraging information from multiple heterogeneous source data sets while adhering to privacy constraints. {We rigorously formulate the notion of \textit{federated differential privacy}, which offers privacy guarantees for each data set without assuming a trusted central server.}  Under this privacy model, we study four statistical problems: univariate mean estimation, low-dimensional linear regression, high-dimensional linear regression, and M-estimation. By investigating the minimax rates and quantifying the cost of privacy, we show that federated differential privacy is an intermediate privacy model between the well-established local and central models of differential privacy. Our analyses account for data heterogeneity and privacy, highlighting the fundamental costs associated with each factor and the benefits of knowledge transfer in federated learning.

\end{abstract}

\noindent%

\newpage
\spacingset{1.83}

\addtocontents{toc}{\protect\setcounter{tocdepth}{0}}

\section{Introduction}\label{sec:intro}

As data availability grows, research on data aggregation is gaining prominence, offering the potential to improve learning from a target data set by gathering useful information from related sources.  This, however, has also resulted in concerns about data privacy and stimulated research on federated learning \citep[e.g.][]{konevcny2016federated, mcmahan2017communication, li2020federatedsurvey}. Federated learning enables such aggregation without sharing raw data \citep[e.g.][]{konevcny2016federated, mcmahan2017communication, li2020federatedsurvey}, but communicated summaries such as gradients and Hessians can still leak sensitive information \citep{wang2019beyond}. In some instances, attackers can reconstruct original images at the pixel level \citep{zhu2019deep, zhao2020idlg}, underscoring the need for a more robust privacy protection mechanism.

Differential privacy (DP) has become a widely adopted framework \citep{dwork2006calibrating, dwork2014algorithmic}. Recent works have connected DP with federated learning to address the privacy challenges highlighted above \citep[e.g.][]{geyer2017differentially, dubey2020differentially, lowy2021private, liu2022privacy, allouah2023privacy, zhou2023differentially}, mainly focusing on empirical risk minimisation and average performance across all participating data sets, whereas transfer learning (TL) focuses on learning on a target data set in the presence of similar and/or dissimilar source data sets. Ignoring such heterogeneity can harm target performance through negative transfer \citep[e.g.][]{rosenstein2005transfer, yao2010boosting, hanneke2019value}.

Identifying the gaps in the TL literature that pertain to rigorous privacy guarantees, we, in this paper, formalise the notion of \emph{federated transfer learning} (FTL) within a novel federated DP framework. The general problem setup and the privacy framework are introduced in the remainder of \Cref{sec:intro}. Under this framework, we investigate the impact of privacy constraints and source data heterogeneity on statistical estimation error rates. In particular, we study four classical statistical problems, including univariate mean estimation in \Cref{sec:mean-estimation}, low-dimensional linear regression in \Cref{sec:lowd-regression}, high-dimensional linear regression in \Cref{sec:highd-regression}, and finally M-estimation in \Cref{sec:m-estimation}.  
Numerical results are presented in \Cref{sec:numerical}. We conclude with discussions in \Cref{sec:discussion}  and defer all the technical details to the Appendices.

\subsection{Federated transfer learning}\label{sec:FTL}

Throughout this paper, we work within an FTL framework, where the goal is to improve the learning performance on a target data set by effectively incorporating auxiliary source data sets from other sites, while protecting the privacy of each individual data set. To be specific, let $D_0$ be the target data set, $\{D_k\}_{k \in [K]}$ be the $K$ source data sets, where $K \in \mathbb{Z}_+$ and $[K] = \{1, \ldots, K\}$, and $\mathcal{P} = \{P_{\theta}: \theta \in \mathbb{R}^d\}$ be a family of distributions.  For $k \in \tkset$, assume observations in $D_k$ are independent and identically distributed (i.i.d.) with distribution $P_{\theta^{(k)}} \in \mathcal{P}$, where  $\theta^{(k)} \in \mathbb{R}^d$ is the parameter of interest.  

An inherent challenge in the multi-source setting is identifying useful source data sets to learn a better model for the target data, and the degree of ``similarity" between target and source data sets typically determines the utility of the source data. In parametric settings, it is natural to measure the similarity between the target and the $k$-th source through $\rho(\theta^{(0)},\theta^{(k)})$, where $\rho$ is some metric in $\mathbb{R}^d$. In this work, we consider the $\ell_2$-distance $\rho(\theta^{(0)},\theta^{(k)}) = \twonorm{\theta^{(0)} - \theta^{(k)}}$ and assume $\max_{k \in \mA} \twonorm{\theta^{(0)} - \theta^{(k)}} \leq h$ and $\min_{k \in \mA^c} \twonorm{\theta^{(0)} - \theta^{(k)}} > h$, with unknown $\mA \subseteq [K]$ and $h \geq 0$.  The parameter $h$ quantifies the similarity level: a smaller $h$ implies greater similarity between sources in $\mA$ and the target.

The $\ell_2$-distance is adopted due to its natural interpretation in Euclidean parametric estimation problems. It is also consistent with prior work on transfer learning and multi-task learning  \citep[e.g.][]{tian2022unsupervised, duan2023adaptive}, allowing clearer comparisons between private and non-private settings. For some inference problems, alternative metrics such as the KL divergence in \cite{hector2024turning} may be more appropriate, which we leave for future work.

We aim to estimate $\theta^{(0)}$ by leveraging all the information in the data $\{D_k\}_{k \in \tkset}$, with the hope of improving upon the \textit{target-only} estimator, i.e.\ those relying solely on the target data $D_0$.  Such settings arise naturally in multi-institutional collaborations where data are distributed, heterogeneous, and sensitive. Examples are prevalent in the public health domain, including medical image analysis \citep{adnan2022federated}, disease prediction \citep{khanna2022privacy},  and breast cancer diagnosis \citep{shukla2025federated}. Directly combining sensitive information from separate data sets, however, can raise serious privacy concerns. Therefore, throughout this paper, we study FTL procedures subject to the federated DP constraint introduced in \Cref{sec:privacy-framework}, which offers privacy guarantees for each data set without a trusted central server.

\subsection{Privacy framework}\label{sec:privacy-framework}

The prevailing framework for developing privacy-preserving methodology is differential privacy \citep[DP,][]{dwork2006calibrating}. The most standard definition of DP considers a centralised setting, where a trusted data curator has access to everyone's raw data. Given a data set $D$ with sample size $n$, a privacy mechanism $Q(\cdot | D)$ is a conditional distribution of the private information given the data. Let $Z \in \mathcal{Z}$ denote the private information. For $\epsilon > 0$ and $\delta \geq 0$, the mechanism $Q$ is said to satisfy {$(\epsilon,\delta)$-central DP}, if for all possible $D$ and $D'$ that differ by at most one data entry, it holds that 
\begin{equation} \label{Eq:alphaDP}
    Q(Z \in S| D) \leq e^{\epsilon} Q(Z \in S |D') + \delta,
\end{equation}
for all measurable sets $S \subseteq \mathcal{Z}$, with smaller $\epsilon$ and $\delta$ corresponding to stronger privacy.

In federated settings, however, giving a central server access to all raw data may be unrealistic or undesirable. We therefore introduce federated differential privacy (FDP), under which each site privatises its information locally before communication.

\begin{figure}[hbt!]
    \centering
    \includegraphics[width=0.7\textwidth]{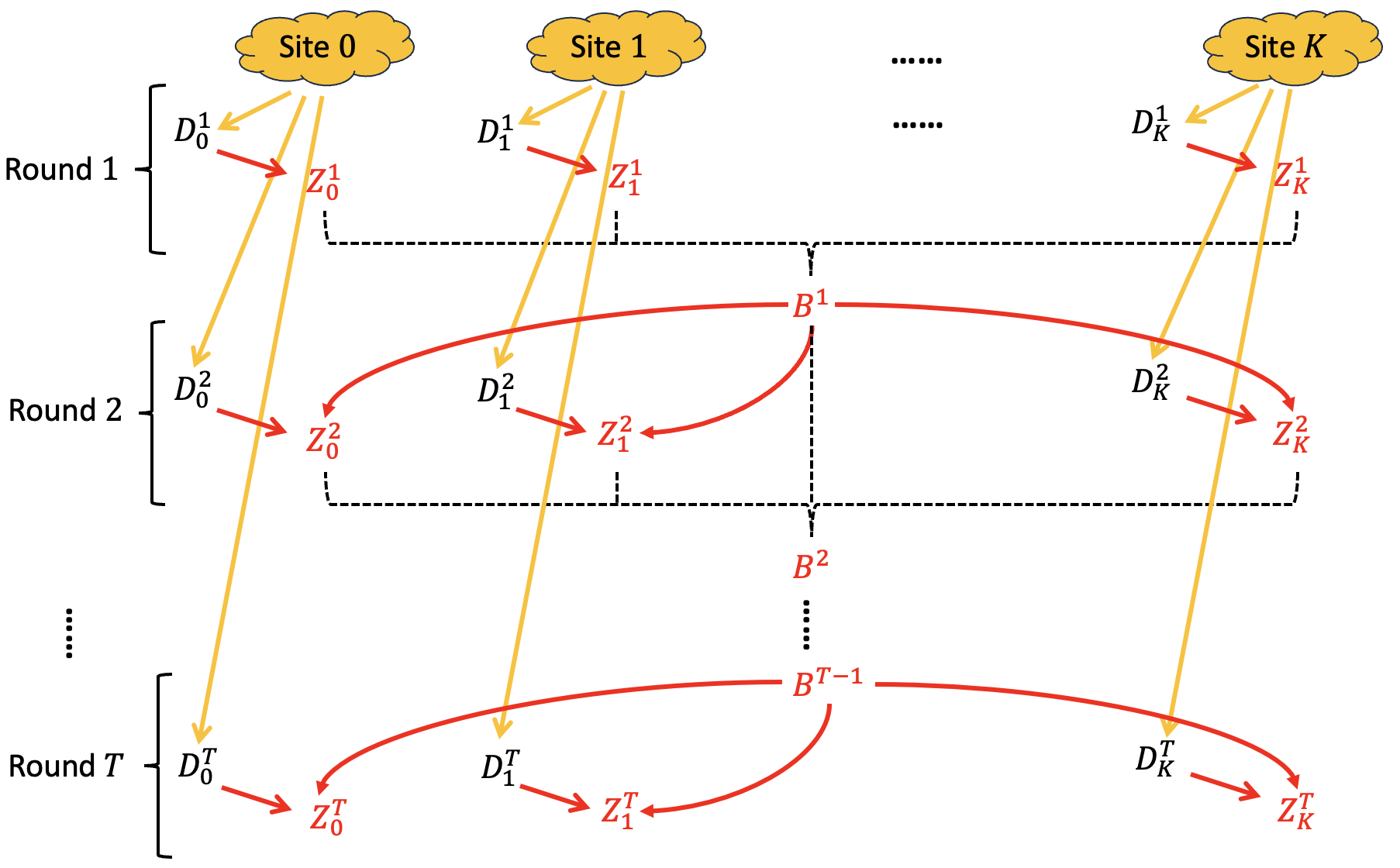}
    \caption{An illustration of the privacy mechanisms that satisfy \Cref{def:interactive-FDP}. For $t \in [T]$ and $k \in \{0\} \cup [K]$, $D_k^t$ and $Z_k^t$ refer to the data used in round $t$ at site $k$ and the communicated private information in round $t$ from site $k$, respectively; $B^t = (\{Z_k^t\}_{k \in \tkset}, B^{t-1})$ is the set of private information from all $K+1$ sites up to round $t$. Privacy mechanisms are applied to obtain each $Z_k^t$ using the information in $D_k^t$ and $B^{t-1}$. }
    \label{fig:dp-schematic-1}
\end{figure}

Recall the setup in \Cref{sec:FTL}. We consider a $T$-round interaction scheme, as shown in \Cref{fig:dp-schematic-1}. In the $t$-th round of communication, $t \in [T]$, private information $Z_k^t$ is produced using some privacy mechanism $Q_k^t$ at each site $k \in \tkset$. Let $\{D_k^{t}\}_{t \in [T]}$ form a partition of $D_k$, i.e.\ $D_k = \cup_{t = 1}^{T} D_k^{t}$ with mutually disjoint and non-empty $\{D_k^{t}\}_{t \in [T]}$. We write $B^t = ( \{Z^t_k\}_{k \in \tkset}, B^{t-1})$ to denote all information communicated across different sites in and before round $t$, and set $B^0 = 1$.

     \begin{defn}[Federated Differential Privacy, FDP]\label{def:interactive-FDP}
        Let $Q = \{Q_k^t\}_{k \in \tkset, t \in [T]}$ denote the collection of all the privacy mechanisms across sites and iterations. For $\epsilon > 0$, $\delta \geq 0$ and $T, K \in \mathbb{Z}_+$, we say $Q$ is FDP with parameter $(\epsilon,\delta)$, denoted as $(\epsilon,\delta)$-FDP, if for all $t \in [T]$, $k \in \{0\} \cup [K]$, it holds that  
         \begin{equation}\label{eq:composition-eachstep}
        Q_k^t(Z_k^t \in S|D_k^t,B^{t-1}) \leq e^{\epsilon}  Q_k^t(Z_k^t \in S|(D_k^{t})',B^{t-1}) + \delta, 
    \end{equation}
     for any measurable set $S$, and all possible $D_k^t$ and $(D_k^{t})'$ that differ in at most one data entry, with~$Z_k^t$,~$D_k^t$ and $B^{t-1}$ defined above.
    \end{defn}

\begin{remark}
Given the parameters $\epsilon,\delta$ and $T$, we denote the class of privacy mechanisms satisfying \Cref{def:interactive-FDP} as $\mathcal{Q}_{\epsilon,\delta,T}$.  The parameters $\epsilon$ and $\delta$ have the same interpretations as those in the central DP setting. The choice of $T$ needs to ensure that $\{D_k^{t}\}_{t \in [T]}$ form a partition of $D_k$ for each $k$, and its value cannot exceed the smallest sample size among $\{D_k\}_{k \in \tkset}$. Writing $Z = \{Z_k^t\}_{k \in \tkset, t \in [T]}$ as the entire private communication transcript, we also say an algorithm is $(\epsilon, \delta)$-FDP,{ if it only depends on $Z$}, due to the post processing property of DP.
\end{remark}

In the special case of $T = 1$, the condition in \eqref{eq:composition-eachstep} reduces to the \emph{non-interactive} version     
\begin{equation}\label{eq:silo-level privacy}
     Q_k(Z_k \in S | D_k) \leq e^{\epsilon} Q_k(Z_k \in S | (D_k)') + \delta,
\end{equation}
 where $Z_k \in \mathcal{Z}_k$ is produced by a privacy mechanism $Q_k$ from each site $ k \in \tkset$, without using any information from other sites. 
The advantage of non-interactive mechanisms is that they incur no communication costs between sites when producing private information. The restriction to non-interactivity, however, excludes interactive privacy mechanisms that could be more efficient for complex problems. For example, interactivity can strictly improve testing and estimation rates in some problems under privacy constraints \citep[e.g.][]{smith2017interaction,joseph2019role,berrett2020locally,butucea2023interactive}. We consider non-interactive FDP mechanisms for univariate mean estimation problems in \Cref{sec:mean-estimation} and general FDP mechanisms (\Cref{def:interactive-FDP}) for other problems.

In another special case, when the sample size of each site, $n_k = 1$ for all $k\in \tkset$, \eqref{eq:silo-level privacy} coincides with the local differential privacy (LDP) \citep[e.g.][]{duchi2018minimax}; see the corresponding definition in \Cref{section:ldp}. In fact, an appealing feature of the FDP framework is that it provides an intermediate privacy model between the well-studied central and local DP models. See our discussions and comparisons in \Cref{sec:contribution,sec:mean-discussion,sec:lowd-lowerbound,sec:discussion-highd}.

\subsection{Minimax risk under FDP constraints}

To investigate the impact of privacy constraints and source heterogeneity on learning the target model, we adopt the minimax framework. Consider the parameter space
\begin{equation}\label{eq:general-space}
    \Theta(\mathcal{A},h) = \left\{\bm{\theta} = \{\theta^{(k)}\}_{k \in \tkset}: \max_{k \in \mathcal{A}} \rho(\theta^{(k)}, \theta^{(0)}) \leq h, \min_{k \in \mA^c} \rho(\theta^{(k)}, \theta^{(0)}) > h \right\}, 
\end{equation} 
specified by $\mathcal{A} \subseteq [K]$ and $h \geq 0$. The source data sets in $\mathcal{A}$ are assumed to be similar to the target data set. Sources that are sufficiently different from the target are collected in $\mA^c$.
Although $\mA$ and $h$ are required to specify the parameter space in \eqref{eq:general-space}, which is crucial for defining the minimax risk, our algorithms do not require such knowledge. As described in \Cref{sec:detection}, we develop a general detection strategy that automatically selects a set $\hat{\mA} \subseteq [K]$, without prior knowledge of $\mA$ or $h$. We then apply appropriately privatised federated learning algorithms to this selected informative set $\hat{\mA}$ and demonstrate their near-optimal performance.

We consider the minimax risk under FDP constraints, defined as 
\[
\inf_{ \substack{Q \in \mathcal{Q}_{\epsilon,\delta,T}}} \inf_{\hat{{\theta}}(Z)} \sup_{P_{\bm{\theta}} \in \mathcal{P}(\Theta(\mathcal{A},\, h))} \mathbb{E}_{P_{\bm{\theta}},Q}\{\rho(\hat{{\theta}}(Z),{\theta}^{(0)})\},
\] 
where $\mathcal{P}(\Theta(\mathcal{A},h))$ denotes the class of joint distributions over target and source data sets with $P_{{\theta}^{(k)}} \in \mathcal{P}$ for all $k \in \tkset$ and $\bm{\theta} = \{\theta^{(k)}\}_{k \in \tkset} \in \Theta(\mathcal{A},h)$. The estimator~$\hat{\theta}(Z)$ is a measurable function of the privatised information $Z$, the entire private communication transcript, generated from some privacy mechanism $Q \in \mathcal{Q}_{\epsilon,\delta,T}$.

\subsection{Contributions}\label{sec:contribution}
In this paper, we study specific problems under the FTL setup, where the goal is to improve learning on the target data set by utilising information from multiple source data sets with potentially heterogeneous data-generating mechanisms.  To provide rigorous privacy guarantees suitable for such settings, we formulate FDP (\Cref{def:interactive-FDP}) that offers site-specific privacy guarantees without a trusted central server. We investigate the minimax risk for the univariate mean estimation problem (\Cref{sec:mean-estimation}) and the low-dimensional linear regression problem (\Cref{sec:lowd-regression}) under such privacy constraints.  In \Cref{sec:extensions}, we extend our framework to high-dimensional linear regression  (\Cref{sec:highd-regression}) and M-estimation (\Cref{sec:m-estimation}), where we develop algorithms that satisfy FDP constraints and provide theoretical guarantees for them to quantify the cost of privacy.

For the univariate mean estimation problem (\Cref{sec:mean-estimation}) and the low-dimensional linear regression problem (\Cref{sec:lowd-regression}), where the covariate dimension $d$ is assumed to be smaller than the sample size at each site, we develop private federated learning procedures and establish their optimality. In particular, the minimax rates take the form of 
\[
\text{Target-only Rate} \,\, \wedge \,\, {\text{FDP Rate}},
\]
where the target-only rate corresponds to learning only from the target data, subject to central DP constraints, and the FDP rate arises when combining information across source data sets, subject to FDP constraints.

{
\begin{table}[!ht]
\centering

\begin{tabular}{ccc}
\hline
Problems          & Target-only & FDP \\ \hline
Univariate mean estimation              &    $\frac{1}{\sqrt{n_0}}+\frac{1}{n_0 \epsilon}$        &  $h+ \frac{1}{\sqrt{\sum_{k\in\mathcal{A}\cup\{0\}}(n_k \wedge n_k^2\epsilon^2)}}$  \\ 
Low-dimensional regression   &    $\sqrt{\frac{d}{n_0}}+\frac{d}{n_0\epsilon}$         &  $h+ \sqrt{\frac{d}{\sum_{k\in\mathcal{A}\cup\{0\}} \left(n_k \wedge n_k^2\epsilon^2/d\right)}}$   \\
\hline
\end{tabular}
\caption{Minimax rates established in \Cref{sec:mean-estimation,sec:lowd-regression}. The parameters~$n_k$ correspond to the sample sizes of each data set, with $n_0$ being the target data sample size, $h$ measures the heterogeneity between the source data sets in $\mA$ and the target data set, $d$ is the dimension of the covariates in regression problems, and $\epsilon$ is the privacy parameter.}
\label{table:contribution-rates}
\end{table}
}
By comparing FDP rates with target-only rates, our results quantify the costs and benefits of private FTL across heterogeneous source data sets. The target-only rate depends only on the target sample size $n_0$, whereas the FDP rate aggregates the effective contributions from the target and informative source data sets and also incurs the heterogeneity cost $h$.
When $h$ is small and $|\mA|$ is large, i.e.\ there are sufficient informative source data sets for learning the target parameter, FDP rates offer improvement compared to target-only ones. 

Notably, each site may contribute to the overall FDP rate differently according to its own sample size $n_k$, through the effective terms $n_k\wedge n_k^2\epsilon^2$ in mean estimation, and $ n_k \wedge n_k^2\epsilon^2/d$ in regression. This site-specific contribution is a distinctive feature of the FDP constraint, which naturally accommodates heterogeneous sample sizes across sites.
We further compare our results with those under central DP and LDP privacy constraints for these problems in \Cref{sec:mean-discussion,sec:lowd-lowerbound,sec:discussion-highd}. These comparisons demonstrate that FDP is an intermediate privacy model between DP and LDP, with FDP rates interpolating between the two.

\subsection{Related work}\label{sec:literature}
In this section, we briefly review three related areas:  transfer learning, federated learning and differential privacy.

Transfer learning (TL) aims to improve a target task by leveraging information from heterogeneous source datasets. In statistical estimation and inference, existing work formulates such heterogeneity through settings such as covariate shift, posterior drift, and concept drift \citep{pan2009survey}. Many methods rely on penalisation techniques to borrow information from different data sources, including \cite{bastani2021predicting, li2022transfer, lin2022transfer, tian2022unsupervised, duan2023adaptive, li2023estimation, hector2024turning, wang2025transfer, gu2025robust}. \cite{maity2022meta} uses robust loss functions to handle outlying source data sets. \cite{hickey2024transfer} and \cite{maity2024linear} use a Cauchy random variable and an exponential function, respectively, to tilt source prediction models toward the target model. \cite{cai2025semi} develops a semi-supervised triply robust inductive transfer learning approach that is robust to misspecification of nuisance models. \cite{gao2025improving} and \cite{han2025federated} propose adaptive aggregation strategies in causal inference to borrow information from heterogeneous data.

Federated learning studies decentralised multi-site learning without combining raw data, often through iterative communication of summary statistics between local sites and a central server \citep[e.g.][]{konevcny2016federated, mcmahan2017communication, li2020federated}, which is similar to the setup described in \Cref{fig:dp-schematic-1}. The problem we consider is closest to personalised or multi-task federated learning, where the goal is to learn site-specific models while borrowing information from other sites \citep[e.g.][]{smith2017federated, t2020personalized, chen2021theorem, li2023targeting}; the key difference is that, in our FDP setting, every transmitted information is privatised before communication.

Differential privacy techniques have recently been widely used in learning with distributed datasets.  Most of the existing literature, however, focuses on providing a central-DP-type guarantee, either at the item-level or user-level \cite[e.g.][]{geyer2017differentially,ghazi2021user,levy2021learning,jain2021differentially}, requiring the presence of a trusted central server.  
On the contrary, in the FDP framework, all information communicated between different sites and the central server is privatised.  Privacy at each site is therefore protected against any inference attack from potential untrusted servers or adversarial sites. Privacy constraints similar to our FDP framework have appeared under various names in the literature \citep[e.g.][]{lowy2021private,zhou2023differentially,allouah2023privacy}. Work concurrent with and subsequent to ours considers various statistical problems under FDP-type constraints, including nonparametric regression \citep{cai2024optimal}, nonparametric classification \citep{auddy2024minimax}, functional data analysis \citep{xue2024optimal,cai2024functional}, and nonparametric hypothesis testing \citep{cai2024federated}, among others.

\subsection{Notation}\label{sec:notation}

For a set $S$, we use $|S|$ to denote its cardinality. A random variable $X$ has a standard Laplace distribution if it has density $f(x) = \exp(-|x|)/2$. For a matrix $\bm{A}$, we use $\lambda_{\min}(\bm{A})$ and $\lambda_{\max}(\bm{A})$ to denote the smallest and largest eigenvalues of $\bm{A}$, respectively, and $\twonorm{\bm{A}} = \lambda_{\max}(\bm{A})$ represents its operator norm. For a vector ${x} = (x_1, \ldots, x_d)^{\top} \in \mathbb{R}^d$, we define its $\ell_0$-pseudo-norm, $\ell_2$- and $\ell_{\infty}$-norms as $\zeronorm{x} = |\{j \in [d]: x_j \neq 0\}|$, $\twonorm{x} = \sqrt{\sum_{j=1}^d x_j^2}$ and $\infnorm{x} = \max_{j\in [d]}|x_j|$, respectively. Given a matrix $\bm{B} \in \mathbb{R}^{d \times d}$, we let $\|x\|_{\bm{B}} = \sqrt{x^\top \bm{B} x}$.
With $R > 0$, we write $\Pi_R(x) = x\min\{1,R/\|x\|_2\}$  as the projection of vector $x\in \mathbb{R}^d \setminus \{0\}$ onto the~$\ell_2$-ball in $\mathbb{R}^d$ of radius $R$ and centred at the origin.  We use $\prod_R^{\infty}(x) = x\min\{1, R/\infnorm{x}\}$ to denote the projection of $x\in \mathbb{R}^d \setminus \{0\}$ onto the  $\ell_{\infty}$-ball in $\mathbb{R}^d$ of radius $R$ and centred at the origin. For two real positive series $\{a_n\}_{n=1}^{\infty}$ and $\{b_n\}_{n=1}^{\infty}$, we write $a_n \lesssim b_n$ or $a_n = \mathcal{O}(b_n)$ when there exist absolute constants $C > 0$ and $N_0 \in \mathbb{Z}_+$ such that $a_n \leq Cb_n$ for all $n \geq N_0$, $a_n \gtrsim b_n$ or $a_n = \Omega(b_n)$ if $b_n \lesssim a_n$, and $a_n \asymp b_n$ if $a_n \lesssim b_n \lesssim a_n$.
  Notations $\widetilde{\mathcal{O}}$ and $\widetilde{\Omega}$ have similar meanings, respectively, up to logarithmic factors. We write $a_n \ll b_n$ to denote $a_n/b_n \rightarrow 0$, as $n \rightarrow \infty$. For a real-valued random variable $X$, the Orlicz-$\psi_2$ norm is defined as $\|X\|_{\psi_2} = \inf \{t > 0: \mathbb{E}[\exp(\{|X|/t\}^2)]\leq 2\}$. We use $\mathrm{SG}(C,\Sigma)$ to denote the class of sub-Gaussian distributions on $\mathbb{R}^d$ that satisfy: (i) $\mathbb{E}(X) = 0$, (ii) $\text{Var}(X) = \Sigma$ and (iii) 
   $ \|u^{\top }X\|_{\psi_2} \leq C\|u\|_{\Sigma}$, for all $u \in \mathbb{R}^d$. For a convex function $f: \mathbb{R}^d \rightarrow \mathbb{R}$ that is twice continuously differentiable, we say it is $L$-smooth if $\nabla^2 f(x) \preceq LI_d$ and it is $\mu$-strongly convex if $\mu I_d \preceq \nabla^2f(x) $. For \(r > 0\), we use $B_r(\theta)$ to 
denote the  Euclidean ball of radius \(r\) centered at \(\theta\). For $a,b \in \mathbb{R}$, we write $a\wedge b = \min\{a,b\}$ and $a \vee b = \max\{a,b\}$.

\section{Univariate Mean Estimation}\label{sec:mean-estimation}

Recall the FTL setup in \Cref{sec:FTL}, where we have target data set $D_0 = \bm{X}^{(0)} = \{X^{(0)}_i\}_{i = 1}^{n_0}$ and $K$ source data sets with the $k$-th source data set denoted as $D_k = \bm{X}^{(k)} = \{X^{(k)}_i\}_{i=1}^{n_k}$, $k \in [K]$.  Assume $X^{(k)}_i$ is sub-Gaussian with unknown mean $\mu^{(k)}$ and $\|X_i^{(k)}\|_{\psi_2}\lesssim 1$, for $i \in [n_k]$, $k \in \tkset$. Further, all data are assumed to be mutually independent. We write $\mu^{(0)}$ as $\mu$ for brevity and denote $\alpha^{(k)} = |\mu^{(k)}-\mu|$ as the $k$-th contrast. 

We are interested in the general parameter space defined in \eqref{eq:general-space} that
\begin{equation}\label{eq-Theta-2-univ}
    \Theta_{\bm{\mu}}(\mathcal{A},h) = \left\{\{\mu^{(k)}\}_{k \in \tkset}: \, \max_{k \in \mathcal{A}}\alpha^{(k)}\leq h,\, \min_{k \in \mA^c}\alpha^{(k)} > h\right\},    
\end{equation}
for the univariate mean estimation problem.  Our goal is to estimate $\mu$ subject to the $(\epsilon,\delta)$-FDP constraint. We will show that the $\ell_2$ minimax estimation error of $\mu$ is of order
\begin{equation*}
  \left(\frac{1}{\sqrt{n_0}}+\frac{1}{n_0\epsilon}\right) \wedge \left(h  +  \frac{1}{\sqrt{\sum_{k\in\mathcal{A}\cup\{0\}}(n_k\wedge n_k^2\epsilon^2)}} \right)
\end{equation*}
up to logarithmic factors.

Mean estimation under central DP for a single data set has been carefully studied in \cite{karwa2017finite}. Suppose we have $n$ i.i.d.~observations $\{X_i\}_{i = 1}^n$ with mean $\mu^*$ and $\psi_2$-norm bounded by $1$. \citet[][Theorem 4.1]{karwa2017finite}\footnote{We note that their result is for normally distributed data, but the same guarantee holds under a sub-Gaussian assumption if the bin length in their Algorithm 1 is adjusted by a multiplicative constant.} shows that if $n \gtrsim \varepsilon^{-1} \log(\delta^{-1}\eta^{-1})$,
then there exists an $(\epsilon,\delta)$-central DP estimator $\hat{\mu}$ such that with probability at least $1-\eta$,
\begin{equation}\label{eq:private_mean_guarantee}
    |\hat{\mu} - \mu^*| \lesssim f(n,\eta,\epsilon) = \sqrt{\log(1/\eta)/n}+\log(1/\eta)\sqrt{\log(n/\eta)}(\epsilon n)^{-1}.
\end{equation}
The estimator they proposed is a noisy truncated mean, i.e. 
\begin{equation}\label{eq:KV_private_mean}
    \hat{\mu} = \frac{1}{n} \sum_{i = 1}^n Y_i + \frac{2(X_{\max}-X_{\min})}{n \epsilon}Z,
\end{equation}
where $Y_i = \max\{X_{\min}, \min\{X_i, X_{\max}\}\}$,
 $Z$ is a standard Laplace-distributed random variable that is independent of the data, and the truncation thresholds $X_{\min}$ and $X_{\max}$ are obtained via a $(\epsilon/2,\delta)$-central DP algorithm using $\{X_i\}_{i=1}^n$ \cite[][Algorithm 1]{karwa2017finite}. The Laplace noise is added to ensure that $\hat{\mu}$ is $(\epsilon/2,0)$-central DP, and by the standard composition property of DP, the overall estimator $\hat{\mu}$ is $(\epsilon,\delta)$-central DP.

In the following, we are to show that a simple detection procedure combined with this base estimator~$\hat{\mu}$ can attain minimax optimality up to logarithmic factors in the FTL mean estimation problem.

\subsection{Federated private mean estimation}

We write $\hat{\mu}^{(k)}$ as the estimate obtained by applying the centrally private estimator  $\hat{\mu}$ in \eqref{eq:KV_private_mean} to the $k$-th data set, $k \in \{0\} \cup [K]$.  In the notation of the general FDP framework (\Cref{fig:dp-schematic-1}), this procedure is non-interactive $(T=1)$, with each site producing a single private output $Z_k^1 = \hat{\mu}^{(k)}.$ Our final estimator $\tilde{\mu}$ is a weighted average of $\hat{\mu}^{(0)}$ and $\hat{\mu}^{(k)}$, $k \in \hat{\mathcal{A}}$, i.e.\
\begin{equation}\label{eq:weighted_mean}
    \tilde{\mu} = \sum_{k \in \{0\}\cup \hat{\mathcal{A}}}v_k \hat{\mu}^{(k)}, 
\end{equation}
where  
\[
v_k = \frac{u_k}{\sum_{k \in \{0\}\cup\hat{\mathcal{A}}}u_k}, \quad u_k = n_k \wedge n_k^2\epsilon^2
\]
and 
\begin{equation}\label{eq:mean-hatA}
    \hat{\mathcal{A}} = \left\{k \in [K]:\, \hat{\alpha}^{(k)} = |\hat{\mu}^{(k)} - \hat{\mu}^{(0)}| \leq \tilde{c}f(n_0, \eta, \epsilon)\right\},
\end{equation}
with $\tilde{c}$ being some constant to be chosen and $f(\cdot, \cdot, \cdot)$ defined in \eqref{eq:private_mean_guarantee}.  The set $\hat{\mA}$ is selected by comparing the private estimate on each source data set $\hat{\mu}^{(k)}$ to the private estimate on the target data set $\hat{\mu}^{(0)}$, and using the accuracy of the target estimate $f(n_0, \eta, \epsilon)$ to form the threshold. We also use the same methodology in the regression problems in \Cref{sec:lowd-regression,sec:highd-regression}. A general description of this selection method, along with detailed heuristics and theoretical justifications, is presented in \Cref{sec:detection}. In particular, we show that, with high probability, $\hat{\mA}$ recovers $\mA$ under a separation condition on the sources in $\mA^c$ and using the information in $\hat{\mA}$ guarantees performance no worse than using only the target data. As for the privacy guarantee, note that each $\hat{\mu}^{(k)}$ is $(\epsilon,\delta)$-central DP \citep[][Theorem 4.1]{karwa2017finite}, and $\tilde{\mu}$ only depends on $\hat{\mu}^{(k)}$, but not on any of the raw data. Therefore, $\tilde{\mu}$ satisfies \eqref{eq:silo-level privacy}, i.e.\ it is a non-interactive, $(\epsilon,\delta)$-FDP estimator. The following theorem establishes the theoretical guarantee for the estimator $\tilde{\mu}$. 

\begin{thm}\label{thm:mean-upperbound}
Given data $D_0$ and $\{D_k\}_{k \in [K]}$, with parameters from $\Theta(\mathcal{A}, h)$ defined in \eqref{eq-Theta-2-univ}, suppose that $\min_{k \not\in \mathcal{A}} \alpha^{(k)} \geq c f(n_0,\eta,\epsilon)$ with $f(\cdot, \cdot, \cdot)$ defined in \eqref{eq:private_mean_guarantee}, for some sufficiently large absolute constant $c >0$, and
\begin{equation}\label{eq:samplesize_condition}
  \log\left(\max_{k \in [K]}n_k\right) \lesssim \log(n_0) \quad \text{and} \quad \min_{k \in [K]}n_k \gtrsim n_0 \gtrsim \frac{1}{\epsilon}\log\left(\frac{1}{\delta \eta}\right).
\end{equation}
Then for $\tilde{\mu}$ defined in \eqref{eq:weighted_mean}, with $\hat{\mathcal{A}}$ defined in \eqref{eq:mean-hatA}, there exists a choice of $\tilde{c} > 0$ such that
\[
      \mathbb{P}\Big(|\tilde{\mu} - \mu| \lesssim \RN{1} \wedge \RN{2} \Big) \geq 1- (2K+4)\eta, 
\]
where
\begin{align*}
	\RN{1} = f(n_0,\eta,\epsilon), \qquad \RN{2} = h+ \frac{\log(1/\eta)\sqrt{\max_{k \in \{0\}\cup {\mathcal{A}}}\log(n_k/\eta)}}{\sqrt{\sum_{k \in \{0\}\cup {\mathcal{A}}}(n_k\wedge n_k^2\epsilon^2)}}.
\end{align*}
\end{thm}

The requirement in \Cref{thm:mean-upperbound} that $\min_{k \not\in \mathcal{A}} \alpha^{(k)} \geq c f(n_0,\eta,\epsilon)$ guarantees the sites that are not in $\mathcal{A}$ are sufficiently different from the target site. It is used to show that $\hat{\mA} = \mA$ with high probability;  see \Cref{lemma:selection-consistency} in \Cref{sec:detection}. The conditions in \eqref{eq:samplesize_condition} are on the sample sizes. We require all the sample sizes at each site to be at least $\log\{1/(\delta \eta)\}/\epsilon$ so that each private mean estimator $\hat{\mu}^{(k)}$ has reliable performance. Assuming the sample sizes of source data sets are larger than that of the target data set ensures that $\hat{\mathcal{A}}$ includes those informative source data sets that can improve the estimation accuracy of $\mu$.  Lastly, the mild condition $ \log\left(\max_{k \in [K]}n_k\right) \lesssim \log(n_0)$ ensures that $\tilde{\mu}$ performs at least as well as the target-only estimator $\hat{\mu}^{(0)}$. 

We note that the upper bound on the estimation error consists of two terms. Term~$\RN{1}$ is the target-only rate, corresponding to the estimation error obtained by using $\hat{\mu}^{(0)}$, computed using the target data alone. Term $\RN{2}$ is the federated learning rate, representing the estimation error when we also leverage information from source data sets.  Term $\RN{2}$ depends on the combined sample size and $h$, which captures the heterogeneity between target and source data sets. \Cref{thm:mean-upperbound} shows that, with $\hat{\mathcal{A}}$, the estimator $\tilde{\mu}$ automatically achieves the minimum of these two error terms. It also quantifies the potential gain of incorporating source data sets with the target data. Indeed, when $h$ is sufficiently small, the estimation error of $\tilde{\mu}$ depends on the total sample size of the target and informative source data sets, which improves the target-only rate.

\subsection{Optimality and minimax lower bound}\label{sec:mean-discussion}

In this section, we first present a minimax lower bound to show that $\tilde{\mu}$ is optimal up to logarithmic factors among all non-interactive FDP estimators.

\begin{thm}\label{lemma:lowerbound-mean}
Let $P_{\bm{\mu}}$ denote the joint distribution of all mutually independent source and target data with mean parameters $\bm{\mu} = \{\mu^{(k)}\}_{k \in \tkset}$ and each distribution has its $\psi_2$-norm bounded by $1$, and  $\mathcal{P}_{\bm{\mu}} = \{P_{\bm{\mu}}: \bm{\mu} \in \Theta_{\bm{\mu}}(\mathcal{A}, h) \}$. Suppose that $0 \leq \delta < \epsilon({K+1})^{-1}$, with $\epsilon > 0$,  then it holds that 
\[
    \inf_{Q \in \mathcal{Q}_{\epsilon,\delta,1}} \inf_{\hat{\mu}}\sup_{P \in \mathcal{P}_{\bm{\mu}}}\mathbb{E}_{P, Q}|\hat{\mu} - \mu| \gtrsim \left(\frac{1}{\sqrt{n_0}}+\frac{1}{n_0\epsilon}\right) \wedge \left(h  +  \frac{1}{\sqrt{\sum_{k\in\mathcal{A}\cup\{0\}}(n_k\wedge n_k^2\epsilon^2)}} \right),
\]
where $\mathcal{Q}_{\epsilon, \delta, 1}$ denotes the class of FDP mechanisms defined in \Cref{def:interactive-FDP} with $T = 1$.
\end{thm}
Note that the parameter space $\Theta_{\bm{\mu}}(\mathcal{A}, h)$ is indexed by a pair $(\mathcal{A}, h)$ and our minimax lower bound holds for any valid $(\mathcal{A}, h)$ in this parameter space.
The minimax lower bound demonstrated in \Cref{lemma:lowerbound-mean} matches the upper bound in \Cref{thm:mean-upperbound} up to logarithmic factors. Broadly speaking, the terms involving $\epsilon$ characterise the effects of privacy on the mean estimation problem.
We discuss the minimax rates for mean estimation under different privacy constraints below.   

 Suppose that there are a total of $|\mathcal{A}| + 1 = K$ sites. For simplicity and clarity, we focus on the case where there is no heterogeneity across sites and each site has $n$ data points. That is, we have a total of $nK$ independent random variables with mean $\mu$ and $\psi_2$-norm bounded by $1$. The minimax rates for estimating $\mu$, omitting logarithmic factors, under different privacy constraints are listed in \Cref{table: boundary}.

\begin{table}[ht]
\centering
\renewcommand{\arraystretch}{1}
\begin{tabular}{cccc}
\toprule
 No privacy & Central DP  &  FDP  & LDP   \\ 
  & \citep{karwa2017finite}  & (Theorems \ref{thm:mean-upperbound} and \ref{lemma:lowerbound-mean}) & \citep{duchi2019lower}\\
 \midrule
$\frac{1}{\sqrt{nK}}$ & $\frac{1}{\sqrt{nK}}+\frac{1}{nK\epsilon}$ & $\frac{1}{\sqrt{nK}}+\frac{1}{n\sqrt{K}\epsilon}$ & $\frac{1}{\sqrt{nK}\epsilon}$\tablefootnote{This specific result under LDP constraint requires $|\mu|$ to be bounded by some absolute constant. The rate shown in the table requires $\epsilon \lesssim 1$. When $\epsilon \gtrsim 1$, the dependence on $\epsilon$ changes from $1/\epsilon$ to $1/\sqrt{\epsilon}$.  }              \\ 
\bottomrule
\end{tabular}
\caption{Minimax rates for estimating $\mu$ (up to poly-logarithmic factors) under different privacy constraints, when $n_k = n$ for $k \in {[K]}$ and there is no heterogeneity, i.e.\ $h = 0$.}
\label{table: boundary}
\end{table}

From left to right in the table, the problem becomes harder, reflecting the fundamental differences between these constraints. In the standard central DP setting, a central server has access to all raw data before applying some privacy mechanism.  In the LDP setting, every data point is privatised before being sent to a central server. The FDP framework is an intermediary between these two extremes. In particular, when $K = 1$, the FDP rate matches the rate under the central DP constraint, and with the LDP rate when $n = 1$.

{Another way to look at the difference among different privacy notions is through the range of privacy parameter $\epsilon$ such that we can get privacy for free. That is, for what value of $\epsilon$, does the corresponding rate under privacy constraints coincide with the non-private rate? Under the central DP constraint, we can obtain the non-private rate $1/\sqrt{nK}$ as long as $\epsilon \gtrsim 1/\sqrt{nK}$. For FDP, this becomes $\epsilon \gtrsim 1/\sqrt{n}$, and for LDP, we can only get privacy for free for $\epsilon \asymp 1$. This shrinkage of the privacy-for-free region can also be seen as a quantification of the difference between these privacy constraints. }

\section{Low-Dimensional Linear Regression}\label{sec:lowd-regression}

In this section, we consider a linear regression problem under the FTL setup.
In particular, we study private gradient descent schemes that satisfy the FDP constraint (\Cref{def:interactive-FDP}). We focus on the low-dimensional case in this section, where the number of features $d$ satisfies $d \leq n_0$, though $d$ may grow with $n_0$. The more challenging high-dimensional counterpart is studied in \Cref{sec:highd-regression}.

Recall the FTL setup in \Cref{sec:FTL}, where we have target data $D_0 = (\bm{Y}^{(0)}, \bm{X}^{(0)}) = \{(Y^{(0)}_i,X^{(0)}_i)\}_{i =1}^{n_0}$ and $K$ source data {sets} with the $k$-th source data, $k \in [K]$, denoted as $D_k = (\bm{Y}^{(k)}, \bm{X}^{(k)}) = \{(Y^{(k)}_i,X^{(k)}_i)\}_{i=1}^{n_k}$. For each $k \in \{0\} \cup [K]$, assume that $(\bm{Y}^{(k)}, \bm{X}^{(k)})$ are drawn independently from the following model
\begin{equation}\label{eq:regression_model}
    Y^{(k)}_i = \langle X^{(k)}_i, \beta^{(k)} \rangle + \xi_i^{(k)}, \qquad X^{(k)}_i \sim P_x^{(k)}, \quad i \in [n_k], 
\end{equation}
where $\langle \cdot, \cdot \rangle$ denotes the vector inner product, $\beta^{(k)} \in \mathbb{R}^d$ is the regression coefficient vector, $P_x^{(k)} \in \mathrm{SG}(C,\Sigma^{(k)})$, defined at the end of \Cref{sec:notation}, for some absolute constant $C$ and $\Sigma^{(k)} \in \mathbb{R}^{d\times d}$ is a positive definite covariance matrix, and $\xi_i^{(k)}$ is a mean-zero sub-Gaussian noise variable with $\|\xi_i^{(k)}\|_{\psi_2}\leq \sigma_k$ that is independent of $ X_i^{(k)}$ for $i \in [n_k]$. As before, we write $\beta$ for $\beta^{(0)}$ and denote $\alpha^{(k)}_r = \|\beta^{(k)} - \beta\|_2$ as the $k$-th contrast. Consider for simplicity $\sigma_k^2 \asymp 1$ for all $k \in \{0\} \cup [K]$.

For a single data set, the linear regression problem under the central DP constraint has been studied in \cite{cai2019cost}, where it is shown that a noisy gradient descent algorithm achieves the minimax optimal rate of convergence up to poly-logarithmic factors. However, their result requires that the sample size satisfies $n = \tilde{\Omega}(d^{3/2}/\epsilon)$ \cite[c.f.][Theorem 4.2]{cai2019cost} due to the projection step in each iteration. The dependence on~$d$ in this condition is somewhat unnatural since it requires more samples than necessary for the estimation error to diminish. In our study, we first modify the procedure in \cite{cai2019cost}, leveraging the idea of adaptive clipping \citep{liu2023near,varshney2022nearly} to relax the sample-size requirement for private linear regression on a single data set. 

Let $\beta^*$ denote the true regression parameter for the single data set of sample size $n$. We show in \Cref{lemma:DPregreesion_singlesite} that there is an estimator $\hat{\beta}$ (\Cref{algorithm:DPregression_single}) that is $(\epsilon,\delta)$-central DP and achieves
  \begin{equation}\label{eq:singleregression_rate}
        \|\hat{\beta} - \beta^*\|_2 \lesssim r(n,d,\epsilon,\delta, \eta) =   \log\left(\frac{\log(n)}{\eta}\right)\sqrt{\frac{d\log(n)}{n}} + \frac{d\log^2(n/\eta)\sqrt{\log(1/\delta)\log(\log(n)/\eta))}}{n\epsilon} 
    \end{equation}
with probability at least $1-7\eta$, when $n = \tilde{\Omega}(d/\epsilon)$. This estimator is then used in the informative source selection step in our study of the FTL problem. Due to space constraints, we defer the section on private linear regression on a single data set to \Cref{sec-reg-single-app}. 

For the FTL linear regression problem, we are interested in the general parameter space defined in \eqref{eq:general-space}, which takes the form 
\begin{equation}\label{eq:regression_parameterspace}
    \Theta_{\bm{\beta}}(\mathcal{A},h) = \left\{\{\beta^{(k)}\}_{k \in \tkset}: \max_{k \in \mathcal{A}}\alpha^{(k)}_r \leq h,\, \min_{k \in \mA^c}\alpha^{(k)}_r > h \right\}.
    \end{equation}
 Our goal is to estimate $\beta$ subject to the $(\epsilon,\delta)$-FDP constraint, for a given pair of $(\epsilon, \delta)$.  We show that the minimax rate of the $\ell_2$-estimation error for $\beta$ is, up to poly-logarithmic factors,
\[
    \left( \sqrt{\frac{d}{n_0}}+\frac{d}{n_0\epsilon}\right) \wedge \left\{ h  \vee  \sqrt{\frac{d}{\sum_{k \in\stset} \{n_k \wedge \{(n_k\varepsilon)^2/d\}\}}} \right\}.
\]

\subsection{Federated private linear regression}

The high-level methodology for the FTL linear regression problem is the same as that used for univariate mean estimation: we first identify the informative sources and then combine the information while respecting the FDP constraint (\Cref{def:interactive-FDP}). To fit into the FDP framework, we use half of the data for detection and the other half to combine information. To be specific, we first consider
\begin{equation}\label{eq:lowd-hatA}
    \hat{\mathcal{A}} = \left\{k \in [K]: \hat{\alpha}^{(k)}_r = \|\hat{\beta}^{(k)} - \hat{\beta}^{(0)}\|_2 \leq \tilde{c}\, r(n_0,d,\epsilon,\delta, \eta)\right\}, 
\end{equation}
with $r(\cdot, \cdot, \cdot, \cdot, \cdot)$ defined in \eqref{eq:singleregression_rate} being the target-only rate and for some $\tilde{c} > 0$, where $\hat{\beta}^{(k)}$ denotes the output of \Cref{algorithm:DPregression_single} applied to half of the data at the $k$-th location, for $k \in \tkset$, say, $\{(\Xk{k}_i, \Yk{k}_i)\}_{i \in [\floor{n_k/2}]}$.  The set $\hat{\mA}$ has the same form as that in the mean estimation problem, and they are both special cases of the general formulation of the detection strategy described in \Cref{sec:detection}. We propose \Cref{algorithm:DPregression_federated} to aggregate information from the detected informative set $\hat{\mathcal{A}}$, with theoretical guarantees in \Cref{thm:lowd-regression-transfer}.

 \Cref{algorithm:DPregression_federated} can be viewed as a private federated mini-batch gradient descent algorithm. In each iteration, every selected site in $\{0\} \cup \hat{\mathcal{A}}$ computes a privatised gradient update $Z^t_k$ using its local data, and sends this information to the central server. The server then combines these private updates and moves the current estimator in the aggregated descent direction. It draws inspiration from \citet[][Algorithm 4.1]{cai2019cost} and \citet[][Algorithm 2]{varshney2022nearly} with two important ingredients. The first is the Gaussian mechanism, which adds appropriately scaled Gaussian noise to truncated gradients at each step. This is also the fundamental privacy-preserving step in many gradient-based algorithms. The second is the PrivateVariance mechanism (\Cref{algorithm:Private_variance} in \Cref{appendix:privatevariance}), which adaptively chooses the truncation level in each iteration. The key idea is that, as the gradient descent steps proceed, the iterate $\beta^t$ is expected to get closer to $\beta$, and the truncation level should reflect this in order to minimise the total amount of noise added.

\begin{thm}\label{thm:lowd-regression-transfer}
    Let $\{\bXk{k}, \bYk{k}\}_{k \in \tkset}$ be generated from \eqref{eq:regression_model}, with $0< 1/L \leq \lambda_{\min}(\Sigma_k) \leq \lambda_{\max}(\Sigma_k) \leq L < \infty$, for some absolute constant $L\geq 1$ and $k \in \tkset$. Initialise \Cref{algorithm:DPregression_federated} with $\beta^0 = 0$, step size $\rho = 18L(1+81L^2)^{-1}$ and choose $T = \lceil C\log(N) \rceil$, for some large enough absolute constant $C > 0$, with $N = \sum_{k \in \{0\} \cup \hat{\mathcal{A}}} n_k$ and $\hat{\mathcal{A}}$ in \eqref{eq:lowd-hatA}. Suppose that $\|\beta\|_2 \leq C'$, $\min_{k \notin \mA}\alpha_r^{(k)} \geq C''r(n_0,d,\epsilon,\delta,\eta)$, where $C',C''>0$ are absolute constants, $r(\cdot, \cdot, \cdot, \cdot, \cdot)$ is defined in \eqref{eq:lowd-hatA}, the sample sizes satisfy 
        \[
        \min_{k \in [K]} u_k \gtrsim u_0 \gtrsim d\log(n_0)\log\big(\frac{n_0 \vee (\log(n_0)/\delta)}{\eta}\big){\log\left(\frac{\log(n_0)}{\eta(\epsilon\wedge\delta)}\right)}, 
        \]
    and $ \log\big(\sum_{k \in [K]}n_k\big) \lesssim \log(n_0)$. Then, \Cref{algorithm:DPregression_federated} satisfies $(\epsilon,\delta)$-FDP and there exists a choice of $\tilde{c} > 0$ in \eqref{eq:lowd-hatA} such that 
    \[
    \mathbb{P}\Big(\|\tilde{\beta}-\beta\|_2 \lesssim \RN{1} \wedge \RN{2}\Big) \geq 1-14(K+1)\eta,
    \]
    where $\RN{1} = r(n_0,d,\epsilon,\delta, \eta)$,
    \begin{align*}
        \RN{2} &= h+ \log^2(N/\eta)\sqrt{\frac{d\log(1/\delta)}{\sum_{k \in \{0\}\cup\mA}\min\{n_k, n_k^2\epsilon^2/d\}}}.
    \end{align*} 
\end{thm}

\begin{algorithm}[!ht]
	\begin{algorithmic}[1]
		\INPUT{Data $\{\{(\Xk{k}_i, \Yk{k}_i)\}_{i \in \floor{n_k/2}+1, \dotsc, n_k}\}_{k \in \{0\}\cup\hat{\mA}}$, number of iteration $T$, step size $\rho$, privacy parameters $\epsilon,\delta$, initialisation $\beta^0$, failure probability $\eta \in (0,1/2)$}
                \State Set batch size $b^{(k)} = \floor{n_k/(2T)}$, for $k \in \{0\} \cup  \hat{\mA}$,  $N = \sum_{k \in \{0\} \cup \hat{\mathcal{A}}} n_k$, truncation radius $R = \sqrt{d\log(N/\eta)} $, weight $v_k = u_k/\sum_{k \in \{0\} \cup  \hat{\mA}}u_k$, with $u_k = \min\{b^{(k)}, (b^{(k)}\epsilon)^2/d\}$, privacy parameters $\epsilon' = \epsilon/2, \delta' = \delta/2$
		\For{$t = 0, \ldots, T-1$} 
                 \For{$k \in \{0\} \cup \hat{\mathcal{A}}$} \Comment{Each site generates the privatised information $Z^t_k$ locally}
            \State Set $\tau^{(k)} = b^{(k)}t$, $R_t^{(k)} = \sqrt{\log(N/\eta)}$PrivateVariance($\{Y^{(k)}_{\tau^{(k)} +i} - X^{(k)\top}_{\tau^{(k)}+i} \beta^t\}_{i=1}^{b^{(k)}},\epsilon',\delta'$) \Comment{See \Cref{algorithm:Private_variance} for PrivateVariance}
			\State Sample $w_t^{(k)} \sim \mathcal{N}(0, I_d)$ and let $\phi_t^{(k)} = \sqrt{2\log(1.25/\delta')}2R R_t^{(k)}/(b^{(k)}\epsilon')$
            \State Compute $Z^t_k = v_k \Big(1/b^{(k)}\sum_{i=1}^{b^{(k)}}\Pi_{R}(X_{\tau^{(k)}+i}^{(k)})\Pi_{R_t^{(k)}}(X_{\tau^{(k)}+i}^{(k)\top} \beta^{t} - Y_{\tau^{(k)}+i}^{(k)})+\phi_t^{(k)} w_t^{(k)}\Big)$
            \EndFor
            \State $\beta^{t+1} = \beta^t - \rho\sum_{k \in \{0\} \cup \hat{\mathcal{A}}}Z^t_k$; \Comment{A central server aggregates the privatised information}
		\EndFor
		\OUTPUT $\tilde{\beta} = \beta^T$. 
		\caption{Federated linear regression with FDP guarantees} \label{algorithm:DPregression_federated}
	\end{algorithmic}
\end{algorithm} 

We conclude this subsection with a few remarks.
\begin{itemize}[leftmargin=*]

    \item (FDP guarantee). To show that \Cref{algorithm:DPregression_federated} satisfies the FDP constraint in \Cref{def:interactive-FDP}, it suffices to show that each iteration, along with the detection step, guarantees $(\epsilon,\delta)$-central DP in each site given the data and the private information from the previous steps. This is given in the detection step since all $\hat{\beta}^{(k)}$'s are $(\epsilon,\delta)$-central DP, and $\hat{\mathcal{A}}$ is a post-processing step of them. For \Cref{algorithm:DPregression_federated}, the privacy in each iteration is obtained by a composition of the Gaussian and PrivateVariance mechanisms.

    \item (The benefit of transfer learning). The upper bound in \Cref{thm:lowd-regression-transfer} has two terms. Term~\RN{1} is the target-only rate, and term \RN{2} is the FDP rate when combining the informative source data sets with the target data set. When there are many source data sets that are sufficiently similar to the target data, i.e.\  $h$ is small and $|\mathcal{A}|$ is large, \Cref{algorithm:DPregression_federated} obtains a substantially faster convergence rate than using the target data alone. Moreover, 
    \Cref{algorithm:DPregression_federated} can adaptively achieve a better rate between these two rates, and as shown in \Cref{sec:lowd-lowerbound}, it is minimax rate-optimal (up to poly-logarithmic terms).  

    \item (Sample-splitting). For our procedure, we use separate samples for the detection step (computing $\hat{\mathcal{A}}$) and for each step in the iteration. Using separate samples in each iteration in \Cref{algorithm:DPregression_federated} avoids dependence in analysing the PrivateVariance procedure. As for the detection step, our proof still works when using the same data for computing $\hat{\mathcal{A}}$ and as input for \Cref{algorithm:DPregression_federated}. See the proof of \Cref{thm:lowd-regression-transfer} for details. We conduct sample splitting for all steps here so that the overall procedure fits into the FDP mechanisms framework (\Cref{def:interactive-FDP}).
\end{itemize}

\subsection{Optimality and minimax lower bound}\label{sec:lowd-lowerbound}

In this subsection, we demonstrate the optimality of \Cref{algorithm:DPregression_federated} in \Cref{thm:lowerbound low-d regression} and compare the costs of different forms of privacy constraints in the context of linear regression. Let the class of distributions under consideration be
    \begin{align*}
        \mathcal{P}_{\bm{\beta}} &= \bigg\{ \prod_{k=0}^K P^{\otimes n_k}_{\beta^{(k)}}: \{\beta^{(k)}\}_{k \in \tkset} 
 \in \Theta_{\bm{\bbeta}}(\mathcal{A},h), \\
  & \hspace{8em} P_{\beta^{(k)}} = P_{y^{(k)}|x^{(k)},\beta^{(k)}} P_x,\, P_{y|x,\beta^{(k)}} = \mathcal{N}(x^{\top}\beta^{(k)}, 1),\,  P_x \in \mathrm{SG}(C,\Sigma) \bigg\},
    \end{align*}
where the parameter space $\Theta_{\bm{\beta}}(\mathcal{A},h)$ for $\bm{\beta} =  \{\beta^{(k)}\}_{k \in \tkset}$ is defined in \eqref{eq:regression_parameterspace}, and $\mathrm{SG}(C,\Sigma)$ denotes a class of sub-Gaussian distributions on $\mathbb{R}^d$ with parameter $C$ and covariance $\Sigma$. See \Cref{sec:notation} for the precise definition. We consider the class of FDP mechanisms $\mathcal{Q}_{\epsilon,\delta,T}$ in \Cref{def:interactive-FDP}. Recall that the interaction scheme consists of $T$ rounds, and, within each round $t \in [T]$, private information $Z_k^t$ is obtained by applying privacy mechanisms $Q_k^t$ at each site to $D_k^t$ and the private information accumulated through the previous rounds $B^{t-1}$.

    \begin{thm}\label{thm:lowerbound low-d regression} Suppose that $\{D_k = \{(X_i^{(k)}, Y_i^{(k)})\}_{i \in [n_k]}\}_{k \in \tkset}$ are generated from the distribution $P_{\bm{\beta}} \in \mathcal{P}_{\bm{\beta}} $. Write $\beta$ for $\beta^{(0)}$. Suppose that 
    \begin{equation}\label{eq:lowerbound-condition1}
          \epsilon \in (0,1), \; \delta < n_0^{-2}, \; d\log(1/\delta) \lesssim n_0 \; \text{and} \;  h \leq \sqrt{d},
    \end{equation}
    then we have for any $T \geq 1$,
    \begin{equation}\label{eq:lowerbound_rate1}
        \inf_{Q \in \mathcal{Q}_{\epsilon,\delta,T}} \inf_{\substack{\hat{\beta}(Z)}} \sup_{P_{\bm{\beta}} \in \mathcal{P}_{\bm{\beta}}} {\mathbb{E}_{P_{\bm{\beta}},Q}\|\hat{\beta}(Z) - \beta\|_2^2} \gtrsim \left( {\frac{d}{n_0}}+\frac{d^2}{n_0^2\epsilon^2}\right) \wedge h^2.
    \end{equation} 
Let $b_k^t$ denote the size of $D_k^t$. If, in addition to \eqref{eq:lowerbound-condition1}, we have
     $\sum_{k \in\stset} \sum_{t = 1}^T\{b^t_k d \wedge (b^t_k\varepsilon)^2\} \gtrsim d^2$ and $d\delta\log(1/\delta) \lesssim \varepsilon^2,$
    then it holds that
  \begin{align}\label{eq:lowerbound_rate2}
        \inf_{Q \in \mathcal{Q}_{\epsilon,\delta,T}} \inf_{\substack{\hat{\beta}(Z)}} \sup_{P_{\bm{\beta}} \in \mathcal{P}_{\bm{\beta}} } {\mathbb{E}_{P_{\bm{\beta}},Q}\|\hat{\beta}(Z) - \beta\|_2^2} &\gtrsim \nonumber\\ & \hspace{-10em} \left( {\frac{d}{n_0}}+\frac{d^2}{n_0^2\epsilon^2}\right) \wedge \left\{ h^2  \vee  \frac{d}{\sum_{k \in\stset} \{n_k \wedge \{(n_k\varepsilon)^2/d\}\}} \right\}.
    \end{align}
\end{thm}

This lower bound holds for any $T$ that satisfies the required condition. 
Note that the upper bound in \Cref{thm:lowd-regression-transfer} is stated for the $\ell_2$-estimation error, which, after squaring each term, matches \eqref{eq:lowerbound_rate2} up to poly-logarithmic factors. Rigorously speaking, we can only guarantee the optimality of \Cref{algorithm:DPregression_federated} in terms of the \emph{squared}-$\ell_2$-metric, but for the sake of consistency with the remaining results of the paper, we focus on the $\ell_2$-norm in our discussions and comparisons with other notions of DP.

To derive the lower bound, we consider two constructions of $\bm{\beta}$, which lead to the two terms in \eqref{eq:lowerbound_rate1} and \eqref{eq:lowerbound_rate2}. The one in \eqref{eq:lowerbound_rate1} is obtained by noting that FDP is a stronger notion than the central DP.  We can therefore apply a modified version of the central DP lower bound (see \Cref{lemma:tonycai}) based on \cite{cai2019cost}, which accounts for the non-informative source data sets $[K] \setminus \mathcal{A}$. The one in \eqref{eq:lowerbound_rate2}, especially the last term, is obtained using \Cref{lemma:van-tree-lowerbound}, which applies the Van-Trees inequality \citep[][Theorem 1]{gill1995applications} with modified arguments based on \cite{xue2024optimal} and \cite{cai2024optimal} to account for the non-informative sources.

The terms involving $\epsilon$ in \eqref{eq:lowerbound_rate2} quantify the effects of privacy on the linear regression problem. 
We compare the fundamental difficulty of estimating the linear regression parameter under different notions of privacy below. We focus on the case where no heterogeneity exists across different sites and each site has $n$ pairs of covariate-response observations. Suppose there are $K$ sites, i.e., we have $nK$ i.i.d. data $(X_i, Y_i)$ from the linear model \eqref{eq:regression_model} with regression parameter $ \beta$. 

\begin{table}[!ht]
\centering
\renewcommand{\arraystretch}{1}
\begin{tabular}{cccc}
\toprule
 No privacy & Central DP   &  FDP  & LDP 
  \\ 
  & \citep{cai2019cost}  & (Theorems \ref{thm:lowd-regression-transfer} and \ref{thm:lowerbound low-d regression}) & \citep{zhu2023improved}\\
 \midrule
$\sqrt{\frac{d}{nK}}$ & $\sqrt{\frac{d}{nK}}+ \frac{d}{nK\epsilon}$ & $\sqrt{\frac{d}{nK}} + \frac{d}{n\epsilon\sqrt{K}}$ & $\frac{{d}}{\sqrt{nK}\epsilon}$\tablefootnote{Upper bound results are also established in \cite{zhu2023improved}, which do not match the lower bound in general. However, this LDP lower bound is already larger than the upper bound under FDP constraints, which demonstrates that FDP indeed allows fundamentally more accurate estimations. }              \\ 
\bottomrule
\end{tabular}
\caption{Convergence rates of $\|\hat{\beta} - \beta\|_2$ (up to poly-logarithmic factors) subject to different privacy constraints, $n_k = n$ for $k \in {[K]}$ and $h = 0$.  The rates for no privacy, central DP, and FDP are minimax rates, while the LDP rate is only a lower bound.}
\label{table: boundary-regression}
\end{table}

From left to right, we observe here the similar phenomenon of increasing difficulty, as discussed in the mean estimation problem (\Cref{table: boundary}).  In particular, the privacy error term under FDP is larger than that under DP by a factor of $\sqrt{K}$, while the privacy error term under LDP is at least larger than that under FDP by a factor of $\sqrt{n}$.

\section{Extensions}\label{sec:extensions}
Having established the fundamental limits of mean estimation and linear regression under FDP constraints, we now extend our framework to broader classes of estimation problems, including high-dimensional regression (\Cref{sec:highd-regression}) and M-estimation (\Cref{sec:m-estimation}). We provide estimators that satisfy FDP constraints and establish their theoretical guarantees. For simplicity, we assume $n_k = n$ throughout this section. Most algorithms used in this section share a similar spirit to \Cref{algorithm:DPregression_federated}, and we leave their details to the appendix due to space constraints. 
\subsection{High-Dimensional Linear Regression}\label{sec:highd-regression}

In this section, we consider a high-dimensional linear regression problem in the context of FTL. Recall the FTL setup from Section~\ref{sec:FTL}, where we have target data $D_0 = (\bm{Y}^{(0)}, \bm{X}^{(0)}) = \{(Y^{(0)}_i,X^{(0)}_i)\}_{i =1}^{n}$ and source data sets $D_k = (\bm{Y}^{(k)}, \bm{X}^{(k)}) = \{(Y^{(k)}_i,X^{(k)}_i)\}_{i =1}^{n}$ collected from the $k$-th source, for $k \in [K]$. We assume that the data $(\bm{Y}^{(k)}, \bm{X}^{(k)})$ are drawn independently from the same linear model as in \eqref{eq:regression_model} in \Cref{sec:lowd-regression}, i.e.
\begin{equation}\label{eq: high-dim reg model}
    \Yk{k}_i = \< \Xk{k}_i, \bbetak{k}\> + \xi^{(k)}_i, \qquad X^{(k)}_i \sim P_x^{(k)}, \quad i \in [n_k]. 
\end{equation}  
In this section, we allow the dimension $d$ to be much larger than the sample sizes of both target and source data sets, and we assume the target coefficient is sparse, i.e.~$\zeronorm{\bbetak{0}} = s < d$.  Note that such a sparsity assumption is only imposed on the target model.  Write $\beta = \beta^{(0)}$.   As in the univariate mean estimation and the low-dimensional linear regression problems, we assume that there exists an unknown source index set $\mA \subseteq [K]$ such that 
\begin{equation}\label{eq: h high-dim}
    \max_{k \in \mA} \twonorm{\bbetak{k} - \beta} \leq h \quad \text{and} \quad \min_{k \in \mA^c} \twonorm{\bbetak{k} - \beta} > h.
\end{equation}
 For notational simplicity, let
\begin{equation}\label{eq:targetrate-hd}
    r_{\textup{HLR}}(n, s', d, \epsilon, \delta, \eta) = \sqrt{\frac{s'\log(d/\eta)\log(n)}{n}} + \frac{s'\log^{1/2}(1/\delta)\log^{5/2}(nd/\eta)}{n\epsilon},
\end{equation}
for any $n > 0$, $s' > 0$, $d > 0$, $\epsilon > 0$, $\delta \in (0, 1)$ and $\eta \in (0, 1)$. Similar to the previous sections, this quantity is motivated by the error rate achieved by a central DP algorithm on a single site. Specifically, given $n$ data points on a single site with true sparse regression parameter $\beta^*$, we show in \Cref{thm: high-dim upper bound single-source} that there is an $(\epsilon, \delta)$-central DP estimator $\hat{\beta}$ (Algorithm \ref{algorithm:DPregression_high_dim_single_source}) that achieves $\|\hat{\beta} - \beta^*\|_2\lesssim r_{\textup{HLR}}(n, s, d, \epsilon, \delta, \eta)$, which is the minimax estimation error rate up to logarithmic factors, with probability at least $1-\eta$. Due to space constraints, we leave this subsection on private high-dimensional linear regression on a single data set to \Cref{sec-hd-app}.

\subsubsection{Federated private high-dimensional linear regression}

For the high-dimensional linear regression problem under the FTL setup \eqref{eq: high-dim reg model}, we adopt a strategy similar to that in \Cref{sec:lowd-regression}. Building on the single-site private estimator \Cref{algorithm:DPregression_high_dim_single_source} in \Cref{sec-hd-app}, we further introduce a federated private estimator \Cref{algorithm:DPregression_high_dim_combined} in \Cref{subsec: hd diff nk appendix}. The latter is a high-dimensional analogue of \Cref{algorithm:DPregression_federated}, which combines noisy mini-batch gradients from the target and source sites and then exploits the sparsity of the regression parameter via a hard-thresholding step.
 As in the previous sections, we consider 
applying \Cref{algorithm:DPregression_high_dim_combined} to the set 
\begin{equation}\label{algorithm:DPregression_high_dim_detection}
    \hat{\mA} := \{k \in [K]: \twonorm{\hbeta^{(k)} - \hbeta^{(0)}} \leq \tilde{c} \,r_{\textup{HLR}}(n, s', d, \epsilon, \delta, \eta/K)\},
\end{equation}
where $r_{\mathrm{HLR}}$ is defined in \eqref{eq:targetrate-hd}, $\tilde{c}>0$ is some constant to be chosen, and $\{\hbeta^{(k)}\}_{k \in \tkset}$ are obtained by applying \Cref{algorithm:DPregression_high_dim_single_source} onto half of the data at the $k$-th location. 
 
 We then combine these two procedures using the meta-algorithm \Cref{algorithm:DPregression_high_dim}, introduced in \Cref{subsec: hd diff nk appendix}, which decides whether aggregating information from candidate source sites is preferable to using the target data alone.

Our theoretical results rely on the following assumptions. 
\begin{asp}\label{asp: x}
    Assume $0< L^{-1} \leq \lambdamin(\bSigmak{k}) \leq \lambdamax(\bSigmak{k}) \leq L < \infty$ with some absolute constant $L \geq 1$, for all $k \in \tkset$.
\end{asp}

\begin{asp}\label{asp: beta}
(i) For all $k \in \mA$, $\onenorm{\bbetak{k} - \beta} \lesssim \sqrt{s}\twonorm{\bbetak{k} - \beta}$.
    (ii) For each $k \in \mA^c$, there exists a $\wbbetak{k} \in \mathbb{R}^d$ with $\zeronorm{\wbbetak{k}} \leq s$, $\onenorm{\bbetak{k} - \wbbetak{k}} \lesssim \sqrt{s}\twonorm{\bbetak{k} - \wbbetak{k}}$, and $\twonorm{\bbetak{k} - \wbbetak{k}} \leq c\twonorm{\bbetak{k} - \beta}$ with a small absolute constant $c > 0$ \footnote{It suffices that the constant $c$ satisfies $cC \leq 1/2$, where $C$ is the absolute constant appearing in Proposition \ref{prop: single_source}.(\rom{4}) in the appendix.}.
    (iii) It holds that $\min_{k \in \mA^c} \twonorm{\bbetak{k} - \beta} \geq Cr_{\textup{HLR}}(n,s, d,\epsilon,\delta, \eta)$, where $C > 0$ is a large enough absolute constant.
\end{asp}

\begin{remark}
    Assumption \ref{asp: x} is a common assumption in high-dimensional linear regression literature with a random design, where the minimum eigenvalue of $\bSigma^{(k)}$ is bounded away from zero to ensure a non-degenerated behaviour of the estimator, and the maximum eigenvalue of $\bSigma^{(k)}$ is bounded above to ensure the geometric convergence rate of gradient descent. Similar conditions can be found in \cite{jain2014iterative}, \cite{loh2015regularized}, and \cite{wainwright2019high} without privacy constraints, in \cite{cai2019cost} with privacy constraints. 
    
    Assumption \ref{asp: beta} consists of a set of technical assumptions. Unlike the low-dimensional case, \Cref{asp: beta}.(\rom{1}) is needed for the single-source algorithm (Algorithm \ref{algorithm:DPregression_high_dim_single_source}) to deliver an accurate estimation for the source data sets, which are not assumed to be generated from a sparse model. Recall that the target coefficient $\beta$ is assumed to be $\ell_0$-sparse, and, under \Cref{asp: beta}.(\rom{1}), the source coefficients in $\mA$ can be approximated by $\beta$ in the sense that $\twonorm{\bbetak{k}-\beta} \leq h$ and $\onenorm{\bbetak{k}-\beta} \lesssim \sqrt{s}\twonorm{\bbetak{k}-\beta}$. Assumption \ref{asp: beta}.(\rom{2}) guarantees that coefficients of sources in $\mA^c$ can be approximated by another $\ell_0$-sparse vector (which could be far away from $\beta$) in the same sense. Assumption \ref{asp: beta}.(\rom{3}) is similar to the condition in Theorem~\ref{thm:lowd-regression-transfer} in the low-dimensional case, imposed to guarantee a sufficiently large gap between sources inside and outside $\mA$. Together, these assumptions ensure that the detection step~\eqref{algorithm:DPregression_high_dim_detection} succeeds. Conditions similar to Assumption \ref{asp: beta} have been used in other high-dimensional transfer learning literature; see \cite{jun2022transfer} and \cite{tian2022transfer}.  
\end{remark}

With these assumptions, we have the following upper bound on the estimation error. 

\begin{thm}\label{thm: high-dim upper bound}
Let $\{\mathbf{X}^{(k)}, \mathbf{Y}^{(k)}\}_{k \in \{0\} \cup [K]}$ be generated from \eqref{eq: high-dim reg model}.  Initialise Algorithm \ref{algorithm:DPregression_high_dim} with $\bbeta^{0} = 0$. Suppose that Assumptions \ref{asp: x} and \ref{asp: beta} hold, $\max_{k \in \{0\} \cup \mA}\twonorm{\bbetak{k}} \leq C$ with some constant $C > 0$, $s \gtrsim s' \geq 4.18L^4s$, $n \gtrsim \epsilon^{-1}s\log^{1/2}(1/\delta)\log^{5/2}(nd/\eta)$, and $\sqrt{|\mA|+1}n \gtrsim \epsilon^{-1}\sqrt{ds}\log^{1/2}(1/\delta)\log^{5/2}[(|\mA|+1)nd/\eta]$.
	We then have the following.
	(i) Algorithm \ref{algorithm:DPregression_high_dim} is $(\epsilon, \delta)$-FDP. 
		(ii) There exists a choice of $\tilde{c} > 0$ in \eqref{algorithm:DPregression_high_dim_detection} such that the output $\hbeta$ from Algorithm \ref{algorithm:DPregression_high_dim} satisfies
          \begin{equation}
              \tp(\twonorm{\hbeta - \beta} \lesssim (\textup{\Rom{1}}) \wedge (\textup{\Rom{2}})) \geq 1-\eta,
          \end{equation}
          where
		\[
			(\textup{\Rom{1}}) = \sqrt{\frac{s\log(d/\eta)\log (n)}{n}} + \frac{s \log^{1/2}(1/\delta)\log^{5/2}(nd/\eta)}{n\epsilon}
   \]
   and
   \begin{equation}\label{eq: aggregation rate high dim}
       (\textup{\Rom{2}}) = \sqrt{\frac{s \log (d/\eta)\log ((|\mA|+1)n) }{(|\mA|+1)n}} + h + \frac{\sqrt{ds}\log^{1/2}(1/\delta)\log^{5/2}[\{(|\mA|+1)nd\}/\eta]}{\sqrt{|\mA|+1}n\epsilon}. 
   \end{equation}
\end{thm}

The rate in \Cref{thm: high-dim upper bound} is the minimum of the target-only rate (\Rom{1}) and the FDP rate (\Rom{2}), so Algorithm \ref{algorithm:DPregression_high_dim} adaptively decides whether to aggregate source information. In particular, \eqref{eq: aggregation rate high dim} improves upon (\Rom{1}) when the sources are sufficiently similar to the target,
\[
    h \ll \sqrt{\frac{s\log(d/\eta)\log (n)}{n}} + \frac{s \log^{1/2}(1/\delta)\log^{5/2}(nd/\eta)}{n\epsilon},
\]
and there are sufficiently many informative sources, namely
\[
    |\mA| \gg \Big(\frac{\log [((|\mA|+1)nd)/\eta]}{\log (nd/\eta)}\Big)^{5}\frac{d}{s}.
\]

Due to space constraints, we leave further discussion on the comparison across different privacy constraints to \Cref{sec:discussion-highd}.

\subsection{M-estimation}\label{sec:m-estimation}
In this section, we extend our framework to $M$-estimation problems with smooth and locally strongly convex functions, encompassing a broader class of statistical models including robust regression and generalised linear models. Suppose we have target data $D_0 = \{W_i^{(0)}\}_{i = 1}^n$ and source data sets 
$D_k = \{W_i^{(k)}\}_{i = 1}^n$. 
For some convex and twice differentiable loss function $\rho(w,\theta)$, let 
\[
\theta^{(k)} = \argmin_{\theta\in \Theta} \mathbb{E}_{P_{\theta^{(k)}}}[\rho(W,\theta)], \quad k \in \{0\}\cup [K].
\]
Unlike in the previous sections where we employed the informative source detection strategy (\Cref{sec:detection}), in this section, we assume for simplicity that the heterogeneity level across all $[K]$ source sites is bounded by $h$, i.e.,
   $ \max_{k\in [K]}\|\theta^{(k)} - \theta^{(0)}\|_2\leq h.$
   
If instead, the heterogeneity level is believed to be bounded by $h$ only on some unknown subset $\mA$ of $[K]$, then we can adopt the informative source selection strategy in \Cref{sec:detection}, as we did in \Cref{sec:mean-estimation,sec:lowd-regression,sec:highd-regression}.

Following the literature on private M-estimation \citep[e.g.][]{kifer2012private,avella2023differentially,xie2025online}, we impose the following assumptions.

\begin{asp}\label{asp:m-estimation} 
Let $\Psi(w,\theta) = \frac{\partial}{\partial\theta}\rho(w,\theta)$ and $L_k(\theta) = \frac{1}{n}\sum_{i=1}^n\rho(W_i^{(k)},\theta)$, $k \in [K] \cup \{0\}$. Consider the following assumptions
    \begin{enumerate}
    \item {The gradient is bounded} 
\begin{equation}\label{eq:Bgradient}
    \sup_{w,\theta}\|\Psi(w,\theta)\|_2 \leq B < \infty
\end{equation}
\item For $k \in [K] \cup \{0\}$, the loss function $L_k$ is $\tau_1$-strongly convex  on $B_r(\theta^{(k)})$ 
and $\tau_2$-smooth on $\Theta$. 
\end{enumerate}
\end{asp}

As discussed and demonstrated in \cite{avella2023differentially}, 
under an appropriate choice of the loss function, 
these assumptions hold either almost surely or with high probability. We present detailed examples in \Cref{sec:m-estimation examples}, including linear regression with Huber loss, logistic regression with Mallow's weight, and squared-loss linear regression under bounded-design assumptions. These examples illustrate that the constants $(\tau_1,\tau_2,B)$ can scale differently across models.

For M-estimation, we introduce \Cref{algorithm:M-est_federated} in \Cref{sec:m-appendix}, which is again a mini-batch gradient descent type algorithm similar to \Cref{algorithm:DPregression_federated}. The theoretical guarantee for \Cref{algorithm:M-est_federated} is presented in \Cref{prop:m-estimation-explicit}, which tracks the dependence on the strong convexity parameter $\tau_1$, the smoothness parameter $\tau_2$, and the gradient bound $B$. The proof is provided in \Cref{sec:m-appendix}.

\begin{thm}\label{prop:m-estimation-explicit}
Suppose that \Cref{asp:m-estimation} holds and the initializer 
$\|\theta_0 - \theta^{(0)}\|_2\lesssim 1$. Assume that $\Theta$ is bounded with diameter $D \lesssim 1$.
Let $\rho= \tau_2^{-1}/2, \tau_2 \geq 1$,
\[
n(K+1) \gtrsim \frac{\tau_2^2B^4{d}}{\epsilon^{2}\min\{\tau_1^4,\tau_1^6\}}\mathrm{polylog}(n,K,\eta,\delta,\tau_1,\tau_2),
\]
$T \asymp \tau_2\log(n(K+1))\tau_1^{-1}$ and $\tau_2h\tau_1^{-1} \lesssim r \asymp 1$.
Then the output of \Cref{algorithm:M-est_federated}, $\theta_T$, satisfies $(\epsilon,\delta)$-FDP and, with probability at least $1-\eta$, it holds that
\[
\|\theta_T - \theta^{(0)}\|_2 \lesssim \frac{\tau_2}{\tau_1}\,h 
+ \bigg[
\frac{B\sqrt{\tau_2}}{\tau_1^{3/2}}\cdot \sqrt{\frac{d}{n(K+1)}}
\;+\;
\frac{B\tau_2}{\tau_1^2}\cdot \frac{\sqrt{d}}{n\epsilon\sqrt{K+1}}
\bigg]\mathrm{polylog}(n,K,\eta,\delta,\tau_1,\tau_2),
\]
where $\mathrm{polylog}(\cdot)$ denotes a quantity that is polynomial in the logarithms of its arguments.
\end{thm}

\begin{remark}
    Several remarks are in order. (i) The assumption that $\Theta$ is bounded is only required to obtain a uniform control of the mini-batch gradient descent. We also provide a result for using the full batch gradient in \Cref{cor:full-batch mestimation}, where it can be removed. 
    \\
    (ii) The minimal sample size condition ensure all iterates stay in an appropriate strongly convexity region centred at $\theta^*$, where $\theta^* = \argmin \mathbb{E}[\overline{L}(\theta)]$ and $\overline{L}(\theta) = (K+1)^{-1}\sum_{k=0}^K L_k(\theta)$. In particular, with high probability, the added Gaussian noise in each step will not push the gradients outside this region. 
    \\
    (iii) The final bound has three terms, a heterogeneity term, a non-private term and a FDP-induced private error term. The heterogeneity term is not simply $h$ in this case, since we consider $\theta^*$ as an intermediate target in order to obtain the effective sample size of $n(K+1)$. Controlling $\|\theta^* - \theta\|$ using properties of $\overline{L}$ produces the term $\tau_2 h/\tau_1$. For the non-private error term, we can, in fact, remove the dependence on $B$ by further assuming that the gradient has constant order sub-Gaussian norm, i.e.\  $\sup_{w,\theta}\|\Psi(w,\theta)\|_{\psi_2}\lesssim 1$. Finally, when comparing against the central DP results under the same setup \citep[][Theorem 2]{avella2023differentially}, we observe that the FDP rate in \Cref{prop:m-estimation-explicit} also features a similar change of dependence from $K$ to $\sqrt{K}$, which is consistent with our findings in other sections.
\\
(iv) Compared to other upper bounds derived in Theorems~\ref{thm:mean-upperbound}, \ref{thm:lowd-regression-transfer} and \ref{thm: high-dim upper bound}, the estimation error rate in \Cref{prop:m-estimation-explicit} does not take the minimum with the target-only rate. This is because we directly assume that all sources satisfy the same heterogeneity bound $h$ and do not run an additional selection step. Incorporating the detection procedure from \Cref{sec:detection} would yield results analogous to those in earlier sections, automatically adapting between target-only and federated learning rates.
\\
(v) Existing results in the DP literature mostly treat $\tau_1,\tau_2,B$ as constants \citep[e.g.][Theorem 2]{avella2023differentially}, but tracking them explicitly is important for concrete models. For example, considering the linear regression setup with squared loss, as we show in \Cref{sec:m-estimation examples}, under a set of assumptions, it can be verified that $B \asymp \sqrt{d}$, $\tau_1\asymp\tau_2\asymp 1$. If we further assume that the covariates have constant-order sub-Gaussian norms, we recover the same FDP rate as in \Cref{sec:lowd-regression}, albeit under stronger model assumptions and minimal sample-size conditions. Note that although our version with tracked dependence of these model parameters clarifies the effects of dimensionality through these parameters when applied to specific statistical models, we do not claim, nor do we expect, the dependence on $\tau_1,\tau_2,B$ in \Cref{prop:m-estimation-explicit} to be optimal. We leave this as an open question for future research.
\end{remark}

\section{Numerical Experiments}\label{sec:numerical}

\subsection{Simulation}
We consider one target data set and $K = 10$ source data sets, where the $i$-th observation from the $k$-th data set (with the target denoted as the $0$-th data set) is generated from a $d$-dimensional linear model:
\begin{equation}\label{eq:simulation-model}
    Y^{(k)}_i = \langle X^{(k)}_i, \beta^{(k)} \rangle + \xi_i^{(k)}, \,\, X^{(k)}_i \sim \mathcal{N}(0, I_d), \,\, \xi_i^{(k)} \sim \mathcal{N}(0, 1), \,\, i \in [n] \mbox{ and } k \in \{0\} \cup [K].
\end{equation}

\subsubsection{Homogeneous settings}
The first simulation compares different privacy notions, namely central DP (CDP), LDP, and FDP, and validates our theoretical results on estimation errors. We set $\beta^{(k)} = \beta = d^{-1/2}(1, \ldots, 1)^\top$ for all $k \in {0} \cup [K]$. We fix $K = 10$ and $d = 20$. First, we vary $n$ from 20,000 to 100,000 by increments of 20,000, while fixing the privacy parameters to $\epsilon = 1$ and $\delta = 0.001$. Next, we vary $\epsilon$ from 0.6 to 2.4 in increments of 0.2, fixing $n = 60{,}000$ and $\delta = 0.001$. We compare the $\ell_2$ estimation error of $\beta$ under CDP, LDP and FDP. For this, we use Algorithm \ref{algorithm:DPregression_single} for CDP, Algorithm 1 from \cite{wang2019sparse} and \cite{wang2021sparse} for LDP, and Algorithm \ref{algorithm:DPregression_federated} for FDP (with $\hat{\mA} = [K]$). Note that when implementing Algorithm~\ref{algorithm:DPregression_single} and \ref{algorithm:DPregression_federated}, we replace the PrivateVariance (\Cref{algorithm:Private_variance}) component with PrivateVarianceGaussian (\Cref{algorithm:Private_varianceSG}), which is a version that achieves the same theoretical guarantees specifically under Gaussian assumptions as in \eqref{eq:simulation-model}. 

For CDP and LDP, we consider two scenarios: one using all data with the corresponding DP guarantee, and the other using only the target data. The first scenario serves as a benchmark for comparing CDP, LDP and FDP when the total sample size is the same. The second scenario highlights the performance gain of the FDP algorithm over private algorithms using only the target data. We append the suffixes `-all' and `-target' to distinguish between these two settings. The failure probability is set to $\eta = 0.01$.

\begin{figure}[!!h]
    \centering
    \includegraphics[width=\linewidth]{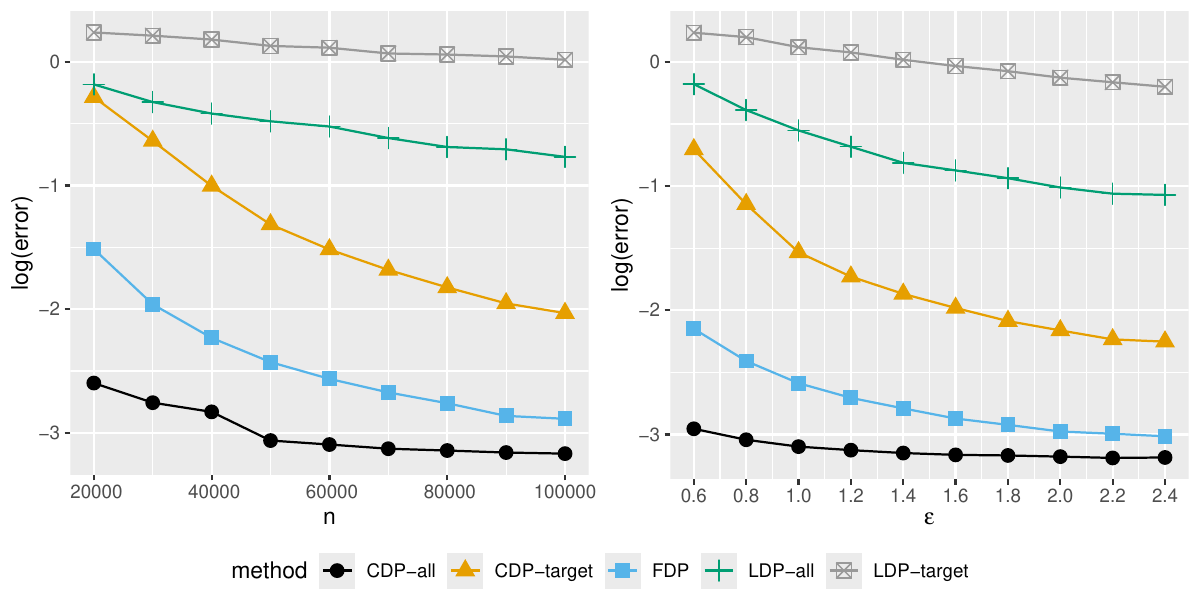}
    \caption{Comparison of estimation errors under different DP notions, when the sample size~$n$ (left) or the privacy parameter $\epsilon$ (right) changes.}
    \label{fig:dp-comparison}
\end{figure}

The results are shown in Figure \ref{fig:dp-comparison}, where the $y$-axis represents the $\ell_2$ estimation error on a logarithmic scale. As $n$ or $\epsilon$ increases, the performance of all methods improves. Notably, FDP leads to a higher estimation error than CDP-all, but lower than LDP-all, CDP-target and LDP-target, aligning well with our theoretical comparisons summarised in Table \ref{table: boundary-regression}.

\subsubsection{Heterogeneous settings}\label{subsubsec: hetero simulation}
The second simulation studies heterogeneous sources by setting $\beta^{(k)} = \beta + (h,0,\ldots,0)^\top$ for $k \in [K]$, while keeping $\beta^{(0)}$ unchanged. A larger $h$ corresponds to a greater distributional shift between the sources and the target. In addition to the methods from the previous setting, we include our Algorithm \ref{algorithm:DPregression_federated} with $\hat{\mA}$ selected via a private transferable source detection algorithm, as described in \eqref{eq:lowd-hatA}. We use the suffix `-detection' to distinguish FDP using all $K$ sources from FDP using only the detected transferable sources. We set $\tilde{c} = 1$ in the detection algorithm for simplicity. A sensitivity analysis for $\tilde{c}$ is provided in Section \ref{sec: sensitivity} of the Appendix, where we evaluate several choices of $\tilde{c}$. The results show very similar performance across different values of $\tilde{c}$, suggesting that our method is fairly robust to its choice in this setting. We vary $h$, fix $n = 100,000$ and use the same privacy parameters $(\epsilon, \delta)$ as in the previous simulation.

The results are summarised in Figure \ref{fig:h}. As $h$ increases, the performance of FDP-detection deteriorates initially but then improves and stabilises. The improvement observed as $h$ increases beyond a certain threshold is due to a clearer separation between the target and uninformative sources, which is essential for the detection algorithm to succeed. In contrast, CDP-all, LDP-all, and FDP continue to worsen due to the negative transfer effect.  Note that FDP-detection can still perform worse than CDP-target, although they are supposed to have the same error rate indicated by our theory. This is primarily due to data splitting for detection, which reduces practical performance, even if our detection algorithm outputs $\hat{\mA} = \emptyset$ when $h$ is large. When $h$ is small, FDP-detection still outperforms CDP-target even with data splitting, leveraging transferable sources. LDP-all performs better than CDP-all and FDP when $h$ is large because the LDP algorithm \citep{wang2019sparse, wang2021sparse} ultimately projects the estimate onto the $\ell_2$-unit ball, thereby constraining worst-case estimation errors.

\begin{figure}[!!ht]
    \centering
    \includegraphics[width=0.85\linewidth]{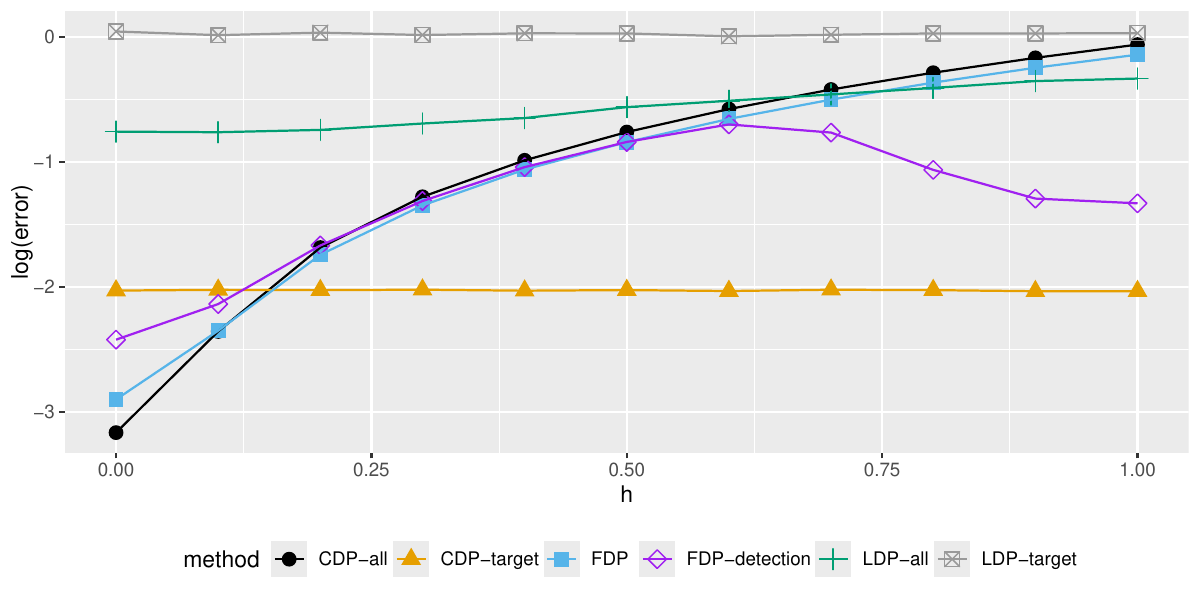}
    \caption{Performance of different methods under varying degrees of heterogeneity between target and sources.}
    \label{fig:h}
\end{figure}

We note that the ability of FDP-detection to identify transferable sources and avoid negative transfer depends on a sufficiently large gap between the sources in $\mA$ and $\mA^c$. Before the cutoff point at which FDP-detection starts to outperform FDP and CDP-all, this gap may not be large enough for the algorithm to reliably detect and exclude unhelpful sources. When $h$ becomes large, all sources are uninformative, and FDP-detection is able to identify and discard more of them, leading to improved performance.

\subsection{Real-data analysis}
We demonstrate our FDP framework using a real-data example in this section. The dataset contains information on students' final exam scores and several factors that may influence performance, such as study hours, class attendance, and sleep quality. \footnote{The full dataset is available at \url{https://www.kaggle.com/datasets/kundanbedmutha/exam-score-prediction-dataset/data}.}

According to the academic program (B.Com, B.Tech, B.Sc, etc.), the data are divided into seven subsets. The objective is to predict exam scores using all available data, while protecting the privacy of individual student records. We select five variables that show significant associations with exam scores: ``study\_hours", ``class\_attendance",  ``sleep\_quality", ``study\_method", and ``facility\_rating". Categorical variables are converted into dummy variables for regression analysis. The final dataset includes 12 covariates (including the intercept) and one response variable. The sample sizes for the seven programs are 2864, 2878, 2798, 2896, 2836, 2902, and 2826, respectively. We iteratively treat each subset as the target dataset and use the remaining six as source datasets. In each of the 200 replications, the target dataset is randomly split into 90\% for training and 10\% for testing. We consider the same methods as in the simulation study and use the same abbreviations. After fitting linear regression models under different privacy frameworks on the training data, we evaluate performance by computing the prediction error on the test data. The data are first standardised, processed under the privacy frameworks, and then transformed back to the original scale to calculate the prediction error.

\begin{figure}[!!ht]
    \centering
    \includegraphics[width=0.85\linewidth]{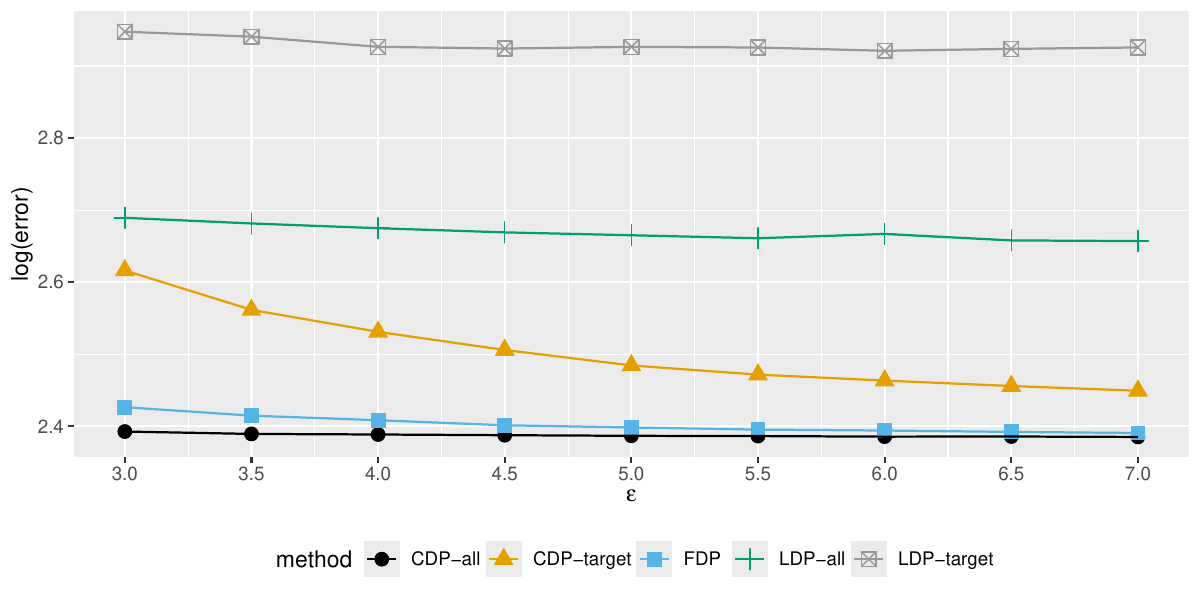}
    \caption{Performance of different methods under varying values of privacy parameter $\epsilon$.}
    \label{fig:exam}
\end{figure}

The performance of the different methods is summarised in Figure \ref{fig:exam}. FDP achieves substantially better performance than the other DP frameworks, except for CDP-all. As $\epsilon$ increases, the performance of FDP approaches that of CDP-all.

Next, we contaminate the first dataset (B.Com) by adding independent noise from $N(5, 1)$ to the response variable of each observation. We evaluate the different methods, including FDP-detection, based on their average performance across the remaining six datasets. The results are presented in Figure \ref{fig:exam-outlier}. FDP-detection remains robust in the presence of contamination, whereas the original FDP and CDP-all perform poorly.

\begin{figure}[!!ht]
    \centering
    \includegraphics[width=0.85\linewidth]{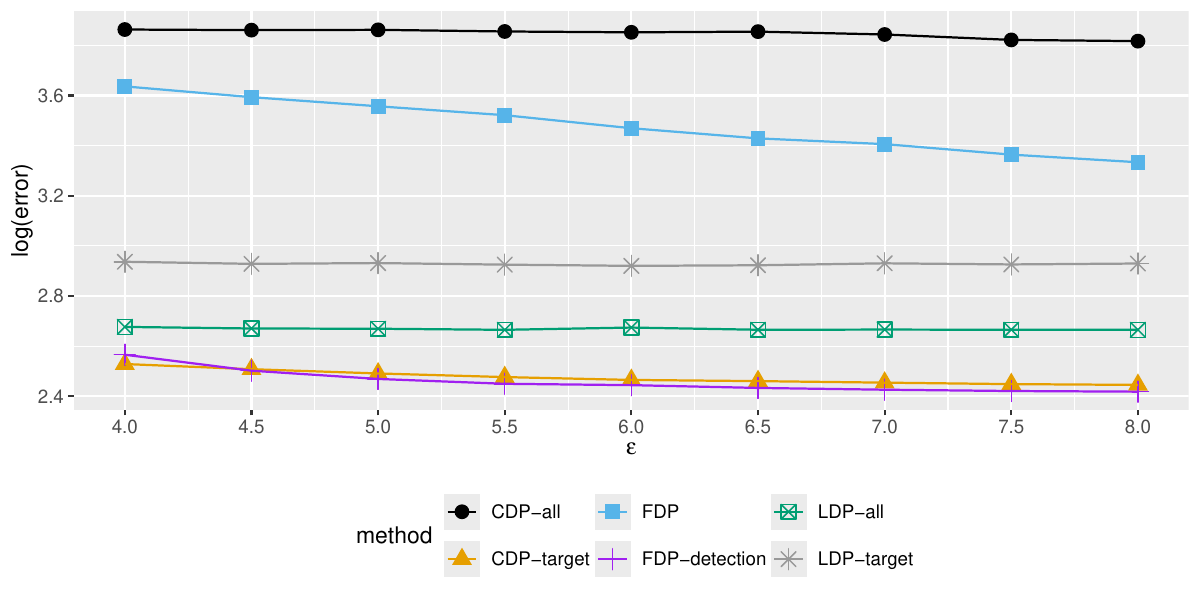}
    \caption{Performance of different methods under varying values of privacy parameter $\epsilon$ under contamination.}
    \label{fig:exam-outlier}
\end{figure}

\section{Conclusion}\label{sec:discussion}

In this work, we study the use of the FDP framework for privacy protection in FTL problems. Under this framework, we study four statistical problems, including univariate mean estimation, low-dimensional and high-dimensional linear regression, and M-estimation, focusing on the effects of privacy and data heterogeneity.  
While this paper primarily addresses the FTL problem with a focus on the target data set, the proposed algorithms can be applied individually to each site. This flexibility then allows extension to the federated multi-task learning paradigm, where the objective is to obtain parameter estimates for all sites. Several promising directions remain open for future research. A natural future direction is to develop statistical inference methods under the FDP constraint. Even without privacy, performing inference by borrowing strength from multiple heterogeneous sources remains challenging. Relevant discussions can be found in \cite{tian2022transfer}, \cite{cai2023statistical}, \cite{guo2023robust} and  \cite{tian2023comment}. In terms of private statistical inference, some recent works such as \cite{avella2021privacy}, \cite{avella2023differentially} and \cite{chadha2024resampling} provide useful starting points for developing FDP-compatible inference procedures.

\spacingset{0.88}

{\footnotesize
\bibliographystyle{abbrvnat}
\bibliography{ref}}

\spacingset{1.5}
\newpage
\appendix
\section*{Appendices}
All technical details are collected in the Appendices. We first introduce the background of local differential privacy in \Cref{section:ldp} and present a general, informative source detection method used throughout the paper in \Cref{sec:detection}. The proofs of results in \Cref{sec:mean-estimation,sec:lowd-regression,sec:highd-regression,sec:m-estimation}, as well as the corresponding single-site central DP results, are provided in \Cref{sec:appendix-2,sec:appendix-3,sec:appendix-4,sec:appendix-m-estimation}, respectively, along with some auxiliary results. 

\renewcommand{\contentsname}{Content of Appendices}
\addtocontents{toc}{\protect\setcounter{tocdepth}{5}}
\tableofcontents

\section{Local Differential Privacy (LDP)}\label{section:ldp}

One appealing feature of FDP is that it provides an intermediate privacy model between central DP and LDP; see the discussion in \Cref{sec:contribution} and the comparisons in \Cref{sec:mean-discussion,sec:lowd-lowerbound,sec:discussion-highd}. With central DP and FDP introduced in \eqref{Eq:alphaDP} and \eqref{eq:composition-eachstep}, we now turn to the concept of LDP. LDP is the strongest notion of privacy among these three.  Each user submits a privatised version of their data to the central server without passing through the site administrator. Formally, suppose that each user $u \in [U]$, $U \in \mathbb{Z}_+$, holds data $X_u \in \mathcal{X}$ and generates private data $Z_u \in \mathcal{Z}$ using some privacy mechanism $Q_u$. The private information $Z_u$ is said to be an $\epsilon$-LDP view of $X_u$, if for all $x_u, x_u' \in \mathcal{X}$, it holds that 
\begin{equation}\label{eq:LDP}
    Q_u(Z_u \in S|x_u) \leq \exp(\epsilon)Q_u(Z_u \in S|x_u'),
\end{equation}
for any measurable set $S$.  The version presented in \eqref{eq:LDP} is arguably the simplest design of LDP schemes, known as non-interactive LDP mechanisms \citep[e.g.][]{duchi2018minimax}. This also coincides with the notion of non-interactive FDP \eqref{eq:silo-level privacy} when $n_k = 1$ for all $k \in \{0\} \cup [K]$. More general mechanisms that allow some form of interaction among users have been considered in the literature \cite[e.g.][]{duchi2018minimax,duchi2019lower,joseph2019role,acharya2020unified}. When comparing with LDP settings in our paper, we focus on the pure $\epsilon$-LDP instead of the approximate $(\epsilon,\delta)$-LDP {since several existing works have shown that moving from pure to approximate in the LDP setting does not yield more accurate algorithms \cite[e.g.][]{bassily2015local,bun2019heavy,duchi2019lower}. }

\section{A general informative source detection strategy}\label{sec:detection}

Recall the general parameter space
\[
    \Theta(\mathcal{A},h) = \left\{\bm{\theta} = \{\theta^{(k)}\}_{k \in \tkset}: \max_{k \in \mathcal{A}} \rho(\theta^{(k)}, \theta^{(0)}) \leq h,\, \min_{k \in \mA^c} \rho(\theta^{(k)}, \theta^{(0)}) > h \right\}, 
\]
defined in \eqref{eq:general-space}. We note that whether the set $\mathcal{A}$ is informative regarding the target data, i.e.\ whether combining information therein can improve learning on the target data, depends on the value of $h$. We present a simple, general and effective procedure, which we exploit in all three problems considered in this paper. It automatically detects the true informative set, under a minor separation condition, so that we can apply appropriate private federated learning algorithms to combine the information.

Intuitively, when $\rho(\theta^{(k)}, \theta^{(0)})$ is small, the $k$-th source indexed by $\theta^{(k)}$ is expected to be informative in learning the target model indexed by $\theta^{(0)}$. However, we lack access to the true parameters $\{\theta^{(k)}\}_{k \in \tkset}$ in practice, and we have to use their estimators $\{\hat{\theta}^{(k)}\}_{k \in \tkset}$, which means we can view source $k$ as informative if $\rho(\hat{\theta}^{(k)}, \hat{\theta}^{(0)})$ is small. Determining how small it should be necessitates setting a threshold.  If $\rho(\hat{\theta}^{(k)}, \hat{\theta}^{(0)})$ falls below that threshold, we will consider source $k$ as informative for learning the target. Suppose $\rho(\setheta,\stheta) \leq r$ for all $k \in \tkset $ with high probability. Later on, we will argue that $\tilde{c}r$ is a good threshold to use, where $\tilde{c}$ is some constant to be specified. Formally, we will select the informative data sources as
\begin{equation}\label{eq:general-hatA}
    \hat{\mathcal{A}} = \{k \in [K]: \, \rho(\hat{\theta}^{(k)},\hat{\theta}^{(0)}) \leq \tilde{c} \,r\}. 
\end{equation}
Note that given the privacy concern, we will consider obtaining $\hat{\mA}$ as part of the FDP framework (\Cref{def:interactive-FDP}), which essentially imposes some privacy constraint on computing $\hat{\theta}^{(k)}$ at each site. To facilitate readers' comprehension, we provide an intuitive schematic (Figure \ref{detection-schematic}) to illustrate the detection strategy.

\begin{figure}[!ht]
    \centering
    \includegraphics[width=0.8\textwidth]{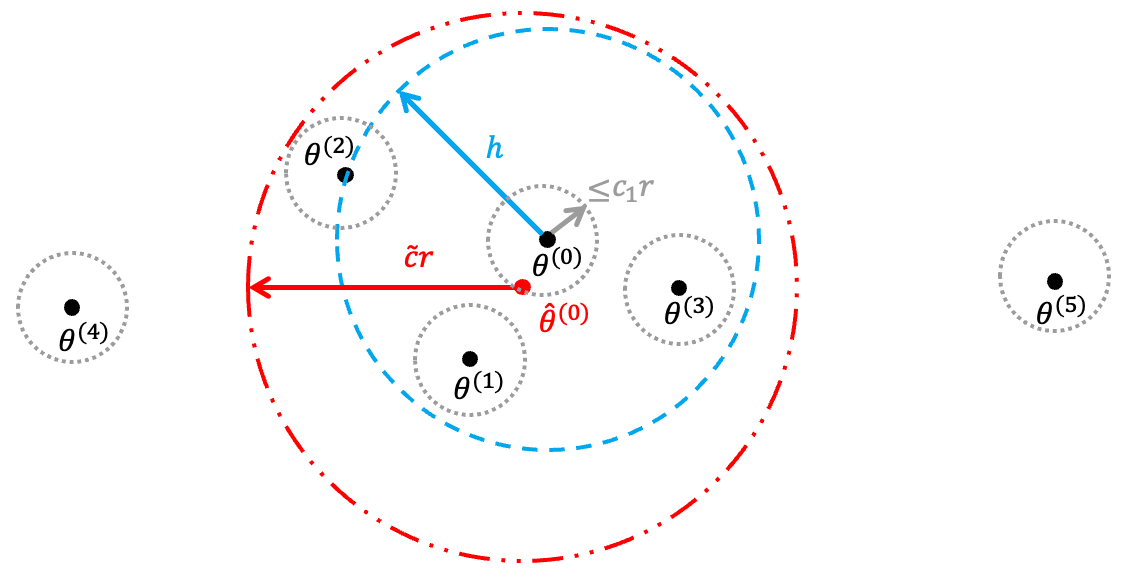}
    \caption{An illustration of the informative source detection strategy. The blue dash-line circle denotes the range of similarity levels between the target and sources in $\mA$. The red dash-dot-line circle represents the threshold for determining the informative set $\hat{\mA}$. Each grey dot-line circle refers to the estimation error range of each $\theta^{(k)}$ by using local data from each site. In this example,  $\hat{\mA} = \mA = \{1,2,3\}$ and the outlier source index set $\mA^c = \{4, 5\}$. }
    \label{detection-schematic}
\end{figure}

The informative set $\hat{\mA}$ is selected by comparing the private parameter estimates on each source data set with those on the target data set, and using the accuracy of the target estimate as a threshold. \Cref{lemma:selection-consistency} below (the second claim) shows that, when $h \lesssim r$, $\hat{\mathcal{A}}$ can identify $\mA$ with high probability, if the source data sets in $\mA^c$ are sufficiently different from the target. In this case, we may regard $\mA$ itself as informative regarding the target data set. However, when $h$ is much larger than $r$, the set~$\mA$ may contain disparate sources, and we should not aim to recover $\mA$. In those cases, we will show later in each specific problem that combining information in $\hat{\mathcal{A}}$ with the target data never does worse than using the target data alone. In particular, we exploit the first claim in \Cref{lemma:selection-consistency}, which says that the source data sets that get selected into $\hat{\mA}$ have the similarity level measured by $\rho$ bounded by $r$, the estimation accuracy on the target data, up to some constant. Together, we have a simple procedure that is adaptive to different values of $h$, which ensures that our main results, as summarised in \Cref{sec:contribution}, hold for any value of $h$.

\begin{lemma}\label{lemma:selection-consistency} Let $\eta \in (0,1)$ and $\tilde{c}$ be the constant in $\hat{\mA}$ defined in \eqref{eq:general-hatA}. Suppose that there exists some absolute constant $c_1 >0$ such that 
\[
    \mathbb{P}( \rho(\hat{\theta}^{(k)},\theta^{(k)}) \leq c_1 r) \geq 1-\eta, \; \forall k \in \{0\} \cup [K], 
    \]
then 
\[
\mathbb{P}\Big(\max_{k \in \hat{\mA}} \rho(\theta^{(k)}, \theta^{(0)}) \leq (2c_1+\tilde{c})r \Big) \geq 1-(K+1)\eta.
\]
In addition, if $h < c_1 r$ and $\rho({\theta}^{(k)},{\theta}^{(0)}) \geq c_2 r, $ for all $k \not\in \mathcal{A}$ with some absolute constant $c_2> 5c_1$, then 
 choosing $\tilde{c}$ such that $3c_1 < \tilde{c} < (c_2-2c_1)$ guarantees that
    \[
    \mathbb{P}(\hat{\mathcal{A}} = \mathcal{A}) \geq 1-(K+1)\eta.
    \]
\end{lemma}

As mentioned at the beginning of \Cref{sec:intro}, the {presence of disparate sources} is often overlooked in federated learning literature. Many existing works either assume that the same model is shared across {all sites} or focus on controlling the average risk.  Among those that account for disparate sites, there are two main approaches. The first one involves weighting multiple estimators, each derived from a different source \citep[e.g.][]{li2022transfer, lin2022transfer, li2023estimation}. The weights are typically determined based on an estimated similarity between the target and sources. 
The second approach is based on a detection method, such as the one used in this paper.  This can be seen as a `hard-thresholding' version of the weighting approach. While similar methods have been applied in the transfer and multi-task learning literature, these efforts typically focus on specific models \citep[e.g.][]{jun2022transfer, tian2022transfer}, whereas we provide a general formulation that can be applied to a broad range of transfer learning problems.

\begin{proof}[Proof of \Cref{lemma:selection-consistency}]
Conditional on the event that $\Big\{\rho(\setheta,\stheta) \leq c_1 r, \forall k \in \{0\}\cup \mathcal{A} \Big\}$, which happens with probability at least $1-(K+1)\eta$. It follows immediately from triangle inequality and the definition of $\hat{\mA}$ that 
\[
\max_{k \in \hat{\mA}} \rho(\theta^{(k)}, \theta^{(0)}) \leq (2c_1+\tilde{c}) r.
\]
We then show that the second claim $\hat{\mA} = \mA$ holds with high probability.  For $k \in \mathcal{A}$, using the triangle inequality and the fact that $h < c_1 r$, we have that
    \[
    \rho(\setheta,\tetheta) \leq 3c_1r. 
    \]
    For $k \in \mathcal{A}^c$, we have that
    \[
    \rho(\setheta,\tetheta) + \rho(\setheta,\stheta) + \rho(\tetheta,\ttheta) \geq \rho(\ttheta,\stheta) \geq c_2 r,
    \]
    and therefore
    $\rho(\setheta,\tetheta) \geq (c_2-2c_1)r$. Combining the two statements above, choosing $\tilde{c}$ such that $3c_1 < \tilde{c} < (c_2-2c_1)$ ensures that
    \[
    \mathbb{P}(\hat{\mathcal{A}} = \mathcal{A}) \geq 1-(K+1)\eta. 
    \]
\end{proof}

\section{Technical Details of Section \ref{sec:mean-estimation}}\label{sec:appendix-2}

\subsection{Proofs of results in Section \ref{sec:mean-estimation}}
\subsubsection{Proof of Theorem \ref{thm:mean-upperbound}}
In the proof below, we use \citet[][Theorem 3.1]{karwa2017finite}, which was originally stated and proved under Gaussian distribution assumptions, but we note that the same guarantee holds under sub-Gaussian assumptions if the bin length in their Algorithm 1 is multiplied by a constant.

Note that we can write 
\begin{align*}
    \tilde{\mu}-\mu & = \sum_{k \in \{0\}\cup \hat{\mathcal{A}}}v_k (\hat{\mu}^{(k)}-\mu) \\ & = \sum_{k \in \{0\}\cup \hat{\mathcal{A}}}\frac{2v_k(X^{(k)}_{\max}-X^{(k)}_{\min})}{n_k \epsilon}Z_k + \sum_{k \in \hat{\mathcal{A}}} v_k ({\mu}^{(k)}-\mu)+\sum_{k \in \{0\}\cup \hat{\mathcal{A}}}\frac{v_k\sum_{i = 1}^{n_k}(Y_i^{(k)}-\mu^{(k)})}{n_k}.
\end{align*}
Further, we can write the last term as 
\[
\sum_{k \in \{0\}\cup \hat{\mathcal{A}}}\frac{v_k\sum_{i = 1}^{n_k}(Y_i^{(k)}-\mu^{(k)})}{n_k} = \sum_{k \in \{0\}\cup \hat{\mathcal{A}}}\frac{v_k\sum_{i = 1}^{n_k}(Y_i^{(k)}-X_i^{(k)})}{n_k} + \sum_{k \in \{0\}\cup \hat{\mathcal{A}}}\frac{v_k\sum_{i = 1}^{n_k}(X_i^{(k)}-\mu^{(k)})}{n_k}.
\]

We consider two cases, i.e.\ $h < c_1f(n_0, \eta, \epsilon)$ and $h \geq c_1 f(n_0, \eta, \epsilon)$, where $c_1$ is the appropriate constant that guarantees the condition in \Cref{lemma:selection-consistency}.  When $h <c_1 f(n_0, \eta, \epsilon)$, we have 
\[
\mathbb{P}(\hat{\mathcal{A}} = \mathcal{A}) \geq 1-(K+1)\eta, 
\]
under appropriate conditions, as shown in \Cref{lemma:selection-consistency}, and therefore
\begin{align*}
    \mathbb{P}(|\tilde{\mu}-\mu|>t) - (K+1)\eta &\leq \mathbb{P}\bigg(\bigg|\sum_{k \in \{0\}\cup {\mathcal{A}}}v_k (\hat{\mu}^{(k)}-\mu) \bigg| > t\bigg) \\
    &  \leq \mathbb{P}\left(\bigg| \sum_{k \in \{0\}\cup {\mathcal{A}}}\frac{2v_kC_0 \sqrt{\log(n_k /\eta)}}{n_k \epsilon}Z_k \bigg| > \frac{t-\max_{k \in \mathcal{A}}\alpha^{(k)}}{2}\right) \\ & \hspace{5em} + \mathbb{P}\left( \bigg|\sum_{k \in \{0\}\cup {\mathcal{A}}}\frac{v_k\sum_{i = 1}^{n_k}(Y_i^{(k)}-\mu^{(k)})}{n_k} \bigg|> \frac{t-\max_{k \in \mathcal{A}}\alpha^{(k)}}{2} \right) \\
    &  = (I) + (II), 
\end{align*}
where we use \citet[][Theorem 3.1]{karwa2017finite} in the second inequality to bound $X^{(k)}_{\max}-X^{(k)}_{\min} \leq C_0 \sqrt{\log(n_k /\eta)}$ for some absolute constant $C_0 >0$.
For the first term, applying Bernstein's inequality \cite[e.g.][Theorem 2.8.2]{vershynin2018high}, with $w_k = \frac{2v_k C_0\sqrt{\log(n_k /\eta)}}{n_k \epsilon}$ for $k \in \{0\}\cup {\mathcal{A}}$, we obtain for $t > h$
\[
(I) \leq \exp\left(-c\min\left\{\frac{(t-h)^2}{(\sum_{k \in \{0\}\cup {\mathcal{A}}}w_k^2)}, \frac{t-h}{\max_{k \in \{0\}\cup {\mathcal{A}}}w_k}\right\}\right) \leq \eta 
\]
by choosing 
\[
t \asymp h + \sqrt{\sum_{k \in \{0\}\cup {\mathcal{A}}}w_k^2}\log(1/\eta) \asymp h + \log(1/\eta)\sqrt{\sum_{k \in \{0\}\cup {\mathcal{A}}}\frac{v_k^2 {\log(n_k /\eta)}}{n_k^2 \epsilon^2}}, 
\]
since $\max_{k \in \mathcal{A}}\alpha^{(k)} \leq h$. 
For the second term, consider the event $B = \{Y_i^{(k)} = X_i^{(k)}, k \in \tkset, i \in [n_k]\}$, i.e.\ there is not any $X_i^{(k)}$ that lies outside the truncation thresholds $X^{(k)}_{\min}$ and $ X^{(k)}_{\max}$. \citet[][Theorem 3.1]{karwa2017finite} shows that under \eqref{eq:samplesize_condition}, $\mathbb{P}(B) \geq 1- (|\mathcal{A}|+1)\eta$. Now, we can bound 
\[
(II) \leq (K+1)\eta+ \mathbb{P}\left(\bigg|\sum_{k \in \{0\}\cup \mathcal{A}}\frac{v_k\sum_{i = 1}^{n_k}(X_i^{(k)}-\mu^{(k)})}{n_k} \bigg| > (t-h)/2\right) \leq (K+2)\eta
\]
by choosing $t \asymp h+ \sqrt{\sum_{k \in \{0\}\cup {\mathcal{A}}}v_k^2\log(1/\eta)/n_k}$. Together, we have when $h < c_1 f(n_0, \eta, \epsilon)$, for some absolute constant $C_1>0$
\[
\mathbb{P}\left(|\tilde{\mu}-\mu| > C_1\left(h+ \sqrt{\sum_{k \in \{0\}\cup {\mathcal{A}}}\frac{v_k^2\log(1/\eta)}{n_k}}+\log(1/\eta)\sqrt{\sum_{k \in \{0\}\cup {\mathcal{A}}}\frac{v_k^2 {\log(n_k /\eta)}}{n_k^2 \epsilon^2}} \right)\right) \leq (2K+4)\eta.
\]
Note that 
\begin{align*}
    h+ \sqrt{\sum_{k \in \{0\}\cup {\mathcal{A}}}\frac{v_k^2\log(1/\eta)}{n_k}}&+\log(1/\eta)\sqrt{\sum_{k \in \{0\}\cup {\mathcal{A}}}\frac{v_k^2 {\log(n_k /\eta)}}{n_k^2 \epsilon^2}}\\
    &\lesssim h+\log(1/\eta)\sqrt{\sum_{k \in \{0\}\cup {\mathcal{A}}}\frac{v_k^2}{n_k} + \frac{v_k^2 {\log(n_k /\eta)}}{n_k^2 \epsilon^2}} \\
    &\lesssim h+\log(1/\eta)\sqrt{\max_{k \in \{0\}\cup {\mathcal{A}}}\log(n_k/\eta)}\sqrt{\frac{\sum_{k \in \{0\}\cup {\mathcal{A}}} u_k}{(\sum_{k \in \{0\}\cup {\mathcal{A}}}u_k)^2}}\\
    &\lesssim h+\frac{\log(1/\eta)\sqrt{\max_{k \in \{0\}\cup {\mathcal{A}}}\log(n_k/\eta)}}{\sqrt{\sum_{k \in \{0\}\cup {\mathcal{A}}}u_k}} \\
    &=h+\frac{\log(1/\eta)\sqrt{\max_{k \in \{0\}\cup {\mathcal{A}}}\log(n_k/\eta)}}{\sqrt{\sum_{k \in \{0\}\cup {\mathcal{A}}}(n_k\wedge n_k^2\epsilon^2)}}
\end{align*}

When $h \geq c_1 f(n_0, \eta, \epsilon)$, using the first part of \Cref{lemma:selection-consistency}, we have 
\[
\mathbb{P}\left(\max_{k \in \hat{\mathcal{A}}}\alpha^{(k)} \lesssim  f(n_0,\eta,\epsilon)\right) \geq 1-(K+1)\eta.
\]
For any subset $S$ of $[K]$,  we also have 
\[
\frac{\log(1/\eta)\sqrt{\max_{k \in S\cup\{0\}}\log(n_k/\eta)}}{\sqrt{\sum_{k\in S\cup\{0\}}(n_k\wedge n_k^2\epsilon^2)}} \lesssim \frac{\log(1/\eta)\sqrt{\max_{k \in S\cup\{0\}}\log(n_k/\eta)}}{\sqrt{(S+1)(n_0 \wedge n_0^2\epsilon^2)}} \lesssim f(n_0,\eta,\epsilon),
\]
where the first inequality is due to $\min_{k \in [K] } n_k \gtrsim n_0$, and the second inequality is due to $\log(\max_{k \in [K]}n_k) \lesssim \log(n_0)$. Therefore, we can use the same concentration argument as before to obtain 
\[
\mathbb{P}(|\tilde{\mu} - \mu| \lesssim f(n_0,\eta,\epsilon) ) \geq 1-(2K+4)\eta,
\]
when $h \geq c_1 f(n_0, \eta, \epsilon)$. Finally, we have 
\begin{equation*}
    \mathbb{P}\left(|\tilde{\mu} - \mu| \lesssim f(n_0,\eta,\epsilon) \wedge \left(h+ \frac{\log(1/\eta)\sqrt{\max_{k \in \{0\}\cup {\mathcal{A}}}\log(n_k/\eta)}}{\sqrt{\sum_{k \in \{0\}\cup {\mathcal{A}}}(n_k\wedge n_k^2\epsilon^2)}}\right) \right)  \geq 1- (2K+4)\eta, 
\end{equation*}
as claimed. 

\subsubsection{Proof of Theorem \ref{lemma:lowerbound-mean}}
Throughout the proof, we choose the distribution to be $X_i^{(k)} \sim \mathcal{N}(\mu^{(k)},\sigma_k^2)$, with $\sigma_k  = \sigma = 1$ for all $k \in \tkset$.
We consider two specifications of $\bm{\mu}$: (a) $\mu^{(k)} = \mu,$ for $ k \in \mathcal{A}$, (b) $\mu^{(k)} = 0,$ for $ k \in {\mathcal{A}}$. {In both cases, we set $\mu^{(k)} = \mu_{\mathrm{out}}$ for $k \in \mA^c$, where $\mu_{\mathrm{out}}$ is a fixed constant satisfying $|\mu_{\mathrm{out}}| > 2h$, ensuring that $\bm{\mu} \in \Theta_{\bm{\mu}}(\mathcal{A}, h)$.  We write $\mathcal{A}_0 = \{0\}\cup\mathcal{A}$.}

Under (a), {data from sites in $\mathcal{A}_0$} are i.i.d.~from $\mathcal{N}(\mu, \sigma^2)$ but the problem does not simply reduce to the estimation of $\mu$ under central DP due to the lack of a trusted central server. For each $k \in \tkset$, let $M_k(\{x_i^{(k)}\}_{i=1}^{n_k})$ be an $(\epsilon,\delta)$-central DP mechanism that is applied to the data realisations $\{x_i^{(k)}\}_{i=1}^{n_k}$ from the $k$-th site. Let $Q_k(\cdot|\{x_i^{(k)}\}_{i=1}^{n_k})$ denote the conditional distribution of $M$ given the data in the $k$-th site. Note that the marginal distribution of $M_k$ can be written as 
\[
\mathbb{M}_{k,\mu}(S) = \int Q_k(S|\{x_i^{(k)}\}_{i=1}^{n_k}) \mathrm{d}P^{\otimes n_k}_{{\mu}}(\{x_i^{(k)}\}_{i=1}^{n_k}),
\]
where $P^{\otimes n_k}_{{\mu}}$ is the $n_k$-fold product measure of $P_\mu = \mathcal{N}(\mu,\sigma^2)$. \Cref{lemma-kv-densityratio} shows that for any $\mu_1,\mu_2 \in \mathbb{R}$
\begin{align*}
    \mathbb{M}_{k,\mu_1}(S) &\leq \exp(6\epsilon n_k \mathrm{TV}(P_{\mu_1},P_{\mu_2}) )\mathbb{M}_{k,\mu_2}(S) + 4\exp(6\epsilon n_k \mathrm{TV}(P_{\mu_1},P_{\mu_2}) )n_k\delta \mathrm{TV}(P_{\mu_1},P_{\mu_2}) \\
    & \leq \exp(3\epsilon n_k|\mu_1-\mu_2|\sigma^{-1} )\mathbb{M}_{k,\mu_2}(S) + 2n\delta\sigma^{-1}|\mu_1-\mu_2|\exp(3\epsilon n_k\sigma^{-1}|\mu_1-\mu_2| ), 
\end{align*}
and 
\[
\mathbb{M}_{k,\mu_2}(S) \leq \exp(3\epsilon n_k|\mu_1-\mu_2|\sigma^{-1} )\mathbb{M}_{k,\mu_1}(S) + 2n_k\sigma^{-1}\delta|\mu_1-\mu_2|\exp(3\epsilon n_k|\mu_1-\mu_2|\sigma^{-1} )
\]
for any $(\epsilon,\delta)$-central DP mechanism $Q_k$ and any measurable set $S$, where we use the fact that $\mathrm{TV}(P_{\mu_1},P_{\mu_2}) \leq |\mu_1-\mu_2|/(2\sigma)$ \citep[e.g.~Theorem 1.3 in][]{devroye2018total}. Write $\epsilon'_k = 3\epsilon n_k|\mu_1-\mu_2|\sigma^{-1} $ and $\delta'_k = 2\sigma^{-1}n_k\delta|\mu_1-\mu_2|\exp(\epsilon'_k )$, we have 
\begin{align}\label{eq:max-divergence}
    D_{\infty}^{\delta'_k}(\mathbb{M}_{k,\mu_1},\mathbb{M}_{k,\mu_2}) &= \max_{S:\, \mathbb{M}_{k,\mu_1}(S) \geq \delta_k'}\log\Bigg(\frac{\mathbb{M}_{k,\mu_1}(S) - \delta'_k}{\mathbb{M}_{k,\mu_2}(S)}\Bigg) \leq \epsilon'_k \nonumber \\ \text{and} \quad D_{\infty}^{\delta'_k}(\mathbb{M}_{k,\mu_2},\mathbb{M}_{k,\mu_1}) &\leq \epsilon'_k.
\end{align}
\Cref{lemma:dwork-approximatedp} shows that \eqref{eq:max-divergence} holds if and only if there exist $\mathbb{M}_{k}', \mathbb{M}_{k}''$ such that $\mathrm{TV}(\mathbb{M}_{k,\mu_1}, \mathbb{M}_{k}') \leq \delta'_k/(e^{\epsilon'_k}+1), \mathrm{TV}(\mathbb{M}_{k,\mu_2}, \mathbb{M}_{k}'') \leq \delta'_k/(e^{\epsilon'_k}+1)$,
\[
\mathbb{M}_{k}'(S) \leq \exp(\epsilon'_k)\mathbb{M}''_{k}(S) \quad \text{and} \quad \mathbb{M}_{k}''(S) \leq \exp(\epsilon'_k)\mathbb{M}'_{k}(S). 
\]
To prepare for the lower bound, we note that the following two bounds are useful. First, standard data processing inequality implies
\[
\mathrm{KL}(\mathbb{M}_{k,\mu_1}, \mathbb{M}_{k,\mu_2}) \leq \mathrm{KL}(P_{\mu_1}^{\otimes n_k},P_{\mu_2}^{\otimes n_k}) = n_k\frac{|\mu_1-\mu_2|^2}{2\sigma^2}
\]
The second bound is on $\mathrm{KL}(\mathbb{M}_{k}'',\mathbb{M}_{k}')$. Without loss of generality, we assume that $\mathbb{M}_{k}''$ and $\mathbb{M}_{k}'$ admit densities $m_{k}''(z)$ and $m_{k}'(z)$ with respect to some measure $\mu$ (e.g.\ $(\mathbb{M}_{k}''+\mathbb{M}_{k}')/2$), respectively.  Following \cite{duchi2018minimax}, we have that
\begin{align*}
    \mathrm{KL}(\mathbb{M}_{k}'',\mathbb{M}_{k}') &\leq \mathrm{KL}(\mathbb{M}_{k}'',\mathbb{M}_{k}')+\mathrm{KL}(\mathbb{M}_{k}',\mathbb{M}_{k}'')\\
    & = \int (m_{k}''(z) - m_{k}'(z))\log\Big(\frac{m_{k}''(z)}{m_{k}'(z)}\Big) \mathrm{d}\mu(z) \\
    & \leq \int \frac{(m_{k}''(z) - m_{k}'(z))^2}{\min\{m_{k}''(z),m_{k}'(z)\}} \mathrm{d}\mu(z)\\
    & \leq \int \frac{(m_{k}''(z))^2}{\min\{m_{k}''(z),m_{k}'(z)\}}\{\exp{(\epsilon'_k)}-1\}^2 \mathrm{d}\mu(z) \\
    & \leq 4\exp(\epsilon'_k)(\epsilon'_k)^2 \leq 12(\epsilon'_k)^2 = 109\epsilon^2n_k^2(\mu_1-\mu_2)^2/\sigma^2,
\end{align*}
provided that $\epsilon'_k \leq 1$, which we shall verify later.  

We write $\mathcal{A}_0 = \{0 \} \cup \mathcal{A}$ and let $S=\{k \in \mathcal{A}_0:n_k< 1/\epsilon^2\}$. We choose $\mu_1$ and $\mu_2$ such that 
\[
|\mu_1-\mu_2|=\frac{c}{\sqrt{\sum_{k\in\mathcal{A}_0}(n_k\wedge n_k^2\epsilon^2)}}.
\]
for some absolute constant $c>0$. Now, for $k \in S$, we have 
\[
\epsilon_k' = \frac{3c\epsilon n_k}{\sqrt{\sum_{k\in\mathcal{A}_0}(n_k\wedge n_k^2\epsilon^2)}} \leq 3c \leq 1,
\]
by choosing $c \leq 1/3$. Moreover, for $\delta \leq \epsilon/(K+1)$, it holds that 
\[
\sum_{k\in S}\delta_k' \leq \sum_{k\in S}6n_k\delta|\mu_1-\mu_2|\leq \frac{\sum_{k \in [K]\cup\{0\}}6n_k\delta c}{\sqrt{\sum_{k\in\mathcal{A}_0}(n_k\wedge n_k^2\epsilon^2)}} \leq \frac{\sum_{k \in [K]\cup\{0\}}6\delta c}{\epsilon} \leq \frac{1}{4},
\]
by choosing $c$ sufficiently small.
Using Le Cam's Lemma \cite[e.g.][]{yu1997assouad}, we obtain
\begin{equation}\label{eq:lecam-mean}
    \inf_{\mathcal{Q}_{\epsilon,\delta,1}} \inf_{\hat{\mu}}\sup_{P \in \mathcal{P}_{\bm{\mu}}}\mathbb{E}|\hat{\mu} - \mu| \geq \frac{|\mu_1 - \mu_2|}{2}\left\{1-\mathrm{TV}\left(\prod_{k \in \mathcal{A}_0}\mathbb{M}_{k,\mu_1}, \prod_{k \in \mathcal{A}_0 }\mathbb{M}_{k,\mu_2}\right)\right\}.
\end{equation}
\begin{align*}
    \mathrm{TV}\left(\prod_{k \in \mathcal{A}_0}\mathbb{M}_{k,\mu_1}, \prod_{k \in \mathcal{A}_0 }\mathbb{M}_{k,\mu_2}\right) &\leq \mathrm{TV}\bigg(\prod_{k \in \mathcal{A}_0 }\mathbb{M}_{k,\mu_1}, \prod_{k \in S }\mathbb{M}_{k}'\prod_{k \in S^c} \mathbb{M}_{k,\mu_1}\bigg)\\& +\mathrm{TV}\bigg(\prod_{k \in \mathcal{A}_0 }\mathbb{M}_{k,\mu_2}, \prod_{k \in S }\mathbb{M}_{k}''\prod_{k\in S^c}\mathbb{M}_{k,\mu_2}\bigg) \\ \hspace{10em} &+ \mathrm{TV}\bigg(\prod_{k \in S }\mathbb{M}_{k}'\prod_{k \in S^c} \mathbb{M}_{k,\mu_1},\prod_{k \in S }\mathbb{M}_{k}''\prod_{k\in S^c}\mathbb{M}_{k,\mu_2}\bigg) \\
    & \leq \sum_{k \in S}\delta_k' + \sqrt{\sum_{k \in S}\mathrm{KL}(\mathbb{M}_{k}'',\mathbb{M}_{k}')+\sum_{k \in S^c}\mathrm{KL}(\mathbb{M}_{k,\mu_1}, \mathbb{M}_{k,\mu_2})} \\
    &\leq \frac{1}{4}+\sqrt{\frac{\sum_{k \in S}109\epsilon^2n_k^2c^2}{\sum_{k \in S}n_k \wedge n_k^2\epsilon^2}+\frac{\sum_{k \in S^c}n_kc^2}{\sum_{k\in S^c}n_k \wedge n_k^2\epsilon^2}} \\
    &\leq \frac{1}{4}+\sqrt{110c^2} < 1/2,
\end{align*}
for $c$ sufficiently small.
Therefore, we have 
\[
\inf_{\mathcal{Q}_{\epsilon,\delta,1}} \inf_{\hat{\mu}}\sup_{P \in \mathcal{P}_{\bm{\mu}}}\mathbb{E}|\hat{\mu} - \mu| \gtrsim \frac{1}{\sqrt{\sum_{k\in\mathcal{A}_0}(n_k\wedge n_k^2\epsilon^2)}}.
\]

\medskip
Under (b), we consider 
\begin{equation}\label{eq-same-mu}
   {\mu^{(k)} = 0 \text{ for } k \in \mathcal{A}, \quad \mu^{(k)} = \mu_{\mathrm{out}} \text{ for } k \in \mA^c,}
\end{equation} 
and it is sufficient to show that 
\begin{equation}\label{eq:mean-singlerate}
    \inf_{\mathcal{Q}_{\epsilon,\delta,1}} \inf_{\hat{\mu}}\sup_{P \in \mathcal{P}_{\bm{\mu}}}\mathbb{E}|\hat{\mu} - \mu| \gtrsim \sigma\frac{1}{n_0\epsilon} \wedge h,
\end{equation}
since the non-private term $(\sigma/\sqrt{n_0}) \wedge h$ can be shown easily using similar arguments.
To establish \eqref{eq:mean-singlerate}, notice that \eqref{eq:lecam-mean} can now be simplified to 
\[
\inf_{\mathcal{Q}_{\epsilon,\delta,1}} \inf_{\hat{\mu}}\sup_{P \in \mathcal{P}_{\bm{\mu}}}\mathbb{E}|\hat{\mu} - \mu| \geq \frac{|\mu_1 - \mu_2|}{2} \left(1- \delta'_0 - \sqrt{\frac{1}{2} \mathrm{KL}(\mathbb{M}_{0}'',\mathbb{M}_{0}')}\right),
\]
due to the design \eqref{eq-same-mu}.
Calculations similar to case (a) suggest that we can choose {$\mu_1 = -\mu_2 = |\mu_1-\mu_2|/2$ with} $|\mu_1-\mu_2| \asymp \sigma(n_0\epsilon)^{-1} \wedge h$ and $\delta \lesssim \epsilon$. This ensures $\inf_{\mathcal{Q}_{\epsilon,\delta,1}} \inf_{\hat{\mu}}\sup_{P \in \mathcal{P}_{\bm{\mu}}}\mathbb{E}|\hat{\mu} - \mu| \gtrsim |\mu_1-\mu_2| \gtrsim \sigma(n_0 \epsilon)^{-1} \wedge h$, as desired. 

\subsection{Auxiliary results}\label{appendix:privatevariance}
\begin{lemma}{\citep[][Lemma 6.1]{karwa2017finite}}\label{lemma-kv-densityratio}
    Let $X_1,\dotsc, X_n$ be i.i.d.~random variables with distribution $P_\theta$ and $\theta \in \Theta$. For any $(\epsilon,\delta)$-central DP mechanism $Q(\cdot|\{x_i\}_{i=1}^n)$, we use $M_\theta$ to denote its marginal distribution, i.e. 
    \[
    M_\theta(\cdot) = \int Q(\cdot|\{x_i\}_{i=1}^n) dP_\theta^{\otimes n}(\{x_i\}_{i=1}^n).
    \]
    Then, for any measurable set $S$, and any pair of $\theta_1,\theta_2 \in \Theta$, it holds that 
    \[
    M_{\theta_1}(S) \leq \exp(\epsilon') M_{\theta_2}(S) + \delta',
    \]
    where $\epsilon' = 6\epsilon n \mathrm{TV}(P_{\theta_1}, P_{\theta_2})$ and $\delta' = 4\exp(\epsilon')n \delta \mathrm{TV}(P_{\theta_1},P_{\theta_2})$.
\end{lemma}

\begin{lemma}{\citep[][Lemma 3.17]{dwork2014algorithmic}}\label{lemma:dwork-approximatedp}
    For two random variables $Y, Z$ and $\delta \geq 0$, consider the following divergence 
    \[
    D_{\infty}^{\delta}(Y || Z) = \max_{S: \mathbb{P}(Y \in S) \geq \delta} \log\left\{\frac{\mathbb{P}(Y \in S) - \delta}{\mathbb{P}(Z \in S)}\right\}.
    \]
    For $\epsilon > 0$, it holds that $D_{\infty}^{\delta}(Y || Z) \leq \epsilon$ and $D_{\infty}^{\delta}(Z || Y) \leq \epsilon$ if and only if there exist random variables $Y'$ and $Z'$ such that 
    \begin{itemize}[leftmargin=*]
        \item $\mathrm{TV}(Y,Y') \leq \delta/(e^\epsilon+1)$, $\mathrm{TV}(Z,Z') \leq \delta/(e^\epsilon+1)$, and
        \item $D_{\infty}^{0}(Y' || Z') \leq \epsilon$ and $D_{\infty}^{0}(Z' || Y') \leq \epsilon$.
    \end{itemize}
   
\end{lemma} 

\section{Technical Details of Section \ref{sec:lowd-regression} }\label{sec:appendix-3}
\subsection{Differentially private linear regression on a single data set}\label{sec-reg-single-app}

In this subsection, we study a single-site linear regression problem under the central DP constraint. Let $\{(X_i, Y_i)\}_{i = 1}^n$ be i.i.d.~from the linear model
\begin{equation}\label{eq:single-site-model}
    Y_i = \langle X_i, \beta^{*} \rangle + \xi_i  \quad X_i \sim P_x, \quad i \in [n],
\end{equation}
with $P_x \in \mathrm{SG}(C,\Sigma)$, and $\xi_i$ being mean-zero and sub-Gaussian with $\|\xi_i\|_{\psi_2}\leq \sigma$.

\begin{algorithm}[!ht]
	\begin{algorithmic}[1]
		\INPUT{Data $\{(X_i, Y_i)\}_{i \in [n]}$, number of iteration $T$, step size $\rho$, privacy parameters $\epsilon,\delta$, initialisation $\beta^0$, failure probability $\eta \in (0,1/2)$}.
        \State Set batch size $b = \floor{n/T}$, truncation radius $R = \sqrt{d\log(n/\eta)} $, privacy parameters $\epsilon' = \epsilon/2, \delta' = \delta/2$
		\For{$t = 0, \ldots, T-1$} 
            \State Set $\tau = bt$ 
            \State Set $R_t = \sqrt{\log(n/\eta)}$PrivateVariance($\{Y_{\tau +i} - X_{\tau+i}^\top \beta^t\}_{i=1}^b,\epsilon',\delta'$) \Comment{See \Cref{algorithm:Private_variance} for PrivateVariance}
			\State Sample $w_t \sim \mathcal{N}(0, I_d)$ and let $\phi_t = \sqrt{2\log(1.25/\delta')}2RR_t/(b\epsilon')$
            \State $\beta^{t+1} = \beta^t - \rho\Big\{\frac{1}{b}\sum_{i=1}^b\Pi_R(X_{\tau+i})\Pi_{R_t}(X_{\tau+i}^\top \beta^{t} - Y_{\tau+i})+\phi_t w_t\Big\}$ 
		\EndFor
		\OUTPUT $\beta^T$. 
		\caption{Differentially private linear regression on a single data set} \label{algorithm:DPregression_single}
	\end{algorithmic}
\end{algorithm} 

The following lemma establishes the theoretical guarantee of the final output of \Cref{algorithm:DPregression_single}. 

\begin{lemma}\label{lemma:DPregreesion_singlesite}
    Let $\{(X_i, Y_i)\}_{i = 1}^n$ be i.i.d.~from the linear model \eqref{eq:single-site-model}. Suppose $0< 1/L \leq \lambda_{\min}(\Sigma) \leq \lambda_{\max}(\Sigma) \leq L < \infty$, for some absolute constant $L\geq 1$, and $\sigma = 1$.  
    \begin{enumerate}[leftmargin=*]
        \item \Cref{algorithm:DPregression_single} is $(\epsilon,\delta)$-central DP.
        \item \label{item-lem-3} Initialise \Cref{algorithm:DPregression_single} with $\beta^0 = 0$ and step size $\rho = 18L(1+81L^2)^{-1}$. Suppose that
        \begin{equation}\label{eq:alg1_minialsamplesize}
        	  n \gtrsim \frac{dT\log\big(\frac{n \vee (T/\delta)}{\eta}\big)\log(T/[\eta(\epsilon\wedge\delta)])}{\epsilon},
        \end{equation}
        and $T = \lceil C\log(n) \rceil$ for some absolute constant $C>0$.  We then have with probability at least $1-7\eta$ that
    \[
     \|\beta^{T} - \beta^*\|_2 \lesssim \frac{\|\beta^*\|_2}{n^{\frac{C}{81L^2+1}}}+  \log\left(\frac{\log(n)}{\eta}\right)\sqrt{\frac{d\log(n)}{n}} +  \frac{d\log^2(n/\eta)\sqrt{\log(\log(n)/\eta)\log(1/\delta)}}{n\epsilon}. 
    \]
    \item In addition, suppose that $\|\beta^*\|_2 \leq C'$ for some absolute constant $C'$ and $C \geq (81L^2+1)/2$, then we have 
    \begin{equation*}
        \|\beta^T - \beta^*\|_2 \lesssim r(n,d,\epsilon,\delta, \eta) =   \log\left(\frac{\log(n)}{\eta}\right)\sqrt{\frac{d\log(n)}{n}} + \frac{d\log^2(n/\eta)\sqrt{\log(1/\delta)\log(\log(n)/\eta))}}{n\epsilon} 
    \end{equation*}
    with probability at least $1-7\eta$.
   
    \end{enumerate}
\end{lemma}

\Cref{lemma:DPregreesion_singlesite} shows that \Cref{algorithm:DPregression_single} achieves the optimal convergence rate up to poly-logarithmic factors. 
Compared to \citet[][Theoerm 4.2]{cai2019cost}, where it is shown that there exists an estimator $\hat{\beta}$ such that
\begin{equation*}
    \|\hat{\beta}- \beta^*\|_2^2 \lesssim  \frac{d}{n}+\frac{d^2\log(1/\delta)\log^3(n)}{n^2\epsilon^2}
\end{equation*}
with high probability, when $n = \tilde{\Omega}(d^{3/2}/\epsilon)$, we see that \Cref{algorithm:DPregression_single} requires a much weaker minimal sample size condition \eqref{eq:alg1_minialsamplesize}.  Moreover, we do not need to assume $\|\beta^*\|_2$ to be bounded by a fixed constant for the theoretical guarantee to hold. In point (\ref{item-lem-3}), it is shown that by setting $T = \lceil C\log(n) \rceil$, with a large enough constant $C >0$, the first term in the upper bound - regarding $\|\beta^*\|_2$ - shall be dominated by the remaining terms, allowing for $\|\beta^*\|_2$ to diverge. We however do assume bounded $\|\beta^*\|_2$ in \eqref{eq:singleregression_rate} to simplify the presentation. 
 Compared to \citet[][Algorithm 2]{varshney2022nearly}, 
 we swap the DP-STAT, which requires the knowledge of $\|\beta^*\|_2$ as an input, with PrivateVariance to perform the adaptive clipping step. Their algorithm adopts 
 a tail-averaging step to output the average of the last $T/2$ iterations, \Cref{algorithm:DPregression_single} simply uses the final iteration as the output. Our analyses are considerably simpler while only sacrificing some poly-logarithmic factors.

\subsection{Proofs of results in Section \ref{sec:lowd-regression}}
\subsubsection{Proof of Lemma \ref{lemma:DPregreesion_singlesite}}

The first claim that \Cref{algorithm:DPregression_single} satisfies $(\epsilon,\delta)$-central DP follows from the parallel composition property of DP \cite[e.g.][Theorem 2]{smith2021making}, since each iteration uses a disjoint set of independent data and satisfies $(\epsilon, \delta)$-central DP via the composition of the Gaussian mechanism and \Cref{algorithm:Private_variance}. 

The third claim in the statement follows directly from the second claim by directly applying the additional assumption $\|\beta^*\|_2 \leq C'$. It therefore suffices to show the second claim. 

Denote $\tau_t = bt, t \in \{0\} \cup [T-1]$ and $Z_i = X_i \xi_i, i \in [n]$.
Consider the following events 
\begin{gather*}
    \mathcal{E}_1 = \bigg\{ \lambda_{\min}\Big(\frac{1}{b}\sum_{i = 1}^b X_{\tau_t+i}X_{\tau_t+i}^\top\Big) \geq \frac{1}{9L}, \lambda_{\max}\Big(\frac{1}{b}\sum_{i = 1}^b X_{\tau_t+i}X_{\tau_t+i}^\top\Big) \leq 9L, \forall t \in  \{0\} \cup [T-1]\bigg\}, \\ \mathcal{E}_2 = \{\Pi_R(X_i) = X_i, \forall i \in [n]\},\\ 
    \mathcal{E}_3 = \Big\{\Pi_{R_t}(X_{\tau_t+i}^\top \beta^{t} - Y_{\tau_t+i}) = X_{\tau_t+i}^\top \beta^{t} - Y_{\tau_t+i}, \;\text{and}\; R_t \leq C_1 \sqrt{\log(n/\eta)}(\sigma+\|\beta^t - \beta^*\|_{{\Sigma}}), \\ \hspace{10cm} \forall t \in  \{0\} \cup [T-1], i \in [b]\Big\}
\end{gather*}    
and
\[
    \mathcal{E}_4 = \Big\{\big\|\frac{1}{b}\sum_{i = 1}^bZ_{\tau_t+i} \big\|_2^2 \leq C_2\sigma^2\frac{d\log^2(T/\eta)}{b}, \forall t \in \{0\} \cup [T-1]\Big\},
\]
where $C_1, C_2 >0$ are some absolute constants. We control the probabilities of these events happening in \Cref{lemma:highprobabilityevent}. In particular, under the conditions required for \Cref{lemma:highprobabilityevent} and the choice of parameters specified in \Cref{algorithm:DPregression_single}, we are guaranteed that the probability of all these events happening is no less than $1-6\eta$. The remainder of the proof is conditional on all of these events happening. 

In the events $\mathcal{E}_2$ and $\mathcal{E}_3$, we can simplify the $t$-th iteration as 
\[
\beta^{t+1} = \beta^t - \rho\bigg\{\frac{1}{b}\sum_{i=1}^bX_{\tau_t+i}(X_{\tau_t+i}^\top \beta^{t} - Y_{\tau_t+i})+\phi_t w_t\bigg\},
\]
which implies 
\begin{align*}
    \beta^{t+1} - \beta^* &= \beta^t-\beta^* - \frac{\rho}{b}\sum_{i=1}^bX_{\tau_t+i}X_{\tau_t+i}^\top(\beta^t-\beta^*)+\frac{\rho}{b}\sum_{i=1}^bZ_{\tau_t+i} - \rho\phi_t w_t \\
    &=\left(I - \frac{\rho}{b}\sum_{i=1}^bX_{\tau_t+i}X_{\tau_t+i}^\top \right)(\beta^t-\beta^*) +\frac{\rho}{b}\sum_{i=1}^bZ_{\tau_t+i} - \rho\phi_t w_t.
\end{align*}

Note that in the event $\mathcal{E}_1$, it holds that
\[
\bigg\|I - \frac{\rho}{b}\sum_{i=1}^bX_{\tau_t+i}X_{\tau_t+i}^\top \bigg\|_2 \leq \max\bigg\{\Big|1-\frac{\rho}{9L}\Big|, \Big|1-9\rho L\Big|\bigg\} = \frac{81L^2-1}{81L^2+1}, 
\]
when choosing $\rho = 18L(1+81L^2)^{-1}$. Write $\gamma = \sqrt{2\log(1.25/\delta')}2Rb^{-1}(\epsilon')^{-1}$, then we have 
\begin{align}
    &\|\beta^{t+1} - \beta^*\|_2 \leq \Big(1-\frac{2}{81L^2+1}\Big)\|\beta^t-\beta^*\|_2 + \rho \log(T/\eta)\sqrt{C_2\sigma^2\frac{d}{b}} + \rho\gamma R_t\|w_t\|_2 \nonumber \\
    & \leq \Big(1-\frac{2}{81L^2+1}\Big)\|\beta^t-\beta^*\|_2 + \rho\log(T/\eta) \sqrt{C_2\sigma^2\frac{d}{b}} + C_1\sqrt{L}\rho\gamma \sqrt{\log(n/\eta)} (\sigma+\|\beta^t - \beta^*\|_{{2}})\|w_t\|_2 \nonumber \\
    & = \left(1-\frac{2}{81L^2+1}+C_1\sqrt{L}\rho\sqrt{\log(n/\eta)}\gamma\|w_t\|_2\right)\|\beta^t-\beta^*\|_2 + \rho \log(T/\eta)\sqrt{C_2\sigma^2\frac{d}{b}} \nonumber \\
    & \hspace{10cm}+ C_1\sqrt{L}\rho\gamma \sqrt{\log(n/\eta)} \sigma \|w_t\|_2, \label{eq-betat1beta*-intermediate}
\end{align}
where the first inequality holds in the event $\mathcal{E}_4$ and the second in $\mathcal{E}_3$, along with the fact that $\|\beta^t - \beta^*\|_{\Sigma}\leq \sqrt{L}\|\beta^t - \beta^*\|_{{2}}$. The Hanson--Wright inequality \cite[e.g.][Theorem 6.2.1]{vershynin2018high} implies that
\[
\mathbb{P}(\|w_t\|_2 \leq C_3\sqrt{d\log(1/\eta)}) \geq 1-\eta,
\]
for some absolute constant $C_3 >0$. Combining with a union bound, we have $\|w_t\|_2 \leq C_3\sqrt{d\log(T/\eta)}$ for any $t \in \{0\} \cup [T-1]$ with probability at least $1-\eta$. Now, provided $b$ is large enough such that 
\[
C_1\sqrt{L}\rho\sqrt{\log(n/\eta)}\gamma\|w_t\|_2 \leq  C_1\sqrt{L}C_3\rho\sqrt{d\log(T/\eta){\log(n/\eta)}}\sqrt{2\log(1.25/\delta')}2Rb^{-1}(\epsilon')^{-1} \leq \frac{1}{81L^2+1},
\]
which can be simplified as 
\[
b \geq C_4\frac{d\log(n/\eta)\sqrt{\log(T/\eta)\log(1/\delta)}}{\epsilon},
\]
for some absolute constant $C_4>0$, we can further upper bound \eqref{eq-betat1beta*-intermediate} that
\begin{align*}
    \|\beta^{t+1} - \beta^*\|_2 &\leq \Big(1- \frac{1}{81L^2+1}
    \Big)\|\beta^t-\beta^*\|_2+ \rho\log(T/\eta) \sqrt{C_2\sigma^2\frac{d}{b}} + C_5\rho\gamma \sigma\sqrt{d\log(T/\eta)\log(n/\eta)} \\
    & \lesssim \Big(1-\frac{1}{81L^2+1}\Big)^{t+1}\|\beta^*\|_2 +  \sigma\log(T/\eta)\sqrt{\frac{d}{b}} + \gamma \sigma\sqrt{d\log(T/\eta)\log(n/\eta)} \\
    & \lesssim \exp\Big(-(t+1)/(81L^2+1)\Big)\|\beta^*\|_2 + \sigma \log(T/\eta)\sqrt{\frac{d}{b}} \\
    & \hspace{2cm}+ \sigma \frac{d\log(n/\eta)\sqrt{\log(T/\eta)\log(1/\delta)}}{b\epsilon}.
\end{align*}
Choosing $T = \lceil C_6\log(n) \rceil$, for some absolute constant $C_6 > 0$, gives that
\[
\|\beta^{T} - \beta^*\|_2 \lesssim \frac{\|\beta^*\|_2}{n^{C_6/(81L^2+1)}}+ \sigma\log(T/\eta) \sqrt{\frac{d}{b}} + \sigma \frac{d\log(n/\eta)\sqrt{\log(T/\eta)\log(1/\delta)}}{b\epsilon}.
\]
\subsubsection{Proof of Theorem \ref{thm:lowd-regression-transfer}}

We first consider the privacy guarantee of \Cref{algorithm:DPregression_federated}. Notice that each iteration, along with the detection step (computing $\hat{\mA}$ in \eqref{eq:lowd-hatA}), uses a fresh batch of samples at each site.  It therefore suffices to verify that \eqref{eq:composition-eachstep} is satisfied for $t \in [T]$ and $k \in \tkset$.  To compute $\hat{\mA}$, each site produces $\hat{\beta}^{(k)}$ using an $(\epsilon,\delta)$-central DP algorithm, i.e.~\Cref{algorithm:DPregression_single}, which satisfies \eqref{eq:composition-eachstep} at $t = 1$. For $t > 1$, each site in each iteration computes a truncated gradient using $R_t^{(k)}$ and then adds scaled Gaussian noise $\phi_t^{(k)}w_t^{(k)}$. Both steps are $(\epsilon/2,\delta/2)$-central DP by the choice of parameters, and together they ensure that \eqref{eq:composition-eachstep} is satisfied by composition. 

We then analyse the performance of \Cref{algorithm:DPregression_federated}. Note that the following required conditions
   \[
        n_0 \gtrsim \frac{d\log(n_0)\log\big(\frac{n_0 \vee (\log(n_0)/\delta)}{\eta}\big)\log\left(\frac{\log(n_0)}{\eta(\epsilon\wedge\delta)}\right)}{\epsilon}, \quad n_k \gtrsim n_0, \quad \log\Big(\sum_{k \in [K]}n_k\Big) \lesssim \log(n_0).
        \]
    are sufficient to guarantee that 
    \[
    n_k \gtrsim \frac{d\log(n_k)\log\big(\frac{n_k \vee (\log(n_k)/\delta)}{\eta}\big)\log\left(\frac{\log(n_k)}{\eta(\epsilon\wedge\delta)}\right)}{\epsilon}
    \]
    holds for any $k \in \{0\} \cup [K]$, this condition on $n_k$ then guarantees we can apply \Cref{lemma:DPregreesion_singlesite} for each site estimator $\hat{\beta}^{(k)}$. 

Our goal is to establish a high probability upper bound on the error $\|\tilde{\beta} - \beta\|_2$ and similar to the proof of \Cref{thm:mean-upperbound}, {we do so by separately considering the cases $h < c_1 r(n_0,d,\epsilon,\delta, \eta)$ and $h \geq c_1 r(n_0,d,\epsilon,\delta, \eta)$}, where $c_1$ is an appropriate constant that guarantees the condition required in \Cref{lemma:selection-consistency}.
    Using Lemmas~\ref{lemma:DPregreesion_singlesite} and \ref{lemma:selection-consistency}, we have that under an appropriate choice of $\tilde{c}$
    \begin{equation}\label{eq:regression-selection}
        \mathbb{P}(\hat{\mathcal{A}} = \mathcal{A}) \geq 1-7(K+1){\eta},
    \end{equation}
    when $h < c_1 r(n_0,d,\epsilon,\delta, \eta)$. With the consistent selection of $\mathcal{A}$, note that for any $t \geq 0$, 
    \begin{align*}
        \mathbb{P}\Big(\|\tilde{\beta}(\hat{\mathcal{A}}) - \beta\|_2 > t\Big) &= \mathbb{P}\Big(\|\tilde{\beta}(\hat{\mathcal{A}}) - \beta\|_2 > t, \hat{\mathcal{A}} = \mathcal{A}\Big) + \mathbb{P}\Big(\|\tilde{\beta}(\hat{\mathcal{A}}) - \beta\|_2 > t, \hat{\mathcal{A}} \neq \mathcal{A}\Big) \\ &  \leq \mathbb{P}\Big(\|\tilde{\beta}({\mathcal{A}})-\beta\|_2 > t\Big) + \mathbb{P}(\hat{\mathcal{A}} \neq \mathcal{A}) \\ & \leq \mathbb{P}\Big(\|\tilde{\beta}({\mathcal{A}})-\beta\|_2 > t\Big) + 7(K+1){\eta},
    \end{align*}
    where we write $\tilde{\beta}$ as $\tilde{\beta}(\cdot)$ to emphasise its dependence on data sets used.  We shall analyse $\mathbb{P}\Big(\|\tilde{\beta}({\mathcal{A}})-\beta\|_2 > t\Big)$, which has the randomness of $\hat{\mathcal{A}}$ removed. 
    
We write $Z_i^{(k)} = X_{i}^{(k)}\xi_i^{(k)}$, $N = n_0 + n_{\mathcal{A}} = \sum_{k \in \stset} n_k$, and consider the following events, where we treat $\sigma_k\asymp \sigma$.
    \begin{align*}
    &\mathcal{E}_1' = \bigg\{ \lambda_{\min}\Big( \sum_{k \in \{0\} \cup {\mathcal{A}}} \frac{v_k}{b^{(k)}}\sum_{i=1}^{b^{(k)}}X^{(k)}_{\tau_t+i}X_{\tau_t+i}^{(k)\top}\Big) \geq \frac{1}{9L}, \\ &\hspace{4cm} \lambda_{\max}\Big(\sum_{k \in \{0\} \cup {\mathcal{A}}}\frac{v_k}{b^{(k)}}\sum_{i=1}^{b^{(k)}}X^{(k)}_{\tau_t+i}X_{\tau_t+i}^{(k)\top}\Big) \leq 9L, \forall t \in  \{0\} \cup [T-1]\bigg\}, \\& \mathcal{E}_2' = \{\Pi_R(X_i^{(k)}) = X_i^{(k)}, \forall i \in [n], k \in \{0\} \cup \mathcal{A}\},\\& 
    \mathcal{E}_3' = \Big\{\Pi_{R_t^{(k)}}(X^{(k)\top}_{\tau_t+i} \beta^{t} - Y^{(k)}_{\tau_t+i}) = X^{(k)\top}_{\tau_t+i} \beta^{t} - Y^{(k)}_{\tau_t+i} \;\text{and}\; R_t^{(k)} \leq C_1 \sqrt{\log(N/\eta)}(\sigma+\|\beta^t - \beta^{(k)}\|_2), \\& 
    \hspace{8cm} \forall  t \in  \{0\} \cup [T-1], i \in [n], k\in \{0\} \cup \mathcal{A} \Big\},\\&
    \mathcal{E}_4' = \Bigg\{\bigg\|\sum_{k \in \{0\} \cup {\mathcal{A}}}\frac{v_k}{b^{(k)}}\sum_{i=1}^{b^{(k)}}Z^{(k)}_{\tau_t+i} \bigg\|_2^2 \leq C_2\sigma^2\frac{d\log^2(1/\eta)}{\sum_{k \in \{0\}\cup\mA}u_k}, \forall  t \in  \{0\} \cup [T-1]\Bigg\}
\end{align*}
where $C_1, C_2 >0$ are some absolute constants. Following a similar roadmap to the proof of \Cref{lemma:highprobabilityevent}, we control the probability of the aforementioned events.  This is done in \Cref{cor:transfer-high-probability-events}, where we show that
\begin{equation}\label{eq:fed-high-probability-events}
    \mathbb{P}\left(\mathcal{E}_1' \cap \mathcal{E}_2' \cap \mathcal{E}_3' \cap \mathcal{E}_4' \right) \geq 1-6\eta. 
\end{equation}

Conditional on the event $\cap_{i = 1}^4 \mathcal{E}_i'$, we can simplify the $t$-th iteration, $t \in [T]$, as
    \[
\beta^{t+1} = \beta^t - \rho\sum_{k \in \{0\} \cup{\mathcal{A}}}v_k\bigg\{\frac{1}{b^{(k)}}\sum_{i=1}^{b^{(k)}}X^{(k)}_{\tau_t+i}(X^{(k)\top}_{\tau_t+i} \beta^{t} - Y^{(k)}_{\tau_t+i})+\phi_t^{(k)} w_t^{(k)}\bigg\},
    \]
    which implies that
   \begin{align} 
    \beta^{t+1} - \beta &= \beta^t-\beta - \sum_{k \in \{0\} \cup {\mathcal{A}}} \frac{\rho v_k}{{b^{(k)}}}
 \sum_{i=1}^{b^{(k)}} X^{(k)}_{\tau_t+i}X_{\tau_t+i}^{(k)\top}(\beta^t-\beta^{(k)}) \nonumber \\
    & \qquad \qquad+\sum_{k \in \{0\} \cup {\mathcal{A}}}\frac{\rho v_k}{{b^{(k)}}}\sum_{i=1}^{b^{(k)}}Z^{(k)}_{\tau_t+i} - \sum_{k \in \{0\} \cup {\mathcal{A}}} \rho v_k\phi_t^{(k)} w_t^{(k)} \nonumber \\
    &=\left(I - \sum_{k \in \{0\} \cup {\mathcal{A}}} \frac{\rho v_k}{{b^{(k)}}}\sum_{i=1}^{b^{(k)}}X^{(k)}_{\tau_t+i}X_{\tau_t+i}^{(k)\top} \right)(\beta^t-\beta) \nonumber \\
    & \qquad \qquad -  \sum_{k \in \{0\} \cup {\mathcal{A}}} \frac{\rho v_k}{{b^{(k)}}}\sum_{i=1}^{b^{(k)}}X^{(k)}_{\tau_t+i}X_{\tau_t+i}^{(k)\top}(\beta-\beta^{(k)}) \nonumber \\ 
    & \qquad \qquad+\sum_{k \in \{0\} \cup {\mathcal{A}}}\frac{\rho v_k}{{b^{(k)}}}\sum_{i=1}^{b^{(k)}}Z^{(k)}_{\tau_t+i} - \sum_{k \in \{0\} \cup {\mathcal{A}}} \rho v_k\phi_t^{(k)} w_t^{(k)}. \label{eq-betat+1errorinter}
\end{align}

Note that the Hanson--Wright inequality shows that
\begin{equation}\label{eq:federated-hansonwright}
    \mathbb{P}\Bigg(\bigg\|\sum_{k \in \{0\} \cup {\mathcal{A}}} \frac{v_k}{b^{(k)}\epsilon} w_t^{(k)}\bigg\|_2 \leq C_3 \sqrt{\frac{\log(1/\eta)}{\sum_{k \in \{0\} \cup {\mathcal{A}}}u_k}}\Bigg) \geq 1-\eta,
\end{equation}
for some absolute constant $C_3>0$. 

Then using the same arguments as in the proof of \Cref{lemma:DPregreesion_singlesite}, it follows that for $T = \lceil C_4\log(N)\rceil$, where $C_4>0$ is some absolute constant,
\begin{align}\label{eq:combinerate-inproof}
    \|\tilde{\beta} - \beta\|_2 &\lesssim \frac{\|\beta\|_2}{N^{C_4/(81L^2+1)}} + h+ \sigma \sqrt{\frac{d\log^2(1/\eta)}{\sum_{k \in \{0\}\cup\mA}u_k}}+\sigma\sqrt{\frac{d\log^2(N/\eta)\log(1/\delta)}{\sum_{k \in \{0\}\cup\mA}u_k}} \nonumber \\
     &\lesssim \frac{\|\beta\|_2}{N^{C_4/(81L^2+1)}} + h+ \log(N/\eta)\sqrt{\frac{dT^2\log(1/\delta)}{\sum_{k \in \{0\}\cup\mA}\min\{n_k, n_k^2\epsilon^2/d\}}},
\end{align}
holds with probability at least
\[
1-7\eta - 7(K+1){\eta} \geq 1-14 (K+1){\eta},
\] 
since $h$ is an upper bound on 
\[
    \sum_{k \in \{0\} \cup {\mathcal{A}}} \frac{\rho v_k}{{b^{(k)}}}\sum_{i=1}^{b^{(k)}}X^{(k)}_{\tau_t+i}X_{\tau_t+i}^{(k)\top}(\beta-\beta^{(k)})
\]
in \eqref{eq-betat+1errorinter}.  We note that in the process of following the proof of \Cref{lemma:DPregreesion_singlesite}, one needs to ensure
\[
 \sum_{k \in \{0\}\cup\mA}u_k \gtrsim {d\log(N/\eta)\sqrt{\log(T/\eta)\log(1/\delta)}},
\]
and this is indeed satisfied by our conditions. Moreover, when $\|\beta\|\lesssim 1$, \eqref{eq:combinerate-inproof} reduces to 
\[
 \|\tilde{\beta} - \beta\|_2 \lesssim h+ \log^2(N/\eta)\sqrt{\frac{d\log(1/\delta)}{\sum_{k \in \{0\}\cup\mA}\min\{n_k, n_k^2\epsilon^2/d\}}}
\]

On the other hand, when $h \geq c_1 r(n_0,d,\epsilon,\delta,\eta)$, applying the first part of \Cref{lemma:selection-consistency}, we have 
\[
\mathbb{P}\Big\{\max_{k \in \hat{\mathcal{A}}} \alpha^{(k)}_r \lesssim r(n_0,d,\epsilon,\delta,\eta)\Big\} \geq 1-7(K+1){\eta}.
\]
Then, provided that $\log(\sum_{k \in [K]}n_k) \lesssim \log(n_0)$, we have for any possible subset $S$ of $[K]$, 
\begin{align*}
\log^2(N/\eta)\sqrt{\frac{d\log(1/\delta)}{\sum_{k \in \{0\}\cup S}\min\{n_k, n_k^2\epsilon^2/d\}}} \lesssim \log^2(n_0/\eta)\sqrt{\frac{d\log(1/\delta)}{(S+1)\min\{n_0, n_0^2\epsilon^2/d\}}}\lesssim     r(n_0,d,\epsilon,\delta,\eta)
\end{align*}
Therefore, applying the same arguments for establishing \eqref{eq:combinerate-inproof} and that in the proof of \Cref{lemma:DPregreesion_singlesite}, we obtain that
\begin{equation}\label{eq:single-inproof}
    \|\tilde{\beta} - \beta\|_2 \lesssim r(n_0,d,\epsilon,\delta,\eta),
\end{equation}
with probability at least $1-14(K+1){\eta}$. Combining  \eqref{eq:combinerate-inproof} and \eqref{eq:single-inproof} yields our final claim.

\subsubsection{Proof of Theorem \ref{thm:lowerbound low-d regression}}

First, note that the rate in \eqref{eq:lowerbound_rate2} can be written as 
\[
\left\{\left(\frac{d}{n}+\frac{d^2}{n^2\epsilon^2}\right) \wedge h^2 \right\} \vee \frac{d}{\sum_{k \in\stset} \{n_k \wedge \{(n_k\varepsilon)^2/d\}\}}) = \RN{1} \vee \RN{2}.
\]
Therefore, we shall directly prove \eqref{eq:lowerbound_rate2} by showing that the minimax risk is lower bounded by \RN{1} and \RN{2} separately. Conditions in \eqref{eq:lowerbound-condition1} are sufficient in establishing the lower bound \RN{1}, validating the claim in \eqref{eq:lowerbound_rate1}. Some additional assumptions are required in establishing \RN{2}, when we apply \Cref{lemma:van-tree-lowerbound}.

    As in the proof of \Cref{lemma:lowerbound-mean}, we shall consider two settings of $\bm{\beta} = \{\beta^{(k)}\}_{k \in \tkset}$: 
    \begin{enumerate}[label=(\alph*), leftmargin=*]
    \item {$\beta^{(k)} = 0$ for $k \in \mathcal{A}$, and $\beta^{(k)} = \beta_{\mathrm{out}}$ for $k \in \mA^c$, where $\beta_{\mathrm{out}} \in \mathbb{R}^d$ is a fixed vector with $\|\beta_{\mathrm{out}}\|_2 > 2h$;}
    	\item $\beta^{(k)} = \beta$ for all $k \in \mathcal{A}$, and $\beta^{(k)} = \beta'$ for $k \not\in \mathcal{A}$ such that $\beta'$ is not a function of $\beta$.
    	
    \end{enumerate}

        \noindent\textbf{Case (a):} In this case we are to establish the rate \RN{1}.  We fix {$\beta^{(k)} = 0$ for $k \in \mathcal{A}$ and $\beta^{(k)} = \beta_{\mathrm{out}}$ for $k \in \mA^c$,} and the generating distribution is therefore  
    \begin{equation*}
         P_{\bm{\beta}} = \prod_{k=0}^K P^{\otimes n}_{\beta^{(k)}} = P_{\beta}^{\otimes n} {\prod_{k \in \mathcal{A}} P_0^{\otimes n} \prod_{k \in \mA^c} P_{\beta_{\mathrm{out}}}^{\otimes n}},  
    \end{equation*}
    where $ \beta^{(0)}=\beta $ with $\|\beta\|_2 \leq h < \sqrt{d}$. {For any $\|\beta\|_2 \leq h$, we have $\max_{k \in \mathcal{A}}\|\beta^{(k)} - \beta\|_2 = \|\beta\|_2 \leq h$ and $\min_{k \in \mA^c}\|\beta^{(k)} - \beta\|_2 = \|\beta_{\mathrm{out}} - \beta\|_2 \geq \|\beta_{\mathrm{out}}\|_2 - h > h$, so this choice of $\bm{\beta}$ belongs to $\Theta_{\bbeta}(\mathcal{A},h)$ defined in \eqref{eq:regression_parameterspace}.} We shall choose the covariance distribution $P_x = \mathcal{N}(0,I) \in \mathrm{SG}(C,I)$, for some absolute constant $C>0$. We  write the composition of $\{Q_k^t\}_{t \in [T],k\in \tkset}$, as $\tilde{Q}(z|D) = \prod_{k,t} Q_k^t(z_k^t|D_k^t, B^{t-1})$. Since each $Q_k^t$ satisfies condition \eqref{eq:composition-eachstep}, it immediately follows that $\tilde{Q}$ is $(\epsilon,\delta)$-central DP and also $(\epsilon,\delta)$-central DP with respect to the target data $D_0$. Therefore, we can apply \Cref{lemma:tonycai} to obtain
    \begin{align*}
        \inf_{Q \in \mathcal{Q}_{\epsilon,\delta,T}} \inf_{\substack{\hat{\beta}(Z) }} \sup_{\bm{\beta} \in \Theta_{\bm{\beta}}(\mA,h)} \mathbb{E}_{P_{\bm{\beta}},Q}\|\hat{\beta} - \beta\|_2^2 \geq & \inf_{\substack{\tilde{Q} \mbox{ is }(\epsilon, \delta)\mbox{-central DP}\\ \text{with respect to} \; D_0}} \inf_{\substack{\hat{\beta}(Z)}} \sup_{\bm{\beta} \in \Theta_{\bm{\beta}}(\mA,h)} \mathbb{E}_{P_{\bm{\beta}},\tilde{Q}}\|\hat{\beta} - \beta\|_2^2 \\
        \gtrsim & \left(\frac{d}{n}+\frac{d^2}{n^2\epsilon^2}\right) \wedge h^2.  
    \end{align*}

    \noindent\textbf{Case (b):}
    In this case, we obtain the rate $\RN{2}$ using \Cref{lemma:van-tree-lowerbound}, where arguments based on the Van-Trees inequality \citep[][Theorem 1]{gill1995applications} are used, in a similar way to \cite{xue2024optimal,cai2024optimal}.  

    Combining the two cases, we obtain the claimed result
    \begin{align*}
        \inf_{Q \in \mathcal{Q}_{\epsilon,\delta,T}} \inf_{\hat{\beta}(Z)} \sup_{P \in \mathcal{P}_{\bm{\beta}}} \mathbb{E}_{P,Q}\|\hat{\beta} - \beta\|_2^2 &\gtrsim \left\{\left(\frac{d}{n}+\frac{d^2}{n^2\epsilon^2}\right) \wedge h^2 \right\} \\  
        & \hspace{3em} \vee  \frac{d}{\sum_{k \in\stset} \{n_k \wedge \{(n_k\varepsilon)^2/d\}\}} \\
        & \hspace{-10em} = \left(\frac{d}{n}+\frac{d^2}{n^2\epsilon^2}\right) \wedge \left\{h^2 \vee  \frac{d}{\sum_{k \in\stset} \{n_k \wedge \{(n_k\varepsilon)^2/d\}\}} \right\}.
    \end{align*}

\subsection{Auxiliary results}\label{subsec: auxi results low dim reg}

\begin{algorithm}
    \begin{algorithmic}
        \INPUT${\{W_i\}_{i \in [2n]}}$, privacy parameters $\epsilon,\delta > 0,$ number of subsets $k>0$.
        \State Partition $[0,\infty)$ into intervals of the form $B_j = (2^j,2^{j+1}], j \in \mathbb{Z}$;
        \State Set $W_i' = W_{2i}-W_{2i-1}$ for $i \in [n]$;
         \State Split $W'_i$ into $k$ subsets of equal size and let $G_\ell$ be the $\ell$-th group;
        \State Set $U_\ell = \frac{1}{|G_\ell|}\sum_{i\in G_\ell}(W_i')^2, \ell \in [k]$
        \For{$j\in \mathbb{Z}$}
        \State Set $\hat{p}_j = \sum_{\ell = 1}^k \mathds{1}\{U_\ell \in B_j\}/k$;
            \If{$\hat{p}_j = 0$}
             \State Set $\tilde{p}_j = 0$;
            \Else \State Set $\tilde{p}_j = \hat{p}_j+Z_j, Z_j \sim 2(\epsilon k)^{-1}\mathrm{Lap}(1)$; 
             \If{$\tilde{p}_j < 2\log(1/\delta)(\epsilon k)^{-1}+1/k$}
             \State Set $\tilde{p}_j = 0$; \EndIf
            \EndIf
        \EndFor
        \State Set $\hat{j} = \argmax_{j \in \mathbb{Z}}\tilde{p}_j$;
        \OUTPUT $\sqrt{2^{\hat{j}}}$. 
		\caption{PrivateVariance}
  \label{algorithm:Private_variance}
    \end{algorithmic}
\end{algorithm}

\begin{algorithm}
    \begin{algorithmic}
        \INPUT${\{W_i\}_{i \in [2n]}}$, privacy parameters $\epsilon,\delta > 0$.
        \State Partition $[0,\infty)$ into intervals of the form $B_j = (2^j,2^{j+1}], j \in \mathbb{Z}$;
        \State Set $W_i' = W_{2i}-W_{2i-1}$ for $i \in [n]$;
        \For{$j\in \mathbb{Z}$}
        \State Set $\hat{p}_j = \sum_{i=1}^n \mathds{1}\{W_i' \in B_j\}/n$;
            \If{$\hat{p}_j = 0$}
             \State Set $\tilde{p}_j = 0$;
            \Else \State Set $\tilde{p}_j = \hat{p}_j+Z_j, Z_j \sim 2(\epsilon n)^{-1}\mathrm{Lap}(1)$; 
             \If{$\tilde{p}_j < 2\log(1/\delta)(\epsilon n)^{-1}+1/n$}
             \State Set $\tilde{p}_j = 0$; \EndIf
            \EndIf
        \EndFor
        \State Set $\hat{j} = \argmax_{j \in \mathbb{Z}}\tilde{p}_j$;
        \OUTPUT $2^{\hat{j}+2}$. 
		\caption{PrivateVarianceGaussian \citep[][Algorithm 2]{karwa2017finite}}
  \label{algorithm:Private_varianceSG}
    \end{algorithmic}
\end{algorithm}

\begin{lemma}\label{lemma:privatevarianceSG}
   \Cref{algorithm:Private_varianceSG} is $(\epsilon,\delta)$-DP. Suppose $W_i$ are independent random variables with variance $\sigma^2$ and $\max_{i\in[n]}\|W_i\|_{\psi_2} \leq C\sigma$ for some absolute constant $C>0$. Then, for any $\eta \in (0,1)$, if
   \[
   \frac{n}{k} \gtrsim \log(k/\eta) \qquad k \gtrsim \frac{\log(1/(\eta\delta))}{\epsilon}
   \]
    then the output of \Cref{algorithm:Private_varianceSG} satisfies with probability at least $1-\eta$
    \[
  \sigma\sqrt{\frac{3}{4}}\leq \mathrm{PrivateVarianceSG}(\{W_i\}_{i \in [2n]},\epsilon,\delta) \leq \sigma\sqrt{\frac{5}{2}}.
    \]
\end{lemma}

\begin{proof}[Proof of \Cref{lemma:privatevarianceSG}]
    Proof of the $(\epsilon,\delta)$-DP property follows from Theorem 3.5 in \cite{vadhan2017complexity}. Note that under our assumption $W_i'$ is sub-Gaussian with $\|W_i'\|_{\psi_2}\leq 2C\sigma$ and $(W_i')^2$ is sub-exponential with $\mathbb{E}(W_i')^2 = 2\sigma^2$. Since $|G_\ell| =n/k $, Bernstein's inequality \cite[e.g.][Theorem 2.8.1]{vershynin2018high} implies for any $t>0$,
    \[
    \mathbb{P}(|U_{\ell} - 2\sigma^2| > t) \leq 2\exp\left(-c\frac{n}{k}\min\Big\{\frac{t^2}{4C^2\sigma^2},\frac{t}{2C\sigma}\Big\}\right),
    \]
    for some absolute constant $c>0$. In other words, with probability at least $1-\eta$, 
    \[
    |U_{\ell} - 2\sigma^2| \leq \max\Big\{\sqrt{\frac{\log(2/\eta)4C^2\sigma^2}{c(n/k)}},\frac{2C\sigma\log(2/\eta)}{c(n/k)}\Big\}.
    \]
    Choosing $k$ such $n/k \geq 4(C\vee C^2)\log(2/\eta)c^{-1}$ ensures that $|U_{\ell} - 2\sigma^2| \leq \sigma^2/2$, for each $\ell \in [k]$. Together with a union bound, we have if $n/k \geq 4(C\vee C^2)\log(2k/\eta)c^{-1}$, then $\max_{\ell \in [k]}|U_{\ell} - 2\sigma^2| \leq \sigma^2/2$, i.e.\
    \[
    \frac{3}{2}\sigma^2\leq U_{\ell} \leq \frac{5}{2}\sigma^2
    \]
    for all $\ell \in [k]$ with probability at least $1-\eta$. Conditioning on this event, we know all $U_{\ell}$ lie in at most two distinct bins, and hence at most two $\hat{p}_j$ are non-zero, and this also holds for $\tilde{p}_j$ since the noise is only added to non-zero bins. Applying Lemma B.2 in \cite{liu2023near}, it holds that if $k \gtrsim \log(1/(\eta\delta))\epsilon^{-1}$, then one out of these two non-zero $\tilde{p}_j$ will be selected as maximum, with probability at least $1-\eta$, and hence 
   $ 3\sigma^2/4\leq 2^{\hat{j}} \leq 5\sigma^2/2$, which implies $\sigma \sqrt{3/4}\leq \sqrt{2^{\hat{j}}}\leq \sigma\sqrt{5/2}$.

\end{proof}

\begin{lemma} \label{lemma:highprobabilityevent}
Consider the events of interest in the proof of \Cref{lemma:DPregreesion_singlesite}:
    \begin{gather*}
    \mathcal{E}_1 = \bigg\{ \lambda_{\min}\Big(\frac{1}{b}\sum_{i = 1}^b X_{\tau_t+i}X_{\tau_t+i}^\top\Big) \geq \frac{1}{9L}, \lambda_{\max}\Big(\frac{1}{b}\sum_{i = 1}^b X_{\tau_t+i}X_{\tau_t+i}^\top\Big) \leq 9L, \forall t \in  \{0\} \cup [T-1]\bigg\},\\ \mathcal{E}_2 = \{\Pi_R(X_i) = X_i, \forall i \in [n]\},\\ 
    \mathcal{E}_3 = \{\Pi_{R_t}(X_{\tau_t+i}^\top \beta^{t} - Y_{\tau_t+i}) = X_{\tau_t+i}^\top \beta^{t} - Y_{\tau_t+i}, \;\text{and}\; R_t \leq C_1\sqrt{\log(n/\eta)}(\sigma+\|\beta^t - \beta^*\|_\Sigma), \\ \hspace{10cm} \forall t \in \{0\} \cup [T-1], i \in [b]\}, \\
    \mathcal{E}_4 = \Big\{\big\|\frac{1}{b}\sum_{i = 1}^bZ_{\tau_t+i} \big\|_2^2 \leq C_2\sigma^2\frac{d\log^2(T/\eta)}{b}, \forall t \in \{0\} \cup [T-1]\Big\}.
\end{gather*}
Under the conditions that 
\begin{align}
    n& \gtrsim \{Td\log(T/\eta)\} \vee \{T\log(T/(\delta\eta))\log(T/(\eta(\epsilon\wedge\delta)))\epsilon^{-1}\} \label{eq-n-cond-lemma12},\\
    R & \gtrsim \sqrt{d\log(n/\eta)}, \nonumber \\
    R_t &\gtrsim \sqrt{\log(n/\eta)}\mathrm{PrivateVariance}(\{X_{\tau_t+i}^\top \beta^{t} - Y_{\tau_t+i}\}_{i=1}^b,\epsilon',\delta'),  \nonumber
\end{align}
we have $$\mathbb{P}(\mathcal{E}_1 \cap \mathcal{E}_2 \cap \mathcal{E}_3 \cap \mathcal{E}_4) \geq 1-6\eta.$$ 

\end{lemma}
\begin{proof}[Proof of \Cref{lemma:highprobabilityevent}]
In the proof, we control the probability of each event separately. 
    
Event $\mathcal{E}_1$ can be controlled using standard results in covariance estimation \cite[e.g.][Theorem 6.5]{wainwright2019high} and Weyl's inequality. In particular, we have
    \[
    \mathbb{P}\Bigg\{\lambda_{\min}\Big(\frac{1}{b}\sum_{i = 1}^b X_{\tau_t+i}X_{\tau_t+i}^\top\Big) \geq \frac{1}{9L},\, \lambda_{\max}\Big(\frac{1}{b}\sum_{i = 1}^b X_{\tau_t+i}X_{\tau_t+i}^\top\Big) \leq 9L,  \forall t \in  \{0\} \cup [T-1] \Bigg\} \geq 1-\eta,
    \]
    as long as 
    \[
    b \geq c_0\{d + \log(T/\eta)\},
    \]
    for some absolute constant $c_0>0$, which holds due to \eqref{eq-n-cond-lemma12} and the design that $b = n/T$. 
    
For $\mathcal{E}_2$, the Hanson--Wright inequality \cite[e.g.][Theorem 6.2.1]{vershynin2018high} implies that 
    \begin{equation}\label{eq:HW-inequality}
         \mathbb{P}\Big(\|X_i\|_2 \leq c_1\sqrt{\mathrm{Tr}(\Sigma)\log(1/\eta_1)}\Big) \geq 1-\eta_1,
    \end{equation}
    for any $\eta_1>0$.  Applying a union bound argument to \eqref{eq:HW-inequality}, we obtain that $\mathbb{P}(\mathcal{E}_2) \geq 1-\eta_1$ with the choice $R = c_2\sqrt{d\log(n/\eta_1)}$. 

For $\mathcal{E}_3$, note that for any fixed $\beta^t$, $t \in \{0\} \cup [T-1]$ and $i \in [b]$, it holds that
    \[
    r_{i,t} = X_{\tau_t+i}^\top \beta^{t} - Y_{\tau_t+i} = X_{\tau_t+i}^\top (\beta^{t} - \beta^*) + \xi_{\tau_t+i}
    \]
has zero-mean, variance $\|\beta^{t} - \beta^*\|^2_{\Sigma}+\sigma^2$, and $\|r_{i,t}\|_{\psi_2} \leq C\sqrt{\sigma^2+\|\beta^{t} - \beta^*\|^2_{\Sigma}}$ for some absolute constant $C$. Therefore $|r_{i,t}| \leq c_2(\sigma+\|\beta^t-\beta^*\|_{\Sigma})\sqrt{\log(1/\eta_2)}$ with probability at least $1-\eta_2$, using the sub-Gaussian tail properties \cite[e.g.][Proposition 2.5.2]{vershynin2018high}.  Applying a union bound argument, we have that
   \begin{equation}\label{eq:E_2_part1}
       |r_{i,t}| \leq c_2(\sigma+\|\beta^t-\beta^*\|_{\Sigma})\sqrt{\log(n/\eta_2)} \qquad \forall i \in [b], t \in \{0\} \cup [T-1],
   \end{equation}
   holds with probability at least $1-\eta_2$.
   
It is worth noting that for any fixed $\beta^t$, we have $\{r_{i,t}\}_{i = 1}^b$ are i.i.d, since in each iteration we use a fresh batch of samples that is independent of $\beta^t$. Hence, applying \Cref{lemma:privatevarianceSG} leads to
   \begin{equation}\label{eq:privatevariance}
       \sqrt\frac{3}{{4}} (\sigma+\|\beta^t-\beta^*\|_{\Sigma}) \leq \text{PrivateVariance}(\{r_{i,t}\}_{i=1}^b,\epsilon',\delta') \leq \sqrt{\frac{5}{2}}(\sigma+\|\beta^t-\beta^*\|_{\Sigma}),
   \end{equation}
   with probability at least $1-\eta_3$, if $b \gtrsim \log\{1/(\delta\eta_3)\}\log(\log[1/(\eta_3\delta)](\eta_3\epsilon')^{-1})/\epsilon'$.
   Combining \eqref{eq:E_2_part1}, \eqref{eq:privatevariance} and a union bound argument, we have that $\mathbb{P}(\mathcal{E}_3) \geq 1-\eta_2-\eta_3$ as long as $R_t \geq c_2\sqrt{2\log(n/\eta_2)}\text{PrivateVariance}(\{r_{i,t}\}_{i=1}^n,\epsilon',\delta')$ and 
   \[
   n \gtrsim \frac{T\log(T/(\delta\eta_3))\log(T\log[T/(\eta_3\delta)](\eta_3\epsilon)^{-1})}{\epsilon}.
   \]
   
For $\mathcal{E}_4$, recall that $Z_{\tau_t+i} = X_{\tau_t+i}\xi_{\tau_t+i}$. Therefore, we use Hanson--Wright inequality again to obtain
   \[
   \mathbb{P}_{Z|\xi}\left(\big\|\frac{1}{b}\sum_{i = 1}^bZ_{\tau_t+i} \big\|_2^2> \mathrm{Tr}(\Sigma_\xi)+2\sqrt{\mathrm{Tr}(\Sigma_\xi^2)\log(1/\eta_4)}+2\|\Sigma_\xi\|_2\log(1/\eta_4)\right) \leq \eta_4,
   \]
   where $\Sigma_\xi = (\frac{1}{b^2}\sum_{i=1}^b\xi^2_{\tau_t+i})\Sigma$. Since 
   \[
   \frac{1}{b}\sum_{i=1}^b\xi^2_{\tau_t+i} \lesssim \sigma^2\log(1/\eta_5)
   \]
   with probability at least $1-\eta_5$, by a standard sub-exponential concentration inequality \cite[e.g.][Corollary 2.8.3]{vershynin2018high}, and $\xi_i$'s and $X_i$'s are independent, we further have 
   \[
   \mathbb{P}\left(\big\|\frac{1}{b}\sum_{i = 1}^bZ_{\tau_t+i} \big\|_2^2>c_3\frac{\sigma^2d\log(1/\eta_4)\log(1/\eta_5)}{b}\right) \leq \eta_4+\eta_5
   \]

   Overall, we have 
   \[
   \mathbb{P}(\mathcal{E}_1 \cap \mathcal{E}_2 \cap \mathcal{E}_3 \cap \mathcal{E}_4) \geq 1-\eta-\eta_1-\eta_2-\eta_3-\eta_4-\eta_5, 
   \]
    and setting $\eta_1 = \eta_2 = \eta_3 = \eta_4 = \eta_5 = \eta$ yields the claim.
   
\end{proof}
\begin{cor}\label{cor:transfer-high-probability-events}
Consider the events defined in the proof of \Cref{thm:lowd-regression-transfer}, namely $\mathcal{E}_1'$, $\mathcal{E}_2'$, $\mathcal{E}_3'$ and $\mathcal{E}_4'$. Let $N = \sum_{k \in \{0\} \cup \mathcal{A}} n_k$. For $\eta \in (0, 1)$, under the conditions that 
\[
    \min_{k \in \stset} n_k \gtrsim Td\log(T/\eta) \vee T\log(T/(\delta\eta))\log(T/(\eta(\epsilon\wedge\delta))\epsilon^{-1}, \quad R \gtrsim \sqrt{d\log(N/\eta)}
\]
and
\[
    R_t \gtrsim \sqrt{\log(N/\eta)}\mathrm{PrivateVariance}(\{X_{\tau_t+i}^\top \beta^{t} - Y_{\tau_t+i}\}_{i=1}^b,\epsilon',\delta'),
\]
we have that $$\mathbb{P}(\mathcal{E}'_1 \cap \mathcal{E}'_2 \cap \mathcal{E}'_3 \cap \mathcal{E}'_4) \geq 1-6\eta.$$ 
\end{cor}

\begin{proof}[Proof of \Cref{cor:transfer-high-probability-events}]
    The proof generalises the single-site result in \Cref{lemma:highprobabilityevent} to the multi-site setting. For brevity, we only point out the differences between controlling $\{\mathcal{E}_i\}_{i \in [4]}$ and $\{\mathcal{E}_i'\}_{i \in [4]}$. 
    
    For $\mathcal{E}_1'$, we note that the population version of 
    \[
    \sum_{k \in \{0\} \cup {\mathcal{A}}} \frac{v_k}{b^{(k)}}\sum_{i=1}^{b^{(k)}}X^{(k)}_{\tau_t+i}X_{\tau_t+i}^{(k)\top}
    \]
    is $\tilde{\Sigma} = \sum_{k \in \stset} n_k \Sigma^{(k)}/N$, which has $\lambda_{\min}(\tilde{\Sigma}) \geq 1/L$ and $\lambda_{\max}(\tilde{\Sigma}) \leq L$. Therefore, the same arguments as in \Cref{lemma:highprobabilityevent} apply, and we obtain {$\mathbb{P}(\mathcal{E}_1') \geq 1-\eta$,} as long as $N \gtrsim Td\log(T/\eta)$. 
    
    For $\mathcal{E}_2'$, the same arguments for controlling $\mathcal{E}_2$ in \Cref{lemma:highprobabilityevent} still work, but with $n$ by $N$ in the choice of $R$ to account for the union bound over $N$ random variables. 
    
    The same arguments for $\mathcal{E}_3$ also works for $\mathcal{E}_3'$ but with $\Sigma$ replaced by $\Sigma^{(k)}$ where appropriate, and notice that $\|\beta^t - \beta\|_{\Sigma^{(k)}} \lesssim \|\beta^t - \beta\|_2$ for any $k$. 
    
    For $\mathcal{E}_4'$, we need to replace $\Sigma_\xi$ in $\mathcal{E}_4$ by 
    \[
    \Sigma_\xi = \sum_{k \in \stset} \frac{n_k^2}{(Nb^{(k)})^2} \sum_{i=1}^{b^{(k)}}\xi_{\tau_t+i}^2\Sigma^{(k)}.
    \]
 Then the same arguments show that $\mathbb{P}(\mathcal{E}_4') \geq 1-2\eta$.
\end{proof}

\begin{defn}
    Given a data set $D$, we say a randomised algorithm $M$ is $(\epsilon,\delta)$-central DP with respect to a set $S \subseteq D$, if 
    \begin{equation*} 
    \mathbb{P}(M(D) \in O|D) \leq e^{\epsilon} \mathbb{P}(M(D') \in O|D')
\end{equation*}
for any measurable set $O$ and any data set $D'$ that can be obtained by altering at most one data entry in $S$. 
\end{defn}
We use $\mathcal{M}_{\epsilon,\delta}^{S}$ to denote the set of all procedures that are $(\epsilon,\delta)$-central DP with respect to $S$. Note that any $(\epsilon,\delta)$-central DP algorithm is, by definition, $(\epsilon,\delta)$-central DP with respect to $S$ for any $S \subseteq D$. This weaker notion is helpful to consider the source data sets in $\mA$ and $\mA^c$ separately. 

\begin{lemma}\label{lemma:tonycai}
Consider the following class of distributions
\[
    \mathcal{P}_{c}(\bm{\beta}) = \left\{P_{\beta}^{\otimes n} P_{\beta'}^{\otimes m}: P_{\beta} = P_{y|x,\beta} P_x, \, P_{y|x,\beta} = \mathcal{N}(x^{\top}\beta, \sigma^2), \, P_x = \mathcal{N}(0, I), \,\|\beta\|_2 \leq c, \,c \geq 0\right\}.
\]
 Let $(\bm{Y}, \bm{X}) = \{(Y_i,X_i)\}_{i \in [n+m]}$ be generated from the distribution $P_{\bm{\beta}} \in  \mathcal{P}_{c}(\bm{\beta})$, and $S$, with $|S| = n$, denote the set of data that corresponds to the parameter $\beta$. Suppose that $\beta' \in \mathbb{R}^d$ is not a function of $\beta$ and following conditions hold
    \[
    c \leq \sqrt{d}, \quad \epsilon \in (0,1), \quad \delta < n^{-2}, \quad d\log(1/\delta) \lesssim n, \quad \text{and} \quad d^2\sigma^2 \gtrsim 1.
    \]
    Then, for every estimator $M(\bm{Y},\bm{X}) \in \mathcal{M}_{\epsilon,\delta}^{S}$, it holds that 
    \[
    \inf_{\substack{M \in \mathcal{M}^{S}_{\epsilon,\delta}}} \sup_{P \in \mathcal{P}_c(\beta)} \mathbb{E}\|M(\bm{Y},\bm{X}) - \beta\|^2_2 \gtrsim \left\{\sigma^2 \left(\frac{d}{n}+\frac{d^2}{n^2\epsilon^2}\right)\right\} \wedge c^2.
    \]
\end{lemma}

\begin{proof}[Proof of \Cref{lemma:tonycai}]
Without loss of generality, we assume $\{(Y_i,X_i)\}_{i \in [n]}$ are generated from the normal linear model with parameter $\beta$, and $\{(Y_i,X_i)\}_{i=n+1}^{m+n}$ are generated from the normal linear model with parameter $\beta'$. For any $M(\bm{Y},\bm{X}) \in \mathcal{M}_{\epsilon,\delta}^{S}$, we consider $\tilde{M}(\bm{Y},\bm{X}) = \Pi_{\sqrt{d}}(M(\bm{Y},\bm{X}))$, which projects the original estimator $M$ onto the $\ell_2$-ball centred at the origin with radius $\sqrt{d}$. Then, it holds that 
\[
\|\tilde{M}(\bm{Y},\bm{X}) - \beta\|_2 \leq 2\sqrt{d} \quad \text{and} \quad \|\tilde{M}(\bm{Y},\bm{X}) - \beta\|_2 \leq \|M(\bm{Y},\bm{X})-\beta\|_2,
\]
for any $M(\cdot)$, any $(\bm{Y},\bm{X})$ and any $\beta \in \mathbb{R}^d$ with $\|\beta\|_2 \leq c \leq \sqrt{d}$. The first inequality follows from the triangle inequality, and the second inequality follows from the non-expansive property of the projection operator $\Pi_{\sqrt{d}}(\cdot)$. Now, we have 
\[
\inf_{\substack{M \in \mathcal{M}^{S}_{\epsilon,\delta}}} \sup_{P \in \mathcal{P}_c(\beta)} \mathbb{E}\|M(\bm{Y},\bm{X}) - \beta\|^2_2 \geq \inf_{\substack{M \in \mathcal{M}^{S}_{\epsilon,\delta}: \\ \|{M}(\bm{Y},\bm{X}) - \beta\|_2 \leq 2\sqrt{d} \; \text{for any}\; \beta \\ \text{with} \; \|\beta\|_2 \leq c }} \sup_{P \in \mathcal{P}_c(\beta)} \mathbb{E}\|M(\bm{Y},\bm{X}) - \beta\|^2_2.
\]
Therefore, in the remaining of the proof, we work under the assumption that {$\|M(\bm{Y},\bm{X}) - \beta\|_2 \leq 2\sqrt{d}$ almost surely}, for all $\beta$ with $\|\beta\|_2 \leq c$. 
Furthermore, 
it is sufficient to show that if $\mathbb{E}\|M(\bm{Y},\bm{X}) - \beta\|^2_2 = o(c^2)$, for all $\beta$ with $\|\beta\|_2 \leq c$, then 
 \[
    \inf_{M \in \mathcal{M}_{\epsilon,\delta}} \sup_{P \in \mathcal{P}_c(\beta)} \mathbb{E}\|M(\bm{Y},\bm{X}) - \beta\|^2_2 \gtrsim \sigma^2 \left(\frac{d}{n}+\frac{d^2}{n^2\epsilon^2}\right),
    \]
since otherwise $\mathbb{E}\|M(\bm{Y},\bm{X}) - \beta\|^2_2 \gtrsim c^2$ and together we have the claimed result. In particular, it suffices to prove the second term that involves $\varepsilon$, since the first term follows from the non-private minimax lower bound.

We follow the main arguments in the proofs of Lemma 4.1 and Theorem 4.1 in \cite{cai2019cost} and make adjustments for our setting. For $i \in [n]$, let $(\bm{Y}_i',\bm{X}_i')$ be the data set obtained by only replacing $(Y_i,X_i)$ in $(\bm{Y}, \bm{X})$ with an independent copy.  Let 
    \[
    A_i = \langle M(\bm{Y}, \bm{X}) - \beta,(Y_i - X_i^{\top}\beta)X_i\rangle \quad \mbox{and} \quad A_i' = \langle M(\bm{Y}_i', \bm{X}_i') - \beta,(Y_i - X_i^{\top}\beta)X_i\rangle.
    \]
    Note that 
    \begin{equation}\label{eq:Ai'}
    	\mathbb{E}A_i' = 0 \quad \text{and} \quad \mathbb{E}|A_i'|\leq \sigma \sqrt{\mathbb{E}\|M(\bm{Y},\bm{X}) - \beta\|_2^2}.
    \end{equation}
    Writing $f_\beta(\bm{y}, \bm{x})$ as the joint density, we have  
    \[
    f_\beta(\bm{y}, \bm{x}) = \bigg(\frac{1}{\sqrt{2\pi}\sigma}\bigg)^{m+n} \exp\bigg\{-\frac{\sum_{i=1}^n(y_i - x_i^{\top}\beta)^2+\sum_{i=n+1}^{n+m}(y_i-x_i^{\top}\beta')^2}{2\sigma^2}\bigg\} \prod_{i=1}^{m+n}\phi(x_i),
    \]
    where $\phi(x_i)$ is the density of $\mathcal{N}(0,I)$. Note that since $\beta'$ is not a function of  $\beta$, we have 
    \[
    \frac{\partial f_\beta(\bm{y}, \bm{x})}{\partial\beta} =  \frac{f_\beta(\bm{y}, \bm{x})}{\sigma^2}\sum_{i=1}^n(y_i - x_i^{\top}\beta)x_i,
    \]
    and therefore we have 
    \begin{align*}
          \sum_{i\in [n]} \mathbb{E} A_i &= \sum_{j \in [d]} \mathbb{E} \bigg\{\{M(\bm{Y}, \bm{X})\}_j \sum_{i \in [n]} (Y_i - X_i^{\top}\beta)x_{ij} \bigg\} \\ &= \sigma^2 \sum_{j \in [d]} \mathbb{E}\left\{\{M(\bm{Y}, \bm{X})\}_j \frac{1}{f_\beta(\bm{Y}, \bm{X})}\frac{\partial f_\beta}{\partial\beta_j} \right\}\\ &= \sigma^2 \sum_{j \in [d]} \frac{\partial}{\partial \beta_j}\mathbb{E}\{M(\bm{Y}, \bm{X})\}_j.
    \end{align*}

Let $\nu_1,\dotsc,v_d$ be i.i.d.~random variables from truncated $\mathcal{N}(0,1)$ distribution with truncation at $-1$ and $1$, and $\beta_j = \nu_jc/\sqrt{d}$ so that $\|\beta\|_2 \leq c$. Denote the distribution on $\beta$ as $\pi$. Following the same arguments as those in the proof of Lemma 4.1 in \cite{cai2019cost}, we obtain that
\[
\mathbb{E}_\pi \sum_{i \in [n]}\mathbb{E}A_i \geq \frac{\sigma^2 d}{c^2} \left(\mathbb{E}_\pi \sum_{j \in [d]} \beta_j^2 - \sqrt{\mathbb{E}_\pi \mathbb{E}_{\bm{Y},\bm{X}|\beta}\|M(\bm{Y},\bm{X}) - \beta\|^2_2} \sqrt{\mathbb{E}_\pi \sum_{j \in [d]} \beta_j^2}\right) \gtrsim \sigma^2 d.
\]
Next, we complement the above with an upper bound on $\sum_{i \in [n]}\mathbb{E}A_i$. Using \eqref{eq:Ai'} with Lemma B.2 in \citet{cai2019cost} for each $i\in [n]$, we have 
\[
\sum_{i \in [n]}\mathbb{E}A_i \leq 2n\epsilon \sigma \sqrt{\mathbb{E}\|M(\bm{Y},\bm{X}) - \beta\|^2_2}+2n\delta T + n \int_T^{\infty}\mathbb{P}(|A_i| > t)\,\mathrm{d}t.
\]
For the last term, for $t > 4$, we have that
\begin{align*}
    \mathbb{P}(|A_i| > t) &= \mathbb{P}\left(|Y_i - X_i^{\top}\beta|\big|\langle X_i,M(\bm{Y},\bm{X}) - \beta \rangle\big| > t\right) \\ &\leq \mathbb{P}(|Y_i - X_i^{\top}\beta| d > \sqrt{t}) + \mathbb{P}(2\|X_i\|_2 \geq \sqrt{dt}) \\ 
    & \leq 2\exp\left(\frac{-t}{2d^2\sigma^2}\right) + \exp(-c_0t),
\end{align*}
for some absolute constant $c_0>0$.
Therefore,
\[
\sum_{i \in [n]}\mathbb{E}A_i \lesssim n\epsilon \sigma \sqrt{\mathbb{E}\|M(\bm{Y},\bm{X}) - \beta\|^2_2}+n\delta T + n d^2\sigma^2\exp\big\{-T/(2d^2\sigma^2)\big\} + n\exp(-T).
\]
Choosing $T \asymp d^2\sigma^2\log(1/\delta)$ and taking expectation on both sides with respect to $\pi$ guarantee that 
\[
n\epsilon \sigma \sqrt{\mathbb{E}_\pi \mathbb{E}_{\bm{Y},\bm{X}|\beta}\|M(\bm{Y},\bm{X}) - \beta\|^2_2} \gtrsim \sigma^2 d - nd^2\sigma^2\delta\log(1/\delta) = \sigma^2d \{1 - nd\delta\log(1/\delta)\} \gtrsim \sigma^2 d,
\]
where we use $d^2\sigma^2 \gtrsim 1$ in the choice of $T$ and the conditions on $\delta$ in the last inequality. As the Bayes risk always lower bounds the supremum risk, the proof is concluded.   
\end{proof}

\begin{lemma}\label{lemma:van-tree-lowerbound}
Suppose $\sum_{k \in\stset} \sum_{t = 1}^T\{b^t_k d \wedge (b^t_k\varepsilon)^2\} \gtrsim d^2$ and $d\delta\log(1/\delta) \lesssim \varepsilon^2 < 1$, then 
    \[ \inf_{Q \in \mathcal{Q}_{\epsilon,\delta,T}} \inf_{\substack{\hat{\beta}(Z)}} \sup_{P_{\bm{\beta}} \in \mathcal{P}_{\bm{\beta}} } {\mathbb{E}_{P_{\bm{\beta}},Q}\|\hat{\beta}(Z) - \beta\|_2^2} \gtrsim 
\frac{d}{\sum_{k \in\stset} \{n_k \wedge \{(n_k\varepsilon)^2/d\}\}}.
\]
\end{lemma}

\begin{proof}[Proof of \Cref{lemma:van-tree-lowerbound}]

We first recall and introduce some notations. Let $D_k^t$ denote the data set used in iteration $t$ at site $k$ with size $b^t_k$. Specifically, in the linear regression problem, we have $\{D_k^t\}_{t=1}^T$ form a partition of $\{X_i^{(k)}, Y_i^{(k)}\}_{i=1}^{n}$, each with sample size $b^t_k$. A private transcript $Z_k^t$ is produced at each site $k$ in iteration $t$. Let $Z_k = \{Z_k^t\}_{t=1}^T$ denote the entire private transcript generated at site $k$ across the $T$ iterations. Let $B^{t} = (B^{t-1}, \{Z^t_k\}_{k \in \tkset})$ denote the set of all private transcripts generated from all sites in and before iteration $t$. 

The setting that we consider is $\beta^{(k)} = \beta$ for all $k \in \mathcal{A}$, and $\beta^{(k)} = \beta'$ for $k \not\in \mathcal{A}$ such that $\beta'$ is not a function of $\beta$. We denote the collection of all regression parameters by $\bm{\beta}$ and $\bm{\beta} \in \Theta_{\bm{\beta}}(\mathcal{A},h)$ for any $\mathcal{A}$ and $h >0$ by choosing $\|\beta'\|_2$ sufficiently large.
Our proof follows a similar structure as in \cite{xue2024optimal} and \cite{cai2024optimal}, which requires an application of the Van-Trees inequality \citep[][Theorem 1]{gill1995applications}. Using an appropriate transformation with $\psi((\beta,\beta')) = \beta$ and treating $\beta'$ as deterministic, we have that for any estimator $\hat{\beta}$ 
\begin{align*}
\int \mathbb{E}_{P_{\bm{\beta}},Q} \|\widehat{\beta}(Z) - \beta\|_2^2 \pi(\beta)\,\mathrm{d}\beta \geq \frac{d^2}{\int \mathrm{Tr} (I_{Z_0,\dotsc,Z_{K}}(\beta))\pi(\beta)\, \mathrm{d}\beta + J(\pi)},
\end{align*}
where $\pi$ is a prior distribution on $\beta$, $I_{Z_0,\dotsc,Z_{K}}(\beta)$ is the Fisher information associated with $\{Z_0,\dotsc,Z_{K}\}$ and $J(\pi)$ is the Fisher information associated with $\pi$. We shall consider a prior distribution on $\beta$ with independent components $\pi(\beta) = \prod_{i=1}^d\pi_i(\beta_i)$, where $\beta_i$ is the $i$-th entry of $\beta$, and in this case 
\[
J(\pi) = \sum_{i=1}^d\int \frac{(\pi_i'(\beta_i))^2}{\pi_i(\beta_i)} \mathrm{d} \beta_i.
\]
Specifically, let $\beta_1\dotsc,\beta_d \overset{i.i.d}{\sim}N(0,1)$, and then we have $\pi_i'(\beta_i) = -\beta_i\pi_i(\beta_i)$ and therefore $J(\pi) = d$. Using the chain rule of the Fisher information, it holds that for any $\beta \in \mathbb{R}^d$,
\[
    I_{Z_0,\dotsc,Z_{K}}(\beta) = \sum_{k = 0}^K \sum_{t = 1}^T I_{Z_k^t | B^{t-1}}(\beta),
\]
where $Z_k^t$ is the transcript generated from the $k$-th data set in the $t$-th round and $B^{t-1}$ is all the private information generated in the previous $(t-1)$ rounds.
Hence, we now have 
\begin{align}\label{eq:van-trees}
     \sup_{P_{\bm{\beta}} \in \mathcal{P}_{\bm{\beta}} } {\mathbb{E}_{P_{\bm{\beta}},Q}\|\hat{\beta} - \beta\|_2^2} &\geq \frac{d^2}{\int \mathrm{Tr} (I_{Z_0,\dotsc,Z_{K}}(\beta))\pi(\beta) \, \mathrm{d}\beta + d^2} \nonumber \\ &\geq \frac{d^2}{\sup_{\beta \in \mathbb{R}^d} \mathrm{Tr} (I_{Z_0,\dotsc,Z_{K}}(\beta)) + d^2} \nonumber\\
        & = \frac{d^2}{\sup_{\beta \in \mathbb{R}^d}\sum_{k = 0}^K \sum_{t = 1}^T \mathrm{Tr}\big(I_{Z_k^t | B^{t-1}}(\beta)\big)  + d^2}.
\end{align}

We are now to upper bound $\sup_{\beta \in \mathbb{R}^d}\sum_{k = 0}^K \sum_{t = 1}^T \mathrm{Tr}\big(I_{Z_k^t | B^{t-1}}(\beta)\big)$. For the data corresponding to the $k$-th site, used in $t$-th round, $D^t_k = \{(X_{t, i}^{(k)},Y_{t, i}^{(k)})\}_{i \in [b^t_k]}$, with $k \in \stset$, $t \in [T]$, define 
\[
    S_{\beta}(D_k^t) = \sum_{i \in [b_k^t]} S_{\beta}(D_{k,i}^t)=\sum_{i \in [b_k^t]} \big(Y^{(k)}_{t,i} - \beta^{\top}X^{(k)}_{t,i}\big)X^{(k)}_{t,i}.
\]
Recall that $Y^{(k)}_{t,i} =\beta^{\top}X^{(k)}_{t,i} + \xi^{(k)}_{t,i}$, where $\xi^{(k)}_{t,i} \sim \mathcal{N}(0,1)$. Hence, we can write 
\[
S_{\beta}(D_k^t) = \frac{\partial}{\partial \beta} \log f_\beta(D_k^t) = \frac{1}{f_{\beta}(D_k^t)}\frac{\partial f_\beta}{\partial \beta},
\]
where $f_{\beta}(D_k^t)$ is the likelihood of $D_k^t$. For $k \not\in \mathcal{A}$, since the true parameter $\beta'$ is not a function of $\beta$, we have the score function $S_{\beta}(D_k^t) = \frac{\partial}{\partial \beta} \log f_{\beta'}(D_k^t) =0, k \not\in \mathcal{A}$.
Furthermore, for $k \in \stset$, let $C_{\beta}(Z_k^t|B^{t-1})$ denote the $d \times d$ matrix
\[
    \mathbb{E}\left\{S_{\beta}(D_k^t) \big| Z_k^t, B^{t-1} \right\} \mathbb{E}\left\{S_{\beta}(D_k^t) \big| Z_k^t, B^{t-1} \right\}^{\top}
\]
and write $C_\beta(D_k^t)$ for the unconditional version of the covariance matrix of $S_{\beta}(D_k^t)$. Following the same calculation as (63) in \cite{xue2024optimal}, we have 
\[
I_{Z_k^t | B^{t-1}}(\beta) = \mathbb{E}\bigg[\mathbb{E}\Big(C_{\beta}(Z_k^t|B^{t-1})\Big| B^{t-1}\Big)\bigg], \quad k \in \stset.
\]
For $k \not\in \mathcal{A}$, we have $I_{Z_k^t | B^{t-1}}(\beta) = 0$. The rest of the proof is concerning finding an upper bound for 
\[
\sup_{\beta \in \mathbb{R}^d}\sum_{k = 0}^K \sum_{t = 1}^T \mathrm{Tr}\big(I_{Z_k^t | B^{t-1}}(\beta)\big) = \sup_{\beta \in \mathbb{R}^d}\sum_{k \in\stset} \sum_{t = 1}^T \mathrm{Tr}\big(I_{Z_k^t | B^{t-1}}(\beta)\big).
\]

From this point on, despite the different notations, we can apply the arguments, Case 1 - Step 2 \& 3, used in the Proof of Proposition 10 in \cite{xue2024optimal} with $m = 1$. We shall highlight only the key steps leading to the final result below. 
With the notation 
    \begin{align*}
        G_{k,i}^{t} = \langle \mathbb{E}\big\{S_{\beta}(D_k^t))\big| Z_k^t, B^{t-1}\big\}, S_\beta(D_{k,i}^{t})\rangle \quad \text{and} \quad \breve{G}_{k,i}^{t} = \langle \mathbb{E}\big\{S_{\beta}(D_k^t))\big| Z_k^t, B^{t-1}\big\}, S_\beta(\breve{D}_{k,i}^{t})\rangle,
    \end{align*}
    where $\breve{D}_{k,i}^{t}$ is an independent copy of ${D}_{k,i}^{t}$, it can be shown that 
    \[
    \mathrm{Tr}\big(I_{Z_k^t | B^{t-1}}(\beta)\big) \lesssim \sum_{i \in [b_k^t]}\left(\varepsilon \mathbb{E}|\breve{G}_{k,i}^{t}| + W\delta+\int_{W}^\infty \mathbb{P}\{|G_{k,i}^{t}| \geq w\}\; \mathrm{d}w\right),
    \]
for any $W>0$. We shall choose a covariate distribution with independent and bounded components, such that it belongs to $\mathrm{SG}(C, I)$, for some absolute constant $C>0$. With this choice of covariate distribution, we can further control 
\[
\mathbb{E}|\breve{G}_{k,i}^{t}| \lesssim \sqrt{\mathrm{Tr}\big(I_{Z_k^t | B^{t-1}}(\beta)\big)},
\]
and for the tail probability in the last term, 
\[
\int_{W}^\infty \mathbb{P}\{|G_{k,i}^{t}| \geq w\} \mathrm{d}w \lesssim \int_{W}^\infty \exp\Big(-\frac{w}{b_k^t d}\Big)\mathrm{d}w = b^t_kd\exp\Big(-\frac{W}{b_k^t d}\Big).
\]
Choosing $W = b^t_kd\log(1/\delta)$, we obtain
\[
\mathrm{Tr}\big(I_{Z_k^t | B^{t-1}}(\beta)\big) \lesssim b^t_k\varepsilon\sqrt{\mathrm{Tr}\big(I_{Z_k^t | B^{t-1}}(\beta)\big)} + (b^t_k)^2d\delta\log(1/\delta)+\delta(b^t_k)^2d.
\]
Now, if $b^t_k\varepsilon\sqrt{\mathrm{Tr}\big(I_{Z_k^t | B^{t-1}}(\beta)\big)} \gtrsim (b^t_k)^2d\delta\log(1/\delta)+\delta(b^t_k)^2d$, then we obtain $\mathrm{Tr}\big(I_{Z_k^t | B^{t-1}}(\beta)\big) \lesssim (b^t_k\varepsilon)^2$. If $b^t_k\varepsilon\sqrt{\mathrm{Tr}\big(I_{Z_k^t | B^{t-1}}(\beta)\big)} \lesssim (b^t_k)^2d\delta\log(1/\delta)+\delta(b^t_k)^2d$ instead, we still have 
\begin{equation}\label{eq:private-rate-lowerbound}
    \mathrm{Tr}\big(I_{Z_k^t | B^{t-1}}(\beta)\big) \lesssim (b^t_k\varepsilon)^2
\end{equation}
 under the assumption that $d\delta\log(1/\delta) \lesssim \varepsilon^2$. Finally, the non-private rate can be obtained using standard matrix algebra and properties of conditional expectations as 
\begin{equation}\label{eq:nonprivate-lowerbound}
    \mathrm{Tr}\big(I_{Z_k^t | B^{t-1}}(\beta)\big) \leq \mathrm{Tr}(C_\beta(D_k^t)) \lesssim b^t_k d.
\end{equation}
Substituting \eqref{eq:private-rate-lowerbound} and \eqref{eq:nonprivate-lowerbound} into \eqref{eq:van-trees}, we obtain
\[
\sup_{P_{\bm{\beta}} \in \mathcal{P}_{\bm{\beta}} } {\mathbb{E}\|\hat{\beta} - \beta\|_2^2} \gtrsim \frac{d^2}{\sum_{k \in\stset} \sum_{t = 1}^T\{b^t_k d \wedge (b^t_k\varepsilon)^2\}+d^2} \gtrsim \frac{d^2}{\sum_{k \in\stset} \sum_{t = 1}^T\{b^t_k d \wedge (b^t_k\varepsilon)^2\}}
\]
 when  $\sum_{k \in\stset} \sum_{t = 1}^T\{b^t_k d \wedge (b^t_k\varepsilon)^2\} \gtrsim d^2$. Moreover, since $\sum_{t=1}^T b_k^t = n_k$ and $\sum_{t=1}^T(b^t_k)^2\leq n_k^2$, we further have 
\[ \inf_{Q \in \mathcal{Q}_{\epsilon,\delta,T}} \inf_{\substack{\hat{\beta}(Z)}} \sup_{P_{\bm{\beta}} \in \mathcal{P}_{\bm{\beta}} } {\mathbb{E}_{P_{\bm{\beta}},Q}\|\hat{\beta}(Z) - \beta\|_2^2} \gtrsim 
\frac{d^2}{\sum_{k \in\stset} \{n_k d \wedge (n_k\varepsilon)^2\}} = \frac{d}{\sum_{k \in\stset} \{n_k \wedge \{(n_k\varepsilon)^2/d\}\}}.
\]
\end{proof}

\section{Technical Details of Section \ref{sec:highd-regression}}\label{sec:appendix-4}
\subsection{Discussion}\label{sec:discussion-highd}

In contrast to the mean estimation and low-dimensional linear regression problems, where small heterogeneity between the target and sources $h$ is sufficient to ensure an improvement over the target-only rate, the high-dimensional linear regression problem demands a more stringent condition, shown in (b) above.  The condition (b) incorporates a dimension-dependent factor $d/s$, which is more demanding compared to the low-dimensional regression case, where {$|\mA| \gg 1$} is adequate to achieve a faster rate when $n_k$ is of the same order as $n_0$ for all $k \in \mA$.

This difference arises due to a factor of $\sqrt{ds}$ that emerges in the privacy term of the aggregation rate \eqref{eq: aggregation rate high dim}. This factor appears because Algorithm \ref{algorithm:DPregression_high_dim_combined} introduces dimension-dependent noises (i.e.\ the variance of the Gaussian noise $w_t^{(k)}$ scales with $d$) to the gradient from each data set before forwarding it for aggregation.  This is vital in showing that the procedure satisfies the FDP constraint. Notably, a similar term has been observed in the context of high-dimensional LDP regression and shown to be unimprovable \citep{wang2019sparse, zhu2023improved}. Our FDP setting reduces to LDP when $n = 1$, and therefore it is not surprising to observe the same term in our setting. 

The main open question that arises is whether we can improve the result when $n$ is large and obtain an estimation error that scales with $s$ instead of $\sqrt{ds}$ in \eqref{eq: aggregation rate high dim}. In fact, with some additional assumptions, we conjecture that this should be feasible. For instance, imposing specific conditions for variable selection consistency so that
the target data could identify the signal variables and privately communicate this set to other source data sets, then the dimensionality of the problem can also be reduced to be independent of~$d$. However, we note that the conditions commonly required for variable selection consistency in high-dimensional regression literature, known as irrepresentable conditions \citep[e.g.][]{van2009conditions,zhao2006model}, are substantially stronger than the assumptions we make here. We leave rigorous analyses of such heuristics and the potential for improving the upper bound when $n_k$'s large as an important future research direction.

\begin{table}[!ht]
\centering
\renewcommand{\arraystretch}{1}
\begin{tabular}{cccc}
\toprule
 No privacy &   Central DP&  FDP&  LDP \\
\scriptsize \citep{negahban2011estimation} & \scriptsize\citep{cai2019cost} & \Cref{thm: high-dim upper bound} & \scriptsize \citep{zhu2023improved} \\
 \midrule
 
$ \sqrt{\frac{s}{nK}}$ & $ \sqrt{\frac{s}{nK}}+ \frac{s}{nK\epsilon}$ & $  \Big(\sqrt{\frac{s}{n}} + \frac{s}{n\epsilon}\Big) \wedge \Big(\sqrt{\frac{s}{nK}} + \frac{\sqrt{ds}}{n\epsilon\sqrt{K}}\Big)$ & $ \frac{\sqrt{ds}}{\sqrt{nK}\epsilon}$  \tablefootnote{\cite{zhu2023improved} also showed upper bounds under both non-interactive and interactive LDP constraints which do not match the lower bound in general. However, even their lower bound exceeds the FDP upper bound, which demonstrates that FDP is a weaker DP notion than LDP.}             \\ 
\bottomrule
\end{tabular}
\caption{Convergence rates of $\|\hbeta - \bbeta\|_2$ (up to logarithmic factors) under different privacy constraints, when $n_k = n$ for $k \in {0}\cup [K]$ and there is no heterogeneity, i.e.\ $h = 0$. The rates for no privacy and central DP are minimax, while the FDP rate is an upper bound and the LDP rate is a lower bound. }
\label{table: high-dim reg}
\end{table}

We conclude by comparing the $\ell_2$-estimation error of $\beta$ under different privacy constraints when $h = 0$, $n_k = n$ for $k \in {0}\cup[K]$ as in \Cref{sec:mean-discussion,sec:lowd-lowerbound}. The results are summarised in Table~\ref{table: high-dim reg}. Similar to our prior observations in the mean estimation problem (\Cref{sec:mean-discussion}) and the low-dimensional regression problem (\Cref{sec:lowd-lowerbound}), the rates deteriorate from left to right as the privacy notion becomes stronger. 

\subsection{Differentially private high-dimensional linear regression on a single data set}\label{sec-hd-app}

Consider the high-dimensional regression for a single data set with the central DP constraint, that
\begin{equation}\label{eq-Gaussian-high-d}
     Y_i = \<X_i,\bbeta^*\> + \xi_i, \quad i \in [n], 
\end{equation}
where $\bbeta^* \in \mathbb{R}^d$, $X_i \sim P_x \in \mathrm{SG}(C,\Sigma)$, and $\xi_i$ is mean-zero and sub-Gaussian with $\|\xi_i\|_{\psi_2}\leq \sigma$. The regression coefficient~$\bbeta$ is assumed to be $s$-sparse, i.e.~$\zeronorm{\bbeta^*} = s < d$.  The objective is to estimate $\bbeta^*$ while adhering to the $(\epsilon, \delta)$-central DP constraint.  This is conducted in \Cref{algorithm:DPregression_high_dim_single_source}, motivated by the private high-dimensional linear regression algorithm proposed in \cite{cai2019cost}. We show in \Cref{thm: high-dim upper bound single-source} that Algorithm \ref{algorithm:DPregression_high_dim_single_source} is $(\epsilon, \delta)$-central DP and achieves the minimax estimation error rate up to logarithmic factors.

\begin{algorithm}[!ht]
	\begin{algorithmic}[1]
		\INPUT{Data $\{X_i, Y_i\}_{i \in [n]}$, number of iteration $T$, step size $\rho$, privacy parameters $(\epsilon,\delta)$, initialisation $\bbeta^0$, failure probability $\eta \in (0,1)$,  hard-thresholding parameter $s'$, constant $L$}
        \State Set batch size $b = \lfloor n/T \rfloor$, truncation radius $R = 2\sqrt{L \log(nd/\eta)}$
		\For{$t = 0, \ldots, T-1$} 
                    \State Set $\tau = bt$ 
			     \State Set $R_t= 2\sqrt{\log(n/\eta)}$PrivateVariance$\Big(\big\{X_{\tau+i}^{\top} \bbeta^t - Y_{\tau+i}\big\}_{i=1}^{b},\epsilon/2,\delta/2\Big)$  \Comment{See Algorithm \ref{algorithm:Private_variance} for PrivateVariance algorithm}
                 \State $\bbeta^{t+0.5} = \bbeta^t - \frac{\rho}{b}\sum_{i=1}^{b}\prod_{R_t}\big(X_{\tau+i}^{\top}\bbeta^t - Y_{\tau+i}\big)\prod_R^{\infty}(X_{\tau+i})$
                 \State $\bbeta^{t+1} = \text{Peeling}(\bbeta^{t+0.5}, s', \epsilon/2, \delta/2, 2\rho R_tR/n)$  \Comment{See Algorithm \ref{algorithm:peeling} for Peeling algorithm}
		\EndFor
		\OUTPUT $\bbeta^T$. 
		\caption{Differentially private high-dimensional linear regression on a single data set} \label{algorithm:DPregression_high_dim_single_source}
	\end{algorithmic}
\end{algorithm} 

In Algorithm \ref{algorithm:DPregression_high_dim_single_source}, we deploy the adaptive clipping strategy \citep{varshney2022nearly}, which truncates the gradient by both an estimated radius $R_t$ and a fixed radius $R$. This approach relaxes the sample size requirement in \cite{cai2019cost} - this will be discussed in more detail later. 
Algorithm \ref{algorithm:DPregression_high_dim_single_source} deviates from \Cref{algorithm:DPregression_single} in the low-dimensional setting with the use of the `Peeling' algorithm (line 6). The `Peeling' algorithm
can be viewed as a noisy hard-thresholding algorithm. It selects a few coordinates of the coefficient estimate with the largest absolute values, adds noise to these coordinates, and truncates the remaining coordinates to zero. Analysing its performance reveals an optimal dependence on the intrinsic dimension $s$ rather than the full dimension $d$ in the estimation error rate. Full details of the `Peeling' algorithm can be found in \Cref{algorithm:peeling} in  \Cref{subsec: alg and theory high-dim appendix}, with similar algorithms adopted in the DP literature, e.g.\ \cite{cai2019cost} and \cite{dwork2021differentially}. 

It is worth noting that as a fresh batch of data is used in each iteration, by the parallel composition theorem, the final output is guaranteed to be $(\epsilon, \delta)$-central DP if $\bbeta^t$, $t \in [T]$, in each iteration is $(\epsilon, \delta)$-central DP. Each iteration divides the $(\epsilon, \delta)$ privacy budget into two halves: one for PrivateVariance and one for Peeling. \Cref{thm: high-dim upper bound single-source} provides the theoretical guarantees for \Cref{algorithm:DPregression_high_dim_single_source}, which matches the lower bound in \citet[][Theorem 4.3]{cai2019cost} up to logarithmic factors.

\begin{lemma}\label{thm: high-dim upper bound single-source}
Let $\{(X_i, Y_i)\}_{i = 1}^n$ be generated from the linear model \eqref{eq-Gaussian-high-d}. Suppose $0< L^{-1} \leq \lambda_{\min}(\Sigma) \leq \lambda_{\max}(\Sigma) \leq L < \infty$, for some absolute constant $L\geq 1$, and $\sigma = 1$. Initialise Algorithm \ref{algorithm:DPregression_high_dim_single_source} with $\bbeta^{0} = 0$. When $T \asymp \log(n)$, $s \gtrsim s' \geq 4.18L^4s$, $\rho = \frac{9}{10L}(1-0.296/L^4)$, $n \gtrsim \epsilon^{-1}s\log^{1/2}(1/\delta)\log^{5/2}(nd/\eta)$, and $\twonorm{\bbeta^0} \leq C$ with some constant $C > 0$, the output $\beta^T$ from Algorithm \ref{algorithm:DPregression_high_dim_single_source} is $(\epsilon, \delta)$-central DP and
\begin{equation*}
	\twonorm{\bbeta^{T}-\bbeta^*} \lesssim  r_{\textup{HLR}}(n, s, d, \epsilon, \delta, \eta) = \sqrt{\frac{s\log(d/\eta)\log(n)}{n}} + \frac{s\log^{1/2}(1/\delta)\log^{5/2}(nd/\eta)}{n\epsilon},
\end{equation*}
with probability at least $1-\eta$.
\end{lemma}

\begin{remark}
    The result in \Cref{thm: high-dim upper bound single-source} requires some prior knowledge about the sparsity level $s$ and the eigenvalue-related constant $L$, as we need to set $\rho = \frac{9}{10L}(1-0.296/L^4)$ and $s' \geq 4.18L^4s$ in Algorithm \ref{algorithm:DPregression_high_dim_single_source}. The requirement $s' \geq 4.18L^4s$ is weaker than the condition $s' \geq 72L^4s$ in \cite{cai2019cost}, as the latter $s'$ is impractical even when $L = 1$. Furthermore, we relax the sample size condition from $n = \widetilde{\Omega}(s^{3/2}\epsilon^{-1})$ in \cite{cai2019cost} to $n = \widetilde{\Omega}(s\epsilon^{-1})$, owing to the adaptive clipping technique, and achieve near optimality over a larger parameter space; See also the discussion after \Cref{lemma:DPregreesion_singlesite}. 
\end{remark}

\subsection{General algorithms and results}\label{subsec: hd diff nk appendix}
In this subsection, we present algorithms and theory that allow the sample sizes to vary across sites. These results are more general than those in Section~\ref{sec:highd-regression} and therefore automatically imply them.

\begin{algorithm}[!ht]
	\begin{algorithmic}[1]
		\INPUT{Data $\{(\Xk{k}_i, \Yk{k}_i)\}_{i \in [n_k], k\in \tkset}$, number of iteration $T$, step size $\rho$, privacy parameters $(\epsilon,\delta)$, initialisation $\bbeta^0$, failure probability $\eta \in (0,1/2)$, hard-thresholding parameter $s'$, subset $\mA' \subseteq [K]$, constant $L$}
        \State Set batch size $b^{(k)} = \lfloor n_k/T \rfloor$, for $k \in \{0\} \cup \mathcal{A}'$, truncation radius $R = 2\sqrt{L d\log(N/\eta)}$ and $N = \sum_{k\in \{0\}\cup \mA'}n_k$;
		\For{$t = 0, \ldots, T-1$} 
               \For{$k \in \{0\} \cup \mA'$}\Comment{Each site generates the privatised information $Z^t_k$ locally}
                    \State Set $\tau^{(k)} = b^{(k)}t$; $R^{(k)}_t = 2\sqrt{\log(N/\eta)}$PrivateVariance$\Big(\big\{\big(X^{(k)}_{\tau^{(k)}+i}\big)^{\top} \bbeta^t - Y^{(k)}_{\tau^{(k)}+i}\big\}_{i=1}^{b^{(k)}},\frac{\epsilon}{2},\frac{\delta}{2} \Big)$ 
                 \State Sample $w^{(k)}_t \sim \mathcal{N}\Big(0, \frac{8\log(2.5/\delta)R^2(R^{(k)}_t)^2}{(b^{(k)})^2(\epsilon/2)^2}\bm{I}_d\Big)$
                 \State Compute $Z^t_k = \frac{n_k}{N}\Big(\frac{1}{b^{(k)}}\sum_{i=1}^{b^{(k)}}\prod_{R_t^{(k)}}\big((X^{(k)}_{\tau^{(k)}+i})^\top\bbeta^t - Y^{(k)}_{\tau^{(k)}+i}\big)\prod_R(X^{(k)}_{\tau^{(k)}+i}) + w^{(k)}_t\Big)$
                 \EndFor
                 \State $\bbeta^{t+0.5} = \bbeta^t - \rho\sum_{k \in \{0\}\cup \mA'}Z^t_k$ \Comment{A central server aggregates the privatised information}
                 \State $\bbeta^{t+1} = \text{Hard-thresholding}(\bbeta^{t+0.5}, s')$, where for $j \in [d]$, 
                 \small
                 \[
                 \big(\text{Hard-thresholding}(\bbeta^{t+0.5}, s')\big)_j = \begin{cases}
                     \bbeta^{t+0.5}_j &\; \text{if} \; |\bbeta^{t+0.5}_j| \; \text{is among the $s'$ largest values of} \; \{|\bbeta^{t+0.5}_j|\}_{j \in [d]} \\
                     0 &\; \text{otherwise}
                 \end{cases}
                 \]
		\EndFor
		\OUTPUT $\bbeta^T$. 
		\caption{Federated high-dimensional linear regression with FDP guarantees} \label{algorithm:DPregression_high_dim_combined}
	\end{algorithmic}
\end{algorithm}

\begin{algorithm}[!ht]
	\begin{algorithmic}[1]
		\INPUT{Data $\{\{(\Xk{k}_i, \Yk{k}_i)\}_{i \in \floor{n_k/2}+1, \dotsc, n_k}\}_{k \in \tkset}$, privacy parameters, $(\epsilon,\delta)$, initialisation $\bbeta^{0}$, failure probability $\eta \in (0,1/2)$, hard-thresholding parameter $s'$, constant $L$, set $\hat{\mA}$ defined in \eqref{algorithm:DPregression_high_dim_detection}.}
        \If{$\frac{\sqrt{|\hat{\mA}|ds'}\log^{1/2}(1/\delta)\log^{5/2}[((n_{\hat{\mA}}+n_0)d)/\eta]}{(n_{\hat{\mA}} + n_0)\epsilon} \leq \tilde{c} r_{\textup{HLR}}(n_0, s', d, \epsilon, \delta, \eta)$, where $n_{\hat{\mA}} = \sum_{k \in \hat{\mA}}n_k$}
        \State $\hbeta \leftarrow$ Algorithm  \ref{algorithm:DPregression_high_dim_combined} on data $\{\{(\Xk{k}_i, \Yk{k}_i)\}_{i \in \floor{n_k/2}+1, \dotsc, n_k}\}_{k \in \tkset}$ with $T \asymp \log n_{\hat{\mA}}$, step size $\rho = \frac{9}{10L}(1-0.296/L^4)$, privacy parameters $(\epsilon,\delta)$, initialisation $\bbeta^0$, failure probability $\eta \in (0,1/2)$, hard-thresholding parameter $s'$, subset $\hat{\mA}$, constant $L$
        \Else
        \State $\hbeta \leftarrow$ Algorithm  \ref{algorithm:DPregression_high_dim_single_source} on target data $\{(\Xk{0}_i, \Yk{0}_i)\}_{i \in \floor{n_0/2}+1, \dotsc, n_0}$ with $T \asymp \log n_0$, step size $\rho = \frac{9}{10L}(1-0.296/L^4)$, privacy parameters $(\epsilon,\delta)$, initialisation $\bbeta^0$, failure probability $\eta \in (0,1/2)$, hard-thresholding parameter $s'$, constant $L$
		\EndIf
		\OUTPUT $\hbeta$. 
		\caption{A meta-algorithm for high dimensional linear regression with FDP guarantees} \label{algorithm:DPregression_high_dim}
	\end{algorithmic}
\end{algorithm}

We, again, consider 
applying \Cref{algorithm:DPregression_high_dim_combined} to the set 
\begin{equation}\label{algorithm:DPregression_high_dim_detection diff nk}
    \hat{\mA} := \{k \in [K]: \twonorm{\hbeta^{(k)} - \hbeta^{(0)}} \leq \tilde{c} \,r_{\textup{HLR}}(n_0, s', d, \epsilon, \delta, \eta/K)\},
\end{equation}
where $r_{\mathrm{HLR}}$ is defined in \eqref{eq:targetrate-hd}, $\tilde{c}>0$ is some constant to be chosen, and $\{\hbeta^{(k)}\}_{k \in \tkset}$ are obtained by applying \Cref{algorithm:DPregression_high_dim_single_source} onto half of the data at the $k$-th location. 

Since Algorithms~\ref{algorithm:DPregression_high_dim_single_source} and \ref{algorithm:DPregression_high_dim_combined} use fundamentally different techniques, we cannot guarantee that combining information in $\hat{\mA}$ improves on the target-only rate, as we did in the previous two problems. We therefore include an additional comparison step: if the private part of the aggregated rate is smaller than the target-only rate $r_{\textup{HLR}}(n, s', d, \epsilon, \delta, \eta)$ in \eqref{eq:targetrate-hd}, then we run Algorithm \ref{algorithm:DPregression_high_dim_combined}; otherwise, we use Algorithm \ref{algorithm:DPregression_high_dim_single_source} on the target data only. As designed, Algorithm \ref{algorithm:DPregression_high_dim} will adaptively decide whether to aggregate the source information, and achieve an estimation error rate that is the minimum between the target-only error rate and the FDP rate.

We restate the assumptions below to accommodate heterogeneous sample sizes.
\begin{asp}\label{asp: x diff nk}
    Assume $0< L^{-1} \leq \lambdamin(\bSigmak{k}) \leq \lambdamax(\bSigmak{k}) \leq L < \infty$ with some absolute constant $L \geq 1$, for all $k \in \tkset$.
\end{asp}

\begin{asp}\label{asp: beta diff nk}
(i) For all $k \in \mA$, $\onenorm{\bbetak{k} - \beta} \lesssim \sqrt{s}\twonorm{\bbetak{k} - \beta}$.
    (ii) For each $k \in \mA^c$, there exists a $\wbbetak{k} \in \mathbb{R}^d$ with $\zeronorm{\wbbetak{k}} \leq s$, $\onenorm{\bbetak{k} - \wbbetak{k}} \lesssim \sqrt{s}\twonorm{\bbetak{k} - \wbbetak{k}}$, and $\twonorm{\bbetak{k} - \wbbetak{k}} \leq c\twonorm{\bbetak{k} - \beta}$ with a small absolute constant $c > 0$ \footnote{It suffices that the constant $c$ satisfies $cC \leq 1/2$, where $C$ is the absolute constant appearing in Proposition \ref{prop: single_source}.(\rom{4}) in the appendix.}.
    (iii) It holds that $\min_{k \in \mA^c} \twonorm{\bbetak{k} - \beta} \geq Cr_{\textup{HLR}}(n_0,s, d,\epsilon,\delta, \eta)$, where $C > 0$ is a large enough absolute constant.
\end{asp}

\begin{asp}\label{asp: sample size diff nk}
	$\min_{k \in [K]}r_{\textup{HLR}}(n_k, s, d, \epsilon, \delta, \eta/K) \lesssim r_{\textup{HLR}}(n_0, s, d, \epsilon, \delta, \eta/K) \leq c$, where $c > 0$ is a sufficiently small absolute constant.
\end{asp}

Note that Assumption \ref{asp: sample size diff nk} needed for the detection step \eqref{algorithm:DPregression_high_dim_detection}, can be satisfied when $n_k = \widetilde{\Omega}(n_0)$, i.e.~the source sample sizes are larger in order than that of the target sample size. 

We restate the upper bound on the estimation error in the following theorem.

\begin{thm}\label{thm: high-dim upper bound diff nk}
Let $\{\mathbf{X}^{(k)}, \mathbf{Y}^{(k)}\}_{k \in \{0\} \cup [K]}$ be generated from \eqref{eq: high-dim reg model}.  Initialise Algorithm \ref{algorithm:DPregression_high_dim} with $\bbeta^{0} = 0$. Suppose that Assumptions \ref{asp: x diff nk}, \ref{asp: beta diff nk}, and \ref{asp: sample size diff nk} hold, $\max_{k \in \{0\} \cup \mA}\twonorm{\bbetak{k}} \leq C$ with some constant $C > 0$, $s \gtrsim s' \geq 4.18L^4s$, $\min_{k \in [K]}n_k \gtrsim n_0$, $n_0 \gtrsim \epsilon^{-1}s\log^{1/2}(1/\delta)\log^{5/2}(nd/\eta)$, and $n_{\mA}+n_0 \gtrsim \epsilon^{-1}\sqrt{|\mA|ds}\log^{1/2}(1/\delta)\log^{5/2}[((n_{\mA}+n_0)d)/\eta]$.
	We then have the following.
	(i) Algorithm \ref{algorithm:DPregression_high_dim} is $(\epsilon, \delta)$-FDP. 
		(ii) There exists a choice of $\tilde{c} > 0$ in \eqref{algorithm:DPregression_high_dim_detection diff nk} such that the output $\hbeta$ from Algorithm \ref{algorithm:DPregression_high_dim} satisfies
          \begin{equation}
              \tp(\twonorm{\hbeta - \beta} \lesssim (\textup{\Rom{1}}) \wedge (\textup{\Rom{2}})) \geq 1-\eta,
          \end{equation}
          where
		\[
			(\textup{\Rom{1}}) = \sqrt{\frac{s\log(d/\eta)\log (n_0)}{n_0}} + \frac{s \log^{1/2}(1/\delta)\log^{5/2}(n_0d/\eta)}{n_0\epsilon}
   \]
   and
   \begin{equation}\label{eq: aggregation rate high dim diff nk}
       (\textup{\Rom{2}}) = \sqrt{\frac{s \log (d/\eta)\log (n_{\mA} + n_0) }{n_{\mA} + n_0}} + h + \frac{\sqrt{|\mA|ds}\log^{1/2}(1/\delta)\log^{5/2}[\{(n_{\mA}+n_0)d\}/\eta]}{(n_{\mA} + n_0)\epsilon}. 
   \end{equation}
\end{thm}

\subsection{Proof of results in Section \ref{sec:highd-regression}}\label{subsec: proof hd appendix}

\subsubsection{Proof of Lemma \ref{thm: high-dim upper bound single-source}}\label{subsubsec: proof hd single source upper bound}
The following Proposition presents a more general upper bound of estimation error in the case of private high-dimensional linear regression on a single data set, which automatically implies \Cref{thm: high-dim upper bound single-source}. 

\begin{prop}\label{prop: high-dim upper bound single-source}
    Suppose Assumptions \ref{asp: x} and \ref{asp: beta} hold,  the parameters $s'$, $\rho$ used in Algorithm \ref{algorithm:DPregression_high_dim_single_source} satisfy 
\begin{equation}\label{eq: high-dim upper bound single-source hd appendix}
    \rho = \frac{9\xi}{10L}, \quad \frac{s'}{s} \geq \frac{100\xi}{81(1-\xi)} \quad \mbox{and} \quad \bigg(\frac{10}{9}\frac{\xi}{1-\xi}+\frac{17}{5}\bigg)\xi^2 > \frac{242}{9}L^4.
\end{equation}
For $\eta \in (0, 1)$, if
\begin{equation*}
	n \gtrsim \epsilon^{-1}Ts\log^{1/2}(1/\delta)\log(nd/\eta)\log^{1/2}(Td/\eta),
\end{equation*}
then the output $\bbeta^{T}$ from Algorithm \ref{algorithm:DPregression_high_dim_single_source} is $(\epsilon, \delta)$-central DP and satisfies
\begin{align*}
	\twonorm{\bbeta^{T}-\bbeta^*} &\leq \sqrt{\frac{11}{9}}L\left(1 -\frac{2s}{s'}\xi -\frac{9s'-10s}{9s'} \frac{9\xi^2}{22L^2} + \frac{10s}{9s'}\cdot \frac{11}{9}L^2\right)^{T/2}\twonorm{\bbeta^{0}-\bbeta^*} \\
	&\quad +C\frac{s\sqrt{\log(1/\delta)\log^2(nd/\eta)\log(Td/\eta)}}{(n/T)\epsilon} (1\vee \twonorm{\bbeta^0-\bbeta^*}) \\
	&\quad + C\sqrt{\frac{s\log(d/\eta)}{n/T}},
\end{align*}
with probability at least $1-\eta$.
\end{prop}

\begin{remark}
Note that
\begin{align*}
	1 -\frac{2s}{s'}\xi-\frac{9s'-10s}{9s'} \cdot \frac{9\xi^2}{22L^2} + \frac{10s}{9s'} \frac{11}{9}L^2 &= 1-\frac{10s}{9s'}\frac{9}{22L^2}\left[\left(\frac{9s'}{10s}-1\right)\xi^2 + \frac{22}{5}L^2\xi - 2\left(\frac{11}{9}\right)^2L^4\right] \\
	&\leq 1-\frac{s}{s'}\frac{5}{11L^2}\left[\bigg(\frac{9s'}{10s} + \frac{17}{5}\bigg)\xi^2 - 2\left(\frac{11}{9}\right)^2L^4\right] \\
    &\leq 1-\frac{s}{s'}\frac{5}{11L^2}\left[\bigg(\frac{10}{9}\frac{\xi}{1-\xi} + \frac{17}{5}\bigg)\xi^2 - 2\left(\frac{11}{9}\right)^2L^4\right]\\
	&< 1,
\end{align*}
by \eqref{eq: high-dim upper bound single-source hd appendix}.
\end{remark}

\begin{remark}
     \eqref{eq: high-dim upper bound single-source hd appendix} holds when $s \gtrsim s' \geq 4.18L^4s$ and $\xi = 1-\frac{0.296}{L^4}$. When $T \asymp \log(n)$, $s \gtrsim s' \geq 4.18L^4s$, $\xi = 1-\frac{0.296}{L^4}$, $n \gtrsim \epsilon^{-1}s\log^{1/2}(1/\delta)\log^{5/2}(nd/\eta)$, and $\twonorm{\bbeta^0-\bbeta} \leq C$, we have
    \begin{align*}
        \twonorm{\bbeta^{T}-\bbeta^*} &\leq C\sqrt{\frac{s\log(d/\eta)}{n/T}} + C\frac{s\log^{1/2}(1/\delta)\log^{5/2}(nd/\eta)}{n\epsilon},
    \end{align*}
    with probability at least $1-\eta$. This proves \Cref{thm: high-dim upper bound single-source}.
\end{remark}

The proof of Proposition \ref{prop: high-dim upper bound single-source} can be found in Section \ref{subsec: proofs high-dim appendix}.

\subsubsection{Proof of Theorem \ref{thm: high-dim upper bound diff nk}}\label{subsubsec: proof hd upper bound}
First, we summarise a few important intermediate results for Theorem \ref{thm: high-dim upper bound diff nk} as propositions. The next proposition concerns the privacy of the intermediate estimators in \eqref{algorithm:DPregression_high_dim_detection diff nk} and their estimation error rates, which are useful to prove Proposition \ref{prop: high dim A}.

\begin{prop}\label{prop: single_source}
	Suppose Assumption \ref{asp: x} holds, $\min_{k \in [K]}n_k \gtrsim n_0$, $\max_{k \in [K]}\twonorm{\bbeta^{(k)0}- \bbetak{k}} \leq C$ with some constant $C > 0$, and the parameter $s$ used in Algorithm \ref{algorithm:DPregression_high_dim_single_source} satisfies $s' \geq 4.18L^4s$.
    Then \eqref{algorithm:DPregression_high_dim_detection} satisfies that: 
    \begin{enumerate}[label=(\roman*), leftmargin=*]
    	\item $\hbeta^{(k)}$ is $(\epsilon, \delta)$-central DP, for $k \in \tkset$; 
    	\item With probability at least $1-\eta$, $\twonorm{\hbeta^{(0)}-\bbetak{0}} \leq Cr_{\textup{HLR}}(n_0, s, d, \epsilon, \delta, \eta)$ with some absolute constant $C > 0$;
    	\item With probability at least $1-\eta$,  for all $k \in \mA$, $\twonorm{\hbeta^{(k)}-\bbetak{k}} \leq Cr_{\textup{HLR}}(n_k, s, d, \epsilon, \delta, \eta/K)+Ch$ with some absolute constant $C > 0$;  
    	\item With probability at least $1-\eta$, for all $k \in \mA^c$, $\twonorm{\hbeta^{(k)}-\bbetak{k}} \leq Cr_{\textup{HLR}}(n_k, s, d, \epsilon, \delta, \eta/K)+ C\twonorm{\bbetak{k}-\bbetak{0}}$ with some absolute constant $C > 0$.
    \end{enumerate}
\end{prop}

\begin{remark}
    Part (\rom{1}) of Proposition \ref{prop: single_source} directly follows from  \Cref{thm: high-dim upper bound single-source}. Parts (\rom{2})-(\rom{4}) present the estimation error rates of $\hbeta^{(k)}$'s towards their population truth $\bbetak{k}$'s, which will be very useful in the next result.
\end{remark}

\begin{prop}\label{prop: high dim A}  
	Suppose Assumption \ref{asp: beta} holds. Denote the events
 \begin{align*}
     \mathcal{E}_1 &= \Big\{\twonorm{\hbeta^{(0)}-\bbetak{0}} \leq C_1r_{\textup{HLR}}(n_0, s, d, \epsilon, \delta, \eta/K)\Big\}, \\
     \mathcal{E}_2 &= \bigcap_{k \in \mA}\Big\{\twonorm{\hbeta^{(k)}-\bbetak{k}} \leq C_2r_{\textup{HLR}}(n_k, s, d, \epsilon, \delta, \eta/K)\Big\}, \\
     \mathcal{E}_3 &= \bigcap_{k \in \mA^c}\bigg\{\twonorm{\hbeta^{(k)}-\bbetak{k}} \leq C_3r_{\textup{HLR}}(n_k, s, d, \epsilon, \delta, \eta/K)+ \frac{1}{2}\twonorm{\bbetak{k}-\bbetak{0}}\bigg\},
 \end{align*}
with constants $C_1, C_2$ and $C_3$ corresponds to the constant $C$ in parts (\rom{2})-(\rom{4}) of Proposition \ref{prop: single_source}, respectively, and $\mathcal{E} = \mathcal{E}_1 \cap \mathcal{E}_2 \cap \mathcal{E}_3$. For the output set $\hat{\mA}$ from \eqref{algorithm:DPregression_high_dim_detection}, we have the following results:
	\begin{enumerate}[label=(\roman*), leftmargin=*]
		\item If part (\rom{1}) of Proposition \ref{prop: single_source} holds, then $\hat{\mA}$ is $(\epsilon, \delta)$-central DP;
		\item In the event $\mathcal{E}$, we have $\hat{\mA} \subseteq \mA$ and $\twonorm{\bbetak{k} - \bbetak{0}} \lesssim r_{\textup{HLR}}(n_0, s, d, \epsilon, \delta, \eta)$ for all $k \in \hat{\mA}$;
		\item In the event $\mathcal{E}$, when $h \leq cr_{\textup{HLR}}(n_0, s, d, \epsilon, \delta, \eta)$ with a small constant $c$, we have $\hat{\mA} = \mA$.
	\end{enumerate}	
\end{prop}

\begin{remark}
    Part (\rom{1}) of Proposition \ref{prop: high dim A} provides the privacy guarantee for $\hat{\mA}$, which is necessary for the privacy of the Algorithm \ref{algorithm:DPregression_high_dim} as in the next proposition. Part (\rom{2}) guarantees that with accurate estimation of regression parameters (which is achieved with high probability as in Proposition \ref{prop: high dim A}), combining sources in $\hat{\mA}$ will not lead to a worse performance than the target-only estimator, which prevents us from negative transfer. Part (\rom{3}) guarantees a correct characterisation of the informative source index set $\mA$ when sources are sufficiently similar (i.e., $h$ is small enough), which ensures an improvement of performance for the aggregated estimator with $\hat{\mA}$ compared to the target-only estimator.
\end{remark}

Next, we state the upper bound of estimation error for $\bbeta^T$ from Algorithm \ref{algorithm:DPregression_high_dim_combined}. Then it is straightforward to obtain Theorem \ref{thm: high-dim upper bound} by combining this result with Proposition \ref{prop: high dim A}.

\begin{prop}\label{prop: high-dim upper bound}
	Suppose Assumptions \ref{asp: x} and \ref{asp: beta} hold,  the parameters $s'$, $\rho$ used in Algorithm \ref{algorithm:DPregression_high_dim_combined} satisfy 
	\begin{equation}\label{eq: high-dim upper bound}
        \rho = \frac{9\xi}{10L}, \quad \frac{s'}{s} \geq \frac{100\xi}{81(1-\xi)} \quad \mbox{and} \quad \bigg(\frac{10}{9}\frac{\xi}{1-\xi}+\frac{17}{5}\bigg)\xi^2 > \frac{242}{9}L^4.
    \end{equation}
	For any $\mA' \subseteq \mA$ in Algorithm \ref{algorithm:DPregression_high_dim_combined}, if
	\begin{equation*}
		n_{\mA'}+n_0 \gtrsim \epsilon^{-1}T\sqrt{Kds}\log^{1/2}(1/\delta)\log(N/\eta)\log^{1/2}(dT/\eta),
	\end{equation*}
	then
	\begin{enumerate}[label=(\roman*), leftmargin=*]
		\item given $\mA'$ as the subset, Algorithm \ref{algorithm:DPregression_high_dim_combined} is $(\epsilon, \delta)$-FDP; 
		\item with probability at least $1-\eta$, the output $\bbeta^{T}$ from Algorithm \ref{algorithm:DPregression_high_dim} satisfies
	\begin{align*}
		\twonorm{\bbeta^{T}-\bbetak{0}} &\leq \sqrt{\frac{11}{9}}L\left(1 -\frac{9s'-10s}{9s'} \frac{9\xi^2}{22L^2} + \frac{10s}{9s'}\frac{11}{9}L^2\right)^{T/2}\twonorm{\bbeta^{0}-\bbetak{0}} \\
		&\quad +C\frac{\sqrt{Kds}T\log^{1/2}(1/\delta)\log((n_{\mA'}+n_0)/\eta)\log^{1/2}(dT/\eta)}{(n_{\mA'} + n_0)\epsilon} (1\vee \twonorm{\bbeta^0-\bbetak{0}}) \\
		&\quad + C\sqrt{\frac{sT\log(d/\eta)}{n_{\mA'} + n_0}} + Ch,
	\end{align*}
 where $C > 0$ is some absolute constant.
	\end{enumerate}
\end{prop}

\begin{remark}
Note that
\begin{align*}
	1 -\frac{9s'-10s}{9s'} \frac{9\xi^2}{22L^2} + \frac{10s}{9s'} \frac{11}{9}L^2 &= 1-\frac{10s}{9s'}\frac{9}{22L^2}\left[\left(\frac{9s'}{10s}-1\right)\xi^2 + \frac{22}{5}L^2\xi - 2\left(\frac{11}{9}\right)^2L^4\right] \\
	&\leq 1-\frac{s}{s'}\frac{5}{11L^2}\left[\bigg(\frac{9s'}{10s} + \frac{17}{5}\bigg)\xi^2 - 2\left(\frac{11}{9}\right)^2L^4\right] \\
    &\leq 1-\frac{s}{s'}\frac{5}{11L^2}\left[\bigg(\frac{10}{9}\frac{\xi}{1-\xi} + \frac{17}{5}\bigg)\xi^2 - 2\left(\frac{11}{9}\right)^2L^4\right]\\
	&< 1,
\end{align*}
by \eqref{eq: high-dim upper bound}.
\end{remark}

\begin{remark}
If $T \asymp \log (n_{\mA'} + n_0)$, $s'/s \asymp 1$, and $\twonorm{\bbeta^{0}-\bbetak{0}} \lesssim 1$, then we have
\begin{align*}
    \twonorm{\bbeta^{T} - \bbetak{0}} \lesssim \sqrt{\frac{s \log (n_{\mA'} + n_0) \log d}{n_{\mA'} + n_0}} + h + \frac{\sqrt{Kds}\log (n_{\mA'} + n_0)\log^{\frac{1}{2}}\big(\frac{1}{\delta}\big)\log^{\frac{3}{2}}\big(\frac{n_{\mA'} + n_0}{\eta}\big)\log^{\frac{1}{2}}\big(\frac{d}{\eta}\big)}{(n_{\mA'} + n_0)\epsilon},
\end{align*}
with probability at least $1-\eta$.
\end{remark}

Now, we can return to the main proof of Theorem \ref{thm: high-dim upper bound}.

\noindent (\rom{1}) It follows from Proposition \ref{prop: high dim A}.(\rom{1}) that $\hat{\mA}$ is $(\epsilon, \delta)$-central DP which satisfies \eqref{eq:composition-eachstep}. Proposition \ref{prop: high-dim upper bound}.(\rom{1}) implies that every communication step in Algorithm \ref{algorithm:DPregression_high_dim_combined} between sources satisfies \eqref{eq:composition-eachstep}. Similarly, Proposition \ref{prop: high dim A}.(\rom{1}) guarantees $\hbeta$ in Step 6 is $(\epsilon, \delta)$-central DP which satisfies \eqref{eq:composition-eachstep}. Putting all the pieces together,  Algorithm \ref{algorithm:DPregression_high_dim} is $(\epsilon, \delta)$-FDP by \Cref{def:interactive-FDP}.
	
\noindent (\rom{2}) Define
\begin{align*}
    [1] &= \sqrt{\frac{s\log(d/\eta)\log (n_0)}{n_0}} + \frac{s \log^{1/2}(1/\delta)\log^{5/2}(n_0d/\eta)}{n_0\epsilon}, \\
    [2] &= \sqrt{\frac{s \log (d/\eta)\log (n_{\mA} + n_0) }{n_{\mA} + n_0}} + h + \frac{\sqrt{|\mA|ds}\log^{1/2}(1/\delta)\log^{5/2}[((n_{\mA}+n_0)d)/\eta]}{(n_{\mA} + n_0)\epsilon}.
\end{align*}
	
\noindent \textbf{Case 1:} When $\frac{\sqrt{|\hat{\mA}|ds'}\log^{1/2}(1/\delta)\log^{5/2}[((n_{\hat{\mA}}+n_0)d)/\eta]}{(n_{\hat{\mA}} + n_0)\epsilon} \leq C_0r_{\textup{HLR}}(n_0, s', d, \epsilon, \delta, \eta) \lesssim r_{\textup{HLR}}(n_0, s, d, \epsilon, \delta, \eta)$ and $h \leq cr_{\textup{HLR}}(n_0, s, d, \epsilon, \delta, \eta)$, where $c$ is the constant in Proposition \ref{prop: high dim A}.(\rom{3}): We have $[2] \lesssim [1]$, $\hat{\mA} = \mA$ with probability at least $1-\eta$ by Proposition \ref{prop: high dim A}.(\rom{3}), and the bound $[2]$ follows from Proposition \ref{prop: high-dim upper bound}.(\rom{2}).
	
\noindent \textbf{Case 2:} When 
\[
    \frac{\sqrt{|\hat{\mA}|ds'}\log^{1/2}(1/\delta)\log^{5/2}[((n_{\hat{\mA}}+n_0)d)/\eta]}{(n_{\hat{\mA}} + n_0)\epsilon} \leq C_0r_{\textup{HLR}}(n_0, s', d, \epsilon, \delta, \eta)
\]
and $h > cr_{\textup{HLR}}(n_0, s, d, \epsilon, \delta, \eta)$, where $c$ is the constant in Proposition \ref{prop: high dim A}.(\rom{3}): $[2] \gtrsim [1]$. By Proposition \ref{prop: high dim A}.(\rom{2}), we know $\twonorm{\bbetak{k} - \bbetak{0}} \lesssim r_{\textup{HLR}}(n_0, s, d, \epsilon, \delta, \eta)$, for all $k \in \hat{\mA}$, with probability at least $1-\eta$. Then the bound $[1]$ follows from Proposition \ref{prop: high-dim upper bound}.(\rom{2}), by taking $\mA = \hat{\mA}$ and noticing that $\max_{k \in \mA}\twonorm{\bbetak{k} - \bbetak{0}} \lesssim r_{\textup{HLR}}(n_0, s, d, \epsilon, \delta, \eta)$ with probability at least $1-\eta$.
	
\noindent \textbf{Case 3:} When 
\[
    \frac{\sqrt{|\hat{\mA}|ds'}\log^{1/2}(1/\delta)\log^{5/2}[((n_{\hat{\mA}}+n_0)d)/\eta]}{(n_{\hat{\mA}} + n_0)\epsilon} > C_0r_{\textup{HLR}}(n_0, s', d, \epsilon, \delta, \eta) \gtrsim r_{\textup{HLR}}(n_0, s, d, \epsilon, \delta, \eta):
\]
we have that $[2] \gtrsim [1]$. The bound $[1]$ follows from Proposition \ref{prop: single_source}(\rom{2}).

\subsection{The peeling algorithm}\label{subsec: alg and theory high-dim appendix}
The peeling algorithm selects a few coordinates of the coefficient estimate with the largest absolute values, adds noise to them, and truncates the remaining coordinates to zero, which can be viewed as a private hard thresholding algorithm and has been used in \cite{cai2019cost} and \cite{dwork2021differentially}. We used the peeling algorithm in the single-source algorithm for high-dimensional regression (Algorithm \ref{algorithm:DPregression_high_dim_single_source}) and summarised the peeling algorithm as follows in Algorithm \ref{algorithm:peeling}. 

\begin{algorithm}[!ht]
	\begin{algorithmic}[1]
		\INPUT{A vector $v  \in \mathbb{R}^d$, sparsity parameter $s$, privacy parameters $(\epsilon, \delta)$, noise level $\lambda$.}
        \State Initialise $S = \emptyset$.
		\For{$j = 1, \ldots, s$} 
			     \State Generate $\bw \in \mathbb{R}^d$ with $w_j \overset{\text{i.i.d.}}{\sim} \textup{Laplace}(2\lambda \sqrt{3s\log(1/\delta)}/\epsilon)$.
			     \State Append $j^* = \argmax_{j \in [d] \backslash S}(|v_j| + w_j$) to the set $S$.
		\EndFor
		\State Generate $\bw \in \mathbb{R}^d$ with $w_j \overset{\text{i.i.d.}}{\sim} \textup{Laplace}(2\lambda\sqrt{3s\log(1/\delta)}/\epsilon)$.
		\OUTPUT $\tilde{v}$ with $\tilde{v}_S = v_S + \bw_S$ and $\tilde{v}_{S^c} = 0$. 
		\caption{Peeling \citep{cai2019cost}} \label{algorithm:peeling}
	\end{algorithmic}
\end{algorithm}

\subsection{Auxiliary results}\label{subsec: proofs high-dim appendix}
 Throughout the proofs in the subsection, we ignore the effect of the $\lfloor \cdot \rfloor$ operation, i.e.\ we treat $\lfloor n_k/T \rfloor = n_k/T$.

\subsubsection{Proof of Proposition \ref{prop: high-dim upper bound single-source}}
In this subsection, we provide the proof of Proposition \ref{prop: high-dim upper bound single-source} which is used in Section \ref{subsubsec: proof hd single source upper bound} as a generalised version of \Cref{thm: high-dim upper bound single-source}.  We first present the necessary additional definitions and notations.

For convenience, we denote $\alpha = \frac{10}{11}L^{-1}$, $\gamma = \frac{10}{9}L$, and $S = \textup{supp}(\beta^*)$. For any $t \in [T]$, and $\bbeta \in \mathbb{R}^d$, define the empirical risk function at iteration $t$ as
\begin{align*}
    \mL_{n}^{t}(\bbeta) &= \frac{1}{2n}\sum_{i=1+(t-1)(n/T)}^{t(n/T)}(Y_i-X_i^\top\bbeta)^2,
\end{align*}
Define $\bm{X}^{t} \in \mathbb{R}^{(n/T) \times d}$ as the predictor data matrix in iteration $t$, where each row is an observation in batch $t$. $Y^{t}\in \mathbb{R}^{n/T}$ is the response vector in iteration $t$.

Recall that the step length of gradient descent equals $\rho = \frac{9\xi}{10L} = \xi/\gamma$. Define the gradient of $\mL_n^t$ at $\bbeta^t$ as
\begin{align*}
    \bmg^t &= \nabla \mL_n^t(\bbeta^t) \\
    &= \frac{1}{n}(\bm{X}^{t})^\top(\bm{X}^{t}\bbeta^t - Y^{t}) \\
    &=\frac{1}{n/T}\sum_{i=1+(t-1)(n/T)}^{t(n/T)}\big(X_i^\top\bbeta^t - Y_i\big)X_i.
\end{align*}
and the sets
\begin{equation}\label{eq: def of It single-source hd}
    I^t = S^{t+1}\cup S^t \cup S, \quad \mbox{where } S^t = \textup{supp}(\bbeta^t).
\end{equation}
Recall that in Algorithm \ref{algorithm:DPregression_high_dim_single_source}, we define $R_t= 2\sqrt{\log(n/\eta)}\text{PrivateVariance}\Big(\big\{X_{i}^\top \bbeta^t - Y_i\big\}_{i=1+(t-1)(n/T)}^{t(n/T)},\allowbreak \epsilon/2,\delta/2\Big)$. 
For simplicity, for any vector $v \in \mathbb{R}^d$ and $s \in \mathbb{N}_+$, abusing the notation a bit, we write peeling operator $\textup{Peeling}(v, s, \epsilon/2, \delta/2, 2\rho R_tR/n)$ used in Algorithm \ref{algorithm:DPregression_high_dim_single_source} as $\widetilde{P}_{s'}(v)$. Also define the Laplace noise added in Step 5 and iteration $t$ of Algorithm \ref{algorithm:DPregression_high_dim_single_source} by calling Algorithm \ref{algorithm:peeling} as $\bm{w}^t$ and $\widetilde{\bm{w}}^t$ where each entry follows $\textup{Laplace}(2\lambda\sqrt{3s\log(1/\delta)/\epsilon})$ independently with $\lambda = 2\rho R_tR/n$. Define the sample covariance matrix at iteration $t$ as
$\widehat{\bSigma}^{t} = n^{-1}(\bm{X}^{t})^\top \bm{X}^{t}$.

Next, we divide the formal proof of Proposition \ref{prop: high-dim upper bound single-source} into a few parts. In part (\Rom{1}), we define some events and show that their intersections hold with high probability. In part (\Rom{2}), we make additional notes for the truncations in Algorithm \ref{algorithm:DPregression_high_dim_single_source} and argue that they are not effective in the high-probability event defined in part (\Rom{1}). In part (\Rom{3}), we demonstrate that Algorithm \ref{algorithm:DPregression_high_dim_single_source} satisfies $(\epsilon, \delta)$-central DP. In part (\Rom{4}), we provide a detailed proof of the estimation error upper bound in Proposition \ref{prop: high-dim upper bound single-source}. In the last part (\Rom{5}), we collect the useful lemmas and their proofs.

\noindent \textbf{(\Rom{1}) Conditioning on some events:}
Define events $\mathcal{E}_{1}$, $\mathcal{E}_{2}$, and $\mathcal{E}_{3}$ as follows: 
\begin{align*}
    \mathcal{E}_{1} &=  \Big\{\alpha \leq \lambdamin\big(\widehat{\bSigma}^{[t]}_{S', S'}\big) \leq \lambdamax\big(\widehat{\bSigma}^{[t]}_{S', S'}\big) \leq \gamma, \forall S' \subseteq [d] \text{ with } |S'| \leq s', \forall t \in [T]\Big\} \\
    &\quad \bigcap \Bigg\{\twonorm{\hSigma^{[t]}_{S',S'} - \bSigma_{S',S'}} \leq C\sqrt{\frac{s'\log (d/\eta)}{n/T}}, \forall S' \subseteq [d] \text{ with } |S'| \leq s', \forall t \in [T]\Bigg\}, \\
    \mathcal{E}_{2} &= \bigg\{\infnorm{X_i} \leq C\sqrt{\log(nd/\eta)}, \forall i = [1+(t-1)(n/T)]:[t(n/T)], \forall t \in [T]\bigg\}, \\
    \mathcal{E}_{3} &= \bigg\{\infnorm{\bw^t}^2 \leq C\frac{s\log(1/\delta)\log(nd/\eta)\log(Td/\eta)}{(n/T)^2\epsilon^2}\cdot (R_t)^2, \forall t \in [T]\bigg\},
\end{align*}
where $C > 0$ is a sufficiently large number such that 
\begin{equation}\label{eq: event 123 single-source hd}
    \tp(\mathcal{E}_1) \geq 1-\frac{\eta}{12}, \quad \tp(\mathcal{E}_2) \geq 1-\frac{\eta}{12},  \quad \tp(\mathcal{E}_3) \geq 1- \frac{\eta}{12}.
\end{equation}
Define event $\mathcal{E}_4$ as
\begin{equation*}
    \mathcal{E}_4 = \Big\{R^{t} \geq |X_i^\top\bbeta^t - Y_i|, \forall i = [1+(t-1)(n/T)]:[t(n/T)], \forall t \in [T]\Big\},
\end{equation*}
where $R^{t} = 2\sqrt{\log(4n/\eta)}\textup{PrivateVariance}\big(\big\{X_i^\top\bbeta^t - Y_i\big\}_{i=1+(t-1)(n/T)}^{t(n/T)}, \epsilon/2, \delta/2\big)$. By \Cref{lemma:privatevarianceSG}, the union bound, we have 
\begin{align}
    \sqrt{\frac{3}{4}}({\sigma + \|\bbeta^t-\bbeta^*\|_{\bSigma}}) &\leq \textup{PrivateVariance}\big(\big\{X_i^\top\bbeta^t - Y_i\big\}_{i=1+(t-1)(n/T)}^{t(n/T)}, \epsilon/2, \delta/2\big) \nonumber\\
    &\leq \sqrt\frac{5}{2}({\sigma + \|\bbeta^t-\bbeta^*\|_{\bSigma}}), \label{eq: pv single-source hd}
\end{align}
with probability at least $1 - \eta/(4T)$ for all $t \in [T]$, if 
\[
  n \gtrsim \frac{T\log(T/(\delta\eta))\log(T\log[T/(\eta\delta)](\eta\epsilon)^{-1})}{\epsilon}.
 \] 
  Note that since $\bbeta^t$ is independent of $\{X_i\}_{i=1+(t-1)(n/T)}^{t(n/T)}$, by conditioning on $\bbeta^t$, we have
\begin{equation*}
    X_i^\top\bbeta^t - Y_i = X_i^\top(\bbeta^t-\bbeta^*) + \xi_i
\end{equation*}
and it has zero-mean, variance $\|\beta^{t} - \beta^*\|^2_{\Sigma}+\sigma^2$, and $\|r_{i,t}\|_{\psi_2} \leq C\sqrt{\sigma^2+\|\beta^{t} - \beta^*\|^2_{\Sigma}}$ for some absolute constant $C$. Therefore the same arguments used in the proof of \Cref{lemma:highprobabilityevent} leads to $\tp(\mathcal{E}_4) \geq 1-\frac{\eta}{2}.$

Finally, we define the event
\begin{align*}
    \mathcal{E}_5 = \bigg\{\sqrt{\frac{3}{4}}({\sigma + \|\bbeta^t-\bbeta^*\|_{\bSigma}}) &\leq \textup{PrivateVariance}\big(\big\{X_i^\top\bbeta^t - Y_i\big\}_{i=1+(t-1)(n/T)}^{t(n/T)}, \epsilon/2, \delta/2\big)\\
    &\leq \sqrt\frac{5}{2}({\sigma + \|\bbeta^t-\bbeta^*\|_{\bSigma}}), \forall t \in [T]\bigg\}.
\end{align*}
By \eqref{eq: pv single-source hd} and union bound over all $t \in [T]$, we have
\begin{equation}\label{eq: event 5 single-source hd}
    \tp(\mathcal{E}_5) \geq 1-\frac{\eta}{4}.
\end{equation}
Combining the above, we obtain
\begin{equation*}
    \tp(\cap_{i=1}^5 \mathcal{E}_i) \geq 1- \frac{\eta}{12} - \frac{\eta}{12} - \frac{\eta}{12}-\frac{\eta}{2} -\frac{\eta}{4} \geq 1-\eta.
\end{equation*}

\medskip
\noindent\textbf{(\Rom{2}) Truncation in Algorithm \ref{algorithm:DPregression_high_dim}:}
In events $\mathcal{E}_2\cap \mathcal{E}_4$, the truncations $\prod_R^{\infty}$ and $\prod_{R_t}$ in Algorithm \ref{algorithm:DPregression_high_dim_single_source} are `not effective', by which we mean $\prod_R^{\infty}(X_i) = X_i$, for all $i \in [n]$, and $\prod_{R^{t}}\big(X_i^\top\bbeta^t - Y_i\big) = X_i^\top\bbeta^t - Y_i$, for all $i = [1+(t-1)(n/T)]:[t(n/T)]$ and $t \in [T]$. 

In the following analysis, we condition on the event $\cap_{i=1}^5\mathcal{E}_i$ and show the upper bound holds given that $\cap_{i=1}^5\mathcal{E}_i$ holds.

\medskip

\noindent\textbf{(\Rom{3}) Privacy:} 
First, $R^{t}$ is $(\epsilon/2, \delta/2)$-central DP by \Cref{lemma:privatevarianceSG}. By the fact that the peeling algorithm in Step 5 is $(\epsilon/2, \delta/2)$-central DP \citep{cai2019cost}, together with the composition theorem (Theorem 3.16 in \cite{dwork2014algorithmic}), the estimate $\bbeta^t$ in each iteration is $(\epsilon, \delta)$-central DP. Therefore by the parallel composition theorem (Theorem 2 in \cite{smith2021making}), $\bbeta^T$ output by Algorithm \ref{algorithm:DPregression_high_dim_single_source} is $(\epsilon, \delta)$-central DP.

\medskip   
\noindent\textbf{(\Rom{4}) Derivation of the estimation error bound:}  

We first summarise the key idea of the proof. First, we upper  bound $\mL_n^t(\bbeta^{t+1}) - \mL_n^t(\bbeta^*)$ by $\mL_n^t(\bbeta^{t}) - \mL_n^t(\bbeta^*)$, $\bbeta^t - \bbeta^*$, and $\infnorm{\bw^t}$. Second, we replace $\mL_n^t(\bbeta^{t+1}) - \mL_n^t(\bbeta^*)$ with a lower bound involving $\twonorm{\bbeta^{t+1} - \bbeta^*}$ and replace $\mL_n^t(\bbeta^{t}) - \mL_n^t(\bbeta^*)$ with an upper bound involving $\twonorm{\bbeta^{t} - \bbeta^*}$. Finally, we simplify the result to obtain an induction relationship between $\|\bbeta^{t+1} - \bbeta^*\|_{\bSigma}$ and $\|\bbeta^{t} - \bbeta^*\|_{\bSigma}$, translate the $\bSigma$-norm to $\ell_2$-norm, then complete the proof.

\noindent {\textbf{Step 1:}} Upper bound $\mL_n^t(\bbeta^{t+1}) - \mL_n^t(\bbeta^*)$ by $\mL_n^t(\bbeta^{t}) - \mL_n^t(\bbeta^*)$, $\bbeta^t - \bbeta^*$, and  $\infnorm{\bw^t}$.
We start by applying Taylor's expansion, and it holds that	
	\begin{align}
		\mL_n^t(\bbeta^{t+1}) - \mL_n^t(\bbeta^{t}) &\leq \<\bbeta^{t+1}-\bbeta^t, \bmg^t\> + \frac{1}{2}\gamma\twonorm{\bbeta^{t+1}-\bbeta^t}^2 \nonumber \\
		&= \frac{1}{2}\gamma\twonorm{\bbeta^{t+1}_{I^t} - \bbeta^t_{I^t} + \xi/\gamma \cdot \bmg_{I^t}^t}^2 - \frac{\xi^2}{2\gamma}\twonorm{\bmg_{I^t}^t}^2 + (1-\xi)\<\bbeta^{t+1}-\bbeta^t,\bmg^t\>, \label{eq: proof thm high-dim eq 1 single-source hd}
	\end{align}
    where the inequality is due to Lemma \ref{lem: fact 8.1}.(\rom{2}), the identity holds following the definition of $I_t$ in \eqref{eq: def of It single-source hd} and $\xi \in \mathbb{R}$.  
    
    We can further upper bound the last term in \eqref{eq: proof thm high-dim eq 1 single-source hd}, in which the last term can be bounded as 
	\begin{align}
		\<\bbeta^{t+1}-\bbeta^t,\bmg^t\>  &=  \<\bbeta^{t+1}_{S^{t+1}} - \bbeta^{t}_{S^{t+1}}, \bmg^t_{S^{t+1}}\> - \<\bbeta^{t}_{S^t\backslash S^{t+1}}, \bmg^{t}_{S^t\backslash S^{t+1}}\> \nonumber \\
		&= \<-\rho\bmg^{t}_{S^{t+1}} + \bw^t_{S^{t+1}}, \bmg^t_{S^{t+1}}\> - \<\bbeta^{t}_{S^t\backslash S^{t+1}}, \bmg^t_{S^t\backslash S^{t+1}}\> \label{eq: hd upper 1 single-source hd} \\
		&\leq -\frac{\xi}{\gamma}\twonorm{\bmg^{t}_{S^{t+1}}}^2 + 10\rho^{-1}\twonorm{\bw^t_{S^{t+1}}}^2  + \frac{\rho}{40}\twonorm{\bmg^{t}_{S^{t+1}}}^2 - \<\bbeta^t_{S^t\backslash S^{t+1}}, \bmg^t_{S^t\backslash S^{t+1}}\>\label{eq: hd upper 2 single-source hd}\\
		&\leq -\frac{\xi}{\gamma}\twonorm{\bmg^{t}_{S^{t+1}}}^2 + 10s\rho^{-1}\infnorm{\bw^t}^2  + \frac{\rho}{40}\twonorm{\bmg^{t}_{S^{t+1}}}^2 - \<\bbeta^t_{S^t\backslash S^{t+1}}, \bmg^t_{S^t\backslash S^{t+1}}\>  \label{eq: hd upper 3 single-source hd new}.
	\end{align}
	Note that \eqref{eq: hd upper 1 single-source hd} holds because $\bbeta^{t+1} = \widetilde{P}_{s'}(\bbeta^{t+0.5})$, $\bbeta^{t+1}_{S^{t+1}} = \bbeta^{t+0.5}_{S^{t+1}} + \widetilde{\bw}^t_{S^{t+1}}$, and $\bbeta^{t+0.5}_{S^{t+1}} = \bbeta^t_{S^{t+1}} - \rho\bmg^t_{S^{t+1}}$, and \eqref{eq: hd upper 2 single-source hd} holds since $\<-\rho\bmg^{t}_{S^{t+1}} + \bw^t_{S^{t+1}}, \bmg^t_{S^{t+1}}\> = -\rho\twonorm{\bmg^t_{S^{t+1}}}^2 +\<\bw^t_{S^{t+1}}, \bmg^t_{S^{t+1}}\> \leq -\rho\twonorm{\bmg^t_{S^{t+1}}}^2 + 10\rho^{-1}\twonorm{\bw^t_{S^{t+1}}}^2  + \frac{\rho}{40}\twonorm{\bmg^{t}_{S^{t+1}}}^2$.
 
 We then upper bound the last term in \eqref{eq: hd upper 3 single-source hd new}:
	\begin{align}
		&- \<\bbeta^t_{S^t\backslash S^{t+1}}, \bmg^t_{S^t\backslash S^{t+1}}\> \nonumber \\
		&\leq \frac{\gamma}{2\xi}\twonorma{\bbeta^t_{S^t\backslash S^{t+1}} - \frac{\xi}{\gamma}\bmg^t_{S^t\backslash S^{t+1}}}^2 - \frac{\xi}{2\gamma}\twonorm{\bmg^t_{S^t\backslash S^{t+1}}}^2 \nonumber \\
		&\leq \frac{\gamma}{2\xi}\left[\frac{1+c}{1-c}\twonorma{\frac{\xi}{\gamma}\bmg_{S^{t+1}\backslash S^t}^t}^2 + \frac{1}{1-c}\left(s'/c + (1+1/c)s'\right)\cdot \infnorm{\bw^t}^2\right] - \frac{\xi}{2\gamma}\twonorm{\bmg^t_{S^t\backslash S^{t+1}}}^2 \label{eq: hd upper 3 single-source hd}\\
		&= \frac{\xi}{2\gamma}\cdot \frac{1+c}{1-c}\twonorm{\bmg^t_{S^{t+1}\backslash S^t}}^2 + \frac{\gamma}{2\xi}\cdot \frac{1}{1-c}\cdot (1+2/c)s'\cdot \infnorm{\bw^t}^2 - \frac{\xi}{2\gamma}\twonorm{\bmg^t_{S^{t}\backslash S^{t+1}}}^2, \nonumber
	\end{align}
	where $c$ is a very small constant. 
 
 Equation \eqref{eq: hd upper 3 single-source hd} is an application of Lemma \ref{lem: modified 3.4} by letting $v = \bbeta^t - \xi/\gamma \cdot \bmg^t$, $u = \bbeta^{t+0.5} + \bw^t = \bbeta^t - \xi/\gamma \cdot \bmg^t +\bw^t$, $S^{t+1} = \text{supp}(\bbeta^{t+1}) =$ the set of top $s$ entries of $|u|$ because $\bbeta^{t+1} = \widetilde{P}_{s'}(\bbeta^{t+0.5}) = H_{s'}(\bbeta^{t+0.5} + \bw^t)-\bw^t_{S^{t+1}}+\widetilde{\bw}_{S^{t+1}}^t$, $S_1 = S^t \backslash S^{t+1}$, $S_2 = S^{t+1}\backslash S^t \subseteq (S^t)^c$, and $|S_1| = |S_2|$. 
 
 Therefore, 
	\begin{align*}
		\<\bbeta^{t+1}-\bbeta^t,\bmg^t\> &\leq -\frac{\xi}{\gamma}\twonorm{\bmg^{t}_{S^{t+1}}}^2 + 10s'\rho^{-1}\infnorm{\bw^t}^2  + \frac{\rho}{40}\twonorm{\bmg^{t}_{S^{t+1}}}^2 + \frac{\xi}{2\gamma}\cdot \frac{1+c}{1-c}\twonorm{\bmg^t_{S^{t+1}\backslash S^t}}^2\\
		&\quad + \frac{\gamma}{2\xi}\cdot \frac{1}{1-c}\cdot (1+2/c)s'\cdot \infnorm{\bw^t}^2 - \frac{\xi}{2\gamma}\twonorm{\bmg^t_{S^{t}\backslash S^{t+1}}}^2 \\
		&\leq -\frac{9\xi}{20\gamma}\twonorm{\bmg^{t}_{S^{t+1} \cup S^t}}^2 + Cs'\infnorm{\bw^t}^2,
	\end{align*}
	when $\frac{1+c}{1-c}\cdot \frac{1}{2} \leq \frac{21}{40}$. Going back to \eqref{eq: proof thm high-dim eq 1 single-source hd}, we have
	\begin{align*}
		\mL_n^t(\bbeta^{t+1}) - \mL_n^t(\bbeta^{t}) &\leq \frac{1}{2}\gamma\twonorm{\bbeta^{t+1}_{I^t} - \bbeta^t_{I^t} + \xi/\gamma\cdot \bmg_{I^t}^t}^2 - \frac{\xi^2}{2\gamma}\twonorm{\bmg_{I^t}^t}^2 +(1-\xi)\<\bbeta^{t+1}-\bbeta^t,\bmg^t\> \\
		&\leq \frac{1}{2}\gamma\twonorm{\bbeta^{t+1}_{I^t} - \bbeta^t_{I^t} + \xi/\gamma\cdot \bmg_{I^t}^t}^2 - \frac{\xi^2}{2\gamma}\twonorm{\bmg_{I^t \backslash (S^t\cup S)}^t}^2 - \frac{\xi^2}{2\gamma}\twonorm{\bmg_{S^t\cup S}^t}^2 \\
		&\quad - \frac{9\xi}{20\gamma}(1-\xi)\twonorm{\bmg^t_{S^{t+1} \cup S^t}}^2+ Cs'\infnorm{\bw^t}^2.
	\end{align*}
	Consider a set $S' \subseteq S^t\backslash S^{t+1}$ with $|S'| = |I^t\backslash (S^t \cup S)| = |S^{t+1}\backslash (S^t \cup S)|$. Applying Lemma \ref{lem: modified 3.4} by setting $v = \bbeta^t - \xi/\gamma \cdot \bmg^t$, $S^t = $ the set of top-$s'$ entries of $|v+\bw^t|$, $S_1 = S'$, $S_2 = S^{t+1}\backslash (S^t \cup S) $, and $|S_1| = |S_2|$, we have
	\begin{equation}\label{eq: hd upper 5 single-source hd}
		\twonorm{\bbeta^t_{S'} - \xi/\gamma \cdot \bmg^t_{S'}}^2 \leq \frac{1+c}{1-c}\frac{\xi^2}{\gamma^2} \twonorm{\bmg^t_{S^{t+1}\backslash (S^t\cup S)}}^2 + \frac{1}{1-c}(1+2/c)s'\cdot \infnorm{\bw^t}^2,
	\end{equation}
	which entails that
	\begin{equation*}
		-\frac{\xi^2}{\gamma}\twonorm{\bmg^t_{S^{t+1}\backslash (S^t\cup S)}}^2 \leq -\frac{1-c}{1+c}\gamma\twonorm{\bbeta^t_{S'} - \xi/\gamma\cdot \bmg^t_{S'}}^2 + \gamma\frac{1+2/c}{1+c}s'\infnorm{\bw^t}^2.
	\end{equation*}
	This leads to
	\begin{align}
		&\frac{1}{2}\gamma\twonorm{\bbeta^{t+1}_{I^t} - \bbeta^t_{I^t} + \xi/\gamma\cdot \bmg_{I^t}^t}^2 - \frac{\xi^2}{2\gamma}\twonorm{\bmg_{I^t \backslash (S^t\cup S)}^t}^2 \nonumber\\
        &= \frac{1}{2}\gamma \twonorm{(\widetilde{P}_{s'}(\bbeta^t - \xi/\gamma\cdot \bmg^t))_{I^t} - \bbeta^t_{I^t} + \xi/\gamma\cdot \bmg_{I^t}^t}^2 - \frac{\xi^2}{2\gamma}\twonorm{\bmg_{I^t \backslash (S^t\cup S)}^t}^2 \nonumber\\
		&= \frac{1}{2}\gamma \twonorm{(H_{s'}(\bbeta^t - \xi/\gamma\cdot \bmg^t + \bw^t) - \bw^t_{S^{t+1}} + \widetilde{\bw}^t_{S^{t+1}})_{I^t} - \bbeta^t_{I^t} + \xi/\gamma\cdot \bmg_{I^t}^t}^2 - \frac{\xi^2}{2\gamma}\twonorm{\bmg_{I^t \backslash (S^t\cup S)}^t}^2 \nonumber\\
		&\leq C\gamma\twonorm{\widetilde{\bw}_{I^t}^t}^2 + \frac{\gamma}{2}\frac{1+c}{1-c}\twonorm{(H_{s'}(\bbeta^t - \xi/\gamma \cdot \bmg^t + \bw^t))_{I^t} - (\bbeta^t - \xi/\gamma \cdot \bmg^t+ \bw^t)_{I^t}}^2 \nonumber \\
        &\quad - \frac{\gamma}{2}\cdot \frac{1-c}{1+c}\cdot \twonorm{\bbeta^t_{S'} - \xi/\gamma \cdot g^t_{S'}}^2 + \frac{\gamma}{2}\cdot \frac{1+2/c}{1+c}s'\infnorm{\bw^t}^2 \nonumber\\
		&\leq C\gamma s'\infnorm{\widetilde{\bw}^t}^2 + \frac{\gamma}{2}\frac{1+c}{1-c}\twonorm{(H_{s'}(\bbeta^t - \xi/\gamma \cdot \bmg^t+ \bw^t))_{I^t} - (\bbeta^t - \xi/\gamma \cdot \bmg^t+ \bw^t)_{I^t}}^2  \nonumber\\
		&\quad - \frac{\gamma}{2}\cdot \bigg(\frac{1-c}{1+c}\bigg)^2 \twonorm{\underbrace{[H_{s'}(\bbeta^t - \xi/\gamma \cdot \bmg^t)]_{S'}}_{ = -\bw_{S'}^t \text{ because } S' \subseteq S^t\backslash S^{t+1}} - (\bbeta^t - \xi/\gamma \cdot \bmg^t)_{S'})}^2 + C\gamma s'\infnorm{\bw^t}^2  \nonumber\\
		&\leq C\gamma s'\infnorm{\widetilde{\bw}^t}^2 + \frac{5\gamma}{9}\twonorm{(H_{s'}(\bbeta^t - \xi/\gamma \cdot \bmg^t + \bw^t))_{I^t\backslash S'} - (\bbeta^t - \xi/\gamma \cdot \bmg^t + \bw^t)_{I^t \backslash S'}}^2  \nonumber\\
        &\quad  + \frac{(1+c)^3-(1-c)^3}{(1+c)^2(1-c)}\frac{\gamma}{2}\twonorm{(\bbeta^t - \xi/\gamma \cdot \bmg^t)_{S'}}^2 + C\gamma s'\infnorm{\bw^t}^2  \nonumber\\
		&\leq C\gamma s'\infnorm{\widetilde{\bw}^t}^2 + \frac{5\gamma}{9}\twonorm{(H_{s'}(\bbeta^t - \xi/\gamma \cdot \bmg^t + \bw^t))_{I^t\backslash S'} - (\bbeta^t - \xi/\gamma \cdot \bmg^t + \bw^t)_{I^t \backslash S'}}^2 \nonumber\\
        &\quad + \frac{c^3+3c}{(1+c)(1-c)^2}\cdot \frac{\xi^2}{\gamma}\cdot \twonorm{\bmg^t_{S^{t+1}\backslash (S^t \cup S)}}^2 + C\gamma s'\infnorm{\bw^t}^2\label{eq: hd upper 4 single-source hd},
	\end{align}
	where we used \eqref{eq: hd upper 5 single-source hd} to obtain \eqref{eq: hd upper 4 single-source hd}. Note that $I^t \backslash S' \supseteq S^{t+1}$, hence we have $(H_{s'}(\bbeta^t - \xi/\gamma \cdot \bmg^t + \bw^t))_{I^t\backslash S'} = H_{s'}((\bbeta^t - \xi/\gamma \cdot \bmg^t + \bw^t)_{I^t\backslash S'})$. Applying Lemma \ref{lem: modified A.3} with $v = (\bbeta^t - \xi/\gamma \cdot \bmg^t + \bw^t)_{I^t\backslash S'}$, $\tilde{v} = \bbeta_{I^t\backslash S'}^*$, $\zeronorm{\tilde{v}} \leq s$, and $s' \geq s$, we have
	\begin{align}
		&\frac{5\gamma}{9}\twonorm{(H_{s'}(\bbeta^t - \xi/\gamma \cdot \bmg^t + \bw^t))_{I^t\backslash S'} - (\bbeta^t - \xi/\gamma \cdot \bmg^t + \bw^t)_{I^t \backslash S'}}^2 \nonumber\\
		&\leq \frac{5\gamma}{9}\cdot \frac{|I^t\backslash S'|-s'}{|I^t\backslash S'|-s}\cdot \twonorm{\bbeta_{I^t\backslash S'}^* - (\bbeta^t - \xi/\gamma \cdot \bmg^t + \bw^t)_{I^t \backslash S'}}^2 \label{eq: hd upper bound beta k diff eq 1}\\
		&\leq \frac{5\gamma}{9}\cdot \frac{s}{s'}\cdot \twonorma{\bbeta_{I^t\backslash S'}^* - \bbeta^t_{I^t\backslash S'} + \frac{\xi}{\gamma}\bmg^t_{I^t\backslash S'}}^2 + C\gamma\cdot \frac{s}{s'}\cdot s'\infnorm{\bw^t}^2 \nonumber\\
		&\leq \frac{5\gamma}{9}\cdot \frac{s}{s'}\cdot \twonorma{\bbeta_{I^t\backslash S'}^* - \bbeta^t_{I^t\backslash S'} + \frac{\xi}{\gamma}\bmg^t_{I^t\backslash S'}}^2 + Cs\infnorm{\bw^t}^2, \nonumber
	\end{align}
	where the second inequality used the fact that $|I^t\backslash S'| \leq s'+s$. This holds because $I^t = S^t \cup S^{t+1} \cup S$, $S' \subseteq S^t \backslash S^{t+1} \subseteq I^t$, $|S'| = |I^t\backslash (S^t \cup S)| = |I^t| - |S^t \cup S|$, leading to $|I^t \backslash S'| \leq |S^t \cup S| \leq s' + s$.
	
	Therefore,
	\begin{align*}
		&\frac{1}{2}\gamma\twonorm{\bbeta^{t+1}_{I^t} - \bbeta^t_{I^t} + \xi/\gamma\cdot \bmg_{I^t}^t}^2 - \frac{\xi^2}{2\gamma}\twonorm{\bmg_{I^t \backslash (S^t\cup S)}^t}^2 \\
		&\leq C\gamma s'\infnorm{\bw^t}^2 + C\gamma s'\infnorm{\widetilde{\bw}^t}^2 + \frac{5\gamma}{9}\cdot \frac{s}{s'}\cdot \twonorma{\bbeta_{I^t\backslash S'}^* - \bbeta^t_{I^t\backslash S'} + \frac{\xi}{\gamma}\bmg^t_{I^t\backslash S'}}^2 \\
        &\quad + \frac{c^3+3c}{(1+c)(1-c)^2}\cdot \frac{\xi^2}{\gamma}\cdot \twonorm{\bmg^t_{S^{t+1}\backslash (S^t \cup S)}}^2 \\
		&\leq C\gamma s'\infnorm{\bw^t}^2+ C\gamma s'\infnorm{\widetilde{\bw}^t}^2  + \frac{5\gamma}{9}\cdot \frac{s}{s'}\cdot \twonorma{\bbeta_{I^t}^* - \bbeta^t_{I^t} + \frac{\xi}{\gamma}\bmg^t_{I^t}}^2 + \frac{c^3+3c}{(1+c)(1-c)^2}\cdot \frac{\xi^2}{\gamma}\cdot \twonorm{\bmg^t_{S^{t+1}\backslash (S^t \cup S)}}^2 \\
		&\leq C\gamma s'\infnorm{\bw^t}^2 + C\gamma s'\infnorm{\widetilde{\bw}^t}^2 + \frac{5}{9}\cdot\frac{s}{s'}\cdot \left(2\xi\<\bbeta_{I^t}^*-\bbeta^t_{I^t}, \bmg_{I^t}^t\> + \gamma\twonorm{\bbeta_{I^t}^* - \bbeta^t_{I^t}}^2 + \frac{\xi^2}{\gamma^2}\twonorm{\bmg^t_{I^t}}^2\right) \\
		&\quad + \frac{c^3+3c}{(1+c)(1-c)^2}\cdot \frac{\xi^2}{\gamma}\cdot \twonorm{\bmg^t_{S^{t+1}}}^2 \\
		&\leq C \gamma s'\infnorm{\bw^t}^2 + C\gamma s'\infnorm{\widetilde{\bw}^t}^2 + \frac{5}{9}\cdot\frac{s}{s'}\cdot \left(2\xi\mL_n^t(\bbeta^*) - 2\xi\mL_n^t(\bbeta^t) + (\gamma - \xi\alpha)\twonorm{\bbeta^* - \bbeta^t}^2 + \frac{\xi^2}{\gamma}\twonorm{\bmg^t_{I^t}}\right) \\
		&\quad + \frac{c^3+3c}{(1+c)(1-c)^2}\cdot \frac{\xi^2}{\gamma}\cdot \twonorm{\bmg^t_{S^{t+1}}}^2 \quad \quad (*).
	\end{align*}
	Hence,
	\begin{align}
		\mL_n^t(\bbeta^{t+1}) - \mL_n^t(\bbeta^t) &\leq \frac{1}{2}\gamma\twonorm{\bbeta^{t+1}_{I^t} - \bbeta^t_{I^t} + \xi/\gamma\cdot \bmg_{I^t}^t}^2 - \frac{\xi^2}{2\gamma}\twonorm{\bmg_{I^t \backslash (S^t\cup S)}^t}^2 - \frac{\xi^2}{2\gamma}\twonorm{\bmg_{S^t\cup S}^t}^2 \nonumber\\
		&\quad - \frac{9\xi}{20\gamma}(1-\xi)\twonorm{\bmg^t_{S^{t+1} \cup S^t}}^2+ Cs\infnorm{\bw^t}^2 + C s'\infnorm{\widetilde{\bw}^t}^2 \nonumber\\
		&\leq (*) - \frac{\xi^2}{2\gamma}\twonorm{\bmg^t_{S^t \cup S}}^2 - \frac{9\xi}{20\gamma}(1-\xi)\twonorm{\bmg^t_{S^{t+1} \cup S^t}}^2+ Cs\infnorm{\bw^t}^2 + C s'\infnorm{\widetilde{\bw}^t}^2 \nonumber\\
		&= \frac{10s}{9s'}\cdot \xi \cdot [\mL_n^t(\bbeta^*) - \mL_n^t(\bbeta^t)] + \frac{s}{s'}\cdot \frac{5(\gamma - \xi\alpha)}{9}\twonorm{\bbeta^* - \bbeta^t}^2 \nonumber\\
		&\quad + \frac{s}{s'}\cdot \frac{5\xi^2}{9\gamma} \cdot (\twonorm{\bmg^t_{S^t \cup S}}^2 + \twonorm{\bmg^t_{S^{t+1}\backslash (S^t \cup S)}}^2) \nonumber\\
		&\quad - \frac{\xi^2}{2\gamma}\twonorm{\bmg^t_{S^t \cup S}}^2 - \frac{9\xi}{20\gamma}(1-\xi)\twonorm{\bmg^t_{S^{t+1} \cup S^t}}^2 + Cs \infnorm{\bw^t}^2+ C s'\infnorm{\widetilde{\bw}^t}^2 \nonumber\\
		&= \frac{10s}{9s'}\cdot \xi \cdot [\mL_n^t(\bbeta^*) - \mL_n^t(\bbeta^t)] + \frac{s}{s'}\cdot \frac{5(\gamma - \xi\alpha)}{9}\twonorm{\bbeta^* - \bbeta^t}^2 \nonumber\\
		&\quad + \left[\frac{s}{s'}\cdot \frac{5\xi^2}{9\gamma} -  \frac{9\xi}{20\gamma}(1-\xi)\right]\twonorm{\bmg^t_{S^{t+1}\backslash (S^t \cup S)}}^2 \nonumber\\
		&\quad + \left(\frac{10s}{9s'}  - 1\right)\frac{\xi^2}{2\gamma}\cdot \twonorm{\bmg^t_{S^t \cup S}}^2 + Cs' \infnorm{\bw^t}^2+ C s'\infnorm{\widetilde{\bw}^t}^2 \nonumber\\
		&\leq \frac{10s}{9s'}\cdot \xi \cdot [\mL_n^t(\bbeta^*) - \mL_n^t(\bbeta^t)] + \frac{s}{s'}\cdot \frac{5(\gamma - \xi\alpha)}{9}\twonorm{\bbeta^* - \bbeta^t}^2 \nonumber\\
		&\quad -\frac{9s'-10s}{9s'}\cdot \frac{\xi^2}{2\gamma}\cdot \twonorm{\bmg^t_{S^t \cup S}}^2 + Cs' \infnorm{\bw^t}^2+ C s'\infnorm{\widetilde{\bw}^t}^2, \label{eq: hd upper 6 single-source hd}
	\end{align}
	where \eqref{eq: hd upper 6 single-source hd} holds since $\frac{s}{s'}\cdot \frac{5\xi^2}{9\gamma} -  \frac{9\xi}{20\gamma}(1-\xi) = \frac{\xi}{\gamma}[\frac{s}{s'}\xi - \frac{81}{100}(1-\xi)] \leq 0$ due to the conditions assumed in \Cref{thm: high-dim upper bound single-source}. On the other hand, note that
	\begin{align}
		\mL_N^t(\bbeta^t) - \mL_N^t(\bbeta^*) &\leq \<\bmg^t, \bbeta^t - \bbeta^*\> - \frac{\alpha}{2}\twonorm{\bbeta^* - \bbeta^t}^2 \nonumber\\
		&\leq \twonorm{\bmg^t_{S^t\cup S}}\cdot \twonorm{\bbeta^t - \bbeta^*} - \frac{\alpha}{2}\twonorm{\bbeta^*-\bbeta^t}^2,\label{eq: hd upper 7 single-source hd}
	\end{align}
	and
	\begin{align*}
		\twonorm{\bmg^t_{S^t \cup S}}^2 - \frac{1}{4}\alpha^2\twonorm{\bbeta^*-\bbeta^t}^2 
		&= \left(\twonorm{\bmg^t_{S^t \cup S}} + \frac{\alpha}{2}\twonorm{\bbeta^*-\bbeta^t}\right)\left(\twonorm{\bmg^t_{S^t \cup S}} - \frac{\alpha}{2}\twonorm{\bbeta^*-\bbeta^t}\right) \\
		&\geq \frac{\mL_n^t(\bbeta^t) - \mL_n^t(\bbeta^*)}{\twonorm{\bbeta^*-\bbeta^t}}\cdot \left(\twonorm{\bmg^t_{S^t \cup S}} + \frac{\alpha}{2}\twonorm{\bbeta^*-\bbeta^t}\right)\\
		&\geq \frac{\alpha}{2}[\mL_n^t(\bbeta^t) - \mL_n^t(\bbeta^*)],
	\end{align*}
	which implies
	\begin{equation}\label{eq: hd upper 8 single-source hd}
		\twonorm{\bmg^t_{S^t \cup S}}^2 \geq \frac{1}{4}\alpha^2\twonorm{\bbeta^*-\bbeta^t}^2 + \frac{\alpha}{2}[\mL_n^t(\bbeta^t) - \mL_n^t(\bbeta^*)].
	\end{equation}
	By adding $\mL_n^t(\bbeta^t) - \mL_n^t(\bbeta^*)$ on both sides of \eqref{eq: hd upper 6 single-source hd}, together with \eqref{eq: hd upper 8 single-source hd}, we obtain
	\begin{align}
		\mL_n^t(\bbeta^{t+1}) - \mL_n^t(\bbeta^*) &\leq \left(1-\frac{10s}{9s'}\cdot \xi\right) \cdot [\mL_n^t(\bbeta^t)-\mL_n^t(\bbeta^*)] + \frac{s}{s'}\cdot \frac{5(\gamma - \xi\alpha)}{9}\twonorm{\bbeta^* - \bbeta^t}^2 \nonumber\\
		&\quad -\frac{9s'-10s}{9s'}\cdot \frac{\xi^2}{2\gamma}\cdot \twonorm{\bmg^t_{S^t \cup S}}^2 + Cs' \infnorm{\bw^t}^2+ C s'\infnorm{\widetilde{\bw}^t}^2 \nonumber\\
		&\leq \left(1-\frac{10s}{9s'}\xi - \frac{9s'-10s}{9s'}\cdot \frac{\xi^2}{4\gamma}\alpha\right)\cdot [\mL_n^t(\bbeta^t)-\mL_n^t(\bbeta^*)] \nonumber\\
		&\quad + \left(\frac{s}{s'}\cdot \frac{5(\gamma - \xi\alpha)}{9} - \frac{9s'-10s}{9s'}\cdot \frac{\xi^2}{8\gamma}\alpha^2\right)\twonorm{\bbeta^t - \bbeta^*}^2 + Cs' \infnorm{\bw^t}^2+ C s'\infnorm{\widetilde{\bw}^t}^2. \label{eq: hd upper 9 single-source hd}
	\end{align}
    \noindent {\textbf{Step 2:}} Replace $\mL_n^t(\bbeta^{t+1}) - \mL_n^t(\bbeta^*)$ with a lower bound involving $\|\bbeta^{t+1} - \bbeta^*\|_{\bSigma}$ and replace $\mL_n^t(\bbeta^{t}) - \mL_n^t(\bbeta^*)$ with an upper bound involving $\|\bbeta^{t} - \bbeta^*\|_{\bSigma}$.
    
	Note that
	\begin{align*}
		\mL_n^t(\bbeta^t)-\mL_n^t(\bbeta^*) &= \frac{1}{2(n/T)}\twonorm{Y^{t} - \bm{X}^{t}\bbeta^t}^2 - \frac{1}{2(n/T)}\twonorm{Y^{t} - \bm{X}^{t}\bbeta^*}^2 \\
		&= \frac{1}{2(n/T)}\twonorm{ \bm{X}^{t}(\bbeta^*-\bbeta^t) + \epsilon^{t}}^2 - \frac{1}{2(n/T)}\twonorm{\epsilon^{t}}^2 \\
		&= \frac{1}{2}(\bbeta^t-\bbeta^*)^\top\hSigma^{t}(\bbeta^t-\bbeta^*) + \frac{1}{n/T}(\bbeta^*-\bbeta^t)^\top(\bm{X}^{t})^\top\epsilon^{t}.
	\end{align*}
 
 Note that
	\begin{align*}
		\norma{(\bbeta^t-\bbeta^*)^\top\hSigma^{t}(\bbeta^t-\bbeta^*) - (\bbeta^t-\bbeta^*)^\top\bSigma(\bbeta^t-\bbeta^*)}
		&\lesssim \twonorm{\bbeta^t-\bbeta^*}^2\cdot \sqrt{\frac{s'\log (d/\eta)}{n/T}}, \\
		\twonorma{(\bm{X}^{t}_{:, S^t\cup S})^\top\epsilon^{t}} \leq \sqrt{s'}\infnorma{(\bm{X}^{t})^\top\epsilon^{t}} &\lesssim \sqrt{\frac{s'\log(d/\eta)}{n/T}}.
	\end{align*}
	Therefore,
	\begin{equation}\label{eq: Ln bound 1 single-source hd}
		\mL_n^{t}(\bbeta^t)-\mL_n^{t}(\bbeta^*) \leq \left[\frac{1}{2}+C\gamma \sqrt{\frac{s'\log(d/\eta)}{n/T}} + c\right]\|\bbeta^t-\bbeta^*\|_{\bSigma}^2 + C'\frac{s'\log(d/\eta)}{n/T}.
	\end{equation}
	Similarly,
	\begin{equation}\label{eq: Ln bound 2 single-source hd}
		\mL_n^{t}(\bbeta^{t+1}) - \mL_n^{t}(\bbeta^*) \geq \left[\frac{1}{2}-C\gamma \sqrt{\frac{s'\log(d/\eta)}{n/T}} - c\right]\|\bbeta^{t+1}-\bbeta^*\|_{\bSigma}^2 - C'\frac{s'\log(d/\eta)}{n/T}.
	\end{equation}

    \noindent {\textbf{Step 3:}} Obtain an induction relationship between $\|\bbeta^{t+1} - \bbeta^*\|_{\bSigma}$ and $\|\bbeta^{t} - \bbeta^*\|_{\bSigma}$, translate the $\bSigma$-norm to $\ell_2$-norm, then complete the proof.
    
	Plugging \eqref{eq: Ln bound 1 single-source hd} and \eqref{eq: Ln bound 2 single-source hd} back in \eqref{eq: hd upper 9 single-source hd}, we get
	\begin{align*}
		&\left[\frac{1}{2}-C\gamma \sqrt{\frac{s'\log(d/\eta)}{n/T}} - c\right]\|\bbeta^{t+1}-\bbeta^*\|_{\bSigma}^2 \\
		&\leq \left(1-\frac{10s}{9s'}\xi - \frac{9s'-10s}{9s'}\cdot \frac{\xi^2}{4\gamma}\alpha\right)\cdot\left[\frac{1}{2}+C\gamma \sqrt{\frac{s'\log(d/\eta)}{n/T}} + c\right]\|\bbeta^t-\bbeta^*\|_{\bSigma}^2 \\
		&\quad + \left(\frac{s}{s'}\cdot \frac{5(\gamma - \xi\alpha)}{9\alpha} - \frac{9s'-10s}{9s'}\cdot \frac{\xi^2}{8\gamma}\alpha\right)\|\bbeta^t - \bbeta^*\|_{\bSigma}^2 + Cs' \infnorm{\bw^t}^2 + C s'\infnorm{\widetilde{\bw}^t}^2 + C\frac{s'\log(d/\eta)}{n/T},
	\end{align*}
	which implies that
	\begin{align*}
		\|\bbeta^{t+1}-\bbeta^*\|_{\bSigma}^2 &\leq \left(1-\frac{20s}{9s'}\xi -\frac{9s'-10s}{9s'}\cdot \frac{\xi^2\alpha}{2\gamma} + \frac{10s}{9s'}\cdot \frac{\gamma}{\alpha}+C'\sqrt{\frac{s'\log(d/\eta)}{n/T}}+C'c\right)\|\bbeta^t-\bbeta^*\|_{\bSigma}^2 \\
		&\quad + Cs' \infnorm{\bw^t}^2 + C\frac{s'\log(d/\eta)}{n/T} + C s'\infnorm{\widetilde{\bw}^t}^2 \\
		&\leq \left(1 -\frac{2s}{s'}\xi -\frac{9s'-10s}{9s'}\cdot \frac{\xi^2\alpha}{2\gamma} + \frac{10s}{9s'}\cdot \frac{\gamma}{\alpha}\right)\|\bbeta^t-\bbeta^*\|_{\bSigma}^2 + Cs' \infnorm{\bw^t}^2 \\
        &\quad + C s'\infnorm{\widetilde{\bw}^t}^2+ C\frac{s'\log(d/\eta)}{n/T}.
	\end{align*}
	By induction, we have
	\begin{align*}
		\|\bbeta^{T}-\bbeta^*\|_{\bSigma}^2 &\leq \left(1 -\frac{2s}{s'}\xi -\frac{9s'-10s}{9s'}\cdot \frac{\xi^2\alpha}{2\gamma} + \frac{10s}{9s'}\cdot \frac{\gamma}{\alpha}\right)^T\|\bbeta^0-\bbeta^*\|_{\bSigma}^2 \\
		&\quad + Cs' \sum_{t=0}^{T-1}\left(1-\frac{2s}{s'}\xi -\frac{9s'-10s}{9s'}\cdot \frac{\xi^2\alpha}{2\gamma} + \frac{10s}{9s'}\cdot \frac{\gamma}{\alpha}\right)^{T-t-1}(\infnorm{\bw^t}^2+\infnorm{\widetilde{\bw}^t}^2) \\
		&\quad + C\frac{s'\log(d/\eta)}{n/T} \\
		&\leq \left(1-\frac{2s}{s'}\xi -\frac{9s'-10s}{9s'}\cdot \frac{\xi^2\alpha}{2\gamma} + \frac{10s}{9s'}\cdot \frac{\gamma}{\alpha}\right)^T\|\bbeta^0-\bbeta^*\|_{\bSigma}^2 \\
		&\quad +C\frac{s^2\log(1/\delta)\log(nd/\eta)\log(Td/\eta)}{(n/T)^2\epsilon^2}\sum_{t=0}^{T-1}\left(1-\frac{2s}{s'}\xi -\frac{9s'-10s}{9s'}\cdot \frac{\xi^2\alpha}{2\gamma} + \frac{10s}{9s'}\cdot \frac{\gamma}{\alpha}\right)^{T-t-1}(R_t)^2\\
		&\quad + C\frac{s'\log(d/\eta)}{n/T}\\
		&\leq \left(1 -\frac{2s}{s'}\xi-\frac{9s'-10s}{9s'}\cdot \frac{\xi^2\alpha}{2\gamma} + \frac{10s}{9s'}\cdot \frac{\gamma}{\alpha}\right)^T\|\bbeta^0-\bbeta^*\|_{\bSigma}^2 \\
		&\quad + C\frac{s^2\log(1/\delta)\log^2(nd/\eta)\log(Td/\eta)}{(n/T)^2\epsilon^2} (1\vee \|\bbeta^0-\bbeta^*\|_{\bSigma}^2)\\
		&\quad + C\frac{s'\log(d/\eta)}{n/T},
	\end{align*}
	where the last inequality comes from the choice of $R_t$ in Algorithm \ref{algorithm:DPregression_high_dim_single_source}. Therefore, conditioned on $\cap_{i=1}^5\mathcal{E}_i$, since $s' \lesssim s$, we have
	\begin{align*}
		\twonorm{\bbeta^{T}-\bbeta^*} &\leq \sqrt{\frac{\gamma}{\alpha}}\left(1 -\frac{2s}{s'}\xi -\frac{9s'-10s}{9s'}\cdot \frac{\xi^2\alpha}{2\gamma} + \frac{10s}{9s'}\cdot \frac{\gamma}{\alpha}\right)^{T/2}\twonorm{\bbeta^{0}-\bbeta^*} \\
		&\quad +C\frac{s\sqrt{\log(1/\delta)\log^2(nd/\eta)\log(Td/\eta)}}{(n/T)\epsilon} (1\vee \twonorm{\bbeta^0-\bbeta^*}) \\
		&\quad + C\sqrt{\frac{s\log(d/\eta)}{n/T}},
	\end{align*}
	which completes the proof of the bound.
\medskip

\noindent {(\Rom{5}) Lemmas and their proofs:}
\begin{lemma}[A modified version of Lemma 3.4 in \citealt{cai2019cost}]\label{lem: modified 3.4}
Consider vectors $v$, $w \in \mathbb{R}^d$, another vector  $u = v + w$, a set $S_1 \subseteq S'$ with $S'$ as the indices of top-$s'$ entries of $|u|$ (absolute value for each entry), and another set $S_2 \subseteq (S')^c$ with $|S_1| = |S_2|$, where $s' \in \mathbb{N}_+$. Then for any $c \in (0,1)$,
	\[
		(1-c)\twonorm{v_{S_2}}^2 - \frac{1}{c}\twonorm{\bw_{S_2}}^2 \leq \twonorm{u_{S_1}}^2 \leq (1+c)\twonorm{v_{S_1}}^2 + \left(1+\frac{1}{c}\right)\twonorm{\bw_{S_1}}^2,
	\]
	which implies
	\[
		\twonorm{v_{S_2}}^2 \leq \frac{1+c}{1-c}\twonorm{v_{S_1}}^2 + \frac{1}{1-c}\left[\frac{1}{c}|S_2| + \left(1+\frac{1}{c}\right)|S_1|\right]\infnorm{\bw}^2.
	\]
\end{lemma}

\begin{lemma}[Lemma 1 in \citealt{jain2014iterative}]\label{lem: modified A.3}
	Suppose $|\textup{supp}(v)| \geq \tilde{s}\vee s'$ where $v \in \mathbb{R}^d$, $\tilde{s}, s' \in \mathbb{N}_+$, and $\tilde{s} \leq s'$. Then for any $\tilde{v}$ with $\zeronorm{\tilde{v}} \leq \tilde{s}$,  we have
	\[
		\twonorm{\textup{Hard-thresholding}(v, s')  - v}^2 \leq \frac{|\textup{supp}(v)| - s'}{|\textup{supp}(v)| - \tilde{s}}\cdot \twonorm{\tilde{v} - v}^2.
	\]
\end{lemma}

\begin{lemma}\label{lem: norm single-source hd}
	Under Assumption \ref{asp: x}, with probability at least $1-\eta$, we have 
	\begin{enumerate}[label=(\roman*), leftmargin=*]
		\item $\max_{i=1:n}\twonorm{X_i} \lesssim \sqrt{d\log(n/\eta)}$;
		\item $\infnorm{\bw^t}^2, \infnorm{\widetilde{\bw}^t}^2  \lesssim \frac{s\log(1/\delta)\log(nd/\eta)\log(Td/\eta)}{(n/T)^2\epsilon^2} (R_t)^2$ for all $t \in [T]$.
	\end{enumerate}
\end{lemma}

\begin{proof}[Proof of Lemma \ref{lem: modified 3.4}]
By the Cauchy-Schwarz inequality,
\begin{align*}
    \twonorm{u_{S_1}}^2 &= \twonorm{v_{S_1}}^2 + \twonorm{\bw_{S_1}}^2 + 2\<v_{S_1}, \bw_{S_1}\> \\
    &\leq (1+c)\twonorm{v_{S_1}}^2 + \bigg(1+\frac{1}{c}\bigg)\twonorm{\bw_{S_1}}^2 \\
    &\leq (1+c)\twonorm{v_{S_1}}^2 + \bigg(1+\frac{1}{c}\bigg)|S_1|\infnorm{\bw_{S_1}}^2, \\
    \twonorm{u_{S_2}}^2 &= \twonorm{v_{S_2}}^2 + \twonorm{\bw_{S_2}}^2 + 2\<v_{S_2}, \bw_{S_2}\> \\
    &\geq (1-c)\twonorm{v_{S_2}}^2 -\frac{1}{c}\twonorm{\bw_{S_2}}^2 \\
    &\geq (1-c)\twonorm{v_{S_2}}^2 -\frac{1}{c}|S_2|\infnorm{\bw_{S_2}}^2.
\end{align*}   
Combining the fact that $\twonorm{u_{S_2}}^2 \leq \twonorm{u_{S_1}}^2$ with the two inequalities above, we obtain the desired result.
\end{proof}

\begin{proof}[Proof of Lemma \ref{lem: norm single-source hd}]

(\rom{1}) $\twonorm{X_i}^2$ is $Cd$-sub-Exponential with mean $C'd$, then the bound is a direct consequence of the tail bound of sub-Exponential variables \citep[see e.g.~Theorem 2.8.1 in][] {vershynin2018high}.

\noindent (\rom{2}) This is by the union bound and the tail of Laplacian variables.

\end{proof}

\subsubsection{Proof of Proposition \ref{prop: single_source}}
In this section, we provide the proof of Proposition \ref{prop: single_source} which is used in Section \ref{subsubsec: proof hd upper bound} as an important intermediate step to prove Theorem \ref{thm: high-dim upper bound}.

\textbf{Part (\rom{1})} follows from a similar privacy argument in the proof of Proposition \ref{prop: high-dim upper bound single-source}. \textbf{The proof of part (\rom{2})} follows the main idea in the proof of Proposition \ref{prop: high-dim upper bound single-source}. We only point out the different arguments here. For convenience, we denote $n_k = 2n$ and $\bSigmak{k} = \bSigma$. We need to replace $\bbeta$ in \eqref{eq: hd upper bound beta k diff eq 1} by $\bbetak{0}$, then the same arguments go through until the end of Step 1. In Step 2, we will lower bound $\mL_n^t(\bbeta^{(k)t+1}) - \mL_n^t(\bbetak{0})$ with $\|\bbeta^{(k)t+1} - \bbetak{0}\|_{\bSigma}$ and upper bound  $\mL_n^t(\bbeta^{(k)t}) - \mL_n^t(\bbetak{0})$ with $\|\bbeta^{(k)t} - \bbetak{0}\|_{\bSigma}$, respectively. Note that
\begin{align}
    \mL_n^t(\bbeta^{(k)t})-\mL_n^t(\bbetak{0}) &= \frac{1}{2(n/T)}\twonorm{Y^{(k)t} - \bm{X}^{(k)t}\bbeta^{(k)t}}^2 - \frac{1}{2(n/T)}\twonorm{Y^{(k)t} - \bm{X}^{(k)t}\bbetak{0}}^2 \nonumber\\
    &= \frac{1}{2(n/T)}\twonorm{\bm{X}^{(k)t}(\bbetak{k} - \bbetak{0}) + \bX^{(k)t}(\bbetak{0} - \bbeta^{(k)t}) + \epsilon^{(k)t}}^2 \nonumber\\
    &\quad - \frac{1}{2(n/T)}\twonorm{\bm{X}^{(k)t}(\bbetak{k} - \bbetak{0}) + \epsilon^{(k)t}}^2 \nonumber\\
    &= \frac{1}{2}(\bbeta^{(k)t} - \bbetak{0})^\top \hSigma^{(k)t}(\bbeta^{(k)t} - \bbetak{0}) \nonumber\\
    &\quad + \frac{1}{n/T}(\bbeta^{(k)t} - \bbetak{0})^\top (\bm{X}^{(k)t})^\top[\bm{X}^{(k)t}(\bbetak{k} - \bbetak{0}) + \epsilon^{(k)t}]. \label{eq: hd upper bound proof prop eq 1}
\end{align}
Define event
\begin{equation*}
    \mathcal{E}_6 = \Bigg\{\norm{u^\top (\hSigma^{(k)t} - \bSigma)u} \leq C\sqrt{\frac{s\log(dKT/\eta)}{n/T}}, \forall t = 0:T, k \in \mA, \forall u \textup{ satisfying } \twonorm{u} \leq 1, \zeronorm{u} \leq s+s'\Bigg\}.
\end{equation*}
By standard arguments and the union bounds, we have $\tp(\mathcal{E}_6) \geq 1-\eta$. Conditioned on the event $\mathcal{E}_6$, we must have
\begin{equation*}
    (\bbeta^{(k)t} - \bbetak{0})^\top \hSigma^{(k)t}(\bbeta^{(k)t} - \bbetak{0}) \leq (\bbeta^{(k)t} - \bbetak{0})^\top \bSigma(\bbeta^{(k)t} - \bbetak{0}) + C\sqrt{\frac{s\log(dKT/\eta)}{n/T}} \twonorm{\bbeta^{(k)t} - \bbetak{0}}^2.
\end{equation*}
Furthermore, with a small constant $c > 0$, 
\begin{align*}
    &\frac{1}{n/T}(\bbeta^{(k)t} - \bbetak{0})^\top(\bX^{(k)t})^\top \bX^{(k)t}(\bbetak{k} - \bbetak{0}) \\
    &\leq c(\bbeta^{(k)t} - \bbetak{0})^\top \bSigma (\bbeta^{(k)t} - \bbetak{0}) + \frac{1}{4c}(\bbetak{k} - \bbetak{0})^\top \hSigma^{(k)t}(\bbetak{k} - \bbetak{0}) \\
    &\leq c\|\bbeta^{(k)t} - \bbetak{0}\|_{\bSigma}^2 + \frac{1}{4c}\|\bbetak{k} - \bbetak{0}\|_{\bSigma}^2 + \frac{1}{4c}(\bbetak{k} - \bbetak{0})^\top (\hSigma^{(k)t}- \bSigma)(\bbetak{k} - \bbetak{0}).
\end{align*}
By Lemma 12 in \cite{loh2012high}, we can upper bound the last term as
\begin{align*}
    \frac{1}{4c}(\bbetak{k} - \bbetak{0})^\top (\hSigma^{(k)t}- \bSigma)(\bbetak{k} - \bbetak{0}) &\leq C\sqrt{\frac{s\log(dKT/\eta)}{n/T}}\twonorm{\bbetak{k} - \bbetak{0}}^2 \\
    &\quad+ C\sqrt{\frac{\log(dKT/\eta)}{sn/T}}\onenorm{\bbetak{k} - \bbetak{0}}^2,
\end{align*}
which entails that
\begin{align*}
    \frac{1}{4c}(\bbetak{k} - \bbetak{0})^\top\hSigma^{(k)t}(\bbetak{k} - \bbetak{0}) &\leq \frac{1}{4c}\|\bbetak{k} - \bbetak{0}\|_{\bSigma}^2 + C\sqrt{\frac{s\log(dKT/\eta)}{n/T}}\twonorm{\bbetak{k} - \bbetak{0}}^2 \\
    &\quad+ C\sqrt{\frac{\log(dKT/\eta)}{sn/T}}\onenorm{\bbetak{k} - \bbetak{0}}^2 \\
    &\leq \frac{1}{4c}\|\bbetak{k} - \bbetak{0}\|_{\bSigma}^2 + C\sqrt{\frac{s\log(dKT/\eta)}{n/T}}\twonorm{\bbetak{k} - \bbetak{0}}^2 \\
    &\quad+ C\sqrt{\frac{\log(dKT/\eta)}{sn/T}}\onenorm{\bbetak{k} - \bbetak{0}}^2 \\
    &\leq \frac{1}{4c}\|\bbetak{k} - \bbetak{0}\|_{\bSigma}^2 + C\sqrt{\frac{s\log(dKT/\eta)}{n/T}}\twonorm{\bbetak{k} - \bbetak{0}}^2 \\
    &\quad+ C\sqrt{\frac{\log(sdKT/\eta)}{n/T}}\twonorm{\bbetak{k} - \bbetak{0}}^2,
\end{align*}
where the last inequality holds due to the assumption that $\onenorm{\bbetak{k} - \bbetak{0}} \lesssim \sqrt{s}\twonorm{\bbetak{k} - \bbetak{0}}$. And the term $\frac{1}{n/T}(\bbeta^{(k)t} - \bbetak{0})^\top (\bm{X}^{(k)t})^\top\epsilon^{(k)t}$ in \eqref{eq: hd upper bound proof prop eq 1} can be similarly bounded as in the proof of Proposition \ref{prop: high-dim upper bound single-source}.  Combining all the pieces, conditioned on event $\mathcal{E}_6$, we have
\begin{align*}
    \mL_n^t(\bbeta^{(k)t})-\mL_n^t(\bbetak{0}) \leq \bigg(\frac{1}{2} + C\sqrt{\frac{\log(sdKT/\eta)}{n/T}} + c\bigg)\|\bbeta^{(k)t} - \bbetak{0}\|_{\bSigma}^2 + C'h^2 + C'\frac{s'\log(dK/\eta)}{n/T}.
\end{align*}
Similarly, conditioned on event $\mathcal{E}_6$, we can show that
\begin{align*}
    \mL_n^t(\bbeta^{(k)t+1})-\mL_n^t(\bbetak{0}) \geq \bigg(\frac{1}{2} -C\sqrt{\frac{\log(sdKT/\eta)}{n/T}} - c\bigg)\|\bbeta^{(k)t+1} - \bbetak{0}\|_{\bSigma}^2 - C'h^2 - C'\frac{s'\log(dK/\eta)}{n/T}.
\end{align*}
The remaining arguments are the same as in Step 3 of the proof of Proposition \ref{prop: high-dim upper bound single-source}, which we do not repeat here. 

\textbf{The statement (\rom{3})} can be similarly proved by following the same analysis above, therefore we omit the details.

\subsubsection{Proof of Proposition \ref{prop: high dim A}}
In this section, we provide the proof of Proposition \ref{prop: high dim A}, which is used in Section \ref{subsubsec: proof hd upper bound} as an important intermediate step to prove Theorem \ref{thm: high-dim upper bound}.

\noindent \textbf{The statement (\rom{1})} follows from Proposition \ref{prop: single_source}.(\rom{1}), Proposition \ref{prop: high-dim upper bound}.(\rom{1}), and the parallel DP composition theorem \citep{smith2021making}. 

\medskip
\noindent \textbf{For (\rom{2})}, to show that $\widehat{\mathcal{A}} \subseteq \mathcal{A}$, we note that for all $k \in \mA^c$, it holds
\begin{align*}
    \twonorm{\hbeta^{(k)}-\hbeta^{(0)}} &\geq \twonorm{\bbetak{k} - \bbetak{0}} - \twonorm{\hbeta^{(k)} - \bbetak{k}} - \twonorm{\hbeta^{(0)} - \bbetak{0}} \\
    &\geq \twonorm{\bbetak{k} - \bbetak{0}} - C\twonorm{\bbetak{k} - \widetilde{\bbeta}^{(k)}} - Cr_{\textup{HLR}}(n_k, s, d, \epsilon, \delta, \eta/K) - Cr_{\textup{HLR}}(n_0, s, d, \epsilon, \delta, \eta/K) \\
    &\geq \twonorm{\bbetak{k} - \bbetak{0}} - Cc\twonorm{\bbetak{k} - \bbetak{0}} - Cr_{\textup{HLR}}(n_k, s, d, \epsilon, \delta, \eta/K) - Cr_{\textup{HLR}}(n_0, s, d, \epsilon, \delta, \eta/K) \\
    &\geq \frac{1}{2}\twonorm{\bbetak{k} - \bbetak{0}} - Cr_{\textup{HLR}}(n_k, s, d, \epsilon, \delta, \eta/K) - Cr_{\textup{HLR}}(n_0, s, d, \epsilon, \delta, \eta/K) \\
    &> C_0 r_{\textup{HLR}}(n_0, s, d, \epsilon, \delta, \eta/K),
\end{align*}
where the first inequality is due to triangle inequality, the second one is due to parts (\rom{2}) and (\rom{4}) of Proposition \ref{prop: single_source}, and the last two are due to Assumption \ref{asp: beta}.(\rom{3}) and the sample size condition $\min_{k \in [K]} \gtrsim n_0$. By the definition of $\widehat{\mathcal{A}}$ in Algorithm \ref{algorithm:DPregression_high_dim}, we have
$\hat{\mA} \subseteq \mA$.

To further bound the difference $\twonorm{\bbetak{k} - \bbetak{0}}$ for each $k \in \widehat{\mathcal{A}}$, it holds that
\begin{align*}
    \twonorm{\bbetak{k} - \bbetak{0}} &\leq \twonorm{\hbeta^{(k)} - \hbeta^{(0)}} + \twonorm{\hbeta^{(k)} - \bbetak{k}} + \twonorm{\hbeta^{(0)} - \bbetak{0}} \\
    &\leq C_0 r_{\textup{HLR}}(n_0, s, d, \epsilon, \delta, \eta/K) + Cr_{\textup{HLR}}(n_k, s, d, \epsilon, \delta, \eta/K) + C r_{\textup{HLR}}(n_0, s, d, \epsilon, \delta, \eta/K) \\
    &\lesssim r_{\textup{HLR}}(n_0, s, d, \epsilon, \delta, \eta/K),
\end{align*}
where the first inequality is due to triangle inequality, the second one is due to parts (\rom{2}) and (\rom{3}) of Proposition \ref{prop: single_source}, and the final one is due to the sample size condition $\min_{k \in [K]} \gtrsim n_0$. 

\medskip
\noindent \textbf{For (\rom{3})}, for all $k \in \mA$, we have
\begin{align*}
    \twonorm{\hbeta^{(k)}-\hbeta^{(0)}} &\leq \twonorm{\bbetak{k} - \bbetak{0}} + \twonorm{\hbeta^{(k)} - \bbetak{k}} + \twonorm{\hbeta^{(0)} - \bbetak{0}} \\
    &\leq C r_{\textup{HLR}}(n_0, s, d, \epsilon, \delta, \eta/K) + Cr_{\textup{HLR}}(n_k, s, d, \epsilon, \delta, \eta/K) + Ch \\
    &\leq C_0  r_{\textup{HLR}}(n_0, s, d, \epsilon, \delta, \eta/K),  
\end{align*}
where the first inequality is due to triangle inequality, the second one is due to Proposition \ref{prop: single_source}.(\rom{3}) and the previous part (\rom{2}), and the final one is due to Assumption \ref{asp: sample size diff nk}. By the definition of $\widehat{\mathcal{A}}$ in Algorithm \ref{algorithm:DPregression_high_dim}, it implies that $\hat{\mA} \supseteq \mA$. Combining this with part (\rom{1}), we have $\hat{\mA} = \mA$, which completes the proof.

\subsubsection{Proof of Proposition \ref{prop: high-dim upper bound}}\label{subsubsec: proof of prop high-dim upper bound}

In this subsection, we provide the proof of \Cref{prop: high-dim upper bound}, which is used in Section \ref{subsubsec: proof hd upper bound} as an important intermediate step to prove Theorem \ref{thm: high-dim upper bound}.  The proof follows the same idea as in the proof of Proposition \ref{prop: high-dim upper bound single-source}. We first present the necessary additional definitions and notations.

Recall that in Proposition \ref{prop: high-dim upper bound}, $\mA'$ can be any subset of $\mA$. For convenience, we denote $\alpha = \frac{10}{11}L^{-1}$, $\gamma = \frac{10}{9}L$, $N = \sum_{k \in \{0\} \cup \mA'}n_k$, and we refer to the target data set when we say `source $0$'. For any $k \in \{0\}\cup \mA'$, $t \in [T]$, and $\bbeta \in \mathbb{R}^d$, define the empirical risk function of source $k$ and the combined empirical risk of all sources in $\{0\} \cup \mA'$ at iteration $t$ as
\begin{align*}
    \mL_{n_k}^{(k)t}(\bbeta) &= \frac{1}{2n_k}\sum_{i=1+(t-1)(n_k/T)}^{t(n_k/T)}[Y^{(k)}_i-(X^{(k)}_i)^\top\bbeta]^2,\\
    \mL_N^t(\bbeta) &= \sum_{k \in \{0\}\cup \mA}\frac{n_k}{N}\mL^{(k)}_{n_k}(\bbeta).
\end{align*}
Define $\bm{X}^{(k)t} \in \mathbb{R}^{(n_k/T) \times d}$ as the predictor data matrix in iteration $t$, where each row is an observation in batch $t$. $Y^{(k)t}\in \mathbb{R}^{n_k/T}$ is the response vector in iteration $t$.

Recall that the step length of gradient descent equals $\rho = \frac{9\xi}{10L} = \xi/\gamma$. Define the gradient of $\mL_N^t$ at $\bbeta^t$ as
\begin{align*}
    \bmg^t &= \nabla \mL_N^t(\bbeta^t) \\
    &= \frac{1}{N}\sum_{k \in \{0\}\cup \mA'} (\bm{X}^{(k)t})^\top(\bm{X}^{(k)t}\bbeta^t - Y^{(k)t}) \\
    &=\frac{1}{N}\sum_{k \in \{0\}\cup \mA'}\sum_{i=1+(t-1)(n_k/T)}^{t(n_k/T)}\big[(X^{(k)}_i)^\top\bbeta^t - Y^{(k)}_i\big]X^{(k)}_i.
\end{align*}
and the sets
\begin{equation}\label{eq: def of It}
    I^t = S^{t+1}\cup S^t \cup S, \quad \mbox{where } S^t = \textup{supp}(\bbeta^t).
\end{equation}
Recall that in Algorithm \ref{algorithm:DPregression_high_dim_combined}, we define 
\begin{align}
    R^{(k)}_t &= 2\sqrt{\log(N/\eta)} \textup{PrivateVariance}\Big(\big\{(X^{(k)}_{i})^\top \bbeta^t - Y^{(k)}_i\big\}_{i=1+(t-1)(n_k/T)}^{t(n_k/T)},\epsilon/2,\delta/2\Big), \nonumber\\
    R_t&= \sum_{k \in \{0\}\cup \mA'}R^{(k)}_t. \nonumber
\end{align}
Recall that in Algorithm \ref{algorithm:DPregression_high_dim_combined}, at iteration $t$, for the gradient of source $k \in \{0\} \cup \mA'$, we generate the Gaussian noise $w^{(k)}_t \sim \mathcal{N}\Big(0, \frac{8\log(2.5/\delta)R^2(R^{(k)}_t)^2}{(n_k/T)^2(\epsilon/2)^2}\bm{I}_d\Big)$. Denote the weighted sum of the noises for the combined gradient from sources in $\{0\} \cup \mA'$ as
\begin{equation*}
    \bw^t = \sum_{k \in \{0\}\cup \mA'}\frac{n_k}{N}w^{(k)}_t \sim \mathcal{N}\Big(0, \frac{8\log(2.5/\delta)T^2R^2}{N^2(\epsilon/2)^2}\sum_{k \in \{0\}\cup \mA'}(R^{(k)}_t)^2\bm{I}_d\Big).
\end{equation*}
For simplicity, for any vector $v \in \mathbb{R}^d$ and $s' \in \mathbb{N}_+$, we write hard-thresholding operator $\textup{Hard-thresholding}(v, s')$ used in Algorithm \ref{algorithm:DPregression_high_dim_combined}  as $H_{s'}(v)$.  Define the sample covariance matrices of source $k \in \{0\}\cup \mA'$ and the combined covariance matrix at iteration $t$ as
\begin{align*}
    \widehat{\bSigma}^{(k)t} &= \frac{1}{n_k}(\bm{X}^{(k)t})^\top \bm{X}^{(k)t},\\
    \widehat{\bSigma}^{t} &= \sum_{k \in \{0\} \cup \mA'}\frac{n_k}{N}\widehat{\bSigma}^{(k)t}.
\end{align*}

Next, we divide the formal proof of Proposition \ref{prop: high-dim upper bound} into a few parts. In part (\Rom{1}), we define some events and show that their intersections hold with high probability. In part (\Rom{2}), we make additional notes for the truncations in Algorithm \ref{algorithm:DPregression_high_dim_combined} and argue that they are not effective in the high-probability event defined in part (\Rom{1}). In part (\Rom{3}), we demonstrate that Algorithm \ref{algorithm:DPregression_high_dim_combined} satisfies the FDP notion in Definition \ref{def:interactive-FDP}. In the last part (\Rom{4}), we provide a detailed proof of the estimation error upper bound in Proposition \ref{prop: high-dim upper bound}.

\noindent\textbf{(\Rom{1}) Conditioning on some events:}
Define events $\mathcal{E}_{1}$, $\mathcal{E}_{2}$, and $\mathcal{E}_{3}$ as follows: 
\begin{align*}
    \mathcal{E}_{1} &= \Big\{\alpha \leq \lambdamin\big(\widehat{\bSigma}^{(k)[t]}_{S', S'}\big) \leq \lambdamax\big(\widehat{\bSigma}^{(k)[t]}_{S', S'}\big) \leq \gamma, \forall k \in \mA', \forall S' \subseteq [d] \text{ with } |S'| \leq s', \forall t \in [T]\Big\} \\
    &\quad \bigcap \Big\{\alpha \leq \lambdamin\big(\widehat{\bSigma}^{[t]}_{S', S'}\big) \leq \lambdamax\big(\widehat{\bSigma}^{[t]}_{S', S'}\big) \leq \gamma, \forall S' \subseteq [d] \text{ with } |S'| \leq s', \forall t \in [T]\Big\} \\
    &\quad \bigcap \Bigg\{\twonorm{\hSigma^{[t]}_{S',S'} - \bSigma_{S',S'}} \leq C\sqrt{\frac{s'\log (d/\eta)}{N/T}}, \forall S' \subseteq [d] \text{ with } |S'| \leq s', \forall t \in [T]\Bigg\}, \\
    \mathcal{E}_{2} &= \bigg\{\twonorm{X^{(k)}_i} \leq C\sqrt{d\log(N/\eta)}, \forall i = [1+(t-1)(n_k/T)]:[t(n_k/T)], \forall k \in \mA', \forall t \in [T]\bigg\}, \\
    \mathcal{E}_{3} &= \bigg\{\infnorm{\bw^t}^2 \leq C\frac{d\log(1/\delta)\log(N/\eta)\log(dT/\eta)}{(N/T)^2\epsilon^2}\cdot (R_t)^2, \forall t \in [T]\bigg\},
\end{align*}
where $C > 0$ is a sufficiently large number such that 
\begin{equation}\label{eq: event 123}
    \tp(\mathcal{E}_1) \geq 1-\frac{\eta}{12}, \quad \tp(\mathcal{E}_2) \geq 1-\frac{\eta}{12},  \quad \tp(\mathcal{E}_3) \geq 1- \frac{\eta}{12}.
\end{equation}
Define event $\mathcal{E}_4$ and $\mathcal{E}_5$ as
\begin{equation*}
    \mathcal{E}_4 = \Big\{R^{(k)}_t \geq |(\Xk{k}_i)^\top\bbeta^t - \Yk{k}_i|, \forall i = [1+(t-1)(n_k/T)]:[t(n_k/T)], \forall k \in \mA', \forall t \in [T]\Big\},
\end{equation*}
where $R^{(k)}_t = 2\sqrt{\log(4N/\eta)}\textup{PrivateVariance}\big(\big\{(\Xk{k}_i)^\top\bbeta^t - \Yk{k}_i\big\}_{i=1+(t-1)(n_k/T)}^{t(n_k/T)}, \epsilon/2, \delta/2\big)$, and 
\begin{align*}
    \mathcal{E}_5 = \bigg\{\sqrt{\frac{3}{4}}({\sigma_k + \|\bbeta^t-\bbeta^*\|_{\bSigma^{(k)}}}) &\leq \textup{PrivateVariance}\big(\big\{(\Xk{k}_i)^\top\bbeta^t - \Yk{k}_i\big\}_{i=1+(t-1)(n_k/T)}^{t(n_k/T)}, \epsilon/2, \delta/2\big)\\
    &\leq \sqrt\frac{5}{2}({\sigma_k + \|\bbeta^t-\bbeta^*\|_{\bSigma^{(k)}}}), \forall k \in \mA', \forall t \in [T]\bigg\}.
\end{align*}
These two events can be controlled in the same way as in the proof of \Cref{prop: high-dim upper bound single-source} with $\mathbb{P}(\mathcal{E}_4 ) \geq 1-\eta/2$ and $\mathbb{P}(\mathcal{E}_5 ) \geq 1-\eta/4$, under the condition 
\[
  n_k \gtrsim \frac{T\log(T/(\delta\eta))\log(T\log[T/(\eta\delta)](\eta\epsilon)^{-1})}{\epsilon}.
\]

Together, we have
\begin{equation*}
    \tp(\cap_{i=1}^5 \mathcal{E}_i) \geq 1- \frac{\eta}{12} - \frac{\eta}{12} - \frac{\eta}{12}-\frac{\eta}{2} -\frac{\eta}{4} \geq 1-\eta.
\end{equation*}

\medskip
\noindent\textbf{(\Rom{2}) Truncation in Algorithm \ref{algorithm:DPregression_high_dim}:}
In events $\mathcal{E}_2\cap \mathcal{E}_4$, the truncations $\prod_R$ and $\prod_{R_t^{(k)}}$ in Algorithm \ref{algorithm:DPregression_high_dim} are `not effective', by which we mean $\prod_R(X^{(k)}_i) = X^{(k)}_i$, for all $i \in [n_k]$ and $k \in \{0\}\cup \mA'$, and $\prod_{R^{(k)}_t}\big((\Xk{k}_i)^\top\bbeta^t - \Yk{k}_i\big) = (\Xk{k}_i)^\top\bbeta^t - \Yk{k}_i$, for all $i = [1+(t-1)(n_k/T)]:[t(n_k/T)]$, $k \in \{0\}\cup \mA'$, and $t \in [T]$. 

In the following analysis, we condition on the event $\cap_{i=1}^5\mathcal{E}_i$ and show the upper bound holds given that $\cap_{i=1}^5\mathcal{E}_i$ holds.

\medskip

\noindent\textbf{(\Rom{3}) Privacy:} 
First, $R^{(k)}_t$ is $(\epsilon, \delta)$-central DP by \Cref{lemma:privatevarianceSG}. By the Gaussian noise added in Step 5 on the empirical gradient $\frac{1}{n_k/T}\sum_{i=1}^{n_k/T}\prod_{R_t^{(k)}}\big((X^{(k)}_i)^\top\bbeta^t - Y^{(k)}_i\big)\prod_R(X^{(k)}_i)$ together with the composition theorem (Theorem 3.16 in \cite{dwork2014algorithmic}), the noisy gradient $\frac{1}{n_k/T}\sum_{i=1}^{n_k/T}\prod_{R_t^{(k)}}\big((X^{(k)}_i)^\top\bbeta^t - Y^{(k)}_i\big)\prod_R(X^{(k)}_i) + w^{(k)}_t$ satisfies \eqref{eq:composition-eachstep}. Therefore by Definition \ref{def:interactive-FDP}, Algorithm \ref{algorithm:DPregression_high_dim_combined} is $(\epsilon, \delta)$-FDP.

\medskip   
\noindent\textbf{(\Rom{4}) Derivation of the estimation error bound:}  

We first summarise the key idea of the proof. First, we upper  bound $\mL_N^t(\bbeta^{t+1}) - \mL_N^t(\bbetak{0})$ by $\mL_N^t(\bbeta^{t}) - \mL_N^t(\bbetak{0})$, $\bbeta^t - \bbetak{0}$, and $\infnorm{\bw^t}$. Second, we replace $\mL_N^t(\bbeta^{t+1}) - \mL_N^t(\bbetak{0})$ with a lower bound involving $\twonorm{\bbeta^{t+1} - \bbetak{0}}$ and replace $\mL_N^t(\bbeta^{t}) - \mL_N^t(\bbetak{0})$ with an upper bound involving $\twonorm{\bbeta^{t} - \bbetak{0}}$. Finally, we simplify the result to obtain an induction relationship between $\|\bbeta^{t+1} - \bbetak{0}\|_{\bSigma}$ and $\|\bbeta^{t} - \bbetak{0}\|_{\bSigma}$, translate the $\bSigma$-norm to $\ell_2$-norm, then complete the proof.

\noindent\textbf{Step 1: } Upper bound $\mL_N^t(\bbeta^{t+1}) - \mL_N^t(\bbetak{0})$ by $\mL_N^t(\bbeta^{t}) - \mL_N^t(\bbetak{0})$, $\bbeta^t - \bbetak{0}$, and  $\infnorm{\bw^t}$.
We start with applying Taylor's expansion and it holds that	
	\begin{align}
		\mL_N^t(\bbeta^{t+1}) - \mL_N^t(\bbeta^{t}) &\leq \<\bbeta^{t+1}-\bbeta^t, \bmg^t\> + \frac{1}{2}\gamma\twonorm{\bbeta^{t+1}-\bbeta^t}^2 \nonumber \\
		&= \frac{1}{2}\gamma\twonorm{\bbeta^{t+1}_{I^t} - \bbeta^t_{I^t} + \xi/\gamma \cdot \bmg_{I^t}^t}^2 - \frac{\xi^2}{2\gamma}\twonorm{\bmg_{I^t}^t}^2 + (1-\xi)\<\bbeta^{t+1}-\bbeta^t,\bmg^t\>, \label{eq: proof thm high-dim eq 1}
	\end{align}
    where the inequality is due to Lemma \ref{lem: fact 8.1}.(\rom{2}), the identity holds following the definition of $I_t$ in \eqref{eq: def of It} and $\xi \in \mathbb{R}$.  
    
    We can further upper bound the last term in \eqref{eq: proof thm high-dim eq 1}, that
	where the last term can be bounded as 
	\begin{align}
		\<\bbeta^{t+1}-\bbeta^t,\bmg^t\>  &=  \<\bbeta^{t+1}_{S^{t+1}} - \bbeta^{t}_{S^{t+1}}, \bmg^t_{S^{t+1}}\> - \<\bbeta^{t}_{S^t\backslash S^{t+1}}, \bmg^{t}_{S^t\backslash S^{t+1}}\> \nonumber \\
		&= \<-\rho\bmg^{t}_{S^{t+1}} - \rho\bw^t_{S^{t+1}}, \bmg^t_{S^{t+1}}\> - \<\bbeta^{t}_{S^t\backslash S^{t+1}}, \bmg^t_{S^t\backslash S^{t+1}}\> \label{eq: hd upper 1} \\
		&\leq -\frac{\xi}{\gamma}\twonorm{\bmg^{t}_{S^{t+1}}}^2 + 10\rho\twonorm{\bw^t_{S^{t+1}}}^2  + \frac{\rho}{40}\twonorm{\bmg^{t}_{S^{t+1}}}^2 - \<\bbeta^t_{S^t\backslash S^{t+1}}, \bmg^t_{S^t\backslash S^{t+1}}\>\label{eq: hd upper 2}\\
		&\leq -\frac{\xi}{\gamma}\twonorm{\bmg^{t}_{S^{t+1}}}^2 + 10s\rho\infnorm{\bw^t}^2  + \frac{\rho}{40}\twonorm{\bmg^{t}_{S^{t+1}}}^2 - \<\bbeta^t_{S^t\backslash S^{t+1}}, \bmg^t_{S^t\backslash S^{t+1}}\>  \label{eq: hd upper 3 new}.
	\end{align}
	Note that \eqref{eq: hd upper 1} holds because $\bbeta^{t+1} = H_{s'}(\bbeta^{t+0.5})$ and $\bbeta^{t+0.5}_{S^{t+1}} = \bbeta^t_{S^{t+1}} - \rho\bmg^t_{S^{t+1}} - \bw^t_{S^{t+1}}$, and \eqref{eq: hd upper 2} holds since $\<-\rho\bmg^{t}_{S^{t+1}} - \rho\bw^t_{S^{t+1}}, \bmg^t_{S^{t+1}}\> = -\rho\twonorm{\bmg^t_{S^{t+1}}}^2 - \rho\<\bw^t_{S^{t+1}}, \bmg^t_{S^{t+1}}\> \leq -\rho\twonorm{\bmg^t_{S^{t+1}}}^2 + 10\rho\twonorm{\bw^t_{S^{t+1}}}^2  + \frac{\rho}{40}\twonorm{\bmg^{t}_{S^{t+1}}}^2$.
 
 We then upper bound the last term in \eqref{eq: hd upper 3 new} that
	\begin{align}
		&- \<\bbeta^t_{S^t\backslash S^{t+1}}, \bmg^t_{S^t\backslash S^{t+1}}\> \nonumber \\
		&\leq \frac{\gamma}{2\xi}\twonorma{\bbeta^t_{S^t\backslash S^{t+1}} - \frac{\xi}{\gamma}\bmg^t_{S^t\backslash S^{t+1}}}^2 - \frac{\xi}{2\gamma}\twonorm{\bmg^t_{S^t\backslash S^{t+1}}}^2 \nonumber \\
		&\leq \frac{\gamma}{2\xi}\left[\frac{1+c}{1-c}\twonorma{\frac{\xi}{\gamma}\bmg_{S^{t+1}\backslash S^t}^t}^2 + \frac{1}{1-c}\left(s'/c + (1+1/c)s'\right)\cdot \infnorm{\bw^t}^2\right] - \frac{\xi}{2\gamma}\twonorm{\bmg^t_{S^t\backslash S^{t+1}}}^2 \label{eq: hd upper 3}\\
		&= \frac{\xi}{2\gamma}\cdot \frac{1+c}{1-c}\twonorm{\bmg^t_{S^{t+1}\backslash S^t}}^2 + \frac{\gamma}{2\xi}\cdot \frac{1}{1-c}\cdot (1+2/c)s'\cdot \infnorm{\bw^t}^2 - \frac{\xi}{2\gamma}\twonorm{\bmg^t_{S^{t}\backslash S^{t+1}}}^2, \nonumber
	\end{align}
	where $c$ is a very small constant. 
 
 Eq.~\eqref{eq: hd upper 3} is an application of Lemma \ref{lem: modified 3.4} by letting $v = \bbeta^t - \xi/\gamma \cdot \bmg^t$, $u = \bbeta^{t+0.5} = \bbeta^t - \xi/\gamma \cdot \bmg^t - \bw^t$, $S^{t+1} = \text{supp}(\bbeta^{t+1}) =$ the set of top $s$ entries of $|u|$ because $\bbeta^{t+1} = H_{s'}(\bbeta^{t+0.5})$, $S_1 = S^t \backslash S^{t+1}$, $S_2 = S^{t+1}\backslash S^t \subseteq (S^t)^c$, and $|S_1| = |S_2|$. 
 
 Therefore, 
	\begin{align*}
		\<\bbeta^{t+1}-\bbeta^t,\bmg^t\> &\leq -\frac{\xi}{\gamma}\twonorm{\bmg^{t}_{S^{t+1}}}^2 + 10s'\rho\infnorm{\bw^t}^2  + \frac{\rho}{40}\twonorm{\bmg^{t}_{S^{t+1}}}^2 + \frac{\xi}{2\gamma}\cdot \frac{1+c}{1-c}\twonorm{\bmg^t_{S^{t+1}\backslash S^t}}^2\\
		&\quad + \frac{\gamma}{2\xi}\cdot \frac{1}{1-c}\cdot (1+2/c)s'\cdot \infnorm{\bw^t}^2 - \frac{\xi}{2\gamma}\twonorm{\bmg^t_{S^{t}\backslash S^{t+1}}}^2 \\
		&\leq -\frac{9\xi}{20\gamma}\twonorm{\bmg^{t}_{S^{t+1} \cup S^t}}^2 + Cs'\infnorm{\bw^t}^2,
	\end{align*}
	when $\frac{1+c}{1-c}\cdot \frac{1}{2} \leq \frac{21}{40}$. Going back to \eqref{eq: proof thm high-dim eq 1}, we have
	\begin{align*}
		\mL_N^t(\bbeta^{t+1}) - \mL_N^t(\bbeta^{t}) &\leq \frac{1}{2}\gamma\twonorm{\bbeta^{t+1}_{I^t} - \bbeta^t_{I^t} + \xi/\gamma\cdot \bmg_{I^t}^t}^2 - \frac{\xi^2}{2\gamma}\twonorm{\bmg_{I^t}^t}^2 +(1-\xi)\<\bbeta^{t+1}-\bbeta^t,\bmg^t\> \\
		&\leq \frac{1}{2}\gamma\twonorm{\bbeta^{t+1}_{I^t} - \bbeta^t_{I^t} + \xi/\gamma\cdot \bmg_{I^t}^t}^2 - \frac{\xi^2}{2\gamma}\twonorm{\bmg_{I^t \backslash (S^t\cup S)}^t}^2 - \frac{\xi^2}{2\gamma}\twonorm{\bmg_{S^t\cup S}^t}^2 \\
		&\quad - \frac{9\xi}{20\gamma}(1-\xi)\twonorm{\bmg^t_{S^{t+1} \cup S^t}}^2+ Cs'\infnorm{\bw^t}^2.
	\end{align*}
	Consider a set $S' \subseteq S^t\backslash S^{t+1}$ with $|S'| = |I^t\backslash (S^t \cup S)| = |S^{t+1}\backslash (S^t \cup S)|$. Applying Lemma \ref{lem: modified 3.4} by setting $v = \bbeta^t - \xi/\gamma \cdot \bmg^t$, $S^t = $ the set of top-$s'$ entries of $|v+\bw^t|$, $S_1 = S'$, $S_2 = S^{t+1}\backslash (S^t \cup S) $, and $|S_1| = |S_2|$, we have
	\begin{equation}\label{eq: hd upper 5}
		\twonorm{\bbeta^t_{S'} - \xi/\gamma \cdot \bmg^t_{S'}}^2 \leq \frac{1+c}{1-c}\frac{\xi^2}{\gamma^2} \twonorm{\bmg^t_{S^{t+1}\backslash (S^t\cup S)}}^2 + \frac{1}{1-c}(1+2/c)s'\cdot \infnorm{\bw^t}^2,
	\end{equation}
	which entails that
	\begin{equation*}
		-\frac{\xi^2}{\gamma}\twonorm{\bmg^t_{S^{t+1}\backslash (S^t\cup S)}}^2 \leq -\frac{1-c}{1+c}\gamma\twonorm{\bbeta^t_{S'} - \xi/\gamma\cdot \bmg^t_{S'}}^2 + \gamma\frac{1+2/c}{1+c}s'\infnorm{\bw^t}^2.
	\end{equation*}
	This leads to
	\begin{align}
		&\frac{1}{2}\gamma\twonorm{\bbeta^{t+1}_{I^t} - \bbeta^t_{I^t} + \xi/\gamma\cdot \bmg_{I^t}^t}^2 - \frac{\xi^2}{2\gamma}\twonorm{\bmg_{I^t \backslash (S^t\cup S)}^t}^2 \nonumber\\
		&= \frac{1}{2}\gamma \twonorm{(H_{s'}(\bbeta^t - \xi/\gamma\cdot \bmg^t))_{I^t} + \bw^t_{I^t} - \bbeta^t_{I^t} + \xi/\gamma\cdot \bmg_{I^t}^t}^2 - \frac{\xi^2}{2\gamma}\twonorm{\bmg_{I^t \backslash (S^t\cup S)}^t}^2 \nonumber\\
		&\leq C\gamma\twonorm{\bw_{I^t}^t}^2 + \frac{1+c}{1-c}\frac{\gamma}{2}\twonorm{(H_{s'}(\bbeta^t - \xi/\gamma \cdot \bmg^t))_{I^t} - (\bbeta^t - \xi/\gamma \cdot \bmg^t)_{I^t}}^2 - \frac{\gamma}{2}\cdot \frac{1-c}{1+c}\cdot \twonorm{\bbeta^t_{S'} - \xi/\gamma \cdot g^t_{S'}}^2 \nonumber\\
		&\quad + \frac{\gamma}{2}\cdot \frac{1+2/c}{1+c}s'\infnorm{\bw^t}^2 \nonumber\\
		&\leq C\gamma s'\infnorm{\bw^t}^2 + \frac{1+c}{1-c}\frac{\gamma}{2}\twonorm{(H_{s'}(\bbeta^t - \xi/\gamma \cdot \bmg^t))_{I^t} - (\bbeta^t - \xi/\gamma \cdot \bmg^t)_{I^t}}^2 \nonumber\\
		&\quad - \frac{\gamma}{2}\cdot \bigg(\frac{1-c}{1+c}\bigg)^2 \twonorm{\underbrace{[H_{s'}(\bbeta^t - \xi/\gamma \cdot \bmg^t)]_{S'}}_{ = -\bw_{S'}^t \text{ because } S' \subseteq S^t\backslash S^{t+1}} - (\bbeta^t - \xi/\gamma \cdot \bmg^t)_{S'})}^2 \nonumber\\
		&\leq C\gamma s'\infnorm{\bw^t}^2 + \frac{5\gamma}{9}\twonorm{(H_{s'}(\bbeta^t - \xi/\gamma \cdot \bmg^t))_{I^t\backslash S'} - (\bbeta^t - \xi/\gamma \cdot \bmg^t)_{I^t \backslash S'}}^2 \nonumber\\
        &\quad + \frac{(1+c)^3-(1-c)^3}{(1+c)^2(1-c)}\cdot \frac{\gamma}{2}\twonorm{(\bbeta^t - \xi/\gamma \cdot \bmg^t)_{S'}}^2 \nonumber\\
		&\leq C\gamma s'\infnorm{\bw^t}^2 + \frac{5\gamma}{9}\twonorm{(H_{s'}(\bbeta^t - \xi/\gamma \cdot \bmg^t))_{I^t\backslash S'} - (\bbeta^t - \xi/\gamma \cdot \bmg^t)_{I^t \backslash S'}}^2 \nonumber\\
        &\quad + \frac{c^3+3c}{(1+c)(1-c)^2}\cdot \frac{\xi^2}{\gamma}\cdot \twonorm{\bmg^t_{S^{t+1}\backslash (S^t \cup S)}}^2 \label{eq: hd upper 4}\\
		&\leq C s'\infnorm{\bw^t}^2 + \frac{5\gamma}{9}\twonorm{(H_{s'}(\bbeta^t - \xi/\gamma \cdot \bmg^t - \bw^t))_{I^t\backslash S'} - (\bbeta^t - \xi/\gamma \cdot \bmg^t - \bw^t)_{I^t \backslash S'}}^2 \nonumber\\
		&\quad + \frac{c^3+3c}{(1+c)(1-c)^2}\cdot \frac{\xi^2}{\gamma}\cdot \twonorm{\bmg^t_{S^{t+1}\backslash (S^t \cup S)}}^2, \nonumber
	\end{align}
	where we used \eqref{eq: hd upper 5} to obtain \eqref{eq: hd upper 4}. Note that $I^t \backslash S' \supseteq S^{t+1}$, hence we have $(H_{s'}(\bbeta^t - \xi/\gamma \cdot \bmg^t - \bw^t))_{I^t\backslash S'} = H_{s'}((\bbeta^t - \xi/\gamma \cdot \bmg^t - \bw^t)_{I^t\backslash S'})$. Applying Lemma \ref{lem: modified A.3} with $v = (\bbeta^t - \xi/\gamma \cdot \bmg^t - \bw^t)_{I^t\backslash S'}$, $\tilde{v} = \bbetak{0}_{I^t\backslash S'}$, $\zeronorm{\tilde{v}} \leq s$, and $s' \geq s$, we have
	\begin{align*}
		&\frac{5\gamma}{9}\twonorm{(H_{s'}(\bbeta^t - \xi/\gamma \cdot \bmg^t - \bw^t))_{I^t\backslash S'} - (\bbeta^t - \xi/\gamma \cdot \bmg^t - \bw^t)_{I^t \backslash S'}}^2 \\
		&\leq \frac{5\gamma}{9}\cdot \frac{|I^t\backslash S'|-s'}{|I^t\backslash S'|-s}\cdot \twonorm{\bbetak{0}_{I^t\backslash S'} - (\bbeta^t - \xi/\gamma \cdot \bmg^t - \bw^t)_{I^t \backslash S'}}^2 \\
		&\leq \frac{5\gamma}{9}\cdot \frac{s}{s'}\cdot \twonorma{\bbetak{0}_{I^t\backslash S'} - \bbeta^t_{I^t\backslash S'} + \frac{\xi}{\gamma}\bmg^t_{I^t\backslash S'}}^2 + C\gamma\cdot \frac{s}{s'}\cdot s'\infnorm{\bw^t}^2 \\
		&\leq \frac{5\gamma}{9}\cdot \frac{s}{s'}\cdot \twonorma{\bbetak{0}_{I^t\backslash S'} - \bbeta^t_{I^t\backslash S'} + \frac{\xi}{\gamma}\bmg^t_{I^t\backslash S'}}^2 + Cs\infnorm{\bw^t}^2,
	\end{align*}
	where the second inequality used the fact that $|I^t\backslash S'| \leq s'+s$. This holds because $I^t = S^t \cup S^{t+1} \cup S$, $S' \subseteq S^t \backslash S^{t+1} \subseteq I^t$, $|S'| = |I^t\backslash (S^t \cup S)| = |I^t| - |S^t \cup S|$, leading to $|I^t \backslash S'| \leq |S^t \cup S| \leq s' + s$.
	
	Therefore,
	\begin{align*}
		&\frac{1}{2}\gamma\twonorm{\bbeta^{t+1}_{I^t} - \bbeta^t_{I^t} + \xi/\gamma\cdot \bmg_{I^t}^t}^2 - \frac{\xi^2}{2\gamma}\twonorm{\bmg_{I^t \backslash (S^t\cup S)}^t}^2 \\
		&\leq C\gamma s'\infnorm{\bw^t}^2 + \frac{5\gamma}{9}\cdot \frac{s}{s'}\cdot \twonorma{\bbetak{0}_{I^t\backslash S'} - \bbeta^t_{I^t\backslash S'} + \frac{\xi}{\gamma}\bmg^t_{I^t\backslash S'}}^2 + \frac{c^3+3c}{(1+c)(1-c)^2}\cdot \frac{\xi^2}{\gamma}\cdot \twonorm{\bmg^t_{S^{t+1}\backslash (S^t \cup S)}}^2 \\
		&\leq C\gamma s'\infnorm{\bw^t}^2 + \frac{5\gamma}{9}\cdot \frac{s}{s'}\cdot \twonorma{\bbetak{0}_{I^t} - \bbeta^t_{I^t} + \frac{\xi}{\gamma}\bmg^t_{I^t}}^2 + \frac{c^3+3c}{(1+c)(1-c)^2}\cdot \frac{\xi^2}{\gamma}\cdot \twonorm{\bmg^t_{S^{t+1}\backslash (S^t \cup S)}}^2 \\
		&\leq C\gamma s'\infnorm{\bw^t}^2 + \frac{5}{9}\frac{s}{s'}\cdot \left(2\xi\<\bbetak{0}_{I^t}-\bbeta^t_{I^t}, \bmg_{I^t}^t\> + \gamma\twonorm{\bbetak{0}_{I^t} - \bbeta^t_{I^t}}^2 + \frac{\xi^2}{\gamma^2}\twonorm{\bmg^t_{I^t}}^2\right) \\
		&\quad + \frac{c^3+3c}{(1+c)(1-c)^2}\cdot \frac{\xi^2}{\gamma}\cdot \twonorm{\bmg^t_{S^{t+1}}}^2 \\
		&\leq C s'\infnorm{\bw^t}^2 + \frac{5}{9}\frac{s}{s'}\cdot \left(2\xi\mL_N^t(\bbetak{0}) - \xi\mL_N^t(\bbeta^t) + (\gamma - \xi\alpha)\twonorm{\bbetak{0} - \bbeta^t}^2 + \frac{\xi^2}{\gamma}\twonorm{\bmg^t_{I^t}}\right) \\
		&\quad + \frac{c^3+3c}{(1+c)(1-c)^2}\cdot \frac{\xi^2}{\gamma}\cdot \twonorm{\bmg^t_{S^{t+1}}}^2 \quad \quad (*).
	\end{align*}
	Hence,
	\begin{align}
		\mL_N^t(\bbeta^{t+1}) - \mL_N^t(\bbeta^t) &\leq \frac{1}{2}\gamma\twonorm{\bbeta^{t+1}_{I^t} - \bbeta^t_{I^t} + \xi/\gamma\cdot \bmg_{I^t}^t}^2 - \frac{\xi^2}{2\gamma}\twonorm{\bmg_{I^t \backslash (S^t\cup S)}^t}^2 - \frac{\xi^2}{2\gamma}\twonorm{\bmg_{S^t\cup S}^t}^2 \nonumber\\
		&\quad - \frac{9\xi}{20\gamma}(1-\xi)\twonorm{\bmg^t_{S^{t+1} \cup S^t}}^2+ Cs\infnorm{\bw^t}^2 \nonumber\\
		&\leq (*) - \frac{\xi^2}{2\gamma}\twonorm{\bmg^t_{S^t \cup S}}^2 - \frac{9\xi}{20\gamma}(1-\xi)\twonorm{\bmg^t_{S^{t+1} \cup S^t}}^2+ Cs\infnorm{\bw^t}^2 \nonumber\\
		&= \frac{10s}{9s'}\cdot \xi \cdot [\mL_N^t(\bbetak{0}) - \mL_N^t(\bbeta^t)] + \frac{s}{s'}\cdot \frac{5(\gamma - \xi\alpha)}{9}\twonorm{\bbetak{0} - \bbeta^t}^2 \nonumber\\
		&\quad + \frac{s}{s'}\cdot \frac{5\xi^2}{9\gamma} \cdot (\twonorm{\bmg^t_{S^t \cup S}}^2 + \twonorm{\bmg^t_{S^{t+1}\backslash (S^t \cup S)}}^2) \nonumber\\
		&\quad - \frac{\xi^2}{2\gamma}\twonorm{\bmg^t_{S^t \cup S}}^2 - \frac{9\xi}{20\gamma}(1-\xi)\twonorm{\bmg^t_{S^{t+1} \cup S^t}}^2 + Cs \infnorm{\bw^t}^2 \nonumber\\
		&= \frac{10s}{9s'}\cdot \xi \cdot [\mL_N^t(\bbetak{0}) - \mL_N^t(\bbeta^t)] + \frac{s}{s'}\cdot \frac{5(\gamma - \xi\alpha)}{9}\twonorm{\bbetak{0} - \bbeta^t}^2 \nonumber\\
		&\quad + \left[\frac{s}{s'}\cdot \frac{5\xi^2}{9\gamma} -  \frac{9\xi}{20\gamma}(1-\xi)\right]\twonorm{\bmg^t_{S^{t+1}\backslash (S^t \cup S)}}^2 \nonumber\\
		&\quad + \left(\frac{10s}{9s'}  - 1\right)\frac{\xi^2}{2\gamma}\cdot \twonorm{\bmg^t_{S^t \cup S}}^2 + Cs' \infnorm{\bw^t}^2 \nonumber\\
		&\leq \frac{10s}{9s'}\cdot \xi \cdot [\mL_N^t(\bbetak{0}) - \mL_N^t(\bbeta^t)] + \frac{s}{s'}\cdot \frac{5(\gamma - \xi\alpha)}{9}\twonorm{\bbetak{0} - \bbeta^t}^2 \nonumber\\
		&\quad -\frac{9s'-10s}{9s'}\cdot \frac{\xi^2}{2\gamma}\cdot \twonorm{\bmg^t_{S^t \cup S}}^2 + Cs' \infnorm{\bw^t}^2, \label{eq: hd upper 6}
	\end{align}
	where \eqref{eq: hd upper 6} holds since $\frac{s}{s'}\cdot \frac{5\xi^2}{9\gamma} -  \frac{9\xi}{20\gamma}(1-\xi) = \frac{\xi}{\gamma}[\frac{s}{s'}\xi - \frac{81}{100}(1-\xi)] \leq 0$ due to \eqref{eq: high-dim upper bound}. On the other hand, note that
	\begin{align}
		\mL_N^t(\bbeta^t) - \mL_N^t(\bbetak{0}) &\leq \<\bmg^t, \bbeta^t - \bbetak{0}\> - \frac{\alpha}{2}\twonorm{\bbetak{0} - \bbeta^t}^2 \nonumber\\
		&\leq \twonorm{\bmg^t_{S^t\cup S}}\cdot \twonorm{\bbeta^t - \bbetak{0}} - \frac{\alpha}{2}\twonorm{\bbetak{0}-\bbeta^t}^2,\label{eq: hd upper 7}
	\end{align}
	And
	\begin{align*}
		\twonorm{\bmg^t_{S^t \cup S}}^2 - \frac{1}{4}\alpha^2\twonorm{\bbetak{0}-\bbeta^t}^2 
		&= \left(\twonorm{\bmg^t_{S^t \cup S}} + \frac{\alpha}{2}\twonorm{\bbetak{0}-\bbeta^t}\right)\left(\twonorm{\bmg^t_{S^t \cup S}} - \frac{\alpha}{2}\twonorm{\bbetak{0}-\bbeta^t}\right) \\
		&\geq \frac{\mL_N^t(\bbeta^t) - \mL_N^t(\bbetak{0})}{\twonorm{\bbetak{0}-\bbeta^t}}\cdot \left(\twonorm{\bmg^t_{S^t \cup S}} + \frac{\alpha}{2}\twonorm{\bbetak{0}-\bbeta^t}\right)\\
		&\geq \frac{\alpha}{2}[\mL_N^t(\bbeta^t) - \mL_N^t(\bbetaks{0})],
	\end{align*}
	which implies
	\begin{equation}\label{eq: hd upper 8}
		\twonorm{\bmg^t_{S^t \cup S}}^2 \geq \frac{1}{4}\alpha^2\twonorm{\bbetak{0}-\bbeta^t}^2 + \frac{\alpha}{2}[\mL_N^t(\bbeta^t) - \mL_N^t(\bbetak{0})].
	\end{equation}
	By adding $\mL_N^t(\bbeta^t) - \mL_N^t(\bbetak{0})$ on both sides of \eqref{eq: hd upper 6}, together with \eqref{eq: hd upper 8}, we obtain
	\begin{align}
		\mL_N^t(\bbeta^{t+1}) - \mL_N^t(\bbetak{0}) &\leq \left(1-\frac{10s}{9s'}\cdot \xi\right) \cdot [\mL_N^t(\bbeta^t)-\mL_N^t(\bbetak{0})] + \frac{s}{s'}\cdot \frac{5(\gamma - \xi\alpha)}{9}\twonorm{\bbetak{0} - \bbeta^t}^2 \nonumber\\
		&\quad -\frac{9s'-10s}{9s'}\cdot \frac{\xi^2}{2\gamma}\cdot \twonorm{\bmg^t_{S^t \cup S}}^2 + Cs' \infnorm{\bw^t}^2 \nonumber\\
		&\leq \left(1-\frac{10s}{9s'}\xi - \frac{9s'-10s}{9s'}\cdot \frac{\xi^2}{4\gamma}\alpha\right)\cdot [\mL_N^t(\bbeta^t)-\mL_N^t(\bbetak{0})] \nonumber\\
		&\quad + \left(\frac{s}{s'}\cdot \frac{5(\gamma - \xi\alpha)}{9} - \frac{9s'-10s}{9s'}\cdot \frac{\xi^2}{8\gamma}\alpha^2\right)\twonorm{\bbeta^t - \bbetak{0}}^2 + Cs' \infnorm{\bw^t}^2. \label{eq: hd upper 9}
	\end{align}
    \noindent\textbf{Step 2: } Replace $\mL_N^t(\bbeta^{t+1}) - \mL_N^t(\bbetak{0})$ with a lower bound involving $\|\bbeta^{t+1} - \bbetak{0}\|_{\bSigma}$ and replace $\mL_N^t(\bbeta^{t}) - \mL_N^t(\bbetak{0})$ with an upper bound involving $\|\bbeta^{t} - \bbetak{0}\|_{\bSigma}$.
    
	Note that
	\begin{align*}
		\mL_N^t(\bbeta^t)-\mL_N^t(\bbetak{0}) &= \frac{1}{2(N/T)}\sum_{k \in \{0\}\cup \mA'}\twonorm{Y^{(k)t} - \bm{X}^{(k)t}\bbeta^t}^2 - \frac{1}{2(N/T)}\sum_{k \in \{0\}\cup \mA'}\twonorm{Y^{(k)t} - \bm{X}^{(k)t}\bbetak{0}}^2 \\
		&= \frac{1}{2(N/T)}\sum_{k \in \{0\}\cup \mA'}\twonorm{\bm{X}^{(k)t}(\bbetak{k}-\bbetak{0}) + \bm{X}^{(k)t}(\bbetak{0}-\bbeta^t) + \epsilon^{(k)t}}^2\\
		&\quad - \frac{1}{2(N/T)}\sum_{k \in \{0\}\cup \mA'}\twonorm{\bm{X}^{(k)t}(\bbetak{k}-\bbetak{0})  + \epsilon^{(k)t}}^2 \\
		&= (\bbeta^t-\bbetak{0})^\top\hSigma^{t}(\bbeta^t-\bbetak{0}) + \frac{1}{N/T}(\bbetak{0}-\bbeta^t)^\top\sum_{k \in \{0\}\cup \mA'}(\bm{X}^{(k)t})^\top\epsilon^{(k)t} \\
		&\quad + \sum_{k \in \{0\}\cup \mA'}\frac{n_k}{N}(\bbetak{k}-\bbetak{0})^\top\hSigma^{(k)t}(\bbetak{0}-\bbeta^t).
	\end{align*}
 
 Note that
	\begin{align*}
		\norma{(\bbeta^t-\bbetak{0})^\top\hSigma^{t}(\bbeta^t-\bbetak{0}) - (\bbeta^t-\bbetak{0})^\top\bSigma(\bbeta^t-\bbetak{0})}
		&\lesssim \twonorm{\bbeta^t-\bbetak{0}}^2\cdot \sqrt{\frac{s'\log (d/\eta)}{N/T}}, \\
		\twonorma{\sum_{k \in \{0\}\cup \mA'}(\bm{X}^{(k)t}_{:, S^t\cup S})^\top\epsilon^{(k)t}} \leq \sqrt{s'}\infnorma{\sum_{k \in \{0\}\cup \mA'}(\bm{X}^{(k)t})^\top\epsilon^{(k)t}} &\lesssim \sqrt{\frac{s'\log(d/\eta)}{N/T}},
	\end{align*}
	\begin{equation*}
		\norma{\sum_{k \in \{0\}\cup \mA'}\frac{n_k}{N}(\bbetak{k}-\bbetak{0})^\top\hSigma^{(k)t}(\bbetak{0}-\bbeta^t)} \lesssim \twonorm{\bbetak{0}-\bbeta^t}\cdot h \lesssim \|\bbetak{0}-\bbeta^t\|_{\bSigma}\cdot h.
	\end{equation*}
	Therefore,
	\begin{equation}\label{eq: Ln bound 1}
		\mL_N^{t}(\bbeta^t)-\mL_N^{t}(\bbetak{0}) \leq \left[\frac{1}{2}+C\gamma \sqrt{\frac{s'\log(d/\eta)}{N/T}} + c\right]\|\bbeta^t-\bbetak{0}\|_{\bSigma}^2 + \frac{s'\log(d/\eta)}{N/T} + h^2.
	\end{equation}
	Similarly,
	\begin{equation}\label{eq: Ln bound 2}
		\mL_N^{t}(\bbeta^{t+1}) - \mL_N^{t}(\bbetak{0}) \geq \left[\frac{1}{2}-C\gamma \sqrt{\frac{s'\log(d/\eta)}{N/T}} - c\right]\|\bbeta^{t+1}-\bbetak{0}\|_{\bSigma}^2 - \frac{s'\log(d/\eta)}{N/T} - h^2.
	\end{equation}

    \noindent\textbf{Step 3: } Obtain an induction relationship between $\|\bbeta^{t+1} - \bbetak{0}\|_{\bSigma}$ and $\|\bbeta^{t} - \bbetak{0}\|_{\bSigma}$, translate the $\bSigma$-norm to $\ell_2$-norm, then complete the proof.
    
	Plugging \eqref{eq: Ln bound 1} and \eqref{eq: Ln bound 2} back in \eqref{eq: hd upper 9}, we get
	\begin{align*}
		&\left[\frac{1}{2}-C\gamma \sqrt{\frac{s'\log(d/\eta)}{N/T}} - c\right]\|\bbeta^{t+1}-\bbetak{0}\|_{\bSigma}^2 \\
		&\leq \left(1-\frac{10s}{9s'}\xi - \frac{9s'-10s}{9s'}\cdot \frac{\xi^2}{4\gamma}\alpha\right)\cdot\left[\frac{1}{2}+C\gamma \sqrt{\frac{s'\log(d/\eta)}{N/T}} + c\right]\|\bbeta^t-\bbetak{0}\|_{\bSigma}^2 \\
		&\quad + \left(\frac{s}{s'}\cdot \frac{5(\gamma - \xi\alpha)}{9\alpha} - \frac{9s'-10s}{9s'}\cdot \frac{\xi^2}{8\gamma}\alpha\right)\|\bbeta^t - \bbetak{0}\|_{\bSigma}^2 + Cs' \infnorm{\bw^t}^2 + \frac{s'\log(d/\eta)}{N/T} + Ch^2,
	\end{align*}
	which implies that
	\begin{align*}
		\|\bbeta^{t+1}-\bbetak{0}\|_{\bSigma}^2 &\leq \left(1-\frac{20s}{9s'}\xi -\frac{9s'-10s}{9s'}\cdot \frac{\xi^2\alpha}{2\gamma} + \frac{10s}{9s'}\cdot \frac{\gamma}{\alpha}+C'\sqrt{\frac{s'\log(d/\eta)}{N/T}}+C'c\right)\|\bbeta^t-\bbetak{0}\|_{\bSigma}^2 \\
		&\quad + Cs' \infnorm{\bw^t}^2 + C\frac{s'\log(d/\eta)}{N/T} + Ch^2 \\
		&\leq \left(1 -\frac{2s}{s'}\xi -\frac{9s'-10s}{9s'}\cdot \frac{\xi^2\alpha}{2\gamma} + \frac{10s}{9s'}\cdot \frac{\gamma}{\alpha}\right)\|\bbeta^t-\bbetak{0}\|_{\bSigma}^2 + Cs' \infnorm{\bw^t}^2 \\
        &\quad + C\frac{s'\log(d/\eta)}{N/T} + Ch^2.
	\end{align*}
	By induction, we have
	\begin{align*}
		&\|\bbeta^{T}-\bbetak{0}\|_{\bSigma}^2 \\
        &\leq \left(1 -\frac{2s}{s'}\xi-\frac{9s'-10s}{9s'}\cdot \frac{\xi^2\alpha}{2\gamma} + \frac{10s}{9s'}\cdot \frac{\gamma}{\alpha}\right)^T\|\bbeta^0-\bbetak{0}\|_{\bSigma}^2 \\
		&\quad + Cs' \sum_{t=0}^{T-1}\left(1 -\frac{2s}{s'}\xi-\frac{9s'-10s}{9s'}\cdot \frac{\xi^2\alpha}{2\gamma} + \frac{10s}{9s'}\cdot \frac{\gamma}{\alpha}\right)^{T-t-1}\infnorm{\bw^t}^2  + C\frac{s'\log(d/\eta)}{N/T} + Ch^2 \\
		&\leq \left(1 -\frac{2s}{s'}\xi-\frac{9s'-10s}{9s'}\cdot \frac{\xi^2\alpha}{2\gamma} + \frac{10s}{9s'}\cdot \frac{\gamma}{\alpha}\right)^T\|\bbeta^0-\bbetak{0}\|_{\bSigma}^2 \\
		&\quad +C\frac{ds'\log(1/\delta)\log(N/\eta)\log(dT/\eta)}{(N/T)^2\epsilon^2}\sum_{t=0}^{T-1}\left(1 -\frac{2s}{s'}\xi-\frac{9s'-10s}{9s'}\cdot \frac{\xi^2\alpha}{2\gamma} + \frac{10s}{9s'}\cdot \frac{\gamma}{\alpha}\right)^{T-t-1}(R_t)^2\\
		&\quad + C\frac{s'\log(d/\eta)}{N/T} + Ch^2 \\
		&\leq \left(1 -\frac{2s}{s'}\xi-\frac{9s'-10s}{9s'}\cdot \frac{\xi^2\alpha}{2\gamma} + \frac{10s}{9s'}\cdot \frac{\gamma}{\alpha}\right)^T\|\bbeta^0-\bbetak{0}\|_{\bSigma}^2 \\
		&\quad + C\frac{Kds'\log(1/\delta)\log^2(N/\eta)\log(dT/\eta)}{(N/T)^2\epsilon^2} (1\vee \|\bbeta^0-\bbetak{0}\|_{\bSigma}^2) + C\frac{s'\log(d/\eta)}{N/T} + Ch^2,
	\end{align*}
	where the last inequality comes from the definition $R_t= \sum_{k \in \{0\}\cup \mA}R^{(k)}_t$ and the definition of $R^{(k)}_t$ in Algorithm \ref{algorithm:DPregression_high_dim_combined}. Therefore, conditioned on $\cap_{i=1}^5\mathcal{E}_i$, since $s' \lesssim s$, we have
	\begin{align*}
		\twonorm{\bbeta^{T}-\bbetak{0}} &\leq \frac{\gamma}{\alpha}\left(1 -\frac{s'-s}{s'}\cdot \frac{\xi^2\alpha}{2\gamma} + \frac{s}{s'}\cdot \frac{\gamma}{\alpha}\right)^{T/2}\twonorm{\bbeta^{0}-\bbetak{0}} \\
		&\quad +C\frac{\sqrt{Kds}\log^{1/2}(1/\delta)\log(N/\eta)\log^{1/2}(dT/\eta)}{(N/T)\epsilon} (1\vee \twonorm{\bbeta^0-\bbetak{0}}) \\
		&\quad + C\sqrt{\frac{s\log(d/\eta)}{N/T}} + Ch,
	\end{align*}
	which completes the proof.

\medskip
\noindent\textbf{(\Rom{5}) Lemmas and their proofs:}

\begin{lemma}\label{lem: fact 8.1}
Suppose Assumption \ref{asp: x} holds. Recall the notations used in Section \ref{subsubsec: proof of prop high-dim upper bound}: $\alpha = \frac{10}{11}L^{-1}$ and $\gamma = \frac{10}{9}L$. The following results hold for any $\eta \in (0,1)$ and $s' \in \mathbb{N}_+$.
\begin{enumerate}[label=(\roman*), leftmargin=*]
    \item For any $k \in [K]$, when $n_k/T \gtrsim s'\log (d)+\log(T/\eta)$, with probability at least $1-\eta$, for any $S' \subseteq [d]$ with $|S'| \leq s'$, and $t \in [T]$, we have 
        \[
            \alpha \leq \lambdamin\big(\widehat{\bSigma}^{(k)t}_{S', S'}\big) \leq \lambdamax\big(\widehat{\bSigma}^{(k)t}_{S', S'}\big) \leq \gamma.
        \]
    \item For any $\mA' \subseteq [K]$, denote $N = \sum_{k \in \{0\}\cup \mA'}n_k$. When $N/T \gtrsim s'\log (d)+\log(T/\eta)$, with probability at least $1-\eta$, for any $S' \subseteq [d]$ with $|S'| \leq s'$ and $t \in [T]$, we have 
        \begin{align}
            \alpha \leq \lambdamin\big(\widehat{\bSigma}^{t}_{S', S'}\big) &\leq \lambdamax\big(\widehat{\bSigma}^{t}_{S', S'}\big) \leq \gamma, \label{eq: lem sigma eigenvalue eq 1}\\
  	    	\twonorm{\hSigma^{t}_{S',S'} - \bSigma_{S',S'}} &\lesssim \sqrt{\frac{s'\log (d/\eta)}{N/T}}. \label{eq: lem sigma eigenvalue eq 2}
  	    \end{align}
\end{enumerate}
\end{lemma}

\begin{lemma}\label{lem: norm}
	Under Assumption \ref{asp: x}, with probability at least $1-\eta$, we have 
	\begin{enumerate}[label=(\roman*), leftmargin=*]
		\item $\max_{k \in [K]}\max_{i=1:n_k}\twonorm{X^{(k)}_i} \lesssim \sqrt{d\log(N/\eta)}$;
		\item $\infnorm{\bw^t}^2 \lesssim \frac{d\log(1/\delta)\log^2(N/\eta)\log(d/\eta)}{(N/T)^2\epsilon^2} (R_t)^2$ for all $t \in [T]$,
	\end{enumerate}
 where $R_t= \sqrt{\sum_{k \in \{0\}\cup \mA}(R^{(k)}_t)^2}$.
\end{lemma}

\begin{proof}[Proof of Lemma \ref{lem: fact 8.1}]
(\rom{1}) By Theorem 6.5 in \cite{wainwright2019high}, for any $S' \subseteq [d]$ with $|S'| \leq s'$ and any $t \in [T]$, 
\[
    \twonorm{\hSigma^{(k)t}_{S',S'} - \bSigmak{k}_{S',S'}} \lesssim \sqrt{\frac{s'}{n_k/T}} + \sqrt{\frac{\log(1/\eta)}{n_k/T}},
\]
with probability at least $1-\eta$. Hence, by a union bound argument, 
\begin{equation}\label{eq: union proof lemma}
    \max_{|S'| \leq s'}\twonorm{\hSigma^{(k)t}_{S',S'} - \bSigmak{k}_{S',S'}} \lesssim \sqrt{\frac{s'}{n_k/T}} + \sqrt{\frac{\log(TN'/\eta)}{n_k/T}},
\end{equation}
with probability at least $1-\eta$, where $N' = \#\{S \subseteq [d]: |S'| \leq s'\} \lesssim s'd^{s'}$. This implies that
\begin{equation}\label{eq: lem hign-dim fact 8.1}
    \max_{|S'| \leq s'}\twonorm{\hSigma^{(k)t}_{S',S'} - \bSigmak{k}_{S',S'}} \lesssim \sqrt{\frac{s'\log d + \log(T/\eta)}{n_k/T}},
\end{equation}
with probability at least $1-\eta$. Therefore when $n_k/T \gtrsim L^2[s'\log(d) + \log(T/\eta)]$,
\begin{equation*}
    \alpha \leq L^{-1} - C\sqrt{\frac{s'\log(d) + \log(T/\eta)}{n_k/T}} \leq \lambdamin(\hSigma^{(k)t}_{S',S'}) \leq \lambdamax(\hSigma^{(k)t}_{S',S'}) \leq L+  C\sqrt{\frac{s'\log(d) + \log(T/\eta)}{n_k/T}} \leq \gamma
\end{equation*}
with probability at least $1-\eta$.

\noindent (\rom{2}) It is easy to see that \eqref{eq: lem sigma eigenvalue eq 1} is implied by Assumption \ref{asp: x} and \eqref{eq: lem sigma eigenvalue eq 2} when $N/T \gtrsim L^2[s'\log(d) + \log(T/\eta)]$. Inequality \eqref{eq: lem sigma eigenvalue eq 2} can be similarly derived as \eqref{eq: lem hign-dim fact 8.1}, and we omit the details.
\end{proof}

\begin{proof}[Proof of Lemma \ref{lem: norm}]

(\rom{1}) $\twonorm{\Xk{k}_i}^2$ is $Cd$-subExponential with mean $C'd$, then the bound is a direct consequence of the tail bound of subExponential variables \citep[see e.g.~Theorem 2.8.1 in][] {vershynin2018high}.

\noindent (\rom{2}) This is by the union bound and the tail of subGaussian variables \citep[see e.g.~Proposition 2.5 in][]{wainwright2019high}.

\end{proof}

\section{Technical Details of Section \ref{sec:m-estimation}}\label{sec:appendix-m-estimation}

\subsection{Examples satisfying \Cref{asp:m-estimation}}\label{sec:m-estimation examples}
We consider three examples that satisfy \Cref{asp:m-estimation}.

First, consider the linear regression setting with $W_i^{(k)} = (x_i^{(k)}, y_i^{(k)})$, $y_i^{(k)} = \langle x_i^{(k)}, \theta^{(k)} \rangle + \xi_i^{(k)}$, where $\xi_i^{(k)}$ is independent noise with variance $\sigma^2$. The Huber loss with parameter $\kappa > 0$ is $\ell_\kappa(u) = u^2/2 \cdot \mathbf{1}(|u| \leq \kappa) + (\kappa|u| - \kappa^2/2)\cdot \mathbf{1}(|u| > \kappa)$. As in \cite{avella2021privacy,xie2025online}, we adopt Mallow's weighted loss function with $w(x_i^{(k)}) = \min(1,R^2\|x_i^{(k)}\|^{-2})$ to down-weight large covariates and ensure the gradient is bounded,
\[
L_k(\theta) = \frac{1}{n}\sum_{i=1}^n \ell_\kappa(y_i^{(k)} - \langle x_i^{(k)}, \theta \rangle)\, w(x_i^{(k)}).
\]
We verify each part of \Cref{asp:m-estimation} in turn.

\textbf{Bounded gradient.} Since $|\ell_\kappa'(u)| \leq \kappa$ and $w(x_i^{(k)})\|x_i^{(k)}\|_2 \leq R$ by the definition of the Mallow's weight, we have
\[
\|\nabla L_k(\theta)\|_2 = \bigg\|\frac{1}{n}\sum_{i=1}^n \ell_\kappa'(y_i^{(k)} - \langle x_i^{(k)}, \theta\rangle)\, w(x_i^{(k)})\, x_i^{(k)}\bigg\|_2 \leq R\kappa,
\]
so the bounded gradient assumption holds with $B = R\kappa$.

\textbf{Smoothness.} The Hessian of $L_k$ is
\[
\nabla^2 L_k(\theta) = \frac{1}{n}\sum_{i=1}^n \ell_\kappa''(y_i^{(k)} - \langle x_i^{(k)}, \theta\rangle)\, w(x_i^{(k)})\, x_i^{(k)} x_i^{(k)\top}.
\]
Since $\ell_\kappa''(u) = \mathbf{1}(|u|\leq \kappa) \leq 1$ and $w(x_i^{(k)})\|x_i^{(k)}\|_2^2 \leq R^2$, we have $\nabla^2 L_k(\theta) \preceq \frac{1}{n}\sum_{i=1}^n w(x_i^{(k)})\, x_i^{(k)} x_i^{(k)\top} \preceq R^2 I$, which ensures that the loss function $L_k$ is smooth with $\tau_2 = R^2$.

\textbf{Local strong convexity.} In a neighbourhood $B_r(\theta^{(k)})$, for observations with $\|x_i^{(k)}\|_2 \leq R$, the residuals satisfy $y_i^{(k)} - \langle x_i^{(k)}, \theta\rangle = \xi_i^{(k)} + \langle x_i^{(k)}, \theta^{(k)} - \theta\rangle$, so $|y_i^{(k)} - \langle x_i^{(k)}, \theta\rangle| \leq |\xi_i^{(k)}| + Rr$. Choosing $\kappa \geq C(\sigma\sqrt{\log(n)} + Rr)$ for a sufficiently large $C$ ensures that, with high probability, all such observations satisfy $|y_i^{(k)} - \langle x_i^{(k)}, \theta\rangle| \leq \kappa$, so that $\ell_\kappa''(\cdot) = 1$ for every $i$ with $\|x_i^{(k)}\|_2 \leq R$. It follows that
\[
\nabla^2 L_k(\theta) \succeq \frac{1}{n}\sum_{i:\,\|x_i^{(k)}\|_2 \leq R} x_i^{(k)} x_i^{(k)\top} \succeq \tau_1 I,
\]
for some $\tau_1 > 0$, provided that $\lambda_{\min}\big(\frac{1}{n}\sum_{i:\,\|x_i^{(k)}\|_2 \leq R} x_i^{(k)} x_i^{(k)\top}\big) > 0$. Under sub-Gaussian covariates with $\mathrm{Var}(X) = \Sigma$ satisfying $\lambda_{\min}(\Sigma) > 0$, this holds with high probability when $R$ and $n$ are large.

Next, consider the logistic regression setting with $W_i^{(k)} = (x_i^{(k)}, y_i^{(k)})$,  
$y_i^{(k)} \sim \text{Bernoulli}(p_i^{(k)}(\theta^{(k)}))$, 
where $p_i^{(k)}(\theta^{(k)}) = \exp(\langle x_i^{(k)},\theta^{(k)} \rangle)/(1+\exp(\langle x_i^{(k)},\theta^{(k)} \rangle))$.
Suppose we adopt a Mallow's weighted version of the usual cross-entropy loss 
as considered in \cite{avella2023differentially} and \cite{xie2025online}: 
\[
L_k(\theta) = \frac{1}{n}\sum_{i=1}^n\Big(\log\Big(1+\exp(\langle x_i^{{(k)}},\theta\rangle)\Big) 
- y_i^{(k)}\langle x_i^{(k)},\theta \rangle\Big)w(x_i^{(k)}),
\]
where $w(x_i^{(k)}) = \min(1,R^2\|x_i^{(k)}\|^{-2})$ for some $R > 0$. We verify \Cref{asp:m-estimation} under a bounded parameter space assumption, which is essential for strong convexity to hold.

\textbf{Bounded gradient.} The gradient of the individual loss takes the form $\Psi(w_i^{(k)}, \theta) = ((1+\exp(-\langle x_i^{(k)},\theta \rangle))^{-1} - y_i^{(k)})\, w(x_i^{(k)})\, x_i^{(k)}$. Since $|(1+\exp(-\langle x_i^{(k)},\theta \rangle))^{-1} - y_i^{(k)}| \leq 1$ and $w(x_i^{(k)})\|x_i^{(k)}\|_2 \leq R$ by the definition of the Mallow's weight, we obtain
\[
\|\nabla L_k(\theta)\|_2 = \bigg\|\frac{1}{n}\sum_{i=1}^n((1+\exp(-\langle x_i^{(k)},\theta \rangle))^{-1}
-y_i^{(k)})x_i^{(k)}w(x_i^{(k)})\bigg\|_2 \leq R,
\]  
so the bounded gradient assumption holds with $B = R$.

\textbf{Smoothness.} The Hessian of $L_k$ is
\[
\nabla^2 L_k(\theta) = \frac{1}{n}\sum_{i=1}^n w(x_i^{(k)})\, \sigma(\langle x_i^{(k)},\theta\rangle)(1-\sigma(\langle x_i^{(k)},\theta\rangle))\, x_i^{(k)} x_i^{(k)\top},
\]
where $\sigma(u) = (1+e^{-u})^{-1}$ is the sigmoid function. Since $\sigma(u)(1-\sigma(u)) \leq 1/4$ for all $u$ and $w(x_i^{(k)})\|x_i^{(k)}\|_2^2 \leq R^2$, we have $\nabla^2 L_k(\theta) \preceq \frac{1}{n}\sum_{i=1}^nw(x_i^{(k)})x_i^{(k)}x_i^{(k)\top}/4 \preceq (R^2/4)I$, which ensures that the loss function $L_k$ is smooth with $\tau_2 = R^2/4$.

\textbf{Local strong convexity.} For observations with $\|x_i^{(k)}\|_2 \leq R$, the inner product satisfies $|\langle x_i^{(k)}, \theta\rangle| \leq \|x_i^{(k)}\|_2\|\theta\|_2 \leq RD:= M$, under a bounded parameter space assumption with diameter $D$. The sigmoid factor is then bounded below:
\[
\sigma(\langle x_i^{(k)},\theta\rangle)(1-\sigma(\langle x_i^{(k)},\theta\rangle)) \geq \frac{e^{M}}{(1+e^M)^2} > 0.
\]
It follows that
\[
\nabla^2 L_k(\theta) \succeq \frac{e^{M}}{(1+e^M)^2}\frac{1}{n}\sum_{i:\,\|x_i^{(k)}\|_2 \leq R} x_i^{(k)} x_i^{(k)\top} \succeq \tau_1 I,
\]
for some $\tau_1 > 0$, provided that $\lambda_{\min}\big(\frac{1}{n}\sum_{i:\,\|x_i^{(k)}\|_2 \leq R} x_i^{(k)} x_i^{(k)\top}\big) > 0$. Under sub-Gaussian covariates with $\mathrm{Var}(X) = \Sigma$ satisfying $\lambda_{\min}(\Sigma) > 0$, this holds with high probability when $R$ and $n$ are large.

Finally, consider the same low-dimensional linear regression model as in \eqref{eq:regression_model}, with $W_i^{(k)} = (x_i^{(k)}, y_i^{(k)})$ and $y_i^{(k)} = \langle x_i^{(k)}, \theta^{(k)} \rangle + \xi_i^{(k)}$, but now with the ordinary squared loss
\[
L_k(\theta) = \frac{1}{2n}\sum_{i=1}^n \big(y_i^{(k)} - \langle x_i^{(k)}, \theta \rangle\big)^2.
\]
In order to satisfy the bounded gradient condition, we assume the covariates are bounded, $\|x_i^{(k)}\|_2 \lesssim \sqrt{d}$ almost surely, the noise satisfies $|\xi_i^{(k)}| \leq 1$ almost surely, and $\sup_{\theta,x}|\langle x, \theta \rangle| \lesssim 1$. For strong convexity and smoothness, we assume the design covariance satisfies $c_- I_d \preceq \Sigma^{(k)} \preceq c_+ I_d$ with $c_- \asymp c_+\asymp 1$.

\textbf{Smoothness and strong convexity.} The Hessian $\nabla^2 L_k(\theta) = \frac{1}{n}\sum_{i=1}^n x_i^{(k)}x_i^{(k)\top}$ does not depend on $\theta$. Standard sample covariance concentration for sub-Gaussian designs \citep[see e.g.~Theorem 6.5 in][]{wainwright2019high} ensures that with high probability $\tau_1 I_d \preceq \nabla^2 L_k(\theta) \preceq \tau_2 I_d$, where $\tau_1\asymp\tau_2\asymp 1$.

\textbf{Bounded gradient.} The individual gradient is $\Psi(w_i^{(k)},\theta) = -(y_i^{(k)}-\langle x_i^{(k)},\theta\rangle)x_i^{(k)}$, which satisfies 
\[
\|\Psi(w_i^{(k)},\theta)\|_2 = |y_i^{(k)}-\langle x_i^{(k)},\theta\rangle|\cdot\|x_i^{(k)}\|_2 \lesssim \sqrt{d} ,
\]
and the bounded gradient condition holds with $B\asymp \sqrt{d}$.

Moreover, if we further assume that $\|x\|_{\psi_2}\lesssim 1$ and use this in the Step 2 of the proof of \Cref{prop:m-estimation-explicit} to control the non-private error term, we obtain that that 
\[
\|\theta_T - \theta^{(0)}\|_2 \lesssim h 
+ \bigg[
 \sqrt{\frac{d}{n(K+1)}}
\;+\;
 \frac{{d}}{n\epsilon\sqrt{K+1}}
\bigg]\mathrm{polylog}(n,K,\eta,\delta),
\]
which matches the minimax rate in \Cref{sec:lowd-regression} up to logarithmic factors, although under much stronger model assumptions and stronger minimal sample size conditions than \Cref{thm:lowd-regression-transfer}.

\subsection{Proofs of results in Section \ref{sec:m-estimation}}\label{sec:m-appendix}
We first present \Cref{algorithm:M-est_federated}, which is the analogue of \Cref{algorithm:DPregression_federated} for M-estimation. 

\begin{algorithm}[!ht]
	\begin{algorithmic}[1]
		\INPUT{Data $\{\{W_i^{(k)}\}_{i\in[n]}\}_{k\in [K]\cup\{0\}}$, 
		number of iteration $T$, step size $\rho$, 
		privacy parameters $\epsilon,\delta$, 
		initialisation $\theta_0$, 
		failure probability $\eta \in (0,1/2)$}
                \State Set batch size $b = \floor{n/T}$.
		\For{$t = 0, \ldots, T-1$} 
                 \For{$k \in \{0\} \cup [K]$} 
		\Comment{Each site generates the privatised information $Z^t_k$ locally}
            \State Set $\tau = bt$
			\State Sample $w_t^{(k)} \sim \mathcal{N}(0, I_d)$ and let $\phi = \sqrt{2\log(1.25/\delta)}2B/(b\epsilon)$
            \State Compute $Z^t_k = \frac{1}{K+1}\Big(1/b\sum_{i=1}^{b}\Psi(W_{\tau+i}^{(k)},\theta_t)+\phi w_t^{(k)}\Big)$
            \EndFor
            \State $\theta_{t+1} = \theta_t - \rho\sum_{k \in \{0\} \cup [K]}Z^t_k$; \Comment{A central server aggregates the privatised information}
		\EndFor
		\OUTPUT $\theta_T$. 
		\caption{Federated M-estimation with FDP guarantees} \label{algorithm:M-est_federated}
	\end{algorithmic}
\end{algorithm}

\begin{proof}[Proof of \Cref{prop:m-estimation-explicit}] The FDP guarantee holds since each $Z_k^t$ is $(\epsilon,\delta)$-DP, and therefore \eqref{eq:composition-eachstep} is satisfied for $t \in [T]$ and $k \in \tkset$.
Write $\overline{L}(\theta) = \frac{1}{K+1}\sum_{k=0}^K L_k(\theta)$, 
$\hat{\theta} = \argmin \overline{L}(\theta)$, and 
$\theta^* = \argmin \mathbb{E}[\overline{L}(\theta)]$. 
We decompose the error via the triangle inequality into three terms:
\[
\|\theta_T - \theta^{(0)}\|_2 
\leq {\|\theta_T - \hat\theta\|_2} 
+ \|\hat\theta - \theta^*\|_2 
+ \|\theta^* - \theta^{(0)}\|_2,
\]
and bound each of them separately. 

\textbf{Step 1: Bound for $\|\theta^*-\theta^{(0)}\|_2$.}

By \Cref{lemma:aloss-property}, $\overline{L}$ is $\tau_1$-strongly convex on 
$B_R(\theta^*)$ with $R = r-h-\frac{\tau_2 h}{2\tau_1}$ and $\tau_2$-smooth. 
The proof of \Cref{lemma:aloss-property} shows that
\begin{equation}\label{eq:heterogeneity-bias}
\|\theta^* - \theta^{(0)}\|_2 
\leq \frac{\|\nabla \overline{f}(\theta)\|_2}{2\tau_1}
\leq \frac{\tau_2 h}{2\tau_1},
\end{equation}
where $\overline{f}(\theta) = \mathbb{E}[\overline{L}(\theta)]$. 

\textbf{Step 2: Bound for $\|\hat\theta-\theta^*\|_2$.}

By the bounded gradient assumption \eqref{eq:Bgradient}, 
independence of data, and the sub-Gaussian concentration inequality with sub-Gaussian parameter $B$, it holds that 
with probability at least $1-\eta$,
\begin{equation}\label{eq:grad-concentration}
\|\nabla \overline{L}(\theta^*)\|_2 
= \bigg\|\frac{1}{n(K+1)}\sum_{k\in[K]\cup\{0\}}\sum_{i\in[n]}
\Psi(W_i^{(k)},\theta^*)\bigg\|_2 
\leq C_1\frac{B(\sqrt{d}+\sqrt{\log(1/\eta)})}{\sqrt{n(K+1)}},
\end{equation}
for an absolute constant $C_1>0$. 
The sample size condition ensures 
$\frac{2}{\tau_1}\|\nabla\overline{L}(\theta^*)\|_2 \leq R/2$, 
so the same argument used in \Cref{lemma:aloss-property} shows 
$\hat\theta \in B_R(\theta^*)$. By $\tau_1$-strong convexity on $B_R(\theta^*)$,
\begin{equation}\label{eq:stat-error}
\|\hat\theta - \theta^*\|_2 
\leq \frac{2}{\tau_1}\|\nabla \overline{L}(\theta^*)\|_2 
\leq \frac{C_2}{\tau_1}\frac{B(\sqrt{d}+\sqrt{\log(1/\eta)})}{\sqrt{n(K+1)}}.
\end{equation}

Combining  \eqref{eq:heterogeneity-bias} and \eqref{eq:stat-error}, we have
\begin{equation}\label{eq:m-estimation eq1}
\|\hat\theta - \theta^{(0)}\|_2 
\leq \|\hat\theta-\theta^*\|_2 + \|\theta^*-\theta^{(0)}\|_2
\leq \frac{C_2}{\tau_1}\frac{B\sqrt{d+\log(1/\eta)}}{\sqrt{n(K+1)}} 
+ \frac{\tau_2 h}{2\tau_1}.
\end{equation}

\textbf{Step 3: Bound for $\|\theta_T - \hat\theta\|_2$.}

Denote by $B_t$ the set of indices used in the $t$-th iteration 
and note that $\cup_t B_t = \{1,\dotsc,n\}$. 
Without loss of generality, we assume each site uses the same index set $B_t$ 
in the $t$-th iteration. Define
\begin{equation}\label{eq:d_t}
D_t(\theta_t) := \bigg\|\frac{1}{b(K+1)}\sum_{k\in[K]\cup\{0\}}\sum_{i\in B_t}
\Psi(W_i^{(k)},\theta_t)
- \frac{1}{n(K+1)}\sum_{k\in[K]\cup\{0\}}\sum_{i\in[n]}
\Psi(W_i^{(k)},\theta_t)\bigg\|_2.
\end{equation}
\Cref{lemma:covering} shows that with probability at least $1-\eta$,
\begin{equation}\label{eq:m-estimation-event2}
\max_{t\in[T]} D_t(\theta_t) 
\leq C_3\,r_b, \quad 
r_b := \frac{B\sqrt{d\log(nT(K+1)/\eta)}}{\sqrt{b(K+1)}} 
= \frac{B\sqrt{dT\log(nT(K+1)/\eta)}}{\sqrt{n(K+1)}},
\end{equation}
where we used $b = n/T$.

Next, let $N_t = \frac{1}{K+1}\sum_{k\in[K]\cup\{0\}}\phi w_t^{(k)}$ 
with $\phi = \frac{2B\sqrt{2\log(1.25/\delta)}}{b\epsilon}$. 
Standard Gaussian concentration gives, with probability at least $1-\eta$,
\begin{equation}\label{eq:m-estimation-event3}
\max_{t\in[T]}\|N_t\|_2 
\leq \frac{\phi}{\sqrt{K+1}}\big(4\sqrt{d}+2\sqrt{2\log(T/\eta)}\big) 
\leq C_4\,r_{priv}, \quad 
r_{priv} := \frac{BT\sqrt{\log(1/\delta)}\,\sqrt{d+\log(T/\eta)}}{n\epsilon\sqrt{K+1}}.
\end{equation}
We can write the update in each step as
\begin{align*}
\theta_{t+1} 
&= \theta_t 
- \frac{\rho}{b(K+1)}\sum_{k\in[K]\cup\{0\}}\sum_{i\in B_t}\Psi(W_i^{(k)},\theta_t)
+ \rho N_t \\
&= \theta_t - \rho\,\nabla\overline{L}(\theta_t) + \rho\,E_t,
\end{align*}
where the error term 
$E_t := \nabla\overline{L}(\theta_t) 
- \frac{1}{b(K+1)}\sum_{k\in[K]\cup\{0\}}\sum_{i\in B_t}\Psi(W_i^{(k)},\theta_t) + N_t$ 
satisfies, under events \eqref{eq:m-estimation-event2} and \eqref{eq:m-estimation-event3},
\begin{equation}\label{eq:Et-bound}
\max_{t\in[T]}\|E_t\|_2 \leq C_3\,r_b + C_4\,r_{priv}.
\end{equation}
Note that the proof of Theorem~2 in \cite{avella2023differentially} is entirely deterministic. The loss function $\aLoss$ satisfies Conditions 1 and 2 therein, 
by the bounded gradient assumption in \Cref{asp:m-estimation} and \Cref{lemma:aloss-property}, 
respectively. 
Therefore, we can follow their arguments to obtain that 
\begin{equation}\label{eq:optimisation-m-estimation}
\|\theta_T-\hat{\theta}\|_2 \lesssim \frac{1}{\tau_1}(r_b+r_{priv}),
\end{equation}
with conditions and choice of tuning parameters including $\rho = (2\tau_2)^{-1}, \tau_2\geq 1$, $T\gtrsim \tau_2\log(nK)/\tau_1$, $\aLoss(\theta_0) -\aLoss(\hat{\theta}) \leq R^2\tau_1/4$, and $n(K+1)$ is large enough such that 
\[
r_b+r_{priv} \leq R \quad r_b+r_{priv}\leq \frac{R^2\tau_1^2}{B^2}.
\]
The minimal sample size condition follows from (36) in the proof of Lemma 18 in \cite{avella2023differentially}. 
A sufficient condition for the above is 
\[
n(K+1) \gtrsim \frac{T^2B^4{d\log(nT(K+1)\eta^{-1})}\log(1/\delta)}{\epsilon^{2}\min\{\tau_1^2,\tau_1^4\}}.
\]
Moreover, the requirement that $\aLoss(\theta_0) -\aLoss(\hat{\theta}) \leq R^2\tau_1/4$ on the initializer can be achieved by assuming 
$\|\theta_0-\theta^{(0)}\|_2\lesssim1$ and running the algorithm for an additional $T_0 \asymp \tau_2/\tau_1$ steps. 
Indeed, note that $\|\theta_0 - \hat\theta\|\leq \|\theta_0-\theta^{(0)}\|+\|\hat\theta - \theta^{(0)}\|\lesssim 1$, 
since $h\tau_2/\tau_1 \lesssim 1$, 
and Proposition 1 in \cite{avella2023differentially} shows that after 
$K_0 = 4\|\theta_0 - \hat\theta\|_2^2/(\eta R^2\tau_1) \asymp \tau_2/\tau_1$ steps, 
$\aLoss(\theta_{K_0}) -\aLoss(\hat{\theta}) \leq R^2\tau_1/4$ with probability at least $1-\eta$.
Adding \eqref{eq:m-estimation eq1} and \eqref{eq:optimisation-m-estimation} yields
\begin{align*}
\|\theta_T - \theta^{(0)}\|_2 
&\leq \frac{1}{\tau_1}(r_b + r_{priv}) 
+ \frac{B\sqrt{d+\log(1/\eta)}}{\tau_1\sqrt{n(K+1)}} 
+ \frac{\tau_2 h}{2\tau_1}.
\end{align*}
Finally, since $r_b 
\geq \frac{B\sqrt{d+\log(1/\eta)}}{\sqrt{n(K+1)}}$, 
we have,
\[
\|\theta_T - \theta^{(0)}\|_2 
\lesssim \frac{\tau_2}{\tau_1}\,h 
+ \frac{1}{\tau_1}\big(r_b + r_{priv}\big),
\]
as claimed.
\end{proof}

\subsection{Full-batch version}

\begin{algorithm}[!ht]
	\begin{algorithmic}[1]
		\INPUT{Data $\{\{W_i^{(k)}\}_{i\in[n]}\}_{k\in [K]\cup\{0\}}$, number of iteration $T$, step size $\rho$, privacy parameters $\epsilon,\delta$, initialisation $\theta_0$, failure probability $\eta \in (0,1/2)$}
		\For{$t = 0, \ldots, T-1$} 
                 \For{$k \in \{0\} \cup [K]$} \Comment{Each site generates the privatised information $Z^t_k$ locally}
			\State Sample $w_t^{(k)} \sim \mathcal{N}(0, I_d)$ and let $\phi = \sqrt{2\log(1.25/\delta)}2B/(n\epsilon)$
            \State Compute $Z^t_k = \frac{1}{K+1}\Big(\frac{1}{n}\sum_{i=1}^{n}\Psi(W_{i}^{(k)},\theta_t)+\phi w_t^{(k)}\Big)$
            \EndFor
            \State $\theta_{t+1} = \theta_t - \rho\sum_{k \in \{0\} \cup [K]}Z^t_k$; \Comment{A central server aggregates the privatised information}
		\EndFor
		\OUTPUT $\theta_T$. 
		\caption{Federated M-estimation with FDP guarantees (full-batch)} \label{algorithm:M-est_federated-fullbatch}
	\end{algorithmic}
\end{algorithm}

\begin{cor}\label{cor:full-batch mestimation} Suppose that \Cref{asp:m-estimation} holds and the initializer $\|\theta_0 - \theta^{(0)}\|_2\lesssim 1$. 
 Let $\rho\leq \frac{1}{2}\min\{\tau_2^{-1},1\}$, $n\sqrt{K+1} \gtrsim B^3d^{1/2}\tau_1^{-2}\epsilon^{-1}, T \asymp \tau_2\log(n(K+1))\tau_1^{-1} $ and $\tau_2h\tau_1^{-1} \lesssim r \asymp 1$. Then the output of \Cref{algorithm:M-est_federated-fullbatch}, $\theta_T$, satisfies with probability at least $1-\eta$
   \[
 \|\theta_T - \theta^{(0)}\|_2 \lesssim \frac{\tau_2}{\tau_1}\,h 
+ \bigg[
\frac{B{\tau_2}}{\tau_1}\cdot \sqrt{\frac{d}{n(K+1)}}
\;+\;
\frac{B}{\tau_1}\cdot \frac{\sqrt{d}}{n\epsilon\sqrt{K+1}}
\bigg]\mathrm{polylog}(n,K,\eta,\delta,\tau_1,\tau_2).
   \]
\end{cor}
To compare the above result with the one in \Cref{prop:m-estimation-explicit}, we note that the difference between the mini-batch gradient descent (\Cref{algorithm:M-est_federated}) and the full-batch gradient descent (\Cref{algorithm:M-est_federated-fullbatch})  leads to different meanings of $\epsilon$ in these two results. In particular, \Cref{algorithm:M-est_federated} accesses separate samples in each iteration, and it naturally fits into our FDP framework. However, in \Cref{algorithm:M-est_federated-fullbatch}, the data points are used in all $T$ rounds of iterations. Therefore, composition theorems \cite[e.g.][]{murtagh2015complexity} are required to calibrate the total privacy leakage. The basic composition result bounds the total privacy leakage of the final output $\theta_T$ by $(\epsilon T,\delta T)$. More advanced composition results improve the scaling of $T$ to $(O(\epsilon\sqrt{T\log(1/\delta)}), O(T\delta))$. Since the number of iterations required in both \Cref{prop:m-estimation-explicit} and \Cref{cor:full-batch mestimation} are of order $O(\tau_2\log(nK)/\tau_1)$, if we adjust $(\epsilon,\delta)$ in \Cref{cor:full-batch mestimation} either according to the basic or advanced composition results, then the mini-batch gradient descent and full-batch algorithm achieve the same bound up to $\tau_2\log(nK)/\tau_1$.
\begin{proof}[Proof of \Cref{cor:full-batch mestimation}]
    The proof directly follows from the proof of \Cref{prop:m-estimation-explicit} with $b = n/T$ replaced by $n$, as we are using the full batch in each iteration. Moreover, since $D_t(\theta_t)$ in the proof of \Cref{prop:m-estimation-explicit} is now zero, there is no need to incur \Cref{lemma:covering} to control its magnitude. Therefore, the bounded-parameter-space assumption is no longer needed.
\end{proof}
\subsection{Auxiliary results}
\begin{lemma}\label{lemma:aloss-property}
    Consider the loss function $\overline{L}(\theta) = \frac{1}{K+1}\sum_{k=0}^KL_k(\theta)$ 
    and let $\theta^* = \argmin \mathbb{E}(\aLoss(\theta))$. 
    Then $\aLoss(\theta)$ is $\tau_1$-strongly convex on $B_R(\theta^*)$ 
    with $R = r-h-\tau_2h(2\tau_1)^{-1}$ and $\tau_2$-smooth.
\end{lemma}

\begin{proof}[Proof of \Cref{lemma:aloss-property}]
    Since each $L_k(\theta)$ is $\tau_1$-strongly convex  on $B_r(\theta^{(k)})$ 
    and $\tau_2$-smooth, the averaged loss $\overline{L}$ is $\tau_1$-strongly convex  
    on $\cap_{k} B_r(\theta^{(k)})$ and $\tau_2$-smooth. 
    Moreover, since $B_{r-h}(\theta^{(0)}) \subseteq \cap_{k} B_r(\theta^{(k)})$, 
    $\aLoss$ is also $\tau_1$-strongly convex on $B_{r-h}(\theta^{(0)})$. 
    Now, we show that $\aLoss$ is $\tau_1$-strongly convex on $B_R(\theta^*)$.
    
    It is sufficient to show that $B_R(\theta^*) \subseteq B_{r-h}(\theta^{(0)})$.  
    For any $k \in [K]$, let $f_k(\theta) = \mathbb{E}_{P_{\theta^{(k)}}}[\rho(W,\theta)]$, 
    and let $\overline{f}(\theta) = \frac{1}{K+1}\sum_{k=0}^Kf_k(\theta)$ 
    using the smoothness property, we have 
    \[
    \|\nabla f_k(\theta^{(0)})\| = \|\nabla f_k(\theta^{(0)}) - \nabla f_k(\theta^{(k)})\| \leq \tau_2\|\theta^{(0)} - \theta^{(k)}\|\leq \tau_2h,
    \]
    and therefore $\|\nabla \overline{f}(\theta^{(0)})\|\leq \tau_2 h$. Now, using strong convexity on $B_{r-h}(\theta^{(0)})$, we have
    \[
    \overline{f}(\theta) \geq \overline{f}(\theta^{(0)}) + \langle\nabla\overline{f}(\theta^{(0)}),\theta-\theta^{(0)}\rangle + \tau_1\|\theta-\theta^{(0)}\|^2,
    \]
    for all $\theta \in B_{r-h}(\theta^{(0)})$. When $\|\theta-\theta^{(0)}\| = r-h$, the above implies 
    \[
     \overline{f}(\theta) \geq \overline{f}(\theta^{(0)}) -(r-h) \|\nabla\overline{f}(\theta^{(0)})\| + \tau_1(r-h)^2 \geq \overline{f}(\theta^{(0)}) -(r-h) \tau_2h + \tau_1(r-h)^2,
    \]
    which implies $\overline{f}(\theta) > \overline{f}(\theta^{(0)})$ when $r-h >\tau_2h/\tau_1$. 
    Now, since $\overline{f}$ is convex, the sublevel set 
    $S = \{\theta:\overline{f}(\theta)\leq \overline{f}(\theta^{(0)})\}$ is convex 
    and $\theta^*\in S$ since $\theta^*$ is the minimiser of $\overline{f}$. 
    Moreover, $S \subseteq B_{r-h}(\theta^{(0)})$ due to the convexity of $\overline{f}$. 
    Now that $\theta^* \in B_{r-h}(\theta^{(0)})$, we can obtain 
    \[
    \|\theta^*-\theta^{(0)}\|\leq \frac{\|\nabla\overline{f}(\theta^{(0)})\|}{2\tau_1} \leq \frac{\tau_2}{2\tau_1}h.
    \]
    Finally, for any $\theta \in B_R(\theta^*)$ with $R = r-h-\tau_2h(2\tau_1)^{-1}$, 
    we have 
    \[
    \|\theta-\theta^{(0)}\| \leq R+\|\theta^*-\theta^{(0)}\| \leq R + \frac{\tau_2}{2\tau_1}h \leq r-h,
    \]
    which implies $B_R(\theta^*) \subseteq B_{r-h}(\theta^{(0)})$.
\end{proof}

 \begin{lemma}\label{lemma:covering}
     Let $D_t(\theta_t)$ be defined in \eqref{eq:d_t}, 
     if $\Theta$ is bounded with diameter $D \asymp 1$, 
     then with probability at least $1-\eta$, 
     \[
     \max_{t\in[T]} D_t(\theta_t) \lesssim \frac{B\sqrt{d\log(nT(K+1)/\eta)}}{\sqrt{b(K+1)}}.
     \]
 \end{lemma}
\begin{proof}[Proof of \Cref{lemma:covering}]
    To control the magnitude of $D_t(\theta_t)$, notice that 
\[
\mu(\theta_t) = \mathbb{E}\left(\frac{1}{b(K+1)}\sum_{k\in[K]\cup\{0\}}\sum_{i\in B_t}\Psi(W_{i}^{(k)},\theta_t)\right) 
= \mathbb{E}\left(\frac{1}{n(K+1)}\sum_{k\in[K]\cup\{0\}}\sum_{i\in[n]}\Psi(W_{i}^{(k)},\theta_t)\right).
\]
And we have 
\begin{align*}
    D_t(\theta_t) &\leq \Bigg\|\frac{1}{b(K+1)}\sum_{k\in[K]\cup\{0\}}\sum_{i\in B_t}\Psi(W_{i}^{(k)},\theta_t) - \mu(\theta_t)\Bigg\|_2+\left\|\frac{1}{n(K+1)}\sum_{k\in[K]\cup\{0\}}\sum_{i\in[n]}\Psi(W_{i}^{(k)},\theta_t) - \mu(\theta_t)\right\|_2 \\
    & =: D_t^{(1)}(\theta_t)+D_t^{(2)}(\theta_t)
\end{align*}
For the first term, since each iteration uses a fresh batch, 
we have for each $t \in [T]$,
\[
\mathbb{P}\left(D_t^{(1)}(\theta_t) \gtrsim \frac{B(\sqrt{d}+\sqrt{\log(T/\eta)})}{\sqrt{b(K+1)}} \bigg|\theta_t\right) \leq \eta/T.
\]
Therefore, a union bound leads to 
\[
\max_{t\in[T]}D_t^{(1)}(\theta_t) \lesssim \frac{B(\sqrt{d}+\sqrt{\log(T/\eta)})}{\sqrt{b(K+1)}}
\]
with probability at least $1-\eta$. For the second term $D_t^{(2)}(\theta_t)$, 
we consider a $c$-cover $\mathcal{N}_c$ of $\Theta$ such that for all $\theta \in \Theta$, 
there exists some $\theta' \in \mathcal{N}_c $ such that $\|\theta - \theta'\|\leq c$ 
and we shall choose $c$ later. 
Note that since $\Theta$ has diameter $D$, we have $|\mathcal{N}_c| \lesssim (D/c)^d$. 
For each fixed $\theta \in \mathcal{N}_c$, we have 
\[
D_t(\theta') \lesssim \frac{B(\sqrt{d}+\sqrt{\log(1/\eta)})}{\sqrt{n(K+1)}}
\]
with probability at least $1-\eta$. With a union bound, we have with probability at least $1-\eta$
\[
\max_{\theta' \in \mathcal{N}_c}D_t(\theta') \lesssim \frac{B(\sqrt{d}+\sqrt{d\log(D/c\eta)})}{\sqrt{n(K+1)}}.
\]
Using the smoothness property of $\aLoss$, we have 
\[
\sup_{\theta\in \Theta}D_t(\theta) \leq \max_{\theta' \in \mathcal{N}_c}D_t(\theta')+2\tau_2c 
\lesssim  \frac{B(\sqrt{d}+\sqrt{d\log(D/c\eta)})}{\sqrt{n(K+1)}}+c.
\]
With the choice of $c \asymp 1/(n(K+1))$, we have 
\[
\max_{t\in[T]}D_t^{(2)}(\theta_t) \leq \sup_{\theta\in \Theta}D_t^{(2)}(\theta)  
\lesssim  \frac{B(\sqrt{d}+\sqrt{d\log(n(K+1)/\eta)})}{\sqrt{n(K+1)}}.\]
Combining the bounds for $\max_{t\in[T]}D_t^{(1)}(\theta_t)$ and 
$\max_{t\in[T]}D_t^{(2)}(\theta_t)$, we obtain the claimed result.
\end{proof}

\section{Sensitivity Analysis}\label{sec: sensitivity}

In the main text, we first detect the transferable source index set $\mA$ using $\hat{\mA}$, then run the FDP algorithm using only sources in $\hat{\mA}$. For the low-dimensional mean estimation problem, low-dimensional linear regression problem, and high-dimensional linear regression problem, $\hat{\mA}$ is defined in equations (\ref{eq:mean-hatA}), (\ref{eq:lowd-hatA}), and (\ref{algorithm:DPregression_high_dim_detection}), respectively. There is a tuning parameter $\tilde{c}$ in the definition of $\hat{\mA}$ which controls the radius of the desired similarity level of selected sources compared to the target. 

\begin{figure}[!ht]
    \centering
    \includegraphics[width=0.8\linewidth]{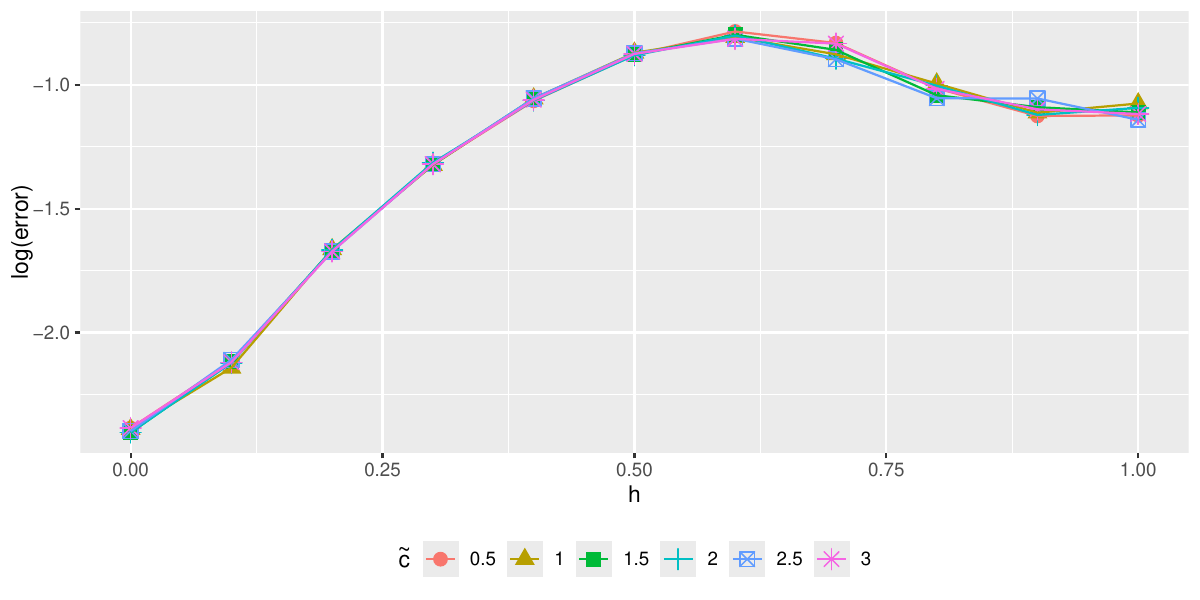}
    \caption{Performance of FDP-detection with different values of the tuning parameter $\tilde{c}$.}
    \label{fig:Figures/detection-sensitivity}
\end{figure}

In our numerical experiments, we set $\tilde{c} = 1$ for convenience and simplicity. In practice, $\tilde{c}$ can be tuned using strategies such as cross-validation. In addition, the performance of our algorithm is relatively robust to the choice of $\tilde{c}$. We tested different $\tilde{c}$ values in the same setting of Section \ref{subsubsec: hetero simulation}. The performance of our FDP algorithm with different $\tilde{c}$ values is summarised in Figure \ref{fig:Figures/detection-sensitivity}. The curves overlap well, indicating similar performance and robustness across choices of $\tilde{c}$.

\end{document}